%% file: example_paper.tex
\documentclass{article}

\usepackage[utf8]{inputenc} 
\usepackage[T1]{fontenc}    
\usepackage{hyperref}       
\usepackage{url}            
\usepackage{booktabs}       
\usepackage{amsfonts}       
\usepackage{nicefrac}       
\usepackage{microtype}      
\usepackage{xcolor}         
\usepackage{graphicx}
\usepackage{subcaption}
\usepackage{mathtools}
\usepackage{amsfonts,amsmath,amssymb,amsthm}
\usepackage{natbib}
\usepackage{hyperref}
\usepackage[capitalize]{cleveref}
\usepackage{nicefrac}       
\usepackage[dvipsnames,svgnames]{xcolor}
\usepackage[most]{tcolorbox}
\usepackage[strict]{changepage}
\usepackage{wrapfig}
\usepackage{framed}

\usepackage{algorithm}
\usepackage{algorithmicx}
\usepackage{algpseudocode}
\usepackage{subcaption, comment}
\usepackage[textsize=tiny]{todonotes}

\usepackage{multirow}   
\expandafter\def\csname ver@algorithmic.sty\endcsname{}
\makeatletter
\let\saved@addcontentsline\addcontentsline
\makeatother
\usepackage{etoc}

\input{preamble}

\usepackage[accepted]{icml2026}

\icmltitlerunning{In-Context Pure Exploration in  Continuous Decision Spaces}

\begin{document}

\twocolumn[
  \icmltitle{In-Context Pure Exploration in  Continuous Decision Spaces}



  \icmlsetsymbol{equal}{*}

  \begin{icmlauthorlist}
    \icmlauthor{Alessio Russo}{equal,yyy}
    \icmlauthor{Yin-Ching Lee}{equal,yyy}
    \icmlauthor{Ryan Welch}{comp}
    \icmlauthor{Aldo Pacchiano}{yyy}
  \end{icmlauthorlist}
  \icmlaffiliation{yyy}{Faculty of Computing and Data Sciences, Boston University, Boston, MA, USA.}
\icmlaffiliation{comp}{Department of Computer Science, Stanford University, Stanford, CA, USA.}

  \icmlcorrespondingauthor{Alessio Russo}{arusso2@bu.edu}
  \icmlcorrespondingauthor{Yin-Ching Lee}{leewill@bu.edu}

  \icmlkeywords{pure exploration, active sequential hypothesis testing, best arm identification, experimental design, reinforcement learning}

  \vskip 0.3in
]



\printAffiliationsAndNotice{}  

\input{sections/abstract}
\input{sections/introduction}
\input{sections/problem_setting}
\input{sections/method}

\input{sections/empirical_evaluation}
\input{sections/conclusions}
\bibliographystyle{abbrvnat}
\bibliography{ref}

\appendix

\input{sections/appendix/appendix}
\end{document}

%% file: preamble.tex
\theoremstyle{plain}
\newtheorem{theorem}{Theorem}[section]
\newtheorem{proposition}{Proposition}
\newtheorem{lemma}{Lemma}

\newtheorem{assumption}{Assumption}
\newtheorem{remark}{Remark}

\theoremstyle{definition}
\newtheorem{definition}{Definition}

\DeclareMathOperator*{\argmax}{arg\,max}
\DeclareMathOperator*{\argmin}{arg\,min}

\newcommand{\A}{\mathcal{A}}
\newcommand{\Y}{\mathcal{Y}}
\newcommand{\X}{\mathcal{X}}

\hypersetup{
    colorlinks = true,
    linkcolor = RedViolet,
    citecolor=NavyBlue,
    }

 \newcommand{\refalgbyname}[2]{{\texttt{\textbf{#2}}}}
\newcommand{\cicpe}{\refalgbyname{algo:icpe_fixed_confidence}{C-ICPE}}
\newcommand{\cicpets}{\refalgbyname{algo:icpe_fixed_confidence}{C-ICPE-TS}}
\newcommand{\cicpettps}{\refalgbyname{algo:icpe_fixed_confidence}{C-ICPE-TTPS}}
\newcommand{\cicpetdthree}{\refalgbyname{algo:icpe_fixed_confidence}{C-ICPE-TD3}}
\newcommand{\cicpeuniform}{\refalgbyname{algo:icpe_fixed_confidence}{C-ICPE-Uniform}}

\definecolor{formalshade}{rgb}{0.95,0.95,1}

%% file: sections/abstract.tex
\begin{abstract}
In active sequential testing, also termed \emph{pure exploration}, a learner is tasked with the goal to adaptively acquire information so as to identify an unknown ground-truth hypothesis with as few queries as possible. This problem has several motivating applications, including Best-Arm Identification (BAI) in bandits, where actions index hypotheses, and generalized search problems, where strategically chosen queries reveal partial information about a hidden label. 
In many modern settings, however, the hypothesis, or recommendation space, is \emph{continuous}: for example, identifying a near optimal action in a continuous-armed bandit, localizing an $\epsilon$-ball contained in a target region, or estimating the minimizer of a function from noisy observations. Existing methods are predominantly frequentist and model-specific, while learned approaches have been limited to finite recommendation spaces.
We introduce \cicpe{}, a theory-guided learned model for Bayesian fixed-confidence pure exploration with continuous recommendations. \cicpe{}  meta-trains sequential architectures over a task prior to jointly learn exploration, stopping and recommendations strategies. At inference time, it actively gathers evidence on tasks and identifies an $\epsilon$-optimal recommendation \emph{without} parameter updates. 
\end{abstract}


%% file: sections/introduction.tex
\section{Introduction}\label{sec:introduction}

Several learning problems are inherently interactive: the learner sequentially performs interventions or stages queries, observes noisy evidence whose distribution depends on that intervention, and stops once the accumulated evidence supports a reliable conclusion. This type of interactive sequential decision-making problem goes back to \citet{chernoff_sequential_1959} and has been formalized through active sequential  hypothesis testing \citep{naghshvar2013Active}  and pure exploration with fixed confidence \citep{audibert2010best,degenne_non-asymptotic_2019}, where the learner minimizes the number of queries subject to returning an $\epsilon$-accurate recommendation with probability at least $1-\delta$.

This regime is well understood in canonical settings with finite   decision spaces, including best-arm identification in stochastic multi-armed bandit models \citep{garivier2016optimal} and best-policy identification in Markov Decision Processes (MDPs) \citep{puterman2014markov}. In these settings the  learner chooses queries (e.g., arms) and outputs an object of interest, often a best action/policy, and
the theoretical guarantees have been studied in a range of settings \citep{degenne_non-asymptotic_2019,poiani_best-arm_2025,al2021navigating,pmlr-v258-russo25a}.

Despite this progress, practical methods for fixed-confidence pure 
exploration remain limited when the learner must return a recommendation 
in a \emph{continuous} space. Existing continuous methods are either 
frequentist and model-specific \citep{garivier2021nonasymptotic,takemori2025instance,poiani2025Pure,russo_adaptive_2025}
or Bayesian but restricted to Gaussian processes and not optimizing 
sample complexity 
\citep{wilson2024stopping}. Even in finite arms, the theory of Bayesian fixed-confidence pure exploration is limited, and results are known only in the finite setting with Gaussian likelihoods/priors \citep{jang2024Fixed}. No analogous 
Bayesian theory, nor practical methods, are known for continuous recommendation spaces under 
general priors. 

Recently, \citep{russo_learning_2025} proposed In-Context Pure Exploration (ICPE), which meta-trains sequential neural policies for finite active-testing problems. However, ICPE is restricted to finite hypothesis and action sets, and does not address the continuous recommendation setting. Hence, it is currently missing a broadly reusable learned method for Bayesian fixed-confidence pure exploration when the recommendation itself is continuous and performance is optimized over a task prior.
We introduce \cicpe{}, a theory-guided method that learns to collect data, stop, and recommend directly from trajectories in continuous recommendation spaces.
This type of $(\epsilon,\delta)$-PAC exploration directly addresses problems in experimental sciences such as materials discovery \citep{liu2017materials} and dose-finding \citep{oquigley1990Continual}, where each trial is costly, observations are noisy, and practitioners need not just a good answer but a guarantee that the answer is $\epsilon$-correct with a given confidence.
\begin{figure}[t]
\centering
\includegraphics[width=0.7\linewidth]{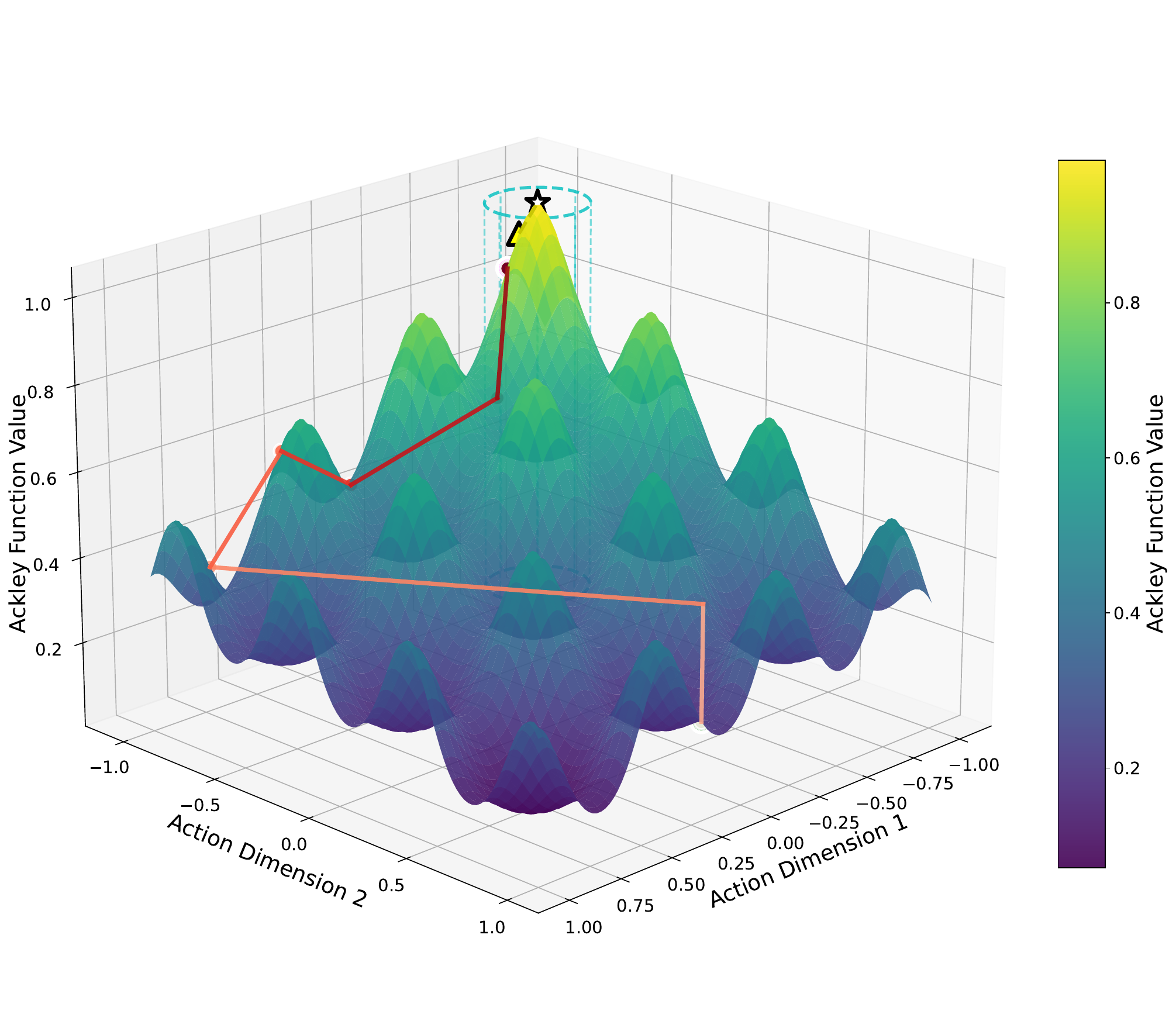}
\caption{ \cicpe{} is able to identify the global maxima (with $\epsilon$-accuracy and $1-\delta$ confidence) of the inverted  Ackley function (with random parameters and observation noise) without gradients while trying to use the least number of data-points.}\label{fig:ackley_example}
\end{figure}
\noindent{\bf Contributions.} First, we formulate Bayesian fixed-confidence pure exploration with continuous recommendations, and establish the corresponding Bellman optimality structure. Second, we prove Bayesian $(\epsilon,\delta)$-correctness under a local closedness condition that is  weaker than the uniqueness assumption in \citep{russo_learning_2025}, using a novel subdifferential argument.
Third, we instantiate this framework into a practical method, \cicpe{},  to train exploration policies that deploy model-free, and evaluate it on noisy binary  search, $\epsilon$-best arm identification on the unit sphere, noisy Ackley minimization, 
value estimation  in Gaussian Processes and a real-world geochemical task where the  goal is to locate peak copper concentration from sparse soil 
measurements \citep{usgs_geochem}.
To our knowledge, this is the first practical learned framework combining continuous recommendations and fixed-confidence stopping.

%% file: sections/problem_setting.tex
\vspace{-5pt}
\section{Problem Setting}\label{sec:problem_setting}
\vspace{-5pt}
We consider a Bayesian family of active sequential decision problems indexed by a latent parameter $\theta\in\Theta$, where $\Theta\subset\mathbb R^d$ is compact and $\theta\sim\nu$. Each environment $M_\theta$ specifies an initial observation law $\rho_\theta\in\Delta(\Y)$ and observation kernels $P_{\theta,t}(\cdot|h_t,a_t)$ over a compact observation space $\Y\subset\mathbb R^n$.

In  sequential testing, the learner interacts with $M_\theta$ over time: in round $t$ it observes the history
\[
H_t\coloneqq(Y_1,A_1,Y_2,\dots,A_{t-1},Y_t),
\]
chooses a query $A_t\in\A\subset\mathbb R^m$ (compact), and observes $Y_{t+1}\sim P_{\theta,t}(\cdot|H_t,A_t)$. The goal is to collect a history that is sufficiently informative to output a high-quality recommendation $\hat x\in\X$, where $\X$ is compact. We refer to $\X$ as the hypothesis or recommendation space. We distinguish $\A$ from $\X$: $\A$ is the query space used to collect information, while $\X$ is the space of objects the learner may return. In many tasks $\X=\A$, but in value-identification tasks $\X$ may instead be a set of possible function values.

\noindent{\bf Risk function.}
Recommendation quality is measured by a task-dependent loss, or risk, function $L_\theta:\X\to[0,\infty)$, satisfying $\inf_{x\in\X}L_\theta(x)=0$. We say that $x$ is $\epsilon$-optimal for task $\theta$ if $L_\theta(x)\leq\epsilon$, and define
\[
\X_\epsilon(\theta)\coloneqq\{x\in\X:L_\theta(x)\leq\epsilon\}.
\]
In the following, we assume that $(\theta,x)\mapsto L_\theta(x)$  is jointly lower semicontinuous (in \cref{sec:appendix:theoretical_results} we state the regularity assumptions used in the theoretical results).
Throughout the paper, $x_\theta^\star$ denotes a selected zero-loss target in the recommendation space $\X$, i.e., $L_\theta(x_\theta^\star)=0$. Depending on
the problem, this object may be an optimizer, a best arm, a threshold,
or an optimal value. When the zero-loss set is not a singleton, we assume a fixed continuous selector $\theta\mapsto x_\theta^\star$ to ensure regularity.

In many examples, the loss is defined through a function $f_\theta$ parametrized by $\theta$. In function optimization problems, we set $\X=\A$, and the goal is to find a point $\hat x\in\X$ that optimizes the function. In this case, one can take the risk loss $L_\theta$ to be value-gap loss  $L_\theta(x)\coloneqq f_\theta(x)-f_\theta(x_\theta^\star)$, with $x_\theta^\star \in F^\star(\theta)$, where $F^\star(\theta)\coloneqq\argmin_{x\in \X}f_\theta(x)$, or a distance loss  in the query space $L_\theta(x)\coloneqq \|x_\theta^\star-x\|$ if the goal is to find $x$ close to a selected optimal point $x_\theta^\star\in F^\star(\theta)$. Other problems include the $\epsilon$-best arm identification problem in multi-armed bandits where $f_\theta(x)$ is the mean reward of arm $x$, or noisy binary search in a continuum, where the agent observes noisy observations of ${\rm sign}(x-x_\theta^\star)$ and the loss is defined in the query space.  Problems where the recommendation space $\X$ is not identical to the query space $\A$ include optimal-value identification, where the learner returns a scalar estimate of the optimal value: we take $\X\subset \mathbb{R}$, and set $L_\theta(x)\coloneqq |x-x_\theta^\star|$ where $x_\theta^\star \coloneqq\max_{a\in\A}f_\theta(a)$.

\noindent{\bf Optimization objective.} We work in the fixed-confidence ($(\epsilon,\delta)$-PAC) regime. A learner is defined by the triplet $(\pi,I,\tau)$:  a sampling policy $\pi=(\pi_t)_{t\ge1}$ such that $A_t=\pi_t(H_t)$; a stopping time $\tau$ with respect to $\mathcal F_t=\sigma(H_t)$, defining when to stop the data acquisition process; an inference rule  $I=(I_t)_{t\geq 1}$ such that $\hat x_\tau = I_\tau(H_\tau)$. The goal of the learner is to adaptively choose queries $A_1,A_2,\dots$ and a stopping time $\tau$, 
 so that the returned $\hat x_\tau$  is $\epsilon$-optimal, i.e. $\hat x_\tau\in \X_\epsilon(\theta)$, with high probability.
Hence, for a given  pair $\epsilon>0,\delta\in(0,1/2)$, we seek to minimize the (expected) number of queries while ensuring $\delta$-correctness: formally, we solve
\begin{equation}\label{eq:fixed_confidence_problem}
\inf_{\tau,\pi,I}\ \mathbb{E}_{\theta\sim \nu}^\pi\left[\tau\right]
\quad\text{s.t.}\quad
\mathbb{P}_{\theta\sim\nu}^\pi\left( \hat x_\tau\in \X_\epsilon(\theta)\right) \geq\ 1-\delta.
\end{equation}
Our formulation is Bayesian: $\nu$ is both a prior over environments and the task distribution used for training and evaluation. This enables amortized learning across tasks: the models are trained on episodes drawn from $\nu$ and transfer to new tasks from the same family without parameter updates. Second, it defines the posterior success probability $q_t(h,x)$ that drives our theory and algorithms. Third, the average-case guarantee under $\nu$ is the natural objective for  applications where the practitioner has domain knowledge about plausible task distributions. For this setup,  we are not aware of analogous Bayes-optimal characterizations for continuous recommendation spaces under general priors, as results are limited to finite settings with Gaussian structure \citep{jang2024Fixed}.

%% file: sections/method.tex
\vspace{-5pt}
\section{Theoretical Background}\label{sec:theoretical_results}
\vspace{-5pt}
This section provides the theoretical foundation for \cicpe{}, where we  characterize an optimal learner. Relative to the finite ICPE analysis of \citet{russo_learning_2025}, the continuous setting introduces two technical issues absent in finite spaces: (i) attainment of the Bellman equation over continuous actions requires establishing a semicontinuity chain from the observation model through the posterior predictive to the $Q$-function, and (ii) the $(\epsilon,\delta)$-correctness proof must handle non-singleton dual optima, which we address via a weaker local closedness condition and a subdifferential argument that replaces the uniqueness and monotonicity assumptions of the finite case. Together, these results provide the first theoretical infrastructure for Bayesian  pure exploration with continuous recommendations.

\noindent{\bf Posterior success and optimal inference.}
In the fixed-confidence setting with continuous $\X$, the relevant posterior object is the posterior probability that $x$ is $\epsilon$-optimal:
\[
q_t(h,x)\coloneqq\mathbb P\left(L_\theta(x)\leq \epsilon\mid H_t=h\right),\; r_t(h)\coloneqq\max_{x\in\X}q_t(h,x).
\]
Here $q_t(h,x)$ is the posterior success probability of recommending $x$, and $r_t(h)$ is the best posterior success probability achievable from history $h$. One can show that the maximum is attained and an optimal inference rule is any  selector (see \cref{prop:optimal_inference_rule} for  a proof)
\[
I_t^\star(h)\in\argmax_{x\in\X}q_t(h,x).
\]

\noindent{\bf Dual formulation and Bellman optimality.}
We study the fixed-confidence problem in \cref{eq:fixed_confidence_problem} through its Lagrangian dual, following the ASHT literature \citep{naghshvar2013Active,russo_learning_2025}. Introducing a multiplier $\lambda\geq 0$ for the correctness constraint gives the dual
\begin{equation}\label{eq:dual_problem}
\inf_{\lambda\geq 0}\sup_{\pi,I,\tau}V_\lambda(\pi,I,\tau),
\end{equation}
where
\[
\resizebox{\hsize}{!}{%
$V_\lambda(\pi,I,\tau)\coloneqq-\mathbb E_{\theta\sim\nu}^\pi[\tau]+\lambda\left(\mathbb P_{\theta\sim\nu}^\pi\left(I_\tau(H_\tau)\in\X_\epsilon(\theta)\right)-1+\delta\right).$%
}
\]
For fixed $\pi$ and $\tau$, optimizing over $I$ replaces the terminal success probability by $\mathbb E^\pi[r_\tau(H_\tau)]$. Hence, for fixed $\lambda$, the inner problem is equivalent up to a constant   to maximizing $\mathbb E^\pi[\lambda r_\tau(H_\tau)-\tau]$.
Stopping can also be embedded as an absorbing action $a_{\rm stop}$: the learner continues with actions in $\A$ until it selects $a_{\rm stop}$, at which point it stops and outputs $I_t^\star(H_t)$; see \cref{lem:stop_as_action}. Thus the fixed-$\lambda$ problem is an optimal-stopping control problem on $\bar\A=\A\cup\{a_{\rm stop}\}$.

For $t\geq 1$ and $h\in{\cal H}_t$, define the optimal reward-to-collect value
\begin{equation}\label{eq:cost_to_go_def}
V_t^\star(h;\lambda)\coloneqq\sup_{\bar\pi}\mathbb E_{\theta\sim\nu}^{\bar\pi}\left[\lambda r_{\bar\tau}(H_{\bar\tau})-(\bar\tau-t)\mid H_t=h\right].
\end{equation}
Define
\begin{align*}
&Q_{t,\rm stop}^\star(h;\lambda)\coloneqq\lambda r_t(h),\\ 
&Q_{t,\rm cont}^\star(h,a;\lambda)\coloneqq -1+\mathbb E\left[V_{t+1}^\star(H_{t+1};\lambda)\mid H_t=h,A_t=a\right],
\end{align*}
where the expectation is under the posterior predictive kernel $\bar P_t(\cdot\mid h,a)$.

\begin{theorem}[Bellman equation and greedy deterministic optimality]\label{thm:main_bellman_greedy}
Assume the regularity conditions of \cref{app:theoretical_results:icpe:fixed_confidence:dual} (\cref{assumption:L_continuity_x_measurability_xtheta} and \cref{assumption:likelihood_continuity}), then
\begin{equation}\label{eq:main_bellman_stop_continue}
V_t^\star(h;\lambda)=\max\left\{Q_{t,\rm stop}^\star(h;\lambda),\sup_{a\in\A}Q_{t,\rm cont}^\star(h,a;\lambda)\right\}.
\end{equation}
Moreover, let $a_t^\star(h)\in\argmax_{a\in\A}Q_{t,\rm cont}^\star(h,a;\lambda)$. The deterministic policy $\bar\pi^\star$ defined by
\[
\bar\pi^\star(h)=a_{\rm stop}\quad\text{if }Q_{t,\rm stop}^\star(h;\lambda)\geq Q_{t,\rm cont}^\star(h,a_t^\star(h);\lambda),
\]
and $\bar\pi^\star(h)=a_t^\star(h)$ otherwise, is optimal for the fixed-$\lambda$ inner problem.
\end{theorem}
Unlike the finite case, where Bellman attainment is automatic, the continuous setting requires verifying that the supremum over $a\in\A$ in \cref{eq:main_bellman_stop_continue} is attained. Our proof (\cref{app:optimal_policy}) establishes this through a value-iteration construction that also handles the coupling with the stopping/continue structure, and  propagates lower semicontinuity from the observation model through the posterior predictive kernel to the $Q$-function. This semicontinuity chain is specific to this problem and does not follow from any existing reference by specialization.


\noindent{\bf Zero duality gap and $(\epsilon,\delta)$-correctness.}
The Bellman theorem characterizes the inner problem for a fixed multiplier $\lambda$. We now state when the Lagrangian relaxation is exact. Let $c(\pi)\coloneqq\mathbb E^\pi[\tau_\pi]$, $\rho(\pi)\coloneqq\mathbb E^\pi[r_{\tau_\pi}(H_{\tau_\pi})]$, and ${\cal K}\coloneqq\{(c(\pi),\rho(\pi)):\pi\in{\cal T}\}$, where $\tau_\pi$ is the stopping time of a policy whose action space embeds the stopping decision. The time-sharing assumption, stated formally in \cref{assump:time_sharing}, says that ex-ante randomization between two admissible policies remains admissible and therefore convexifies ${\cal K}$.

\begin{theorem}[Zero duality gap and $(\epsilon,\delta)$-correctness]\label{thm:main_zero_gap_correctness}
Assume time-sharing (\cref{assump:time_sharing}), and assume strict feasibility (\cref{assump:strict_feasibility_perturbation}), i.e.,  there exists a feasible $\pi_{\rm sf}\in{\cal T}$ such that $\rho(\pi_{\rm sf})>1-\delta$. Then, the duality gap is zero.
Furthermore, let $g(\lambda)\coloneqq\inf_{\pi\in{\cal T}}\{c(\pi)+\lambda(1-\delta-\rho(\pi))\}$, $\lambda^\star\in\argmax_{\lambda\geq0}g(\lambda)$, and let ${\cal S}(\lambda^\star)$ be the set of dual minimizers at $\lambda^\star$. If there exists $\epsilon_0>0$ such that
\[
{\cal K}_{\epsilon_0}(\lambda^\star)\coloneqq\{(c,\rho)\in{\cal K}:c+\lambda^\star(1-\delta-\rho)\leq g(\lambda^\star)+\epsilon_0\}
\]
is closed in $\mathbb R^2$, then there exists a dual-optimal policy $\pi^\star\in{\cal S}(\lambda^\star)$ that is primal optimal. Consequently, with the posterior-optimal inference rule $I_t^\star(h)\in\argmax_{x\in\X}q_t(h,x)$,
\[
\mathbb P_{\theta\sim\nu}^{\pi^\star}\left(L_\theta\left(I_{\tau_{\pi^\star}}^\star(H_{\tau_{\pi^\star}})\right)\leq\epsilon\right)\geq 1-\delta.
\]
\end{theorem}
This theorem improves on the corresponding result in \citep{russo_learning_2025} in two ways. First, we replace the assumption that the dual-optimal policy is unique with a weaker local closedness condition on the near-optimal set ${\cal K}_{\epsilon_0}(\lambda^\star)$: this allows multiple dual-optimal policies, which is  more natural. Second, while \citep{russo_learning_2025} derives correctness via a monotonicity argument on the optimal cost, our proof (\cref{app:theoretical_results:icpe:fixed_confidence:strict_feasibility_and_correctness}) uses a direct subdifferential characterization  to show that if all near-optimal policies have $\rho<1-\delta$, then every subgradient of the dual value is strictly negative, contradicting optimality. Zero duality gap follows from a standard perturbation argument \citep{rockafellar1998variational,borwein2006convex}; see \cref{app:zero_duality_gap_perturbation}.

\vspace{-5pt}
\section{Continuous ICPE: C-ICPE}\label{sec:method}
\vspace{-5pt}

In this section we describe \cicpe{}, a practical method based on the theory of the previous section. \cicpe{} has three components: inference, stopping, and exploration. Each of these are implemented via learned sequential neural architectures, trained end-to-end from interaction data. Once trained, we use \cicpe{} at deployment time (a.k.a. test-time or inference-time) to perform pure exploration. We now describe the models learned by \cicpe{} and the training procedure. More details   can be found in \cref{sec:appendix:algorithms} of the appendix.

\noindent{\bf Training Protocol.} We adopt a similar meta-training protocol as the one used in \citep{russo_learning_2025} to train the models. Briefly,  we assume access to a simulator over $\nu$ from which we can sample trajectories.  We use a meta-training, where we sample tasks $\theta\sim\nu$ and assume access to a zero-loss target $x_\theta^\star\in\X$, collect trajectories using \cicpe, store the data  in  a replay buffer ${\cal B}$, and perform off-policy updates: (1) a likelihood update of the  parameters of the inference model, (2) a DQN-like update of the critic and (3) an update of the actor based on the learned $Q$ function. After training, we freeze all models; deployment requires no access to $x_\theta^\star$ or the prior.  We now discuss the modeling of these components more in detail.

\noindent{\bf Gaussian inference model.}
For a history $h$, the ideal inference rule maximizes the posterior success probability $q_t(h,x)\coloneqq \mathbb P(L_\theta(x)\leq \epsilon|H_t=h)$ over recommendations $x\in \X$. However, computing $q_t$ is not straightforward, as the posterior distribution may have a complex shape. Instead, we train the inference model to learn the posterior law of $x_\theta^\star$ from trajectories and outputs a diagonal Gaussian distribution characterizing the uncertainty around $x_\theta^\star$
\[
 I_\phi(\cdot|h)=\mathcal N\left(\mu_\phi(h),\operatorname{diag}(\sigma^2_\phi(h))\right).
\]
The recommendation at stopping is defined as the mean   $\hat x=\mu_\phi(h)$. The covariance characterizes the uncertainty around this point, and, as shown in \cref{prop:gaussian_nll_moment_projection},  the optimal mean and covariance are the posterior mean and covariance of $x_{\theta}^\star$ given $H_t=h$. Thus the Gaussian is a moment projection of the posterior law of $x_\theta^\star$.  The deployed recommendation $\hat x=\mu_\phi(h)$ should therefore be viewed as a tractable approximation to $\argmax_{x\in\X}q_t(h,x)$, rather than as an exact maximizer of $q_t$. In \cref{app:theoretical_results:gaussian_inference_reward}, we show that $\mu_\phi$ is near-optimal when the posterior uncertainty on $x_\theta^\star$ is small relative to the $\epsilon$-success margin (see \cref{prop:nll_mean_near_optimality}). This justification is most direct for localization losses, where
$\X_\epsilon(\theta)$ is a neighborhood of $x_\theta^\star$.

We train $\phi$ using a log-likelihood loss on a batch of partial trajectories $B=(x_i^\star, H_{t_i})_i$ sampled from the buffer ${\cal B}$, where $x_i^\star$ is the optimal point for trajectory $i$ and $t_i$ is a timestep sampled uniformly at random for that trajectory
\begin{equation}\label{eq:loss_inference}
\mathcal L_{\rm inf}(B;\phi)
=- \sum_{i=1}^{|B|}\log  I_\phi\left(x_{i}^\star|H_{t_i}\right).
\end{equation}
In the following, we denote by $\bar \phi$ the target parameter of the inference model, updated via a Polyak update $\bar \phi \gets (1-\tau_I)\bar\phi +\tau_I \phi$ with $\tau_I\in (0,1)$.

\input{sections/algorithm}
\noindent{\bf Critic and reward definition.} We parametrize the critic by $\psi$, and model it with two heads:  a continuation head $Q_\psi(h,a)$ for $a\in\A$ and a stopping head $Q_\psi(h,a_{\rm stop})$. We also define the value: let $a_{\rm tgt}(h')$ be the continuation target action proposed at the next history by the current policy, then the value is defined as
\[
V_{\psi}(h')\coloneqq
\max\left\{Q_{\psi}(h',a_{\rm stop}),Q_{\psi}(h',a_{\rm tgt}(h'))\right\},
\]
Using the definition of the $Q$ function from \cref{thm:main_bellman_greedy}, we learn the paramter $\psi$ using TD-learning. While  the ideal reward would be $r_t(h)=\max_x q_t(h,x)$, to  better capture the uncertainty around the recommendation $\mu_t(h)$, we use a sampled reward from the target inference model $I_{\bar \phi}$:
\begin{equation}\label{eq:sampled_selector_reward}
\hat r(h,\theta)
\coloneqq
\frac{1}{K}\sum_{k=1}^K
{\bf 1}\left\{L_\theta(X^{(k)})\leq\epsilon\right\},
\qquad
X^{(k)}\sim I_{\bar\phi}(\cdot|h).
\end{equation}
Conditionally on $(h,\theta)$, this is an unbiased estimate of the probability that a sample from the inference distribution lies in $\X_\epsilon(\theta)$. After  averaging over the posterior, this reward is
 $\mathbb E_{X\sim I_{\bar\phi}(\cdot|h)}[q_t(h,X)]\leq r_t(h)$, 
so it is a conservative version of the ideal posterior reward and allows the model to capture the current uncertainty in the recommendation rule (\cref{prop:second_moment_robustness_and_gap} in the appendix makes this precise, as the gap between the practical reward $\widehat r_t(h)$ and the ideal Bellman reward $r_t(h)$ is controlled by the second moment of the inference model).

Then, we sample a batch of partial trajectories $B=(H_{t_i}, A_{t_i}, H_{t_i+1},d_{t_i},x_i^\star)_i$ from the buffer, where $t_i$ is a uniformly sampled timestep for the $i$-th trajectory and $d_{t_i}=1$ when the  maximum horizon is reached. Then, we use targets
\begin{align*}
&y_i^{\rm stop}=\hat r(H_{t_i},\theta),\\
&y_i^{\rm cont}=-c+d_{t_i}\hat r(H_{t_i+1},\theta)+(1-d_{t_i})V_{\bar\psi}(H_{t_i+1}),
\end{align*}
where we reparametrized the Lagrange multiplier as a per-step cost $c=1/\lambda$ and the inner Lagrangian objective remains the same. Then, the  critic loss is
\begin{align}\label{eq:critic_loss_mse}
\mathcal L_{\rm critic}(B;\psi)
&= \frac{1}{2|B|}\sum_{i=1}^{|B|}\Big[
\left(Q_\psi(H_{t_i},A_{t_i})-y_i^{\rm cont}\right)^2\\
\qquad &+
\left(Q_\psi(H_{t_i},a_{\rm stop})-y_i^{\rm stop}\right)^2
\Big].
\end{align}
The critic also decides when to stop: at rollout time the current policy $\pi$ first proposes $A_t$, and the learner stops iff
\[
Q_\psi(H_t,a_{\rm stop})\geq Q_\psi(H_t,A_t).
\]
Hence,  at deployment, no access to $\theta$ is required as the stopping comparison uses only $H_t$.

\noindent{\bf Policy and actors.} The policy, or actor rule, decides what continuation action $a\in \A$ to choose next. Depending on the whether $\X=\A$, we propose three possible actor rules.

\noindent \underline{\it Thompson Sampling (TS) rule}: this rule can be used when the recommendation space and the query space coincide $\X=\A$, and eliminates the need for a separate actor. This rule draws inspiration from classical Thompson Sampling \citep{thompson1933likelihood}: the policy is implicitly represented by the inference model, and the actions are sampled according to $A_t\sim I_\phi(\cdot |H_t)$. This is useful when informative queries are themselves plausible recommendations. Early in an episode the exploration is more spread; later, as the posterior target law contracts, TS concentrates near the current recommendation. As target action for the critic we use the mean value of the inference model $a_{\rm tgt}(h)=\mu_t(h)$.
\vspace{-0.5pt}

\noindent \underline{\it Top Two Posterior Sampling (TTPS) rule}: also this rule can be used when the recommendation space and the query space coincide $\X=\A$. This rule is similar to  classical TTPS \citep{russo_simple_2016}, but we extend it to the continuous case. This rule  draws a sample $A_t\sim I_\phi(\cdot |H_t)$ and, with probability $1/2$, it samples until the new sample  is farther from the mean $\mu_\phi(H_t)$. The logic is that the posterior mean is the current recommendation, while samples farther from it represent plausible alternatives. TTPS therefore spends some probability mass checking alternatives instead of repeatedly querying near the current mean before the stopping critic is confident. As target action for the critic we use the mean value of the inference model $a_{\rm tgt}(h)=\mu_t(h)$.
\vspace{-0.5pt}

\noindent \underline{\it  TD3 rule} \citep{fujimoto2018addressing}: this rule can be used for general recommendation spaces when $\A\neq\X$, and formally tries to solve the Bellman equation in \cref{thm:main_bellman_greedy}. It learns a parametric  actor $\pi_\vartheta(h)\in\A$ from the critic.  The deterministic actor is trained by
\begin{equation}\label{eq:actor_loss_td3}
\mathcal L_{\rm act}(B;\vartheta)=-\frac{1}{|B|}\sum_{i} Q_\psi(H_{t_i},\pi_\vartheta(H_{t_i})).
\end{equation}
The critic target uses the usual TD3 stabilizers: a target actor, target-action smoothing, and twin critics. In case $\X=\A$ we can use a stochastic TD3 variants, where TD3 learns the mean $\bar \mu$ and covariance $\bar \Sigma$ of a Gaussian actor. In this case the  loss is augmented with an imitation learning loss ${\rm KL}(I_\psi(\cdot|h) \| \pi_{\vartheta}(\cdot|h))$ that provides a rich signal:  the actor can increase variance when sampled actions have higher continuation value and shrink it when exploration is no longer useful.


\noindent{\bf Cost update.}  We update the per-step cost $c$ by simply performing a gradient step on the dual variable. We sample a fresh batch of trajectories, and estimate the success rate $\hat p =\frac{1}{|B|}\sum_{i=1}^{|B|}
{\bf 1}\left\{\mu_\phi(H_\tau^{(i)})\in\X_\epsilon(\theta^{(i)})\right\}$, 
and update the cost as follows
\begin{equation}\label{eq:cost_update}
c\gets {\rm Proj}_{[0,1]}\left(c-\eta_c\left((1-\delta)-\hat p\right)\right).
\end{equation}
If empirical success is below $1-\delta$, the cost decreases and trajectories become longer; otherwise, if above the target, the cost increases and stopping becomes more aggressive. 

\noindent{\bf Correctness certification.}
The zero-duality result from the previous section justifies the ideal Lagrangian objective, but a trained model is still approximate. In \cref{app:training_time_certification} we outline how to obtain formal $(\epsilon,\delta)$-guarantees on the trained model (see \cref{prop:checkpointwise_mixture_certification}).

%% file: sections/algorithm.tex
\begin{algorithm}[t!]
 \footnotesize 
   \caption{\cicpe{}  }
   \label{algo:icpe_fixed_confidence}
\begin{algorithmic}[1]
    
   \Statex \texttt{\color{blue}//  Training phase}
     \State Initialize buffer ${\cal B}$, networks $Q_\psi,I_\phi$, actor $\pi$.
   \While{Training is not over}

   \State Sample environment $M_\theta\sim \nu$ and hypothesis $x_\theta^\star$; observe $Y_1\sim\rho$ and  set $t\gets 1$. 
   \Repeat
   \State Execute action $A_t \sim \pi(\cdot\mid H_t)$ according to actor $\pi$ and observe $Y_{t+1}$.
   
   \State Add partial trajectory $(H_t,A_t,Y_{t+1},x_\theta^\star)$ to ${\cal B}$ and set $t\gets t+1$.
   \Until{$Q_{\psi}(H_t,a_{\rm stop}) \geq Q_\psi(H_t,A_t)$.}
   \State In the fixed confidence,   update  $c$ according to \cref{eq:cost_update}.
    \State Sample batch $B\sim {\cal B}$ and update models using ${\cal L}_{\rm inf}(B;\phi)$ (\cref{eq:loss_inference}) and ${\cal L}_{\rm critic}(B;\psi)$ (\cref{eq:critic_loss_mse}); for TD3, train $\pi$ according to ${\cal L}_{\rm act}(B;\psi)$ (\cref{eq:actor_loss_td3}).
   \EndWhile
    \Statex \hrulefill 
   \Statex \texttt{\color{blue}//  Inference/Deployment phase (models are fixed here)}
   \State Sample unknown environment $M\sim \nu$ and collect a trajectory $H_\tau$ using $\pi$ (until  $Q_\psi(H_t,A_{t})\leq Q_\psi(H_t,a_{\rm stop})$).
   \State {\bf Return} $\hat x_\tau = \mu_\phi(H_\tau)$ (recommendation).
\end{algorithmic}
\vspace{-4pt}

\end{algorithm}

%% file: sections/empirical_evaluation.tex
\vspace{-5pt}
\section{Empirical Evaluation}\label{sec:empirical_evaluation}
\vspace{-5pt}

\begin{figure*}
    \centering
    \includegraphics[width=0.8\linewidth]{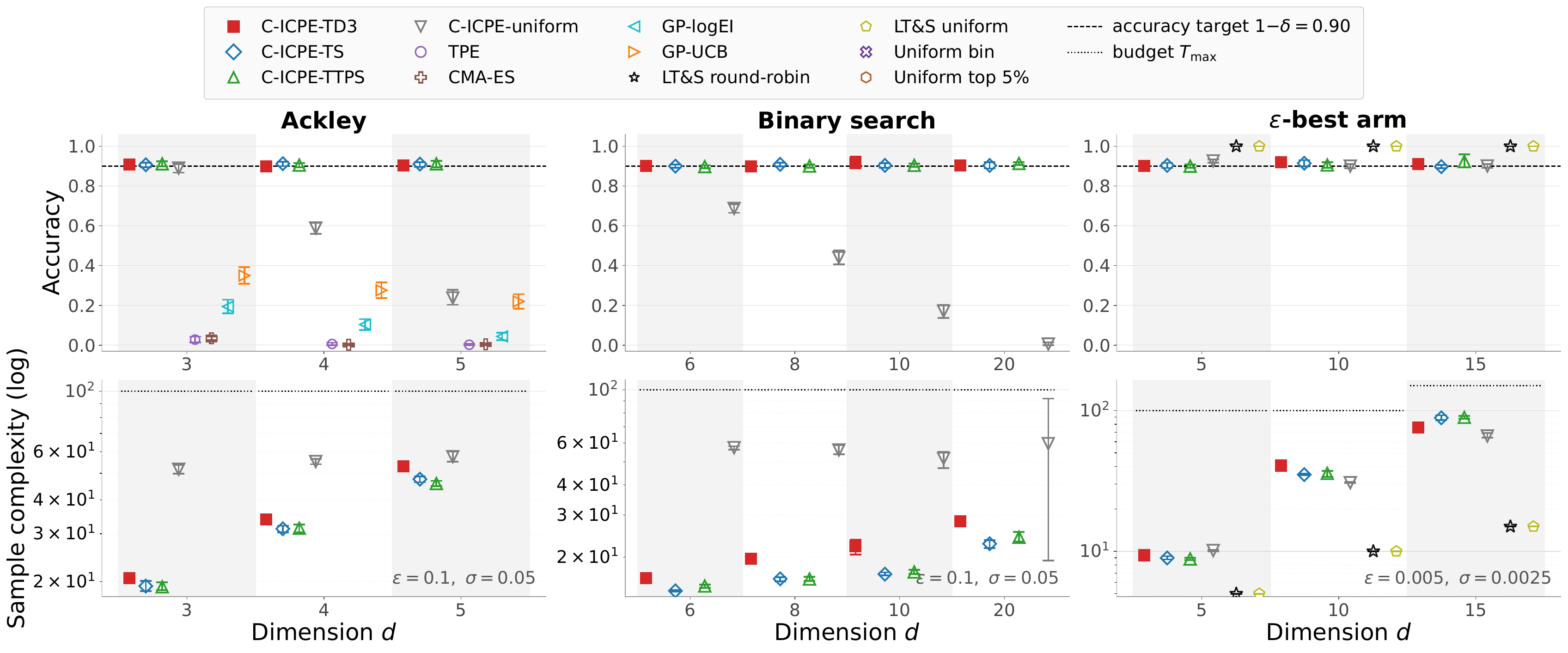}
    \caption{%
  Accuracy (top) and sample complexity (bottom) at the hardest
  $(\varepsilon, \sigma)$ per benchmark.}\label{fig:main_results}
\end{figure*}
We evaluate \cicpe{} on various  benchmarks: noisy binary search, 
$\epsilon$-best arm identification on the unit sphere, Ackley 
minimization, and GP max-value estimation. We also validate on on a real-world 
geochemical exploration task \citep{usgs_geochem}. We compare four exploration rules: TS, 
TTPS, TD3, and uniform sampling. For all experiments we set a maximum 
sample complexity $t_{\max}$ (details in  \cref{sec:appendix:numerical_results}) and report 
$95\%$ confidence intervals using bootstrap.
We also compare against Bayesian optimization baselines whenever possible: 
Tree-structured Parzen Estimator (TPE) \citep{bergstra2011algorithms}, 
Gaussian Process (GP) with UCB or Expected Improvement \citep{srinivas2010gaussian,ament2023unexpected}, and CMA-ES 
\citep{hansen2016cma}. 
Each is given a fixed sample budget larger than the median stopping time 
of \cicpe{}. These methods are not $(\epsilon,\delta)$-correct 
competitors, and optimize a fixed-budget objective. We include them to test whether standard methods
already attains the target correctness at comparable  budgets. 
For $\epsilon$-best arm on the sphere, we additionally compare against 
Lazy Track-and-Stop \citep{jedra2020optimal}, an optimal frequentist 
fixed-confidence baseline.

\vspace{-5pt}
\subsection{Synthetic Benchmarks}
\vspace{-3pt}
We now provide  a brief description of the benchmarks, and then discuss the results.

\noindent {\bf Noisy binary search.}
The environment parameter $\theta$ is drawn uniformly from $[-1,1]^d$, 
with selector $x^\star(\theta) = \theta$. The agent queries $a \in [-1,1]^d$ and observes, per  coordinate, $y_i = \xi_i\cdot\mathrm{sign}(\theta_i - a_i)$, where 
$\xi_i \in \{-1, +1\}$ are i.i.d.\ Rademacher random variables with 
$\mathbb{P}(\xi_i = +1) = 1 - p$. The loss is 
$L_\theta(x) = \|x - \theta\|_2$, so the $\epsilon$-optimal set is 
$\mathcal{X}_\epsilon(\theta) = \{x : \|x - \theta\|_2 \leq 
\epsilon\}$. Here $\mathcal{X} = \mathcal{A} = [-1,1]^d$. The 
difficulty is controlled by the noise rate $p$ and the dimension $d$: 
each coordinate provides one bit of corrupted information per query, and the agent must simultaneously localize all $d$ coordinates.

\noindent {\bf $\epsilon$-best arm on the sphere.}
The environment parameter $\theta$ is drawn uniformly on the unit sphere 
$\mathbb{S}^{d-1}$, with selector $x^\star(\theta) = \theta$. The agent 
queries $a \in [-1,1]^d$ and observes a noisy linear reward 
$y = f_\theta(a) + \xi$, where $f_\theta(a)=\theta^\top a$ and $\xi \sim \mathcal{N}(0, \sigma^2)$. The 
loss is defined via the inner product: 
$L_\theta(x) = 1 - f_\theta(x)$, so the $\epsilon$-optimal set is 
$\mathcal{X}_\epsilon(\theta) = \{x : f_\theta(x) \geq 1 - 
\epsilon\}$. Here $\mathcal{X} = \mathcal{A}$, and the difficulty lies in estimating a direction from scalar projections.

\noindent {\bf Ackley minimization.}
The agent must locate the  global minimizer of a randomly parametrized Ackley 
function \citep{naser2025review}, a standard multimodal 
benchmark for global optimization. The  parameter is 
$\theta = (\alpha, \beta, \gamma, \theta^\star)$, where 
$(\alpha, \beta, \gamma)$ control the function shape and 
$\theta^\star \in [-1,1]^d$ is the global minimizer (selector 
$x^\star(\theta) = \theta^\star$). The agent queries $a \in [-1,1]^d$ 
and observes a normalized function evaluation 
$y = \tilde{f}_{\alpha,\beta,\gamma}(a - \theta^\star) + \xi$, 
$\xi \sim \mathcal{N}(0, \sigma^2)$. The loss is 
$L_\theta(x) = \|x - \theta^\star\|_2$. Here 
$\mathcal{X} = \mathcal{A} = [-1,1]^d$. The difficulty arises from the 
function's many local optima and nearly flat outer region, which can 
trap greedy strategies; the agent must explore globally before 
converging.

\noindent {\bf GP max-value estimation.}
A function $f$ is sampled from a Gaussian process 
$\mathrm{GP}(0, k_{\mathrm{RBF}}(\ell, \sigma_f))$ on $[0,1]^d$, with 
lengthscale $\ell\sim {\rm Unif}[0.05,0.2]$ and output scale $\sigma_f=1$. The target is the scalar maximum value 
$\theta^\star = \max_{x} f(x)$, with selector 
$x^\star(\theta) = \theta^\star \in \mathbb{R}$. The agent queries 
$a \in [0,1]^d$ and observes $y = f(a) + \xi$, 
$\xi \sim \mathcal{N}(0, \sigma^2)$. The loss is 
$L_\theta(x) = |x - \theta^\star|$. This is the 
$\mathcal{X} \neq \mathcal{A}$ setting: the recommendation space 
$\mathcal{X} \subseteq \mathbb{R}$ is scalar while the action space 
$\mathcal{A} = [0,1]^d$ is $d$-dimensional, requiring the TD3 actor to 
learn an exploration policy decoupled from the inference model. The 
difficulty is twofold: the agent must both explore to find the 
region of high function values and estimate the peak value to within 
$\epsilon$, without knowing the function's 
lengthscale in advance. For this problem we compare against two non-parametric baselines: 
(1) uniform sampling over the domain, reporting the trimmed mean of 
the top-$5\%$ observed values; (2) partitioning the domain into 
uniform bins, sampling uniformly within each bin, and reporting the 
highest bin average as the value estimate.

\noindent {\bf Results.} Figs. ~\ref{fig:main_results}-\ref{fig:gp_geochem_result} report accuracy and 
sample complexity  (and their 95\% confidence intervals) at the hardest $(\varepsilon, \sigma)$ per benchmark; 
full sweeps over $(\varepsilon,\sigma,d)$, experimental details, and robustness to prior 
misspecification are in \cref{sec:appendix:numerical_results}.
Across all four tasks, C-ICPE with learned 
exploration (TS, TTPS, or TD3) consistently meets the $1-\delta$ 
accuracy target, while BO baselines and non-parametric estimators fall 
well below,  confirming that fixed-budget optimization does not yield 
$(\varepsilon,\delta)$-correctness. C-ICPE-uniform is competitive at 
low $d$ but degrades as dimension increases on Ackley and binary search, 
where directed exploration matters. On $\varepsilon$-best arm, uniform 
exploration is optimal by rotational symmetry \cite{jedra2020optimal}, and Lazy Track-and-Stop achieves 
optimal sample complexity by exploiting the linear structure; C-ICPE 
matches the accuracy target but uses more samples, reflecting the cost 
of a model-agnostic stopping rule. On GP max-value estimation 
($\mathcal{X}\neq\mathcal{A}$), C-ICPE-TD3 meets the target while 
non-parametric baselines fall short; BO methods are inapplicable here 
as they return locations rather than values. Particularly, in \cref{subsec:sample_complexity_value_vs_argmax} we prove that (\cref{thm:separation}), under an RBF-GP prior 
with interior regularity, max-value estimation is asymptotically 
harder than argmax localization: we establish a sample complexity 
lower bound for value estimation and an upper bound for argmax 
identification via a two-stage algorithm (T-BAL; \cref{alg:argmax}), showing that the 
geometric structure of the function helps localization but not value 
estimation.

\vspace{-8pt}
\subsection{Geochemical Exploration}\label{subsec:geochemical}
\vspace{-3pt}
\begin{figure*}
    \centering
     \begin{minipage}{0.3\linewidth}
         \centering
         \includegraphics[width=\linewidth]{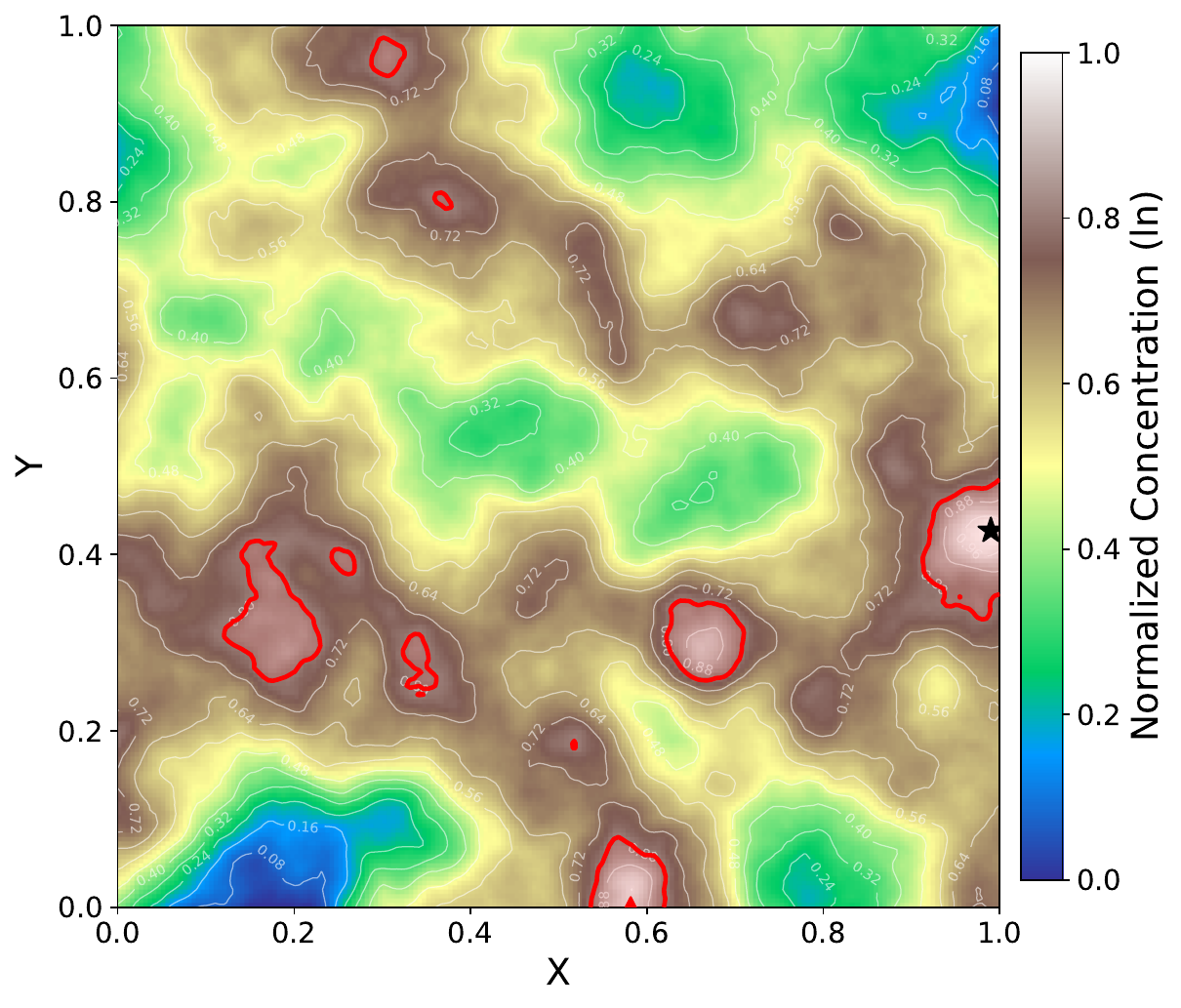}
         \caption{ copper concentration in a 2D region in the geochemical exploration task. Red regions indicate concentration of copper within $\epsilon$ of the maximum value.}
         \label{fig:placeholder}
     \end{minipage}
 \hfill
     \begin{minipage}{0.55\linewidth}
         \centering
        \includegraphics[width=\linewidth]{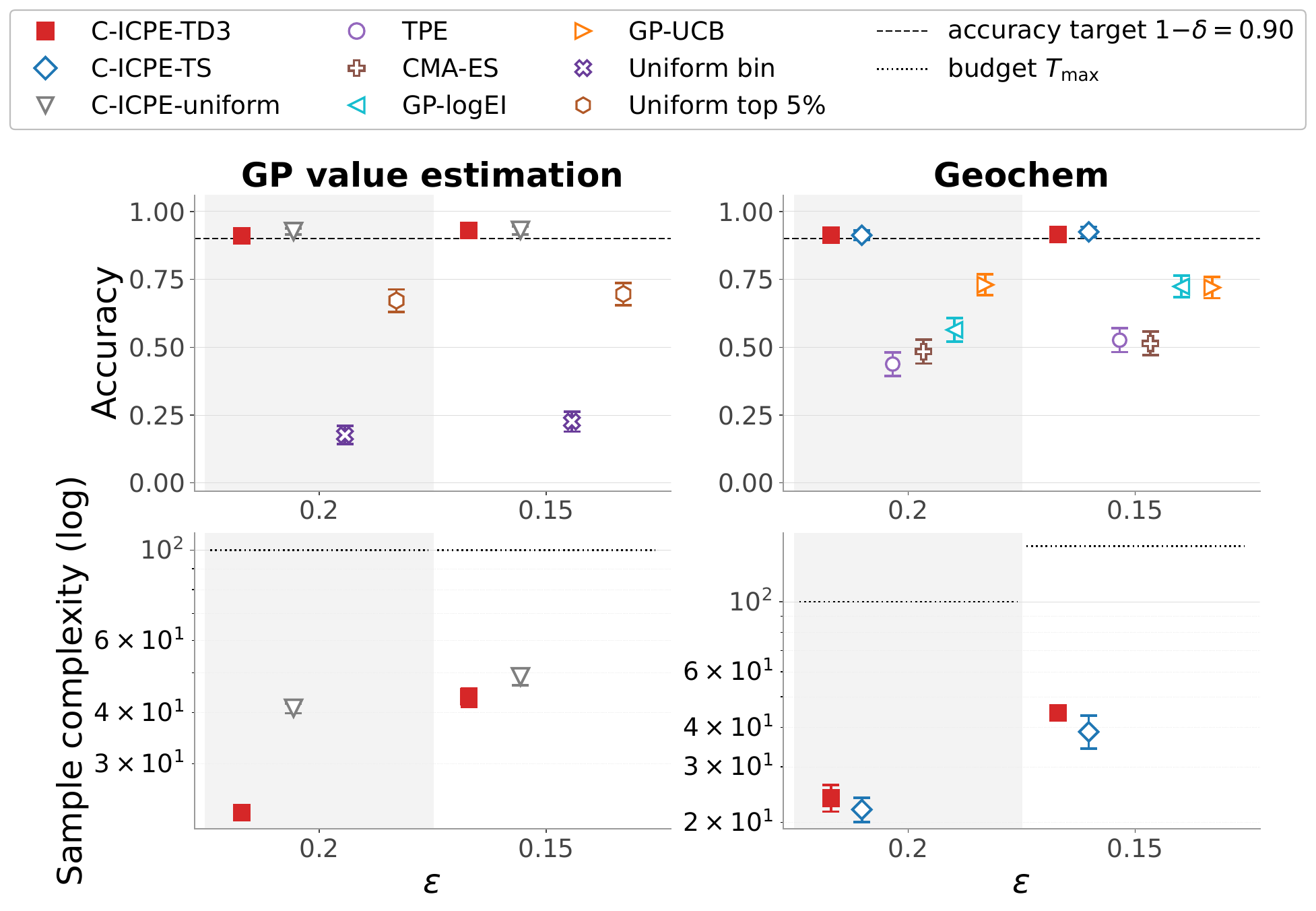}
        \caption{%
  Accuracy (top) and sample complexity (bottom) at the hardest
  $\sigma$ per benchmark.}
        \label{fig:gp_geochem_result}
    \end{minipage}
    \vspace{-7pt}
\end{figure*}


Lastly, we construct a realistic task using data from the 
USGS  Geochemical Survey \citep{usgs_geochem}, which provides  
 measurements of copper concentration across the  
United States. The goal is to identify the location of peak copper 
concentration in an unknown region with $(\epsilon,\delta)$-guarantees. We 
partition the data into geographic regions, fit a sparse variational 
Gaussian process to each, and split regions into training and evaluation. 
Training regions are used to meta-train \cicpe{}; we evaluate on held-out 
regions whose spatial structure was not seen during training. 
See also \cref{app:geochem} for more details.

\noindent {\bf Results.} In \cref{fig:gp_geochem_result} we present the results. In this problem, for sake of simplicity we only test  the TD3 and TS actors for ICPE, and compare with respect to classical Bayesian baselines. The results show accuracy and sample complexity on held-out regions:
C-ICPE-TD3, C-ICPE-TS achieve the $1-\delta$ accuracy target, while
BO baselines  fall below the target while using a larger sample budget, 
demonstrating that both learned exploration and learned stopping 
contribute on this real-data task. This shows evidence that \cicpe{} 
transfers across tasks with genuine distribution shift. In \cref{app:geochem} we report the experimental details, and results for different values of $\epsilon$.

%% file: sections/conclusions.tex
\vspace{-10pt}
\section{Discussion, Related Work and Conclusions}\label{sec:conclusions}
\vspace{-5pt}
 Active sequential hypothesis testing (ASHT) provides the broad conceptual umbrella for this paper. In ASHT, a learner adaptively selects experiments and decides when to stop and declare a hypothesis, with the objective of minimizing expected sample size subject to a correctness constraint \citep{chernoff_sequential_1959,wald_optimum_1948,ghosh1991brief,naghshvar2013Active,naghshvar2012noisy}. A key methodological theme in this literature is that fixed-confidence constraints can be handled via Lagrangian duality. Closely related line of works include Bayesian experimental design and Bayesian active learning, which study adaptive data acquisition when the unknown is drawn from a known prior, typically optimizing expected utility or information gain \citep{lindley1956measure,golovin2011adaptive,rainforth2024modern}, and active learning, which emphasizes selecting informative queries/labels to reduce uncertainty efficiently \citep{cohn1996active}. 
Despite the shared emphasis on adaptive measurement and sequential stopping, most classical ASHT results assume substantial knowledge of the observation model: one typically has access to likelihoods (or at least to a parametric family) for every experiment under every hypothesis, enabling explicit likelihood-ratio statistics and model-based allocation \citep{naghshvar2013Active}. While there are efforts toward relaxing this assumption to partial model knowledge \citep{cecchi2017adaptive}, the need for explicit likelihood structure remains a limiting factor for modern continuous environments with complex, history-dependent feedback. In many practical settings, the learner must instead infer both (i) which latent environment/task it is facing and (ii) which actions are informative, using only interaction data and function approximation. This motivates data-driven approaches that preserve the ASHT objective while reducing dependence on  fully specified likelihoods.

The most developed special case of ASHT is fixed-confidence pure 
exploration in bandits, where the hypothesis is the identity of an 
optimal action. In finite-armed bandits \citep{lattimore2020bandit}, best-arm identification (BAI) at $\epsilon=0$ is characterized by a mature theory: instance-dependent lower bounds quantify the intrinsic complexity of identifying the best arm \citep{garivier2016optimal,degenne_non-asymptotic_2019,wang_optimal_2020,jedra2020optimal,kocak_best_2021,russo2023sample,poiani_best-arm_2025,pmlr-v258-russo25a}, and a family of algorithms achieves near-optimal sample complexity by coupling adaptive allocation with statistically valid stopping rules \citep{audibert2010best,russo_simple_2016,garivier2016optimal,wang2021Fast,jourdan_top_2022}. These results provide both sharp guidance and strong baselines, but they  rely on a finite decision set and model-specific likelihood constructions. Similar themes arise in pure exploration for Markov decision processes, 
where the goal is to identify an optimal policy with probability at 
least $1-\delta$ 
\citep{al2021navigating,taupin2023best,amini_complexity_2023,russo2023model,russomulti}.
This literature yields sharp insights into exploration complexity but 
is likewise developed for finite state-action structure and frequentist 
guarantees.

Even under well-specified models, moving from $\epsilon=0$  to $(\epsilon,\delta)$-PAC identification in continuous decision spaces can be technically demanding.
Several recent works address continuous pure exploration in bandit models
\citep{garivier2021nonasymptotic}. \citet{takemori2025instance} give a 
tractable algorithm for continuous-arm linear bandits and 
\citet{poiani2025Pure} derive lower bounds and a Track-and-Stop 
framework for infinite-answer problems; both are frequentist, 
model-specific, and require explicit likelihood structure. In MDPs, some  work begun to treat $(\epsilon,\delta)$-PAC objectives in finite MDP settings for best policy identification \citep{tirinzoni2022near} and  optimal data-collection for policy evaluation \citep{russo_adaptive_2025}, both in the frequentist setting.
 To our 
knowledge, no general Bayesian theory is known for continuous 
recommendation spaces under general priors; current Bayesian 
fixed-confidence results are limited to finite bandits with Gaussian 
likelihoods and Gaussian priors \citep{jang2024Fixed}.
While Bayesian ideas drive exploration in finite bandit (with frequentist guarantees), e.g. 
posterior sampling and top-two methods 
\citep{russo2014learning,russo_simple_2016,shang2020ttps}, the most developed 
Bayesian framework in continuous spaces is Bayesian optimization (BO) 
\citep{garnett_bayesoptbook_2023}, which maintains a posterior over an 
unknown objective and selects queries via acquisition functions 
\citep{hernandez-lobato_predictive_2014,hennig2012Entropy}. BO aims to identify an optimizer, but is typically 
posed as fixed-budget optimization without a correctness constraint.
\citet{wilson2024stopping} recently introduced a Bayesian 
$(\epsilon,\delta)$-stopping rule for BO, but it is restricted to GP 
surrogates and does not optimize sample complexity. 
Despite this progress, no existing method combines three elements:  fixed-confidence $(\epsilon,\delta)$ stopping, continuous  recommendations, and a learned  exploration and recommender procedure under a Bayesian formulation. \cicpe{} addresses this  gap.

\noindent{\bf Conclusions.}
\cicpe{} is a theory-inspired  method for Bayesian 
fixed-confidence pure exploration with continuous recommendations. On 
the theoretical side, we establish that the Lagrangian duality and 
Bellman optimality structure of finite ASHT carries over to continuous 
spaces under  regularity conditions, and prove 
$(\epsilon,\delta)$-correctness under a local closedness assumption 
that is weaker than the uniqueness condition required in \citep{russo_learning_2025}. On the algorithmic side, we show that \cicpe{} achieves the desired guarantees across different  
tasks, including a real-world geochemical exploration problem,
while using  fewer samples than  standard optimization baselines. To our knowledge, no prior method 
combines continuous recommendations, fixed-confidence stopping  in a single practical framework. 
Limitations and broader impact are discussed in 
\cref{app:limitations}.

%% file: sections/appendix/appendix.tex
\onecolumn
\appendix
\makeatletter
\let\addcontentsline\saved@addcontentsline
\makeatother

\etocsettocstyle{\section*{Table of Contents for the Appendix}}{}
\etocsettocdepth{subsection}
\tableofcontents
\newpage
\input{limitations_broader_impact}
\newpage
\input{sections/appendix/theory}

\newpage
\input{sections/appendix/algorithms}
\newpage
\input{sections/appendix/results}

%% file: limitations_broader_impact.tex
\section{Limitations and Broader Impact}\label{app:limitations}

\paragraph{Limitations.}

\noindent \underline{\it Parametric inference model.}
\cicpe{} models the posterior law of the target $x_\theta^\star$ with a
diagonal Gaussian. In the ideal NLL objective, this corresponds to a
moment projection: the mean matches $\mathbb E[x_\theta^\star|H_t]$ and the
diagonal covariance matches the posterior coordinate variances. This is only a
surrogate for the Bayes $\epsilon$-rule $\argmax_{x\in\X}q_t(H_t,x)$. When the
posterior over $x_\theta^\star$ is multimodal, or poorly
summarized by first and second moments, the mean recommendation may lie between
plausible targets and the covariance may misrepresent uncertainty. This can
affect both exploration and stopping, since the critic evaluates samples from
the same inference distribution. Richer posterior families, such as mixtures or
normalizing flows, or a direct model of $q_t(h,x)$, could reduce this mismatch
at the cost of additional optimization and training complexity.

Note that this limitation is most benign in localization benchmarks, where
$\X_\epsilon(\theta)$ is a ball or interval around $x_\theta^\star$. It is more
pronounced in value-gap tasks such as Ackley or geochemical optimization, where
the $\epsilon$-optimal set can be anisotropic, nonconvex, or multimodal. These
tasks are therefore useful stress tests for the Gaussian inference model.

\noindent \underline{\it  Selector availability.}
Training the inference model requires access to the selector 
$x^\star(\theta)$ for each sampled task $\theta \sim \nu$. In our 
benchmarks this is available in closed form (the shifted minimizer for 
Ackley, the parameter itself for binary search, the GP maximum for 
max-value estimation). In general, computing $x^\star(\theta)$ may 
require numerical optimization, introducing approximation error in the 
NLL targets. If the selector can only be evaluated approximately or 
with noise, the inference model may learn a biased posterior, 
potentially affecting both recommendation quality and stopping 
calibration. Extending \cicpe{} to settings where only noisy or 
approximate selectors are available is an important direction for 
future work.

\noindent \underline{\it Cost  calibration.}
The dual variable $c$, which controls the exploration--stopping 
tradeoff, is updated online during training via primal feedback on the 
empirical success rate. In practice, the convergence of $c$ and the 
sensitivity of stopping behavior to its value require careful tuning of 
learning rates and update schedules. 

\noindent \underline{\it Decoupled action and recommendation spaces.}
When $\mathcal{X} \neq \mathcal{A}$, \cicpe{} requires a separate TD3 
actor to learn the exploration policy. This adds architectural 
complexity and an additional source of approximation error. Our GP 
max-value experiment exercises this setting.

\noindent \underline{\it Bayesian guarantee and prior dependence.}
The $(\epsilon,\delta)$-correctness guarantee is average-case under 
the task prior $\nu$. When the deployment distribution differs 
substantially from $\nu$, correctness may degrade. Our robustness 
experiments and the prior-shift bounds in \cref{sec:prior_misspecification} 
provide some quantitative control, but worst-case guarantees for 
individual task instances are not provided. This limitation is shared 
by all Bayesian methods and is analogous to the dependence of 
frequentist methods on their parametric assumptions.

\noindent \underline{\it Scalability.}
 The method's behavior in 
higher dimensions (larger than $d\geq 50$), where posterior concentration is slower and 
exploration is harder, remains to be investigated. The LSTM (or transformer)
architecture scales with the maximum horizon $t_{\max}$, which may need to 
grow with dimension, increasing both training and inference cost.

\paragraph{Broader impact.}
\cicpe{} is a general-purpose tool for adaptive experimentation with 
correctness guarantees. Potential applications include materials 
discovery, dose-finding in clinical trials, environmental monitoring, 
and any setting where sequential experiments are costly and the 
practitioner requires a principled stopping criterion. In such 
settings, reducing sample complexity directly translates to reduced 
cost, time, and resource consumption.

We do not foresee direct negative societal impacts from the method 
itself. However, as with any system that automates experimental 
decisions, users should be aware that the $(\epsilon,\delta)$ guarantee 
is conditional on the modeling assumptions (the task prior $\nu$ and 
the observation model). Deploying \cicpe{} in safety-critical domains, such as clinical dose-finding, would require careful validation 
of these assumptions and, where appropriate, additional safeguards 
beyond the Bayesian guarantee.

%% file: sections/appendix/theory.tex
\section{Appendix: Theoretical Results}\label{sec:appendix:theoretical_results}
\paragraph{Roadmap and novelty guide.}
The theoretical analysis proceeds in four stages. We summarize what is standard and what is new relative to the finite ICPE framework of \citet{russo_learning_2025}.

\begin{itemize}
\item \textbf{\S\ref{app:theoretical_results:icpe:posterior_distribution_and_inference_rule}: Posterior success probability.} We define $q_t(h,x)=\mathbb{P}(L_\theta(x)\leq\epsilon\mid H_t=h)$ and establish its regularity properties. The proofs use standard tools (Radon--Nikodym, reverse Fatou, Portmanteau); the object $q_t(h,x)$ itself, the natural continuous analogue of posterior mass, has not previously been studied in the pure exploration literature.

\item \textbf{\S\ref{app:theoretical_results:icpe:fixed_confidence:dual}: Bellman optimality.} We prove that the optimal value satisfies a stop/continue Bellman equation and that the supremum over continuation actions is attained by a measurable selector. The proof adapts the value-iteration framework of \citet{bertsekas1996stochastic} to our setting, with the main technical content being a semicontinuity induction that threads likelihood continuity through posterior weak continuity, predictive weak continuity, $Q$-function lower semicontinuity, and a sup--inf interchange. None of these steps appear in \citet{russo_learning_2025}, where attainment is automatic for finite action spaces. The resulting dependency chain and its coupling with the posterior success payoff are specific to this problem.

\item \textbf{\S\ref{app:theoretical_results:icpe:fixed_confidence:strict_feasibility_and_correctness}--\ref{app:zero_duality_gap_perturbation}: Correctness and zero duality gap.} We prove $(\epsilon,\delta)$-correctness under a local closedness condition that is strictly weaker than the uniqueness assumption in \citet{russo_learning_2025}. Our proof uses a subdifferential characterization from \citet{hantoute2008Characterizations} to derive a contradiction without the monotonicity argument of \citet{russo_learning_2025}; this is the strongest theorem-level novelty in the appendix. Zero duality gap follows using a perturbation argument \citep{rockafellar1974conjugate}.

\item \textbf{\S\ref{app:training_time_certification}: Model certification.} We introduce a checkpointwise certification protocol based on mixture supermartingales \citep{kaufmann2021mixture}. Unlike the pooled approach in \citet{russo_learning_2025}, it tests each frozen checkpoint independently and requires no monotonicity assumption on the training trajectory.

\item \textbf{\S\ref{app:theoretical_results:gaussian_inference_reward}: Inference model and sampled reward.}
We relate the implemented Gaussian inference model to the ideal posterior quantities used in the Bellman characterization. We show that the sampled reward is an unbiased estimate of the success probability of the stochastic selector and is conservative relative to the ideal reward $r_t(h)$; the gap is controlled by second moments of the inference distribution (\cref{prop:second_moment_robustness_and_gap}). We also characterize the Gaussian NLL as a moment projection of the posterior law of $x_\theta^\star$ and give conditions under which the NLL mean is near-optimal  (\cref{prop:gaussian_nll_moment_projection,prop:nll_mean_near_optimality}).
\item \textbf{\S\ref{sec:prior_misspecification}: Robustness to prior misspecification.} We  analyse the robustness to prior misspecification, and what is the predicted impact on sample complexity and accuracy.

\item \textbf{\S\ref{subsec:sample_complexity_value_vs_argmax}: Sample Complexity of Value Estimation vs Argmax Localization in Gaussian Processes.} In this section we prove that, under a hierarchical RBF-GP prior with a high-probability interior regularity condition, max-value estimation is asymptotically harder than argmax localization. Specifically, under this regularity assumption, we establish that the  value estimation  problem cannot be circumvented by the geometric structure of the function. 
\begin{itemize}
    \item The analysis is based on showing a lower bound on the sample complexity of estimating the max-value, and an upper bound on estimating the argmax.
    \item To this aim, we introduce an algorithm, Two-Stage Bayesian Argmax Localization (T-BAL) \cref{alg:argmax}, for locating the argmax of a Gaussian Process.
    \item T-BAL works in two phases: first, searches the domain for a region where $X^\star$ may be located, and then perform gradient ascent using noisy finite differences to approximate the gradients.
    \item We provide a sample complexity upper bound of T-BAL and provide $(\epsilon,\delta)$-guarantees.
\end{itemize}
\end{itemize}

\subsection{Problem Modeling}
We specialize to the fixed-confidence ($(\epsilon,\delta)$-PAC) setting introduced in Section~\ref{sec:problem_setting}, and provide a self-contained definition of the induced probability measures.

We now provide a formal definition of the underlying probability measures of the problem we consider. To that aim, it is important to formally define what a model $M$ is, as well as the definition of policy $\pi$ and inference rule $I$ (infernece rules are also known as  recommendation rules).

\paragraph{Spaces and histories.}
Let $\Theta\subset\mathbb R^d$ be compact. Let $\A\subset\mathbb R^m$ be a compact action (query) space and $\Y\subset\mathbb R^n$ a compact observation space, each endowed with the Borel $\sigma$-algebra. Let $\X$ be a compact hypothesis/decision space (in our experiments $\X=\A$).
For $t\in\mathbb N$, define the history space
\[
\mathcal H_t \coloneqq (\Y\times\A)^{t-1}\times \Y,
\qquad
h_t=(y_1,a_1,\ldots,a_{t-1},y_t),
\]
with its product Borel $\sigma$-algebra. We also write $\mathcal H_\infty \coloneqq \Y\times(\A\times\Y)^{\mathbb N}$ for infinite histories.
Since $\A,\Y$ are compact metric spaces, $\mathcal H_t$ and $\mathcal H_\infty$ are standard Borel spaces.

\paragraph{Environment (observation model).}
An environment is indexed by $\theta\in\Theta$ and specified by an initial observation law $\rho_\theta\in\Delta(\Y)$ and a sequence of (possibly history-dependent) observation kernels
\[
P_{\theta,t}(\cdot|h_t,a_t)\in\Delta(\Y),\qquad t\ge 1,
\]
such that for every Borel $C\subset\Y$ the map $(h_t,a)\mapsto P_{\theta,t}(C| h_t,a)$ is measurable.
Optionally, one may assume weak continuity in $\theta$. However, we do assume weak continuity in $a$, as this is later used to prove optimality.
\begin{assumption}[Weak continuity of the transition]\label{assump:weak_continuity_transition}
For all $\theta\in \Theta$ we assume $a\mapsto P_{\theta,t}(\cdot|h_t,a)$ to be weakly continuous.\footnote{That is, for all continuous bounded functions $f$ we have that $a\mapsto \int_{{\cal Y}} f(y) P_{\theta,t}({\rm d}y|h_t,a)$ is continuous.}
\end{assumption}

\paragraph{Learner: policy, stopping time, inference rule.}
A (possibly randomized) sampling policy is a sequence of probability kernels
\[
\pi_t(\cdot| h_t)\in\Delta(\A),\qquad t\geq 1,
\]
measurable as maps $\mathcal H_t\to \Delta(\A)$.
Let $H_t=(Y_1,A_1,\ldots,A_{t-1},Y_t)$ be the random history and $\mathcal F_t=\sigma(H_t)$.
A stopping time $\tau$ is defined w.r.t. $(\mathcal F_t)_{t\ge1}$.
An inference rule is a sequence of measurable maps $I_t:\mathcal H_t\to \X$, and the learner outputs
\[
\hat x_\tau \coloneqq I_\tau(H_\tau).
\]

\paragraph{Loss and $\epsilon$-optimal set.}
For each $\theta\in\Theta$, the environment induces a loss function $L_\theta:\X\to[0,\infty)$ with $\inf_{x\in\X}L_\theta(x)=0$.
Define the $\epsilon$-optimal set
\[
\X_\epsilon(\theta)\coloneqq\{x\in\X:\ L_\theta(x)\le\epsilon\}.
\]
In the following we the following conditions on $L_\theta(x)$.
\begin{assumption}\label{assumption:L_continuity_x_measurability_xtheta}
    We assume joint lower-semicontinuity of $(x,\theta)\mapsto L_\theta(x)$   and joint Borel measurability.
\end{assumption}

\begin{remark}[Regularity of selector-based localization losses]
\label{rem:selector_based_loss_regular}
Some of our localization examples define the loss through a selected target
$x_{\rm sel}^\star(\theta)\in\X$, for instance $L_\theta(x)=\|x-x_{\rm sel}^\star(\theta)\|$. In this case, the standing lower-semicontinuity assumption on $L(\theta,x)=L_\theta(x)$ is satisfied whenever the selector $\theta\mapsto x_{\rm sel}^\star(\theta)$ is continuous. There are two standard ways to obtain such a selector. 
\begin{enumerate}
    \item First, if $F(\theta)\coloneqq\argmax_{x\in\X} f_\theta(x) = \{x^\star(\theta)\}$ is singleton for every $\theta$, $f(\theta,x)$ is jointly continuous, and $\X$ is compact, then the maximum theorem implies that the unique optimizer  is continuous, and thus $x_{\rm sel}^\star(\theta)=x^\star(\theta)$ is continuous. This covers the shifted Ackley benchmark when the shift determines a unique optimizer continuously, the linear bandit benchmark on the unit sphere where $x^\star(\theta)=\theta$, and noisy binary search when the target map is continuous. 
    \item Second, if $F(\theta)$ is not singleton but has nonempty compact convex values and is Hausdorff-continuous in $\theta$, then the minimum-norm selector
$
x_{\rm sel}^\star(\theta)\coloneqq\argmin_{x\in F(\theta)}\|x\|^2
$
is well-defined and continuous: closed convex values give uniqueness of the minimum-norm point, while Hausdorff continuity gives stability of this point as $\theta$ varies. 
\end{enumerate}
Thus the selector-based distance loss is jointly continuous in these cases. This selector also gives a canonical target for the Gaussian inference model: the NLL objective learns a moment projection of the posterior law of $x_{\rm sel}^\star(\theta)$, while correctness remains defined through the loss $L_\theta$ and the set $\X_\epsilon(\theta)$.
\end{remark}

\paragraph{Path measures (Ionescu--Tulcea).}
Fix $\theta\in\Theta$ and a policy $\pi$. By the Ionescu--Tulcea theorem, there exists a unique probability measure $\mathbb P_{\theta,t}^\pi$ on $(\mathcal H_t,\mathcal B(\mathcal H_t))$ such that for all cylinder sets
$C=C_1\times B_1\times\cdots\times B_{t-1}\times C_t$
(with $C_i\in\mathcal B(\Y)$ and $B_i\in\mathcal B(\A)$),
\begin{align*}
\mathbb P_{\theta,t}^\pi(C)
&=\int_{C_1}\rho_\theta(dy_1)
\prod_{s=1}^{t-1}\left[
\int_{B_s}\pi_s({\rm d}a_s | h_s)\,
\int_{C_{s+1}}P_{\theta,s}({\rm d}y_{s+1}| h_s,a_s)\right].
\end{align*}
Analogously, one obtains a unique path measure $\mathbb P_{\theta}^\pi$ on $(\mathcal H_\infty,\mathcal B(\mathcal H_\infty))$.

\paragraph{Mixture law over tasks.}
Given a prior $\nu$ on $\Theta$, define the joint law on $\Theta\times\mathcal H_t$ by
\[
\mathbf P_t^\pi({\rm d}\theta,{\rm d}h_t)\coloneqq \nu({\rm d}\theta)\,\mathbb P_{\theta,t}^\pi({\rm d}h_t),
\]
and the trajectory marginal $\mathbb P_t^\pi(\cdot)=\int \mathbb P_{\theta,t}^\pi(\cdot)\,\nu({\rm d}\theta)$.
We use $\mathbb E_{\theta\sim\nu}^\pi[\cdot]$ and $\mathbb P_{\theta\sim\nu}^\pi(\cdot)$ for expectations/probabilities under this mixture.

\paragraph{Fixed-confidence objective.}
The learner is $(\epsilon,\delta)$-correct (under $\nu$) if
\[
\mathbb P_{\theta\sim\nu}^\pi\!\left(\hat x_\tau\in\X_\epsilon(\theta)\right)\ge 1-\delta.
\]
In the fixed-confidence regime, we seek to minimize the expected number of queries subject to $(\epsilon,\delta)$-correctness:
\[
\inf_{\pi,I,\tau}\ \mathbb E_{\theta\sim\nu}^\pi[\tau]
\quad\text{s.t.}\quad
\mathbb P_{\theta\sim\nu}^\pi\!\left(\hat x_\tau\in\X_\epsilon(\theta)\right)\ge 1-\delta.
\]

\subsection{Posterior distribution over the true hypothesis and inference rule optimality}\label{app:theoretical_results:icpe:posterior_distribution_and_inference_rule}
We first record a domination assumption that allows us to express likelihoods w.r.t.\ fixed reference measures.

\begin{assumption}[Domination]\label{assump:domination_param}
There exist probability measures $\lambda_0,\lambda$ on $(\Y,\mathcal B(\Y))$ such that, for all $\theta\in\Theta$, all $t\ge1$, and all $(h_t,a)\in\mathcal H_t\times\A$,
\[
\rho_\theta(\cdot)\ll \lambda_0(\cdot)
\quad\text{and}\quad
P_{\theta,t}(\cdot| h_t,a)\ll \lambda(\cdot).
\]
Let $p_{\theta,0}(y)\coloneqq \frac{{\rm d}\rho_\theta}{{\rm d}\lambda_0}(y)$ and
$p_{\theta,t}(y'| h_t,a)\coloneqq \frac{{\rm d}P_{\theta,t}(\cdot| h_t,a)}{{\rm d}\lambda}(y')$
be versions of the corresponding densities, chosen jointly measurable in their arguments.

\end{assumption}

\noindent\emph{Remark.} The assumption holds, for instance, when all $\rho_\theta$ and $P_{\theta,t}(\cdot| h_t,a)$ admit densities w.r.t.\ a common reference measure (e.g., Lebesgue on $\Y\subset\mathbb R^n$ or counting measure when $\Y$ is finite).

Under \cref{assump:domination_param}, define the (policy-independent) likelihood of a realized history
$h_t=(y_1,a_1,\ldots,a_{t-1},y_t)\in\mathcal H_t$ under parameter $\theta$:
\[
\ell_t(\theta,h_t) \coloneqq p_{\theta,0}(y_1)\prod_{s=1}^{t-1} p_{\theta,s}(y_{s+1}| h_s,a_s),
\qquad h_s=(y_1,a_1,\ldots,a_{s-1},y_s).
\]

We now give a posterior kernel representation that is independent of $\pi$.

\begin{lemma}[Posterior kernel over $\Theta$]\label{lemma:posterior_distribution_param}
Consider \cref{assump:domination_param}. For each $t\in\mathbb N$ there exists a probability kernel $R_t:\mathcal H_t\times\mathcal B(\Theta)\to[0,1]$, independent of $\pi$, such that for every policy $\pi$, all $A\in\mathcal B(\Theta)$ and $Z\in\mathcal B(\mathcal H_t)$,
\[
\mathbf P_t^\pi(\theta\in A, H_t\in Z)
=\int_Z R_t(A| h)\mathbb P_t^\pi({\rm d}h),
\]
where $\mathbf P_t^\pi({\rm d}\theta,{\rm d}h)=\nu({\rm d}\theta)\,\mathbb P_{\theta,t}^\pi({\rm d}h)$ and $\mathbb P_t^\pi$ is its $\mathcal H_t$-marginal.
Moreover, for $\mathbb P_t^\pi$-a.e.\ $h\in\mathcal H_t$,
\[
R_t(A| h)=
\frac{\int_A \ell_t(\theta,h)\nu({\rm d}\theta)}{\int_\Theta \ell_t(\theta,h)\nu({\rm d}\theta)}.
\]
Consequently, for any measurable map $g:\Theta\to\mathcal S$ into a standard Borel space $\mathcal S$ and any $B\in\mathcal B(\mathcal S)$,
\[
\mathbb P(g(\theta)\in B| H_t=h)
=R_t(\{\theta:\ g(\theta)\in B\}| h)
\quad\text{for }\mathbb P_t^\pi\text{-a.e.\ }h.
\]
\end{lemma}

\begin{proof}
Fix $\pi$ and $t$. Define the reference measure on $\mathcal H_t$ (depending on $\pi$)
\[
\nu_t^\pi({\rm d}h_t)\coloneqq \lambda_0({\rm d}y_1)\prod_{s=1}^{t-1}\big[\pi_s({\rm d}a_s| h_s)\,\lambda({\rm d}y_{s+1})\big].
\]
By construction and \cref{assump:domination_param}, $\mathbb P_{\theta,t}^\pi\ll \nu_t^\pi$ for every $\theta$, with Radon--Nikodym density
\[
\frac{{\rm d}\mathbb P_{\theta,t}^\pi}{{\rm d}\nu_t^\pi}(h_t)=\ell_t(\theta,h_t),
\]
which does not depend on $\pi$.
Therefore, for $A\in\mathcal B(\Theta)$ and $Z\in\mathcal B(\mathcal H_t)$,
\[
\mathbf P_t^\pi(\theta\in A, H_t\in Z)
=\int_A\int_Z \ell_t(\theta,h)\,\nu_t^\pi({\rm d}h)\nu({\rm d}\theta),
\]
and
\[
\mathbb P_t^\pi(Z)=\int_Z\int_\Theta \ell_t(\theta,h)\,\nu({\rm d}\theta)\,\nu_t^\pi({\rm d}h).
\]
Hence $\mathbf P_t^\pi(\theta\in A,\cdot)\ll \mathbb P_t^\pi(\cdot)$ and the Radon-Nikodym derivative is the displayed Bayes ratio, which defines the kernel $R_t(A|h)$. Measurability and the fact that $R_t(\cdot| h)$ is a probability measure follow from standard properties of Radon--Nikodym derivatives. Independence of $\pi$ is immediate from the explicit formula.
\end{proof}

In the following we also need to consider in what cases the mapping $h\mapsto R_t(\cdot |h)$ is weakly continuous. To that aim, we require a further assumption.

\begin{assumption}[Likelihood continuity]\label{assumption:likelihood_continuity}
For each state $s<t$, the kernel $p_{\theta,s}(y|h_s,a_s)$ is jointly continuous in $(\theta, y,h_s,a_s)$ with strictly positive density. Hence, $(\theta,h_t)\mapsto \ell(\theta,h_t)$ is jointly continuous.
\end{assumption}

Under this assumption, we have the following.

\begin{lemma}[Weak continuity of the posterior]\label{lemma:weak_continuity_posterior}
Consider \cref{assump:domination_param} and \cref{assumption:likelihood_continuity}. For each $t$ we have that the mapping $h\mapsto  R_t(\cdot |h),\; h\in {\cal H}_t$, is weakly continuous.
\end{lemma}
\begin{proof}
    Consider any sequence $(h_n)_n \in {\cal H}_t$ such that $h_n\to h, \; h\in {\cal H}_t$. Fix $f\in C_b(\Theta)$ (continuous and bounded). Since for $\mathbb{P}_t^\pi$-a.e. $h'\in {\cal H}_t$ we have $R_t({\rm d}\theta|h')=\frac{\ell_t(\theta,h') \nu({\rm d}\theta)}{\int_\Theta \ell_t(\theta,h')\nu({\rm d}\theta)}$, we have that
    \begin{align*}
    \int_\Theta f(\theta) R_t({\rm d}\theta| h_n) = \frac{\int_\Theta f(\theta) \ell_t(\theta,h_n) \nu({\rm d}\theta) }{\int_\Theta \ell_t(\theta,h_n)\nu({\rm d}\theta) }.
    \end{align*}
    Now, since $\Theta$ is compact and $\theta\mapsto f$ is continuous and $\theta\mapsto\ell_t(\theta,h')$ is continuous for each $h'\in {\cal H}_t$, we have that $\sup_{\theta\in \Theta}|f(\theta) \ell_t(\theta,h_n)|<\infty$, therefore by dominated convergence we have
    \[
    \int_\Theta f(\theta) \ell_t(\theta,h_n) \nu({\rm d}\theta)  \to \int_\Theta f(\theta) \ell_t(\theta,h) \nu({\rm d}\theta)  \quad \hbox{ and }\quad \int_\Theta\ell_t(\theta,h_n) \nu({\rm d}\theta) \to \int_\Theta \ell_t(\theta,h) \nu({\rm d}\theta) .
    \]
    Since $\int_\Theta \ell_t(\theta,h) \nu({\rm d}\theta) >0$ by strict positivity of the density $p_s$, we have that
    \[
    \int_\Theta f(\theta) \ell_t(\theta,h_n) \nu({\rm d}\theta) \to \int_\Theta f(\theta) \ell_t(\theta,h) \nu({\rm d}\theta),
    \]
    so $h\mapsto R_t(\cdot |h)$ is weakly continuous.
 \end{proof}

\paragraph{Optimal inference rule.}
Fix $\epsilon>0$ and $t\in\mathbb N$. Recall that, for each $\theta\in\Theta$, the $\epsilon$-optimal set is
$\X_\epsilon(\theta)=\{x\in\X: L_\theta(x)\leq \epsilon\}$.
Given a realized history $h\in\mathcal H_t$, define the \emph{posterior success probability} of recommending $x\in\X$ as
\begin{equation}\label{eq:posterior_success_prob}
q_t(h,x) \coloneqq \mathbb P_{\theta\sim\nu}^\pi\big(x\in\X_\epsilon(\theta)| H_t=h\big)
= R_t\big(\{\theta\in\Theta: L_\theta(x)\leq \epsilon\}| h\big),
\end{equation}
where $R_t(\cdot| h)$ is the posterior kernel from \cref{lemma:posterior_distribution_param}.
We also define
\begin{equation}\label{eq:rt_def}
r_t(h) \coloneqq \sup_{x\in\X} q_t(h,x).
\end{equation}

\begin{lemma}\label{lemma:uppersemicontinuity_q_reward}
    Under \cref{assumption:L_continuity_x_measurability_xtheta}, we have that $q_t(h,x)$ is jointly Borel measurable in $(h,x)$  and upper semicontinuous in $x$ for each fixed $h$. If in addition to \cref{assumption:L_continuity_x_measurability_xtheta} we also assume \cref{assumption:likelihood_continuity}, then $q_t(h,x)$ is jointly upper semicontinuous.
\end{lemma}
\begin{proof}
We prove the 3 properties separately.
\begin{itemize}
    \item Regarding measurability,  note that  $(x,\theta)\mapsto {\bf 1}\{L_\theta(x)\leq \epsilon\}$ is Boreal measurable by \cref{assumption:L_continuity_x_measurability_xtheta}.  Using that for every $A\in {\cal B}(\Theta)$ we have that $h\mapsto R_t(A | h)$ is Borel-measurable, then $q_t(h,x)=\int_\Theta {\bf 1}\{L_\theta(x)\leq \epsilon\} R_t({\rm d}\theta |h)$   is jointly measurable since integration preserves measurability.
    \item  Consider now the u.s.c. property of $x\mapsto q_t(h,x)$ for each $h$. If, for every $\theta$, the map $x\mapsto L_\theta(x)$ is continuous on the compact set $\X$, then $\X_\epsilon(\theta)$ is closed and $x\mapsto \mathbf 1\{x\in\X_\epsilon(\theta)\}$ is upper semicontinuous. Consequently, $x\mapsto q_t(h,x)$ is upper semicontinuous  $\mathbb{P}_t^\pi$-a.s.. To see this, let $(x_n)_n$ be a sequence in $\X$ such that $x_n\to x^\star$. Define $y_n = {\mathbf 1}\{x_n \in \X_\epsilon(\theta)\}$. By Fatou's reverse lemma we have
\[
\limsup_{n} \mathbb{E}_t[y_n|H_t=h] \leq \mathbb{E}_t[\limsup_n y_n |H_t=h]\leq{\mathbb P}_t(x\in \X_\epsilon(\theta)|H_t=h).
\]
where the last inequality follows from the fact that $\limsup_{n} y_n\leq {\mathbf 1}\{x \in \X_\epsilon(\theta)\}$ from the upper semicontinuity. Thus the posterior is upper semicontinuous on $\mathbb{P}_t^\pi$-a.s.
\item Define $C_\epsilon=\{(\theta,x)\in \Theta\times \X: L_\theta(x)\leq \epsilon\}$. By assumption on $L$ (joint continuity), we have that $C_\epsilon$ is closed. Consider any sequence $(h_n,x_n)\to (h,x)$ and define the distribution $\mu_n(\cdot)\coloneqq R_t(\cdot|h_n) \otimes \delta_{x_n}$ and $\mu(\cdot)\coloneqq R_t(\cdot| h_n)\otimes \delta_x$. Since $h\mapsto R_t(\cdot|h)$ is weakly continuous (\cref{lemma:weak_continuity_posterior}, follows from \cref{assumption:likelihood_continuity}), by Portmanteau's lemma we have
\[
\limsup_{n\to\infty} \mu_n(C_\epsilon) \leq \mu(C_\epsilon).
\]
But $\mu_n(C_\epsilon)= \int_\Theta {\bf 1}\{L_\theta(x_n)\leq \epsilon\} R_t({\rm d}\theta| h_n)=q_t(h_n,x_n)$ and $\mu(C_\epsilon)=\int_\Theta {\bf 1}\{L_\theta(x)\leq \epsilon\} R_t({\rm d}\theta| h)=q_t(h,x)$, hence $(h,x)\mapsto q_t(h,x)$ is jointly upper semicontinuous.
\end{itemize}

\end{proof}

 Since $\X$ is compact, by the Extreme Value theorem we have that the supremum in \eqref{eq:rt_def} is attained (so one may replace $\sup$ by $\max$).

\begin{proposition}[Optimal inference]\label{prop:optimal_inference_rule}
Consider a fixed policy $\pi$ and a fixed time $t\in\mathbb N$. Assume \cref{assump:domination_param} and \cref{assumption:L_continuity_x_measurability_xtheta}. Among  measurable inference rules $I_t:\mathcal H_t\to\X$, the maximal value of
$\mathbb P_{\theta\sim\nu}^\pi\big(I_t(H_t)\in\X_\epsilon(\theta)\big)$
is achieved by any rule satisfying, for $\mathbb P_t^\pi$-a.e. $h\in\mathcal H_t$,
\[
I_t(h) \in \arg\max_{x\in\X} q_t(h,x).
\]
Furthermore, a measurable $\argmax$ selector exists and $I_t(h)$ is measurable.
\end{proposition}

\begin{proof}
Fix $\pi$ and $t$, and let $\hat x_t\coloneqq I_t(H_t)$. Using the posterior kernel,
\begin{align*}
\mathbb P_{\theta\sim\nu}^\pi\big(\hat x_t\in\X_\epsilon(\theta)\big)
&=\int \mathbf 1\{\hat x_t\in\X_\epsilon(\theta)\}\,\mathbf P_t^\pi({\rm d}\theta,{\rm d}h)\\
&=\int_{\mathcal H_t}\left[\int_\Theta \mathbf 1\{I_t(h)\in\X_\epsilon(\theta)\}\,R_t({\rm d}\theta| h)\right]\mathbb P_t^\pi({\rm d}h)\\
&=\int_{\mathcal H_t} q_t\big(h, I_t(h)\big)\,\mathbb P_t^\pi({\rm d}h)\\
&\le \int_{\mathcal H_t}\sup_{x\in\X} q_t(h,x)\,\mathbb P_t^\pi({\rm d}h)
=\int_{\mathcal H_t} r_t(h)\,\mathbb P_t^\pi({\rm d}h).
\end{align*}
If $I_t(h)\in\arg\max_{x\in\X} q_t(h,x)$ for $\mathbb P_t^\pi$-a.e.\ $h$, then the inequality holds with equality, yielding the optimal value.

The existence of an $\argmax$ rule, and measurability of $I_t$, follows from an application of 
\citep[Proposition D.5]{hernandez-lerma1996DiscreteTime}.
\end{proof}

We also note the following lower bound on the posterior $q_t$.

\begin{lemma}[Markov lower bound on posterior success]\label{lem:markov_lower_bound}
Fix $\epsilon>0$, $t\in\mathbb N$, and a realized history $h\in\mathcal H_t$. For any decision $x\in\X$, define the posterior success probability
\[
q_t(h,x)\coloneqq \mathbb P_{\theta\sim\nu}^\pi\big(L_\theta(x)\leq \epsilon | H_t=h\big),
\]
and the posterior mean loss
\[
\bar L_t(h,x)\coloneqq \mathbb E_{\theta\sim\nu}^\pi\big[L_\theta(x)| H_t=h\big].
\]
Then
\[
q_t(h,x)\geq 1-\frac{\bar L_t(h,x)}{\epsilon}.
\]
Equivalently, with the (clipped) shaped reward $r_\epsilon(\theta,x)\coloneqq\big[1-L_\theta(x)/\epsilon\big]_+$,
\[
q_t(h,x)\geq \mathbb E_{\theta\sim\nu}^\pi\big[r_\epsilon(\theta,x)| H_t=h\big].
\]
\end{lemma}

\begin{proof}
Since $L_\theta(x)\ge 0$, Markov's inequality yields
\[
\mathbb P_{\theta\sim\nu}^\pi\big(L_\theta(x)>\epsilon| H_t=h\big)\leq \frac{\bar L_t(h,x)}{\epsilon}.
\]
Taking complements gives the first claim. For the second, note that
$\mathbf 1\{L_\theta(x)\le\epsilon\}\geq [1-L_\theta(x)/\epsilon]_+$ pointwise, and take conditional expectations.
\end{proof}

\subsection{Fixed-confidence setting:  formulation and optimal rules}\label{app:theoretical_results:icpe:fixed_confidence:dual}

We consider the fixed-confidence problem from Section~\ref{sec:problem_setting} in its Bayesian (task-averaged) form.
A learner is a triplet $(\pi,I,\tau)$ with sampling policy $\pi$, inference rule $I=(I_t)_{t\ge1}$, and stopping time $\tau$.
The objective is
\begin{equation}\label{eq:problem_fixed_confidence_cont}
\inf_{\pi,I,\tau}\ \mathbb E^\pi_{\theta\sim\nu}[\tau]
\qquad\text{s.t.}\qquad
\mathbb P^\pi_{\theta\sim\nu}\!\big(I_\tau(H_\tau)\in\X_\epsilon(\theta)\big)\ \ge\ 1-\delta,
\quad \mathbb E^\pi_{\theta\sim\nu}[\tau]<\infty.
\end{equation}
Throughout this section, $H_t=(Y_1,A_1,\ldots,A_{t-1},Y_t)$ is the history, $\hat x_\tau = I_\tau(H_\tau)$ and $\mathcal F_t=\sigma(H_t)$.

\paragraph{Posterior success.}
For each $t$ and realized history $h\in\mathcal H_t$, define the posterior success probability of recommending $x\in\X$ as
\[
q_t(h,x)\ \coloneqq\ \mathbb P^\pi_{\theta\sim\nu}\big(x\in\X_\epsilon(\theta)| H_t=h\big),
\]
and recall $r_t(h)\coloneqq \sup_{x\in\X} q_t(h,x)$. We have the following lemma that relates the success probability to the expected posterior success.

\begin{lemma}[Stopped success as expected posterior success]\label{lem:stopped_success_decomp}
For any policy $\pi$, stopping time $\tau$, and inference rule $I$,
\[
\mathbb P_{\theta\sim\nu}^\pi\left(\hat x_\tau\in \X_\epsilon(\theta)\right)=
\mathbb E^\pi\left[q_\tau\left(H_\tau,\hat x_\tau\right)\right].
\]
\end{lemma}
\begin{proof}
By the tower rule and $\hat x_\tau$ being $\sigma(H_\tau)$-measurable,
\[
\mathbb P_{\theta\sim\nu}^\pi(\hat x_\tau\in\X_\epsilon(\theta))
=\mathbb E^\pi\left[\mathbb E^\pi\!\left[\mathbf 1\{\hat x_\tau\in\X_\epsilon(\theta)\}\mid H_\tau\right]\right]
=\mathbb E^\pi\left[q_\tau(H_\tau,\hat x_\tau)\right].
\]
\end{proof}
Therefore, we have that $\mathbb P^\pi_{\theta\sim\nu}\big(I_\tau(H_\tau)\in\X_\epsilon(\theta)\big)
=\mathbb E^\pi\left[q_\tau\big(H_\tau, I_\tau(H_\tau)\big)\right]$.

\paragraph{Lagrangian dual.}
Define, for $\lambda\ge0$, the Lagrangian value
\[
V_\lambda(\pi,I,\tau)\coloneqq\mathbb E^\pi_{\theta\sim\nu}[\tau]\ +\lambda\Big((1-\delta) - \mathbb P^\pi_{\theta\sim\nu}\big(I_\tau(H_\tau)\in\X_\epsilon(\theta)\big)\Big).
\]
Using \cref{lem:stopped_success_decomp}, this can be written as
\begin{equation}\label{eq:lagrangian_value_cont}
V_\lambda(\pi,I,\tau)
=\lambda(1-\delta) + \mathbb E^\pi\left[\tau - \lambda q_\tau\big(H_\tau, I_\tau(H_\tau)\big)\right].
\end{equation}
The Lagrangian dual of \eqref{eq:problem_fixed_confidence_cont} is then
\begin{equation}\label{eq:dual_cont}
\sup_{\lambda\ge0}\ \inf_{\pi,I,\tau}\ V_\lambda(\pi,I,\tau).
\end{equation}

\subsubsection{Optimal Inference Rule}
Fix $\pi,\lambda$ and $t$. Among all measurable inference rules $I_t:\mathcal H_t\to\X$, the maximal probability of $\epsilon$-success at time $t$ is achieved by any
\[
I_t(h)\in \arg\max_{x\in\X} q_t(h,x),
\]
equivalently, $q_t(h,I_t(h))=r_t(h)$ for $\mathbb P_t^\pi$-a.e. $h$.
(see  \cref{prop:optimal_inference_rule}.)

Since $\tau$ is adapted, plugging the optimal inference rule into \cref{eq:lagrangian_value_cont} yields the simplified dual objective
\begin{equation}\label{eq:dual_after_inference}
\sup_{\lambda\ge0} \inf_{\pi,\tau} \lambda(1-\delta) + \mathbb E^\pi\left[\tau - \lambda r_\tau(H_\tau)\right].
\end{equation}

\subsubsection{Stopping as an action (equivalence)} The additional optimization over stopping rules can be avoided by introducing an additional stopping aciton $a_{\rm stop}$.
Introduce an augmented action space $\bar\A\coloneqq \A\cup\{a_{\rm stop}\}$, where choosing $a_{\rm stop}$ terminates interaction (no new observation is collected).
Let $\bar\tau\coloneqq \inf\{t\geq 1: A_t=a_{\rm stop}\}$.

\begin{lemma}[Embedding stopping times as a stop action]\label{lem:stop_as_action}
For every triplet $(\pi,I,\tau)$ with $\tau< \infty$ a.s., there exists a policy $\bar\pi$ on $\bar\A$ such that, under $\bar\pi$,
(i) $\bar\tau=\tau$ a.s., and (ii) the stopped history $H_{\bar\tau}$ has the same distribution as $H_\tau$ under $\pi$.
In particular, for every $\lambda\ge 0$, $V_\lambda(\pi,I,\tau)=V_\lambda(\bar\pi,I,\bar\tau)$.
\end{lemma}
\begin{proof}
Since $\tau$ is a stopping time w.r.t.\ $\mathcal F_t=\sigma(H_t)$, the event $\{\tau=t\}$ belongs to $\mathcal F_t$; hence there exists a measurable set $S_t\subset \mathcal H_t$ such that $\{\tau=t\}=\{H_t\in S_t\}$.
Define $\bar\pi$ as follows:
at time $t$, given history $h\in\mathcal H_t$,
\[
\bar\pi_t(a_{\rm stop}| h)=\mathbf 1\{h\in S_t\},\qquad
\bar\pi_t(\cdot| h)=\pi_t(\cdot| h)\ \text{ on }\A\ \text{ when }h\notin S_t.
\]
Then $\{A_t=a_{\rm stop}\}=\{H_t\in S_t\}=\{\tau=t\}$, so $\bar\tau=\tau$ a.s.
Moreover, on the event $\{\tau>t\}$ the  action distribution and observation kernel coincide with those under $\pi$, so the induced law of $(H_t)_{t\le\tau}$ is the same; in particular $H_{\bar\tau}$ under $\bar\pi$ has the same distribution as $H_\tau$ under $\pi$. Hence, one can easily show that the equality $V_\lambda(\pi,I,\tau)=V_\lambda(\bar\pi,I,\bar\tau)$ follows. 
\end{proof}

\subsubsection{Optimal Policy}\label{app:optimal_policy}
\cref{lem:stop_as_action} shows that (for fixed $\lambda$) the inner problem in \cref{eq:dual_after_inference} can be viewed as an optimal-stopping control problem on the augmented action space:
each continuation step incurs unit cost, while stopping at history $h\in\mathcal H_t$ incurs terminal cost $-\lambda\,r_t(h)$.

Define the optimal cost-to-go (for fixed $\lambda$) from a history $h\in\mathcal H_t$ as
\begin{equation}
V_{t}^\star(h;\lambda) \coloneqq \inf_{\bar\pi=(\bar\pi_i)_{i\geq t}}\ \mathbb E_{\theta\sim\nu}^{\bar\pi}\left[\sum_{s=t}^{\bar\tau-1}1- \lambda r_{\bar\tau}(H_{\bar\tau}) \bigg|\ H_t=h\right],
\end{equation}
where the infimum is over policies on $\bar\A$ and $\bar\tau$ is the first time $a_{\rm stop}$ is chosen.
Similarly to \citep{russo_learning_2025}, we can define the following optimal $Q$-functions
\[
Q_{t,\rm stop}^\star(h;\lambda)\coloneqq -\lambda r_t(h),
\qquad
Q_{t,\rm cont}^\star(h,a;\lambda)\coloneqq 1+\mathbb E\left[V_{t+1}^\star(H_{t+1};\lambda)| H_t=h, A_t=a\right].
\]
where the latter expectation is over the posterior mixture, defined as
\[
\bar{P}_t(y'\in Y| H_t=h, A_t=a) = \int P_{\theta,t}(y'\in Y| H_t,A_t=a)R_t({\rm d} \theta| H_t=h),\quad \forall Y\in {\cal B}({\cal Y}).
\]
Furthermore, similarly to \citep{russo_learning_2025}, a standard decomposition yields the Bellman optimality relation
\begin{equation}\label{eq:bellman_stop_continue}
V_{t}^\star(h;\lambda)=\min\left\{
Q_{t,\rm stop}^\star(h;\lambda), \inf_{a\in\A} Q_{t,{\rm cont}}^\star(h,a;\lambda)
\right\}.
\end{equation}
However, in order to guarantee that the infimum  $, \inf_{a\in\A} Q_{t,{\rm cont}}^\star(h,a;\lambda)$ is attained, since ${\cal A}$ is compact, we need to guarantee that the $Q$-value is lower semicontinuous.  We begin by showing that $V_t^\star$ is lower semicontinuous. To that aim, we need some results first. We begin by showing that the mixture posterior is weakly continuous.
\begin{lemma}[Weak continuity of the mixture posterior]\label{lem:mixture_weak_cont}
Fix $t$ and $h\in\mathcal H_t$. Let $R_t(\cdot| h)$ be the posterior on $\Theta$ and define the posterior predictive kernel
\[
\bar P_t(\cdot| h,a)\coloneqq\int_\Theta P_{\theta,t}(\cdot| h,a)\,R_t({\rm d}\theta| h).
\]
Under \cref{assump:weak_continuity_transition}, $a\mapsto \bar P_t(\cdot| h,a)$ is weakly continuous. If in addition we assume \cref{assumption:likelihood_continuity}, then $(h,a)\mapsto \bar P_t(\cdot|h,a)$ is jointly weakly continuous.
\end{lemma}

\begin{proof}
We prove weak continuity in the action first.
Fix $f\in C_b(\mathcal Y)$ (continuous and bounded) and a sequence $(a_n)_n$ such that $a_n\to a$. For each $\theta$, define
\[
g_n(\theta)\coloneqq\int f(y) \,P_{\theta,t}({\rm d}y| h,a_n),\qquad g(\theta)\coloneqq \int f(y) \,P_{\theta,t}({\rm d}y| h,a).
\]
By \cref{assump:weak_continuity_transition},
\[
g_n(\theta)\to g(\theta).
\]
Moreover, $\big|\int f(y)\, P_{\theta,t}({\rm d}y| h,a_n)\big|\leq \|f\|_\infty<\infty$ for all $\theta,n$, and thus is also bounded.
By dominated convergence,
\begin{align*}
\int f(y)\,\bar P_t({\rm d}y| h,a_n)&=\int_\Theta g_n(\theta) \,R_t({\rm d}\theta\mid h),\\
&\to
\int_\Theta \underbrace{\left(\int f(y)\,P_{\theta,t}({\rm d}y| h,a)\right)}_{=g(\theta)}R_t({\rm d}\theta\mid h),\\
&=\int f(y)\bar P_t({\rm d} y| h,a). \tag{Fubini-Tonelli}
\end{align*}

Regarding the second part, it is less straightforward to prove. We assume  \cref{assumption:likelihood_continuity}.
Recall that $R_t({\rm d}\theta\mid h) =  \frac{ \ell_t(\theta,h) \nu({\rm d}\theta)}{\int_\Theta \ell_t(\theta,h) \nu({\rm d}\theta)}$.  We omit the normalization constant for simplicity, and simply include it in $\ell_t$.

Define some sequence  $(h_n,a_n)\to (h,a)$, and consider
\[
\int f(y)\bar P_t({\rm d}y\mid h_n,a_n) = \int_\Theta \underbrace{\int f(y) P_{\theta,t}({\rm d}y\mid h_n,a_n)}_{\eqqcolon G_f(h_n,a_n,\theta)} \ell_t(\theta,h_n) \nu({\rm  d}\theta)=\int_\Theta G_f(h_n,a_n,\theta)\ell_t(\theta,h_n)\nu({\rm d}\theta).
\]
Clearly $G_f$ is bounded and $\ell_t$ is jointly continuous by assumption. Since $\Theta$ is compact, we have that along any convergence sequence $h_n\to h$, the likelihood $\ell_t$ is uniformly bounded. If $G_f$ is jointly continuous in $(h,a)$, then, by dominated convergence, we obtain
\[
\int_\Theta G_f(h_n,a_n,\theta)\ell_t(\theta,h_n)\nu({\rm d}\theta)\to \int_\Theta G_f(h,a,\theta)\ell_t(\theta,h)\nu({\rm d}\theta)=\int f(y)\bar P_t({\rm d}y\mid h,a), 
\]
which shows the claim. Hence, we need to show that $(h,a)\mapsto P_{\theta,t}(\cdot\mid h,a)$ is weakly continuous for each $\theta$.

Using that ${\rm d}P_{\theta,t}(\cdot\mid h,a) = p_{\theta,t}(y\mid h,a) {\rm d}\lambda$, and recalling that by  \cref{assumption:likelihood_continuity} $p_{\theta,t}$ is jointly continuous. We have
\[
\int f(y) P_{\theta,t}({\rm d}y\mid h_n,a_n)=\int f(y)p_{\theta,t}(y\mid h_n,a_n) {\rm d}\lambda
\]
using again compactness (of $\Y$) and continuity of the arguments, we derive weak continuity of $(h,a)\mapsto P_{\theta,t}(\cdot\mid h,a)$, which concludes the proof.

\end{proof}


The main technical challenge in the continuous case is showing that the infimum $\inf_{a\in\A} Q_{t,\rm cont}^\star(h,a)$ in the Bellman equation is attained. In finite $\A$ this is trivial; in compact continuous $\A$ it requires lower semicontinuity of $a\mapsto Q_{t,\rm cont}^\star(h,a)$. We establish this through a value-iteration construction that propagates lower semicontinuity from the observation model through the posterior predictive to the $Q$-function. The proof adapts the general template of negative dynamic programming \citep{bertsekas1996stochastic,hernandez-lerma1996DiscreteTime} to our stop/continue structure, where the stopping payoff involves $r_t(h)=\sup_x q_t(h,x)$ --- itself a supremum over a continuous set whose upper semicontinuity must be established first (\cref{lemma:uppersemicontinuity_q_reward}). We build a sequence of values $W_t^{(n)}(h)$ and show that these are l.s.c. in $h$.

Starting from $V_t^\star$, we construct $W_t^\star$ and build an increasing sequence $W_t^{(n)}$ from below. We show that each $W_t^{(n)}$ is l.s.c. in $h$ and that for each fixed $h$, the map $a\mapsto Q_t^{(n)}(h,a)$ is l.s.c. on $\A$. We then show that $W_t^{(n)}$ approaches $W_t^\star$. We conclude by showing that $W_t^\star=V^\star+\lambda$, proving that $V_t^\star$ is $l.s.c$.

Define $W_t^{(0)}(h)\coloneqq 0$ for all $t,h$. Recursively for $n\geq 0$, define
\begin{align*}
    Q_t^{(n)}(h,a) &\coloneqq 1+ \int_{\Y}W_{t+1}^{(n)}(h,a,y) \bar P_t({\rm d}y\mid h,a),\\
    C_t^{(n)}(h) &\coloneqq  \inf_{a\in \A}Q_t^{(n)}(h,a),\\
    W_t^{(n+1)}(h) &\coloneqq \min\{\lambda -\lambda r_t(h); C_t^{(n)}(h)\}.
\end{align*}
Also define $W_t^\star(h) = V_t^\star(h;\lambda)+\lambda$. We also define the operator ${\cal T}_t$ as follows:
\[
({\cal T}_t u)(h) \coloneq \min\left\{ \lambda(1-r_t(h)),  1+\inf_{a\in \A}\int_{\Y}u(h,a,y) \bar P_t({\rm d}y\mid h,a)\right\},
\]
therefore $W_{t}^{(n+1)}={\cal T}_{t} W_{t+1}^{(n)}$. One can clearly show that the operator is monotone, as for $u(h)\leq v(h)$ we have $({\cal T}_t u)(h)\leq ({\cal T}_tv)(h)$.  Therefore $U_t(h)\coloneqq \sup_{n\geq 0} W_t^{(n)}(h)$ exists pointwise, and satisfies $0\leq U_t(h)\leq \lambda$.

\begin{lemma}\label{lemma:W_is_lsc}
Assume \cref{assumption:L_continuity_x_measurability_xtheta} and \cref{assumption:likelihood_continuity}.
    For every $t \in \mathbb{N}$, the iterates $W_t^{(n)}$ and $C_t^{(n)}$ are lower semicontinuous, and for each fixed $h$, $a\mapsto Q_t^{(n)}(h,a)$ is lower semicontinuous on $\A$ and attains its minimum.
\end{lemma}
\begin{figure}[t]
    \centering
    \includegraphics[width=\linewidth]{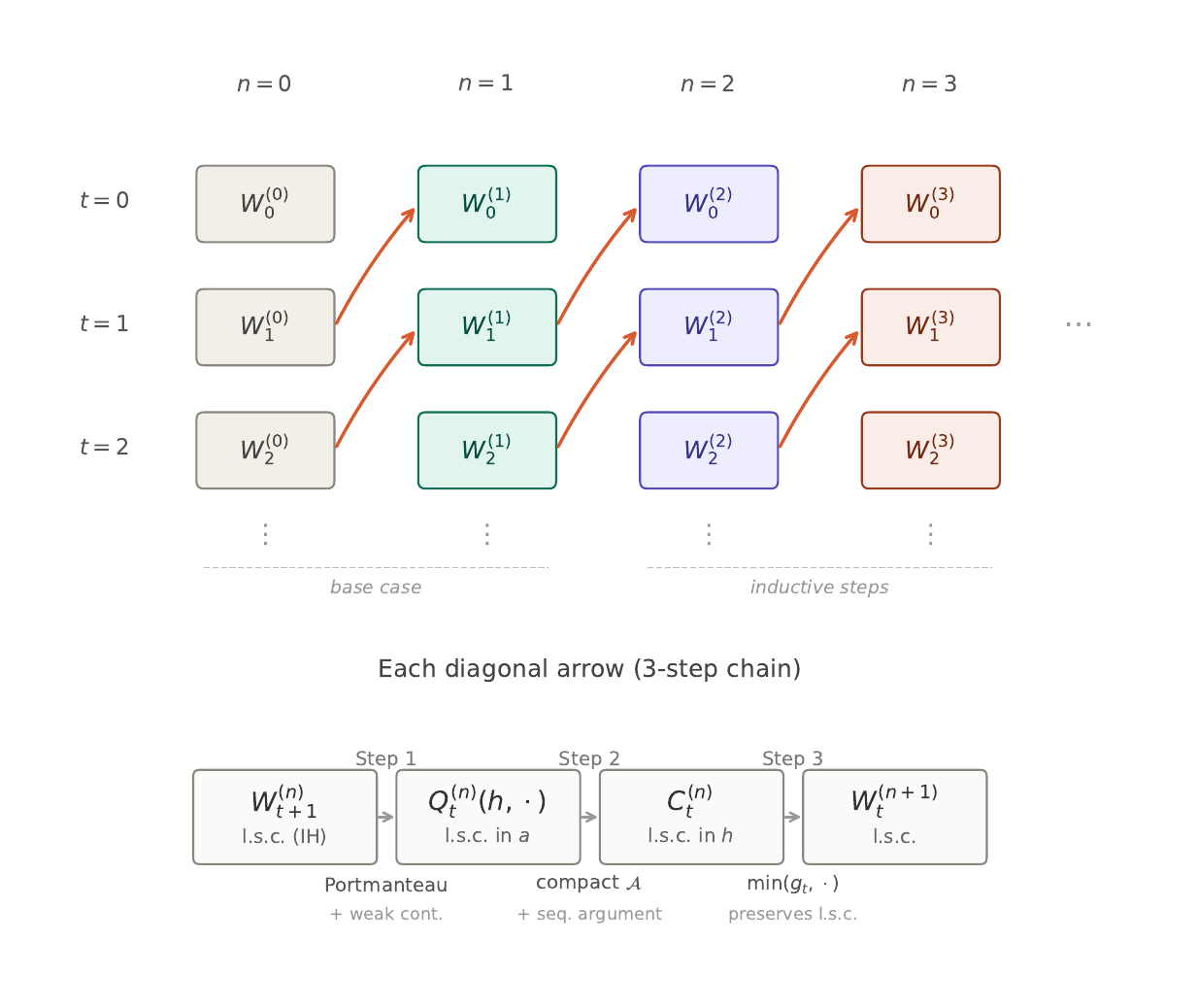}
    \caption{Induction diagram used in the proof of \cref{lemma:W_is_lsc}.}
    \label{fig:induction_diagram}
\end{figure}
\begin{proof}
Let $g_t(h)=\lambda(1-r_t(h))$ and fix $t$.
    We prove the argument by induction on $n$: for a fixed $n$, $W_s^{(n)}$ is l.s.c. for every $s$.
    
    \noindent \underline{\it Base Step.} Clearly, for all $t$, $W_t^{(0)}$ is l.s.c. , and since $Q_t^{(0)}(h,a)=1$ and $C_t^{(0)}(h)=1$, we have $W_t^{(1)}(h)=\min(g_t(h),1)$ which is l.s.c. if $-r_t$ is l.s.c. (which is, by \cref{lemma:uppersemicontinuity_q_reward}).

    Then, assume the induction hypothesis. We use that $W_{t+1}^{(n)}$ is l.s.c. and prove that $W_{t}^{(n+1)}$ is l.s.c. We do so in 3 steps: first we prove that $a\mapsto Q_t^{(n)}(h,a)$ is l.s.c. Then, we prove that $h\mapsto C_t^{(n)}(h)$ is l.s.c. Lastly we show that $W_t^{(n+1)}$ is l.s.c.
    
      \noindent \underline{\it  Step 1.}  Fix $h$ and consider a sequence $a_m\to a$. Define $\mu_m \coloneqq \delta_{a_m}\otimes \bar P_t(\cdot \mid h, a_m)$ and $\mu \coloneqq \delta_a \otimes \bar P_t(\cdot \mid h,a)$.
    By \cref{lem:mixture_weak_cont} we have that $a\to \bar P_t(\cdot\mid h,a)$ is weakly continuous, and thus $\mu_m \Rightarrow \mu$. Since $W_{t+1}^{(n)}$ is l.s.c. and bounded,  the Portmanteau theorem yields
    \[
    \liminf_{m\to \infty} \int_{\Y\times \A} W_{t+1}^{(n)}(h,a',y) {\rm d}\mu_m \geq \int_{\Y\times \A} W_{t+1}^{(n)}(h,a',y){\rm d}\mu.
    \]
    therefore $a\mapsto Q_t^{(n)}(h,a)$ is  l.s.c. on compact $\A$, and the infimum is attained.

      \noindent \underline{\it  Step 2.} We now show that  $h\mapsto C_t^{(n)}(h)$ is l.s.c. Define a sequence $(h_m)_m$ such that $h_m\to h$. Choose a subsequence $(h_{m_k})_{k}$ such that  $C_t^{(n)}(h_{m_k})\to \liminf_{m\to\infty} C_t^{(n)}(h_m)$. For each $k$ choose $a_{m_k}\in \argmin_{a\in \A} Q_t^{(n)}(h_{m_k},a)$. Since $\A$ is compact, by passing to a further subsequence if necessary, we may assume $a_{m_k}\to a$.  Define $\mu_{m_k}\coloneqq \delta_{(h_{m_k},a_{m_k})}\otimes \bar P_t(\cdot\mid h_{m_k},a_{m_k})$ and $\mu\coloneqq \delta_{(h,a)}\otimes \bar P_t(\cdot\mid h,a)$.
      Since by assumption we have that $(h,a)\mapsto \bar P_t(\cdot\mid h,a)$ is jointly weakly continuous, then $\mu_{m_k}\Rightarrow \mu$. Similarly to before, we obtain
      \[
      \liminf_{k\to\infty} Q_t^{(n)}(h_{m_k},a_{m_k}) \geq Q_t^{(n)}(h,a)\geq C_t^{(n)}(h).
      \]
      But  $Q_t^{(n)}(h_{m_k},a_{m_k})=C_t^{(n)}(h_{m_k})$, therefore $\liminf_{m\to\infty} C_t^{(n)}(h_m)=\lim_{k\to\infty} C_t^{(n)}(h_{m_k})\geq C_t^{(n)}(h)$. Therefore $C_t^{(n)}$ is l.s.c. in $h$.

      \noindent \underline{\it  Step 3.} Lastly, consider $W_t^{(n+1)}=\min(g_t(h), C_t^{(n)}(h))$. Since both arguments are l.s.c., then $W_t^{(n+1)}$ is l.s.c. Since $t$ was arbitrary, the statements holds for all $t$.
    
\end{proof}

Then, since each $W_t^{(n)}$ is l.s.c., we get that $U_t(h)$ is l.s.c. (arbitrary suprema of l.s.c. functions are l.s.c.).

\begin{lemma}\label{lemma:QtU_Ut}
   Assume \cref{assumption:L_continuity_x_measurability_xtheta} and \cref{assumption:likelihood_continuity}.
    For every $t \in \mathbb{N}$,  define 
    \[
    Q_t^U(h,a)\coloneqq 1+\int_{\Y} U_{t+1}(h,a,y)\bar P_t({\rm d}y\mid h,a),\qquad C_t^U(h)\coloneqq \inf_{a\in \A} Q_t^U(h,a),
    \]
    then  $U_t(h)=\min\{g_t(h), C_t^U(h)\}$. Furthermore, we have that $a\mapsto Q_t^U(h,\cdot)$ is lower semicontinuous  and $Q_t^U(h,a)$ is jointly Borel measurable.
\end{lemma}
\begin{proof}
Let $g_t(h)=\lambda(1-r_t(h))$ and fix $t$.
    Since $W_{t+1}^{(n)} \uparrow  U_{t+1}$ and $W_{t+1}^{(n)}\geq 0$, by monotone convergence for each $(h,a)$ we have $Q_t^{(n)}\uparrow Q_t^U$. By \cref{lemma:W_is_lsc}, each $Q_t^{(n)}(h,\cdot)$ is l.s.c., thus $Q_t^U$ is also l.s.c. being the suprema of l.s.c. functions. Furthermore, since each $U_{t+1}(h,a,y)$ is l.s.c., and $(h,a)\mapsto \bar P_t(\cdot\mid h,a)$ is  a Borel kernel, then $Q_t^U(h,a)$ is jointly Borel measurable.

    Next, we show that $\sup_n \inf_a Q_t^{(n)}(h,a)=\inf_a\sup_n Q_t^{(n)}(h,a)= C_t^U(h)$. To show this, let $m_n =\inf_a Q_t^{(n)}(h,a)$ and $m=\sup_n m_n$.
    First, note that by monotonicity for all $N$ we have $ m_n=\inf_a Q_t^{(n)}(h,a) \leq \inf_a Q_t^U(h,a)=C_t^U(h)$, and thus $m\leq C_t^U(h)$.
    Now, choose some minimizers $a_n\in \argmin_a Q_t^{(n)}(h,a)$. By compactness, there is some subsequence satisfying $a_{n_k}\to a$. Choose some integers $i_0$, then by monotonicity for all large $k$ we have
    \[
    Q_t^{(i_0)}(h,a_{n_k}) \leq Q_t^{(n_k)}(h,a_{n_k}) = m_{n_k}.
    \]
    Therefore $\liminf_{k} m_{n_k}=m\geq Q_t^{(i_0)}(h,a)$. Since this holds for any $i_0$, we have $m\geq \sup_{i_0} Q_t^{(i_0)}(h,a)= Q_t^U(h,a)\geq C_t^U(h) $, which shows that $\sup_n \inf_a Q_t^{(n)}(h,a)=\inf_a\sup_n Q_t^{(n)}(h,a)= C_t^U(h)$.

    Hence, we have that $U_t(h)=\sup_n W_t^{(n+1)}(h)=\sup_n \min \{g_t(h), C_t^{(n)}(h)\}$. Since $C_t^{(n)}(h)\uparrow C_t^U(h)$, we have $U_t(h)=\min\{g_t(h), C_t^U(h)\}$.
\end{proof}

Finally, we show that $U_t$ is actually the true optimal value $W_t^\star$.
\begin{theorem}
       Assume \cref{assumption:L_continuity_x_measurability_xtheta} and \cref{assumption:likelihood_continuity}.
    For every $t \in \mathbb{N}$,  we have that $U_t(h)=V_t^\star(h)+\lambda$, and $V_t^\star$ is l.s.c.
\end{theorem}
\begin{proof}
Let $g_t(h)=\lambda(1-r_t(h))$ .
    We first show that $U_t$ is actually a lower bound on the cost of any policy. Then we construct  a policy that achieves equality.

    \noindent{\it Step 1 (lower bound).} Consider any admissible policy $\pi$. At any history $h$ at time $t$, 
    \begin{enumerate}
        \item if $\pi$ stops, then $U_t(h)\leq g_t(h)$ since $U_t(h)=\min\{g_t(h), C_t^U(h)\}$.
        \item if $\pi$ continues with some action $a$, then
        \[
        U_t(h) \leq 1+\int_{\Y} U_{t+1}(h,a,y)\bar P_t({\rm d}y\mid h,a),
        \]
        since $C_t^U(h)\leq Q_t^U(h,a)$.
    \end{enumerate}
    Let $\tau$ be the first timestep the policy $\pi$ decides to stop. Then, iterating up to $\min(\tau,t+N)$
    \[
    U_t(h)\leq \mathbb{E}^\pi\left[
    \sum_{s=t}^{\min(\tau,t+N)-1} 1+{\bf 1}_{\{\tau\leq t+N\}} g_\tau(H_\tau) + {\bf 1}_{\{\tau > t+N\}} U_{t+N}(H_{t+N})\Big | H_t=h
    \right]\eqqcolon J_N^\pi(h).
    \]
    If the policy achieves infinite cost, then the inequality is true for all $N$ since $U_t(h)$ is bounded. Then, consider the case where the cost is finite as $N\to\infty$. Since $g_\tau\geq 0$, this implies $\mathbb{E}^\pi[\tau-t\mid H_t=h]<\infty$, therefore $\tau<\infty$ almost surely. Consequently, by dominated convergence the remainder term $+ {\bf 1}_{\{\tau > t+N\}} U_{t+N}(H_{t+N})$ vanishes as $N\to\infty$.
    Therefore, letting $N\to \infty$ we obtain $U_t(h)\leq \lim_{N\to\infty} J_N^\pi(h)\coloneqq J^\pi(h)$, and taking infimum over $\pi$ we get $U_t(h)\leq W_t^\star(h)$.

    \noindent{\it Step 2 (equality).} From \cref{lemma:QtU_Ut},  $Q_t^U(h,a)$ is l.s.c. on $\A$ for each $h$, and $\A$ is compact. Furthemore $Q_t^U(h,a)$ is jointly Borel-measurable.  Then,  there exists a  measurable selector  $a_t^\star(h)\in \argmin_a Q_t^U(h,a)$ \citep{feinberg2021mdps} . Then both $C_t^U(h)=Q_t^U(h,a_t^\star(h))$ is Borel measurable. Consider then a policy $\pi^\star$ that stops at $h$ if $g_t(h) \leq C_t^U(h)$, and otherwise continue with $a_t^\star(h)$. Denote by $\tau^\ast$ this stopping rule. Along this policy, the Bellman minimum is attained with equality at every step, therefore
    \[
    U_t(h)=\mathbb{E}^{\pi^\ast}\left[
    \sum_{s=t}^{\min(\tau^\ast,t+N)-1} 1+{\bf 1}_{\{\tau^\ast\leq t+N\}} g_{\tau^\ast}(H_{\tau^\ast}) + {\bf 1}_{\{\tau^\ast > t+N\}} U_{t+N}(H_{t+N})\Big | H_t=h
    \right].
    \]
    Since $U_t(h)$ is bounded, we have
    \[
    \lambda \geq U_t(h)  \geq \mathbb{E}^{\pi^\ast}[\min(\tau^\ast,t+N) - t\mid H_t=h].
    \]
    Letting $N\to\infty$, we obtain  $\mathbb{E}^{\pi^\ast}[\tau-t\mid H_t=h]\leq \lambda$. Therefore the remainder term $ {\bf 1}_{\{\tau^\ast > t+N\}} U_{t+N}$ vanishes as $N\to\infty$ by dominated convergence since $U_{t+N}\leq \lambda$ and $\tau^\ast<\infty$ almost surely. Therefore we obtain $U_t(h)=J^{\pi^\ast}(h)$. Since $W^\star(h)$ is the minimal cost we obtain $W_t^\star(h)\leq U_t(h)$, but from the previous step we also have $U_t(h)\leq W_t^\star(h)$. Therefore $U_t(h)=W_t^\star(h)=V_t^\star(h;\lambda)+\lambda$. Finally, $V_t^\star$ is l.s.c. since $U_t$ is l.s.c.
\end{proof}

Hence, we conclude with the following result
\begin{proposition}[Actor over $\A$ and stopping action]\label{prop:actor_only_stop_comparison}
Assume \cref{assumption:L_continuity_x_measurability_xtheta} and \cref{assumption:likelihood_continuity}.
and let
 $a^\star(h)\in\arg\min_{a\in\A} Q_{t,\rm cont}^\star(h,a)$.
Then an optimal policy can be implemented by:
(i) selecting $a^\star(h)$ as the continuation action, and (ii) stopping iff
$Q_{t,\rm stop}^\star(h)\le Q_{t,\rm cont}^\star(h,a^\star(h))$.
\end{proposition}
\begin{proof}
By \eqref{eq:bellman_stop_continue}, at each history $h$ the optimal action is whichever attains the minimum between
$Q_{t,\rm stop}^\star(h)$ and $\inf_{a\in\A}Q_{t,\rm cont}^\star(h,a)$.
If $a^\star(h)$ attains the infimum over $\A$, then the comparison
$Q_{t,\rm stop}^\star(h)\le Q_{t,\rm cont}^\star(h,a^\star(h))$ is equivalent to
$Q_{t,\rm stop}^\star(h)\leq \inf_{a}Q_{t,\rm cont}^\star(h,a)$, i.e., stopping is optimal.
Otherwise continuing with $a^\star(h)$ is optimal.

\end{proof}
This last result shows that an algorithm can act in the augmented space $\bar\A$ while never explicitly parameterizing a ``stop action'' in the actor.

Hence an optimal policy can be implemented by:
\begin{enumerate}
\item choose a continuation action $a^\star(h)\in\arg\min_{a\in\A} Q_{t,\rm cont}^\star(h,a)$ (this is the actor over $\A$ only);
\item stop if $Q_{t,\rm stop}^\star(h)\le Q_{t,\rm cont}^\star(h,a^\star(h))$, otherwise continue with $a^\star(h)$.
\end{enumerate}
This is mathematically equivalent to having a single policy over $\bar\A$ that selects $\arg\min\{Q_{t,\rm stop}^\star(h),\min_a Q_{t,\rm cont}^\star(h,a)\}$, but it decomposes the decision into (i) a continuous control over $\A$ and (ii) an optimal-stopping comparison against a scalar stop value. Therefore,  one can learn  a $Q$-value for the stop decision and comparing it to the $Q$-value of the actor-chosen action: this is consistent with the optimal control structure of \eqref{eq:bellman_stop_continue}, provided the actor approximates the minimizer of $Q_{t,\rm cont}^\star(h,\cdot)$ and the critic estimates are calibrated.

\subsubsection{Reward Shaping and Removal of the Stop Action}
The stop action need not be explicitly parameterized by the actor. Indeed, for fixed $\lambda>0$, the Lagrangian objective
\[
\mathbb E^\pi\left[\sum_{s=t}^{\tau-1}1-\lambda r_\tau(H_\tau)\mid H_t=h\right]
\]
is equivalent, up to an additive constant independent of the policy, to maximizing the shaped return obtained from the one-step reward
\[
R_s^{\rm sh}\coloneqq -1+\lambda\left(r_{s+1}(H_{s+1})-r_s(H_s)\right).
\]
Thus the actor may be restricted to continuation actions $a\in\A$, while the stopping decision is implemented by comparing the best continuation advantage against the stop value, which is zero in the shaped formulation. Equivalently, using the rescaled reward
\[
-\frac1\lambda+r_{s+1}(H_{s+1})-r_s(H_s)
\]
gives the same optimal policies.

\begin{lemma}[Reward shaping]
\label{lemma:reward_shaping_stop_continue}
Fix $\lambda>0$.
For a continuation policy $\pi$ and stopping time $\tau$, define the original Lagrangian cost
\[
J_t^\lambda(h;\pi,\tau)
\coloneqq
\mathbb E^\pi\left[
\sum_{s=t}^{\tau-1}1-\lambda r_\tau(H_\tau)
\Big |H_t=h
\right],
\]
and the shaped return
\[
G_t^\lambda(h;\pi,\tau)
\coloneqq
\mathbb E^\pi\left[
\sum_{s=t}^{\tau-1}R_s^{\rm sh}(H_s,A_s,H_{s+1})
\Big | H_t=h
\right].
\]
Then
\[
\argmin_{\pi,\tau}J_t^\lambda(h;\pi,\tau)
=
\argmax_{\pi,\tau}G_t^\lambda(h;\pi,\tau).
\]
\end{lemma}
\begin{proof}
First, the shaped reward telescopes:
\begin{align*}
\sum_{s=t}^{\tau-1}
R_s^{\rm sh}(H_s,A_s,H_{s+1})
=
\sum_{s=t}^{\tau-1}
\left[
\lambda\left(r_{s+1}(H_{s+1})-r_s(H_s)\right)-1
\right]=
\lambda r_\tau(H_\tau)-\lambda r_t(H_t)
-
\sum_{s=t}^{\tau-1}1.
\end{align*}
Conditioning on $H_t=h$ gives
\[
G_t^\lambda(h;\pi,\tau)
=
\mathbb E^\pi\!\left[
\lambda r_\tau(H_\tau)-\sum_{s=t}^{\tau-1}1
\,\middle|\,H_t=h
\right]
-\lambda r_t(h).
\]
Since
\[
J_t^\lambda(h;\pi,\tau)
=
\mathbb E^\pi\!\left[
\sum_{s=t}^{\tau-1}1-\lambda r_\tau(H_\tau)
\,\middle|\,H_t=h
\right],
\]
we obtain
\[
G_t^\lambda(h;\pi,\tau)
=
-J_t^\lambda(h;\pi,\tau)-\lambda r_t(h).
\]
The term $-\lambda r_t(h)$ does not depend on $(\pi,\tau)$, so maximizing $G_t^\lambda$ is equivalent to minimizing $J_t^\lambda$.

\end{proof}

Using this reward shaping, we do not need to learn the stopping $Q$-value, as shown in the next proposition. The actor only needs to output a continuation action $a\in\A$. The stopping rule is a scalar gate:
\[
\text{continue iff } Q_{t,\rm cont}^{\rm sh}(h,a_t(h))>0.
\]
Equivalently, the stop value in the shaped formulation is identically zero, and the state-value target is
\[
S_t(h)=\max\{0,Q_{t,\rm cont}^{\rm sh}(h,a_t(h))\}.
\]

\begin{proposition}[Reward shaping does not require learning a stopping $Q$-value]
    Let
\[
S_t^\star(h)
\coloneqq
\sup_{\pi,\tau}G_t^\lambda(h;\pi,\tau)
\]
be the optimal shaped value. Then
\[
S_t^\star(h)
=
\max\left\{
0,
\sup_{a\in\A} \mathbb E\left[
-1+\lambda\left(r_{t+1}(H_{t+1})-r_t(h)\right)
+
S_{t+1}^\star(H_{t+1})
\Big| H_t=h,A_t=a
\right]
\right\}.
\]
Define the shaped continuation $Q$-value
\[
Q_{t,\rm cont}^{\rm sh}(h,a)
\coloneqq -1+
\mathbb E\left[
\lambda\left(r_{t+1}(H_{t+1})-r_t(h)\right)
+
S_{t+1}^\star(H_{t+1})
\Big  | H_t=h,A_t=a
\right].
\]
If the supremum over $a\in\A$ is attained by a measurable selector
$
a_t^\star(h)\in \argmax_{a\in\A}Q_{t,\rm cont}^{\rm sh}(h,a)$, 
then an optimal policy is implemented by
\[
\text{stop at }h
\quad\Longleftrightarrow\quad
Q_{t,\rm cont}^{\rm sh}(h,a_t^\star(h))\le 0,
\]
and otherwise continuing with $a_t^\star(h)$.
\end{proposition}
\begin{proof}
    
Next, in the shaped formulation, stopping immediately produces the empty sum and hence shaped return $0$. If instead the learner continues with action $a$, it receives the immediate shaped reward
\[
\lambda\bigl(r_{t+1}(H_{t+1})-r_t(h)\bigr)-1
\]
and then obtains the optimal future shaped value $S_{t+1}^\star(H_{t+1})$. Therefore
\[
S_t^\star(h)
=
\max\left\{
0,\;
\sup_{a\in\A}Q_{t,\rm cont}^{\rm sh}(h,a)
\right\}.
\]
Thus stopping is optimal exactly when
\[
\sup_{a\in\A}Q_{t,\rm cont}^{\rm sh}(h,a)\le 0.
\]
If $a_t^\star(h)$ attains the supremum, this is equivalent to
\[
Q_{t,\rm cont}^{\rm sh}(h,a_t^\star(h))\le 0.
\]
Otherwise, if
\[
Q_{t,\rm cont}^{\rm sh}(h,a_t^\star(h))>0,
\]
continuing with $a_t^\star(h)$ attains the continuation branch and is optimal.
\end{proof}
\subsection{Fixed-confidence setting: $(\epsilon,\delta)$-correctness of  dual-optimal points}\label{app:theoretical_results:icpe:fixed_confidence:strict_feasibility_and_correctness}
In this section we provide fixed-confidence guarantees for the continuous ICPE setting.
Compared to the correctness argument in \citep{russo_learning_2025}, our proof differs in three aspects, each addressing a limitation of the finite-space analysis.
\begin{enumerate}
    \item \textbf{Posterior object.} The finite ICPE framework uses $\mathbb{P}(x^\star=x\mid H_t)$, which is ill-defined in continuous $\X$. We replace it with the posterior success probability $r_t(h)=\sup_{x\in\X}\mathbb{P}(L_\theta(x)\leq\epsilon\mid H_t=h)$. This is the natural continuous analogue and the only change forced by the setting.
    \item \textbf{Non-singleton dual optima (main technical improvement).} \citet{russo_learning_2025} assume the dual-optimal policy set $\mathcal{S}(\lambda)$ is a singleton for each $\lambda$. We relax this to local closedness of the near-optimal cost-reward set $\mathcal{K}_{\epsilon_0}(\lambda^\star)$ (\cref{assump:closed_attainable_cost_reward_set}). This is meaningful in continuous spaces where the policy class is richer and uniqueness is harder to verify. The key proof step is showing that if all policies in $\mathcal{K}_{\epsilon_0}(\lambda^\star)$ have $\rho<1-\delta$, then a uniform gap propagates to all near-optimal policies (Step~2 below), which requires the closedness assumption in an essential way.
    \item \textbf{Subdifferential argument (alternative proof technique).} Rather than deriving a contradiction from monotonicity of the optimal cost in $\lambda$ (as in \citep{russo_learning_2025}), we use the subdifferential characterization of \citet{hantoute2008Characterizations} directly. This avoids the monotonicity argument entirely: we show that every subgradient of $f(\lambda)=-g(\lambda)$ at $\lambda^\star$ is strictly negative, contradicting $0\in\partial f(\lambda^\star)$. This argument works with multiple dual optima without modification.
\end{enumerate}

We begin by stating our assumptions. To do so, let $\Pi$ denote the class of admissible policies on the augmented action space $\bar \A$, and for each $\pi$ let $\tau_\pi$ be  the corresponding stopping time, i.e., the first tiem the policy chooses the stop action. Define
\[
{\cal T}\coloneqq \{\pi\in \Pi: \mathbb{E}^\pi[\tau_\pi] < \infty\},
\]
be the set of admissible policies with finite expected sample complexity. 

\begin{assumption}[Strict feasibility]\label{assump:strict_feasibility}
We assume there exists a policy $\pi\in {\cal T}$ such that
\[
\mathbb{E}^\pi[r_{\tau_\pi}(H_{\tau_\pi})] > 1-\delta.
\]
\end{assumption}

We also make the following assumption of closedness on the attainable cost-reward set. For any $\pi\in {\cal T}$ define  $c(\pi)\coloneqq \mathbb{E}^\pi[\tau_\pi]$ and $\rho(\pi)\coloneqq \mathbb{E}^\pi[r_{\tau_\pi}(H_{\tau_\pi})]$. Define also
\[
g(\lambda)= \inf_{\pi\in \Pi: c(\pi)<\infty } c(\pi) + \lambda(1-\delta-\rho(\pi)),\qquad {\cal S}(\lambda)=\argmin_{\pi\in \Pi:c(\pi)<\infty} c(\pi)+\lambda(1-\delta-\rho(\pi)),
\]
and the set of attainable cost-rewards:
\[
{\cal K}\coloneqq \{(c(\pi),\rho(\pi)): \pi\in {\cal T}\}\subset [0,\infty)\times [0,1].
\]
We have the following properties on $g(\lambda)$.
\begin{lemma}\label{lemma:properties_g_lambda}
Under \cref{assump:strict_feasibility},  $g(\lambda)$ is  finite, concave on $[0,\infty)$ and attains a maximum on $[0,\infty)$.
\end{lemma}
\begin{proof}
      For simplicity, we write the dual value directly on the attainable cost-reward set:
    \[
    g(\lambda) = \inf_{(c,\rho)\in {\cal K}} c+ \lambda(1-\delta-\rho),
    \]
    Then, $g(\lambda)\geq -\delta\lambda$. Let $(c,\rho)$ be a point satisfying \cref{assump:strict_feasibility}. Then $\rho>1-\delta$, and thus $g(\lambda)\leq c-\lambda\eta$ for $\eta=-(1-\delta-\rho)$. Since $\eta>0$, and $-\lambda \delta \leq g(\lambda)\leq c-\lambda\eta$, then $g(\lambda)\to-\infty$ as $\lambda\to\infty$.

    Finally, since $\lambda\mapsto c+ \lambda(1-\delta-\rho)$ is affine, the infimum of any family of concave functions is concave, therefore $g(\lambda)$ is concave.  Then, since $g$ is finite and concave, it's continuous on $(0,\infty)$ and u.s.c. in $0$. Therefore $g$ is u.s.c. on $[0,\infty)$, and since $\lim_{\lambda\to\infty}g(\lambda)\to-\infty$, there exists  $\lambda^\star\in[0,\infty)$ such that $g(\lambda^\star)=\max_{\lambda\geq 0} g(\lambda)$.
\end{proof}
We impose then the following assumption.
\begin{assumption}[Closed optimal cost-reward set]\label{assump:closed_attainable_cost_reward_set}
Let ${\cal G}=\argmax_\lambda g(\lambda)$. We assume that for every maximizer $\lambda^\star\in {\cal G}$ there exists $\epsilon_0>0$ such that
\[
{\cal K}_{\epsilon}(\lambda^\star)\coloneqq \{(c,\rho)\in {\cal K}: c+\lambda^\star(1-\delta-\rho)\leq g(\lambda^\star)+\epsilon\}
\]
is closed for all $\epsilon\in (0,\epsilon_0]$
\end{assumption}
Essentially, this geometric assumption is used to relax the assumption used in  \citep{russo_learning_2025} that ${\cal S}(\lambda)$ is a singleton for each $\lambda$, and allows us to generalize the argument to multiple optimal dual policies. 

To verify correctness, we use the fact that the sub-gradient of the optimal value of the dual problem is non-decreasing. To show this result, we employ the following proposition from \citep{hantoute2008Characterizations} (see Prop. 3.1 therein), which characterizes the subdifferential of the supremum of a family of affine functions. \begin{proposition}[Subdifferential of the supremum of affine functions \citep{hantoute2008Characterizations}]\label{prop:subdifferential_affine_functions} 
Given a non-empty set $\{(a_t,b_t):t\in {\cal T}\} \subset \mathbb{R}^2$, and the supremum function $f(x):\mathbb{R}\to\mathbb{R}\cup\{\infty\}$ 
\[ f(x)=\sup \{ a_tx- b_t: t\in {\cal T}\}, \] 
for every $x\in {\rm dom}f$ we have
\[ \partial f(x)= \bigcap_{\epsilon>0} {\rm cl}\left({\rm conv}\{a_t: t\in {\cal T}_\epsilon(x)\} + B(x)\right) \] 
with 
\[ {\cal T}_\epsilon(x)\coloneqq \{t\in {\cal T}: a_tx -b_t \geq f(x)-\epsilon\},\]
and 
\[ B(x)\coloneqq \left\{y\in \mathbb{R}: (y,yx) \in \left(\overline{\rm conv}\{(a_t,b_t):t\in {\cal T}\}\right)_\infty \right\}, \]
where $C_\infty$ is the recession cone of a set $C$ and $\overline{{\rm conv}(\cdot)}$ denotes the closed convex hull of a set. In particular, if $x\in {\rm int}({\rm dom} f)$ we have \[ \partial f(x)=\bigcap_{\epsilon>0} \overline{{\rm conv}}\left\{a_t: t\in {\cal T}_\epsilon(x)\right\}. \] 
\end{proposition}

This last proposition permits us to define the subdifferential of the supremum of affine functions, and, as we see now, we can also find a lower bound on any subdifferential $d\in \partial f(x)$. \\

We are now ready to prove $(\epsilon,\delta)$-PAC guarantees for ICPE in the continuous setting.

\begin{theorem}[$(\epsilon,\delta)$-correctness in the continuous ICPE case]\label{thm:icpe_correctness}
    Under \cref{assump:strict_feasibility}  and \cref{assump:closed_attainable_cost_reward_set} for all $\lambda^\star\in {\cal G}\coloneqq \argmax_\lambda g(\lambda)$ there exists a dual-optimal policy $\pi^\star\in {\cal S}(\lambda^\star)$ such that
    \[
    \mathbb{E}^{\pi^\star}[r_{\tau_{\pi^\star}}(H_{\tau_{\pi^\star}})] \geq 1-\delta,\quad \hbox{ and }\quad \mathbb{P}^{\pi^\star}\left(L_\theta\left(I_{\tau^\star}^\star(H_{\tau^\star}) \right) \leq \epsilon\right)\geq 1-\delta.
    \]
    where $I_t^\star(h)\in \argmax_h q_t(h,x)$ is the optimal inference selector.
\end{theorem}
\begin{proof}
    For simplicity, we write the dual value directly on the attainable cost-reward set:
    \[
    g(\lambda) = \inf_{(c,\rho)\in {\cal K}} c+ \lambda(1-\delta-\rho),
    \]
    and the active set $M(\lambda) \coloneqq \argmin_{(c,\rho)\in {\cal K}}c+ \lambda(1-\delta-\rho)$.

    Consider \cref{lemma:properties_g_lambda}, then there exists $\lambda^\star\in \argmax_\lambda g(\lambda)$ such that $\lambda\in [0,\infty)$. We now prove that $M(\lambda^\star)$  contains at-least a point $(c,\rho)$ with reward at-least $1-\delta$.

    By contradiction, ({\bf HYP}) assume that all points $(c,\rho)\in M(\lambda^\star)$ satisfy $\rho<1-\delta$. 
    \\
    
    \noindent\underline{\it First step: there exists $\alpha>0$ such that $\rho\leq 1-\delta-\alpha$ for optimal policies.}
    Under  ({\bf HYP}), we claim that there exists $\alpha>0$ such that $\rho\leq 1-\delta-\alpha$ for all $(c,\rho)\in M(\lambda^\star)$. First, note that by \cref{assump:closed_attainable_cost_reward_set} there exists $\epsilon_0$ such that $M(\lambda^\star)=\bigcap_{\epsilon\in (0,\epsilon_0]} {\cal K}_{\epsilon}(\lambda^\star)$. Note that  each $ {\cal K}_{\epsilon}(\lambda^\star)$ is closed, non-empty (since $g$ is finite), and since $\rho\in [0,1]$ and $c\leq g(\lambda^\star)+\epsilon+\lambda^\star\delta$, it's bounded. Thus each $ {\cal K}_{\epsilon}(\lambda^\star)$ is compact: therefore, also the intersection $M(\lambda^\star)$ is compact.

    Since $M(\lambda^\star)$ is compact, we have that the maximum is achieved, and since $(c,\rho)\mapsto\rho$ is continuous,  we set $\rho_{\rm max}= \max\{\rho: (c,\rho)  \in M(\lambda^\star)\} $ and  $\alpha=1-\delta-\rho_{\rm max} > 0$.
     \\
    
    \noindent\underline{\it Second step: near optimal policies satisfy $\rho\leq 1-\delta-\alpha/2$.} Always under  ({\bf HYP}), we now claim that for near-optimal points we actually have $\rho\leq 1-\delta-\alpha/2$. Suppose that this is not true:
    then, by \cref{assump:closed_attainable_cost_reward_set} there exists $\epsilon_0$ such that  for every $\epsilon\in (0,\epsilon_0]$ there exists a point  $(c,\rho)$ in ${\cal K}_{\epsilon}(\lambda^\star)$ such that $\rho >1-\delta-\alpha/2$.
    Let $\epsilon_n=1/n$ and consider such a sequence $(c_n,\rho_n)\in {\cal K}_{\epsilon_n}(\lambda^\star)$ satisfying $\rho_n>1-\delta-\alpha/2$ for all $n$ Therefore, for each $n$ we have that
    \[
    c_n + \lambda^\star(1-\delta-\rho_n)\leq g(\lambda^\star)+\frac{1}{n}.
    \]
    Since $\rho_n\in [0,1]$ and $c_n \leq g(\lambda^\star)+1/n  +\lambda^\star\delta$, thus bounded, we have that the sequence $(c_n,\rho_n)$ is bounded.  Since for each $\epsilon_n$ the set ${\cal K}_{\epsilon_n}(\lambda^\star)$ is closed, we can take a convergent subsequence that satisfies $(c_n,\rho_n)\to (c,\rho)$ for some $(c,\rho)\in {\cal K}_{\epsilon_0}$. Therefore, we have that
    \[
    \lim_{n\to\infty} c_n+\lambda^\star(1-\delta-\rho_n)\leq g(\lambda^\star)+1/n \Longrightarrow c+\lambda^\star(1-\delta-\rho)\leq g(\lambda^\star).
    \]
    However, for any $(c,\rho)$ we also have $g(\lambda^\star)\leq c+\lambda^\star(1-\delta-\rho)$, thus $c+\lambda^\star(1-\delta-\rho)=g(\lambda^\star)$, while also having $\rho \geq  1-\delta-\alpha/2$, which contradicts $\rho \leq 1-\delta-\alpha$. Therefore, for all $\epsilon_0$ close points we have $\rho\leq 1-\delta-\alpha/2$.

    We now distinguish the cases $\lambda^\star=0$ and $\lambda^\star>0$.\\

     \noindent\underline{\it Third step: case $\lambda^\star=0$ is not optimal.} Under  ({\bf HYP}), assume $\lambda^\star=0$ is optimal. We show that this is not the case, and there exists $\lambda: g(\lambda)>g(0)$. We proceed by showing there exists $\lambda>0$ small such that for any $(c,\rho)\in {\cal K}$ we have $c+\lambda(1-\delta-\rho)>g(0)$, and thus there exists $\lambda >0$ such that $g(\lambda)>g(0)$.

     Let $\lambda >0$ to be chosen. By \cref{assump:closed_attainable_cost_reward_set} and the previous step there exists $\epsilon_0$ such that for all $(c,\rho)\in {\cal K}_\epsilon(0)$ with $\epsilon\in (0,\epsilon_0]$ we have $1-\delta-\alpha/2\geq\rho$ and $c\leq g(0)+\epsilon$. Consider then the following two cases:
     \begin{enumerate}
         \item Case where $(c,\rho)$ satisfies $c \geq g(0)+\epsilon_0$ (far away point). Since $\rho\in [0,1]$, we have $1-\delta-\rho\in [-1,1]$, and thus $\lambda(1-\delta-\rho)\geq -\lambda$. Hence
         \[
         g(0)+\epsilon_0-\lambda \leq c+\lambda(1-\delta-\rho).
         \]

         \item Case where $(c,\rho)$ satisfies $c\leq g(0)+\epsilon_0$ (near optimal point). Then, in this case we  have that $(1-\delta-\rho)\geq \alpha/2$, thus, using that $g(0)\leq  c$, combining the two inequalities we have
         \[
         g(0)+\lambda \alpha/2 \leq  c+\lambda(1-\delta-\rho).
         \]
     \end{enumerate}
     Then, choose $\lambda \in (0,\epsilon_0/2)$: we obtain $-\lambda >-\epsilon_0/2$, and $\lambda >0$, thus
     \[
     g(0)<g(0)+\frac{\epsilon_0}{2} < c+\lambda(1-\delta-\rho),\qquad g(0) < g(0)+\lambda\frac{\alpha}{2} \leq c+\lambda(1-\delta-\rho).
     \]
   Therefore $g(0) < c+\lambda(1-\delta-\rho)$ for all $(c,\rho)\in {\cal K}$. This shows that $g(\lambda)>g(0)$, and contradicts the optimality of $\lambda^\star=0$.

   \noindent\underline{\it Fourth step: case $\lambda^\star>0$ is not optimal.} Define the function $f(\lambda)=-g(\lambda)$. From \cref{lemma:properties_g_lambda} then $f$ is convex. Since $\lambda^\star>0$, it lies in the interior of ${\rm dom}(f)$. Then, by \cref{prop:subdifferential_affine_functions} , we have
   \[
   \partial f(\lambda^\star)=\bigcap_{\epsilon\in (0,\epsilon_0]} \overline{\rm conv}\{\rho-(1-\delta): (c,\rho)\in {\cal K}_\epsilon(\lambda^\star)\}
   \]
   By step $2$ we have $1-\delta-\alpha/2\geq \rho$, therefore $\rho-(1-\delta)\leq -\alpha/2$, implying that
   \[
   d \leq -\frac{\alpha}{2} \qquad \forall d\in \partial f(\lambda^\star).
   \]
   But $\lambda^\star$ minimizes the convex function $f$ and is an interior point of the domain, so necessarily $0\in \partial f(\lambda^\star)$, which is a contradiction.

   \noindent \underline{\it Last step.} Since both cases are impossible, we conclude there exists $(c,\rho)\in M(\lambda^\star)$ with $\rho\geq 1-\delta$. Hence, there exists a dual optimal policy $\pi\in {\cal S}(\lambda^\star)$  such that $c(\pi)=c$ and $\rho(\pi)\geq 1-\delta$. Lastly, since $I_t^\star$ attains the supremum in the definition of $r_t(h)$, we have $r_t(h)=q_t(h,I_t^\star(h))= \mathbb{P}(L_\theta(I_t^\star(h))\leq \epsilon\mid H_t=h)$. Thus, we get
   \[
   1-\delta \leq \rho(\pi) =\mathbb{E}^\pi[r_{\tau_\pi}(H_{\tau_\pi})] = \mathbb{E}_{H_{\tau_\pi}}^\pi\left[ \mathbb{P}\left(L_\theta\left(I_{\tau_\pi}^\star(H_{\tau_\pi})\right)\leq \epsilon\Big  | H_{\tau_\pi}\right)\right]=\mathbb{P}\left(L_\theta\left(I_{\tau_\pi}^\star(H_{\tau_\pi})\right)\leq \epsilon\right).
   \]
\end{proof}

\begin{remark}
The above result yields a \emph{Bayesian fixed-confidence} guarantee under the prior $\nu$. It is therefore the natural continuous counterpart of the fixed-confidence correctness statement in the discrete ICPE setting.
\end{remark}

\subsection{Fixed-confidence setting: zero duality gap via perturbation values}
\label{app:zero_duality_gap_perturbation}

In the previous sections we reduced the Bayesian fixed-confidence problem to an
optimal-stopping problem with terminal posterior-success reward. We now study
when the corresponding Lagrangian relaxation is exact. The main point of this
section is that, although the policy space is infinite-dimensional, the duality
question can be analyzed through the two-dimensional attainable cost-reward set
\[
{\cal K}
\coloneqq
\{(c(\pi),\rho(\pi)):\pi\in{\cal T}\}
\subset [0,\infty)\times[0,1],
\]
where
\[
c(\pi)\coloneqq \mathbb E^\pi[\tau_\pi],
\qquad
\rho(\pi)\coloneqq
\mathbb E^\pi[r_{\tau_\pi}(H_{\tau_\pi})].
\]
The zero duality gap result uses a standard one-dimensional perturbation argument from convex duality \citep{rockafellar1974conjugate,ekeland1976convex,rockafellar1998variational,borwein2006convex}.  While the general technique idea is not new, what is new is its application to this specific problem studied in this manuscript, where the attainable cost-reward set $\mathcal{K}$ arises from an optimal-stopping control problem over continuous spaces.

After optimizing over the inference rule, the reduced primal problem is
\begin{equation}
\label{eq:reduced_primal_cost_reward}
P^\star\coloneqq\inf_{\pi\in{\cal T}} c(\pi)\qquad
\text{s.t.} \qquad \rho(\pi)\ge 1-\delta. \end{equation}
Equivalently, $P^\star=\inf\{c:(c,\rho)\in{\cal K},\ \rho\ge 1-\delta\}$.
The associated Lagrangian dual value is
\begin{equation}
\label{eq:reduced_dual_value}
D^\star\coloneqq\sup_{\lambda\geq 0} g(\lambda),\qquad g(\lambda)\coloneqq\inf_{\pi\in{\cal T}}\left\{c(\pi)+\lambda(1-\delta-\rho(\pi))\right\}.
\end{equation}
Equivalently, $g(\lambda)=\inf_{(c,\rho)\in{\cal K}}\{c+\lambda(1-\delta-\rho)\}$, and the multiplier satisfies $\lambda\geq 0$. Clearly, by Lagrangian duality, we have that the weak duality easily holds $D^\star\leq P^\star$.
\paragraph{Assumptions.}
We state some assumptions. The first one is a convexification assumption on the
policy class.

\begin{assumption}[Time-sharing] \label{assump:time_sharing}
For every $\pi_0,\pi_1\in{\cal T}$ and every $\alpha\in[0,1]$, there exists
a policy $\pi_\alpha\in{\cal T}$ such that
\[
c(\pi_\alpha)=\alpha c(\pi_1)+(1-\alpha)c(\pi_0),\qquad\rho(\pi_\alpha)=\alpha \rho(\pi_1)+(1-\alpha)\rho(\pi_0).
\]
\end{assumption}

This assumption is natural for randomized sequential policies: before the
episode starts, the learner samples an independent Bernoulli random variable and
then follows either $\pi_1$ or $\pi_0$ for the entire episode. We use
a-priori randomization over complete policies, rather than only randomized
actions at each history, because the latter does not automatically convexify the
full stopping-time law. Hence, under this assumption one can easily show that the cost-reward set ${\cal K}$ is convex.

\begin{assumption}[Strict feasibility]
\label{assump:strict_feasibility_perturbation}
There exists $\pi_{\rm sf}\in{\cal T}$ such that $\rho(\pi_{\rm sf})>1-\delta$.
\end{assumption}
This is the same assumption as in \cref{assump:strict_feasibility}.

The next assumption is not needed for zero duality gap. It is only needed when
we want to extract an actual primal-dual saddle point from the zero-gap identity.
Unlike global closedness of all bounded slices of ${\cal K}$, it only asks for
closedness of one near-optimal dual level set around a dual maximizer.

\begin{assumption}[Local closedness of near-optimal dual level sets]
\label{assump:closed_bounded_slices}
Let ${\cal G}\coloneqq\argmax_{\lambda\geq 0} g(\lambda)$. For every $\lambda^\star\in{\cal G}$, there exists $\epsilon_0>0$ such that
\[
{\cal K}_{\epsilon_0}(\lambda^\star)\coloneqq \{(c,\rho)\in{\cal K}: c+\lambda^\star(1-\delta-\rho)\leq g(\lambda^\star)+\epsilon_0\}
\]
is closed in $\mathbb R^2$.
\end{assumption}

\subsubsection{Zero Duality by Perturbation}
We now show how \cref{assump:time_sharing} and \cref{assump:strict_feasibility_perturbation} can be used to prove zero duality gap via a perturbed value.
Define the scalar perturbation value function
\begin{equation}
\label{eq:perturbation_value_phi}
\varphi(u)\coloneqq\inf\left\{c:(c,\rho)\in{\cal K},\ 1-\delta -\rho\leq u\right\},\qquad u\in\mathbb R,
\end{equation}
with the convention $\inf\emptyset=+\infty$. The original primal value is
\[
P^\star=\varphi(0).
\]
The variable $u$ is the allowed violation of the confidence constraint. If
$u>0$, the constraint is relaxed; if $u<0$, the constraint is strengthened.

\begin{lemma}[Basic properties of the perturbation value]
\label{lem:perturbation_value_properties}
Consider \cref{assump:time_sharing} and \cref{assump:strict_feasibility_perturbation}. Then
$\varphi$ is convex, nonincreasing, and finite on an open interval containing
$0$. In particular, $\varphi$ is continuous at $0$, and
$\partial\varphi(0)\neq\emptyset$.
\end{lemma}

\begin{proof}
We prove the claims one by one.

\noindent\underline{\it Step 1: monotonicity.}
If $u_1\le u_2$, then
\[
\{(c,\rho)\in{\cal K}:1-\delta-\rho\leq u_1\}
\subseteq
\{(c,\rho)\in{\cal K}:1-\delta-\rho\leq u_2\}.
\]
Therefore
\[
\varphi(u_2)\le\varphi(u_1).
\]
Thus $\varphi$ is nonincreasing.

\noindent\underline{\it Step 2: convexity.}
Using \cref{assump:time_sharing} one can easily show that  ${\cal K}$ is convex. Take
$u_0,u_1\in\mathbb R$, $\alpha\in[0,1]$, and $\eta>0$. If
$\varphi(u_i)<\infty$, choose $(c_i,\rho_i)\in{\cal K}$ such that
\[1-\delta-\rho_i\le u_i,\qquad c_i\le \varphi(u_i)+\eta,\qquad i\in\{0,1\}.\]
If one of the two values is $+\infty$, the convexity inequality is trivial.
By convexity of ${\cal K}$,
\[(c_\alpha,\rho_\alpha)\coloneqq\alpha(c_1,\rho_1)+(1-\alpha)(c_0,\rho_0)\in{\cal K}.\]
Moreover,
\[1-\delta-\rho_\alpha=\alpha(1-\delta-\rho_1)+(1-\alpha)(1-\delta-\rho_0)\leq\alpha u_1+(1-\alpha)u_0.\]
Hence $(c_\alpha,\rho_\alpha)$ is feasible for
$\varphi(\alpha u_1+(1-\alpha)u_0)$, and so
\[
\varphi(\alpha u_1+(1-\alpha)u_0)
\leq
c_\alpha
\leq
\alpha\varphi(u_1)+(1-\alpha)\varphi(u_0)+\eta.
\]
Letting $\eta\downarrow0$ gives convexity.

\noindent\underline{\it Step 3: finiteness near zero.}
By strict feasibility, there exists
$(c_{\rm sf},\rho_{\rm sf})\in{\cal K}$ such that
\[
\rho_{\rm sf}>1-\delta.
\]
Let
\[
\eta_{\rm sf}\coloneqq \rho_{\rm sf}-(1-\delta)>0.
\]
Then
\[
1-\delta-\rho_{\rm sf}=-\eta_{\rm sf}.
\]
Therefore, for every $u\ge -\eta_{\rm sf}$, the same point
$(c_{\rm sf},\rho_{\rm sf})$ is feasible for $\varphi(u)$. Hence
\[
\varphi(u)\le c_{\rm sf}<\infty,
\qquad
\forall u\ge -\eta_{\rm sf}.
\]
Also, since $c\ge0$ for all $(c,\rho)\in{\cal K}$, we have
$\varphi(u)\ge0$ whenever $\varphi(u)<\infty$. Thus $\varphi$ is finite
on the open interval $(-\eta_{\rm sf},\infty)$, which contains $0$.

\noindent\underline{\it Step 4: continuity and existence of a subgradient.}
A proper convex function that is finite on an open interval is continuous on
that interval. Since $0\in{\rm int}({\rm dom}\,\varphi)$, the
one-dimensional subdifferential $\partial\varphi(0)$ is nonempty. Equivalently,
one may take any slope between the left and right derivatives of $\varphi$ at
zero.
\end{proof}

\begin{theorem}[Zero duality gap by perturbation]
\label{thm:zero_duality_gap_perturbation}
Assume \cref{assump:time_sharing} and \cref{assump:strict_feasibility_perturbation}. Then
there exists $\lambda^\star\geq0$ such that
\[
g(\lambda^\star)=P^\star.
\]
Consequently,
\begin{equation}
\label{eq:zero_duality_gap}
\inf_{\pi\in{\cal T}:\rho(\pi)\ge 1-\delta} c(\pi)=\sup_{\lambda\ge0}\inf_{\pi\in{\cal T}}\left\{c(\pi)+\lambda(1-\delta-\rho(\pi))\right\}.
\end{equation}
In particular, the Lagrangian dual has no duality gap, and the dual supremum is
attained.
\end{theorem}

\begin{proof}
By \cref{lem:perturbation_value_properties}, choose
$s^\star\in\partial\varphi(0)$. Since $\varphi$ is nonincreasing, every
subgradient at zero is nonpositive. Indeed, for $h>0$, the subgradient
inequality gives
\[
\varphi(h)\ge \varphi(0)+s^\star h,
\]
while monotonicity gives $\varphi(h)\le\varphi(0)$. Therefore
$s^\star h\le0$, and hence $s^\star\le0$.

Define $\lambda^\star\coloneqq -s^\star\geq 0$. The subgradient inequality gives, for every $u\in\mathbb R$,
\[
\varphi(u)\ge \varphi(0)+s^\star u=P^\star-\lambda^\star u.
\]
Now fix any $(c,\rho)\in{\cal K}$ and set
$u=1-\delta-\rho$.
Since $(c,\rho)$ is feasible for $\varphi(u)$, we have $c\ge\varphi(u)$.
Thus
\[c\ge\varphi(u)\geq P^\star-\lambda^\star u=P^\star-\lambda^\star(1-\delta -\rho).\]
Equivalently,
$c+\lambda^\star(1-\delta-\rho)\ge P^\star$.
Taking the infimum over all $(c,\rho)\in{\cal K}$ gives
\[g(\lambda^\star)=\inf_{(c,\rho)\in{\cal K}}\{c+\lambda^\star(1-\delta-\rho)\}\geq P^\star.\]
By weak duality,  we also have
$g(\lambda^\star)\le P^\star$. Hence $g(\lambda^\star)=P^\star$.
Taking the supremum over $\lambda\ge0$ proves \cref{eq:zero_duality_gap}.
\end{proof}

\begin{remark}
\label{rem:why_slater}

Strict feasibility (\cref{assump:strict_feasibility_perturbation}) is a simple sufficient condition. The exact perturbation
condition is lower semicontinuity of $\varphi$ at the origin by the Fenchel-Moreau theorem. In fact,  $\varphi^{\ast\ast}$ is the lower
semicontinuous convex envelope of $\varphi$. Since $\varphi$ is already convex, equality $\varphi^{\ast\ast}(0)=\varphi(0)$ holds exactly when $\varphi$ is lower semicontinuous at $0$. Finally,
\cref{lem:perturbation_value_properties} shows that strict feasibility implies
continuity, hence lower semicontinuity, at $0$.
 We state \cref{thm:zero_duality_gap_perturbation} under strict
feasibility because it is easier to interpret and check. Strict feasibility
places $0$ in the interior of ${\rm dom}\,\varphi$, and finite convex
functions are continuous on the interior of their domain. The price is that
strict feasibility is stronger than necessary.
\end{remark}

\subsubsection{Primal attainment and KKT}
The zero-gap result above does not require any closedness assumption on ${\cal K}$. However, zero duality gap is only a statement about values. To obtain an actual policy that is both primal feasible and dual optimal, we need an attainment condition. The local closedness assumption above is enough: it compactifies a near-optimal dual level set around a dual maximizer, and a feasible minimizing sequence can be extracted inside this compact set.

\begin{lemma}[Primal-dual attainment from local closedness]
\label{lem:primal_attainment_closed_slices}
Assume ~\ref{assump:time_sharing}, ~\ref{assump:strict_feasibility_perturbation}, and ~\ref{assump:closed_bounded_slices}. Let $\lambda^\star\in{\cal G}$ be any dual maximizer. Then there exists $(c^\star,\rho^\star)\in{\cal K}$ such that
\[
\rho^\star\geq 1-\delta,\qquad c^\star=P^\star,\qquad c^\star+\lambda^\star(1-\delta-\rho^\star)=g(\lambda^\star).
\]
Equivalently, there exists a policy $\pi^\star\in{\cal T}$ such that
\[
c(\pi^\star)=P^\star,\qquad \rho(\pi^\star)\geq 1-\delta,\qquad \pi^\star\in{\cal S}(\lambda^\star).
\]
Moreover, $\lambda^\star(1-\delta-\rho(\pi^\star))=0$.
\end{lemma}

\begin{proof}
By \cref{thm:zero_duality_gap_perturbation}, $g(\lambda^\star)=P^\star$. Let $(c_n,\rho_n)\in{\cal K}$ be a feasible minimizing sequence for the primal problem, so that
\[
\rho_n\geq 1-\delta,\qquad c_n\to P^\star.
\]
Define $L_{\lambda^\star}(c,\rho)\coloneqq c+\lambda^\star(1-\delta-\rho)$. Since $\rho_n\geq 1-\delta$ and $\lambda^\star\geq 0$, we have
\[
L_{\lambda^\star}(c_n,\rho_n)\leq c_n.
\]
On the other hand, by definition of $g(\lambda^\star)$,
\[
L_{\lambda^\star}(c_n,\rho_n)\geq g(\lambda^\star)=P^\star.
\]
Therefore,
\[
P^\star\leq L_{\lambda^\star}(c_n,\rho_n)\leq c_n\to P^\star,
\]
and hence $L_{\lambda^\star}(c_n,\rho_n)\to P^\star=g(\lambda^\star)$.

Let $\epsilon_0>0$ be given by \cref{assump:closed_bounded_slices}. For all sufficiently large $n$,
\[
L_{\lambda^\star}(c_n,\rho_n)\leq g(\lambda^\star)+\epsilon_0,
\]
so $(c_n,\rho_n)\in{\cal K}_{\epsilon_0}(\lambda^\star)$.

We now show that ${\cal K}_{\epsilon_0}(\lambda^\star)$ is compact. It is closed by assumption. It is also bounded: for every $(c,\rho)\in{\cal K}_{\epsilon_0}(\lambda^\star)$, since $\rho\in[0,1]$, we have $1-\delta-\rho\geq-\delta$, and thus
\[
c\leq g(\lambda^\star)+\epsilon_0+\lambda^\star\delta.
\]
Also $c\geq0$ and $\rho\in[0,1]$. Hence ${\cal K}_{\epsilon_0}(\lambda^\star)$ is closed and bounded in $\mathbb R^2$, and therefore compact.

Passing to a subsequence, we may assume $(c_n,\rho_n)\to(c^\star,\rho^\star)$ for some $(c^\star,\rho^\star)\in{\cal K}_{\epsilon_0}(\lambda^\star)\subseteq{\cal K}$. Since $c_n\to P^\star$, we have $c^\star=P^\star$. Since $\rho_n\geq 1-\delta$ for all $n$, we have $\rho^\star\geq 1-\delta$. By continuity of $L_{\lambda^\star}$,
\[
L_{\lambda^\star}(c^\star,\rho^\star)=\lim_n L_{\lambda^\star}(c_n,\rho_n)=P^\star=g(\lambda^\star).
\]
Because $(c^\star,\rho^\star)\in{\cal K}$, there exists $\pi^\star\in{\cal T}$ such that $c(\pi^\star)=c^\star$ and $\rho(\pi^\star)=\rho^\star$. The equality $L_{\lambda^\star}(c^\star,\rho^\star)=g(\lambda^\star)$ implies $\pi^\star\in{\cal S}(\lambda^\star)$. Finally, since $c^\star=P^\star$ and $L_{\lambda^\star}(c^\star,\rho^\star)=P^\star$, we obtain
\[
\lambda^\star(1-\delta-\rho^\star)=0.
\]
This proves the claim.
\end{proof}

Define the dual-optimal policy set
\[{\cal S}(\lambda)\coloneqq\argmin_{\pi\in{\cal T}}\left\{c(\pi)+\lambda(1-\delta-\rho(\pi))\right\}.\]
\begin{theorem}[KKT and Bayesian correctness]
\label{thm:kkt_bayesian_correctness}
Assume ~\ref{assump:time_sharing}, ~\ref{assump:strict_feasibility_perturbation}, and ~\ref{assump:closed_bounded_slices}. Let $\lambda^\star\in{\cal G}$ be any dual maximizer. Then there exists $\pi^\star\in{\cal S}(\lambda^\star)$ such that
\[
c(\pi^\star)=P^\star,\qquad \rho(\pi^\star)\geq 1-\delta,\qquad \lambda^\star(1-\delta-\rho(\pi^\star))=0.
\]
Consequently, if $I_t^\star(h)\in\argmax_{x\in\X}q_t(h,x)$ is a measurable posterior-optimal inference selector, then
\[
\mathbb P_{\theta\sim\nu}^{\pi^\star}\left(L_\theta\left(I_{\tau_{\pi^\star}}^\star(H_{\tau_{\pi^\star}})\right)\leq \epsilon\right)\geq 1-\delta.
\]
\end{theorem}

\begin{proof}
By \cref{lem:primal_attainment_closed_slices}, there exists $\pi^\star\in{\cal T}$ such that
\[
c(\pi^\star)=P^\star,\qquad \rho(\pi^\star)\geq 1-\delta,\qquad \pi^\star\in{\cal S}(\lambda^\star),
\]
and
\[
\lambda^\star(1-\delta-\rho(\pi^\star))=0.
\]
It remains only to translate $\rho(\pi^\star)\geq 1-\delta$ into the desired correctness statement.

By definition,
\[
\rho(\pi^\star)=\mathbb E^{\pi^\star}\left[r_{\tau_{\pi^\star}}(H_{\tau_{\pi^\star}})\right].
\]
Since $I_t^\star$ is posterior-optimal, $r_t(h)=q_t(h,I_t^\star(h))$. Hence
\[
\rho(\pi^\star)=\mathbb E^{\pi^\star}\left[q_{\tau_{\pi^\star}}\left(H_{\tau_{\pi^\star}},I_{\tau_{\pi^\star}}^\star(H_{\tau_{\pi^\star}})\right)\right].
\]
By \cref{lem:stopped_success_decomp}, the last display equals
\[
\mathbb P_{\theta\sim\nu}^{\pi^\star}\left(L_\theta\left(I_{\tau_{\pi^\star}}^\star(H_{\tau_{\pi^\star}})\right)\leq \epsilon\right).
\]
Since $\rho(\pi^\star)\geq 1-\delta$, the result follows.
\end{proof}

\begin{remark}[What is stronger than the previous correctness statement]
\label{rem:zero_gap_vs_correctness}
The result above separates three issues that were previously entangled:
(i) zero duality gap, (ii) attainment of a primal-dual saddle point, (iii) Bayesian- $(\epsilon,\delta)$-correctness.
Time-sharing plus strict feasibility gives zero duality gap and dual attainment. The local closedness assumption is only used to extract an actual feasible dual-optimal policy from the zero-gap identity. Finally, Bayesian correctness follows from the posterior-optimal inference identity.
\end{remark}

\paragraph{Weaknesses and limitations.}
There are several limitations to keep in mind.
\begin{enumerate}
    \item First, the theorem is a statement about the randomized, time-sharing closure of the policy class. If one restricts to deterministic policies without ex-ante randomization, ${\cal K}$ need not be convex and a Lagrangian duality gap can occur.
    \item Second, zero duality gap does not imply that every dual-optimal policy is feasible. The theorem guarantees the existence of a feasible dual-optimal policy under the local closedness assumption. A learned policy that approximately minimizes the Lagrangian at $\lambda^\star$ is not automatically certified by this theorem.
    \item Third, local closedness is still an attainment assumption. It is weaker than requiring all bounded slices of ${\cal K}$ to be closed, but it is not automatic. Without some local closedness, the primal value may be approached by policies whose cost-reward pairs converge to a boundary point outside ${\cal K}$, so no exact optimal policy need exist.
\end{enumerate}

\subsection{Training-time certification of $(\epsilon,\delta)$-correctness}
\label{app:training_time_certification}

We now give a training-time correctness guarantee for ICPE that treats the learned policy, stopping rule, and recommender as a black box. The guarantee is designed for the practical workflow used in our experiments: every few training epochs we freeze the current model, evaluate it on a finite batch of fresh task realizations, and certify whether the checkpoint satisfies the $(\epsilon,\delta)$-guarantees.

The key point is that certification must account for repeated testing. A fixed-level certificate applied repeatedly at many checkpoints is not valid without correction. We therefore assign a certification budget $\alpha_m$ to checkpoint $m$, with total budget at most $\alpha\in (0,1)$, and certify each checkpoint using a checkpointwise mixture martingale.

Fix a target success threshold
\[
q^\star = 1-\delta,
\]
and an additional parameter $\alpha\in(0,1)$ that controls the probability of ever falsely certifying a checkpoint.

\paragraph{Protocol.}
We index certification checkpoints by $m\ge 1$. In the experiments, for example, checkpoint $m$ may correspond to every $2500$ training epochs. At checkpoint $m$, let
\[
w_m=(\phi_m,\psi_m,\vartheta_m)
\]
denote the frozen parameters of the inference model $I_\phi$, the critic $Q_\psi$ and the actor $\pi_{\vartheta}$. We denote by $I_m$, $\pi_m$, and $\tau_m$, respectively, the induced inference rule, sampling policy, and stopping time. The probability of success of checkpoint $m$ is
\[
p_m
\coloneqq
\mathbb P_{\theta\sim \nu}^{\pi_m}
\!\left(
I_m(H_{\tau_m})\in {\cal X}_\epsilon(\theta)
\right).
\]
Thus $p_m$ is the probability of success of the actual recommender rule used by the frozen checkpoint.

For each checkpoint $m$, after $w_m$ is fixed, we draw i.i.d. a batch of $B$ evaluation episodes. For $j=1,\dots,B$, let $\theta^{(m,j)}\sim\nu$ be the task in the $j$-th evaluation episode for checkpoint $m$, and let $H_{\tau_m}^{(m,j)}$ be the stopped history obtained by running the frozen snapshot $w_m$ on that task. Define the success indicator
\[
Z_{m,j}
\coloneqq
\mathbf 1
\!\left\{
I_m\!\left(H_{\tau_m}^{(m,j)}\right)
\in {\cal X}_\epsilon(\theta^{(m,j)})
\right\}.
\]
We write
\[
S_{m,t}
\coloneqq
\sum_{j=1}^t Z_{m,j},
\qquad
t=0,1,\dots,B.
\]

Hence,
For each checkpoint $m$, conditional on the training history before evaluating checkpoint $m$ and conditional on the frozen checkpoint $w_m$, the variables
\[
Z_{m,1},\dots,Z_{m,B}
\]
are independent and identically distributed as ${\rm Ber}(p_m)$. Moreover, the certification batch at checkpoint $m$ is not reused for training updates before the certification decision for checkpoint $m$ is made.

\paragraph{Checkpointwise mixture martingale.}
At checkpoint $m$, we test the null hypothesis
\[
H_{0,m}:\quad p_m\le q^\star.
\]
Let $\Pi_m$ be a probability distribution on $[q^\star,1]$, chosen before observing the certification batch at checkpoint $m$. For $r\in[q^\star,1]$, define
\[
L_{m,t}(r)
\coloneqq
\left(\frac{r}{q^\star}\right)^{S_{m,t}}
\left(\frac{1-r}{1-q^\star}\right)^{t-S_{m,t}},
\qquad
t=0,1,\dots,B,
\]
with the usual convention $0^0=1$. The checkpointwise mixture martingale is
\[
M_{m,t}
\coloneqq
\int_{q^\star}^{1} L_{m,t}(r)\,\Pi_m(dr),
\qquad
M_{m,0}=1.
\]
In implementation, $\Pi_m$ may be a finite grid distribution, in which case no numerical integration is needed. If
\[
\Pi_m = \sum_{\ell=1}^L \omega_{m,\ell}\,\delta_{r_{m,\ell}},
\qquad
\sum_{\ell=1}^L \omega_{m,\ell}=1,
\]
then
\[
M_{m,t}
=
\sum_{\ell=1}^L
\omega_{m,\ell}
\left(\frac{r_{m,\ell}}{q^\star}\right)^{S_{m,t}}
\left(\frac{1-r_{m,\ell}}{1-q^\star}\right)^{t-S_{m,t}}.
\]

\paragraph{Certification budgets.}
Let $(\alpha_m)_{m\ge1}$ be nonnegative certification budgets, chosen predictably before observing the certification batch at checkpoint $m$, and satisfying
\[
\sum_{m\ge1}\alpha_m\le \alpha.
\]

We declare checkpoint $m$ certified if
\[
\max_{1\le t\le B} M_{m,t}\ge \alpha_m^{-1}.
\]
Let
\[
{\cal C}
\coloneqq
\left\{
m\ge1:\ \max_{1\le t\le B}M_{m,t}\ge \alpha_m^{-1}
\right\}
\]
be the set of certified checkpoints.

\begin{proposition}[Training-time certification by checkpointwise mixture martingales]
\label{prop:checkpointwise_mixture_certification}
Suppose  that the certification budgets satisfy
$
\sum_{m\ge1}\alpha_m\le\alpha$ almost surely.
Then
\[
\mathbb P\!\left(
\forall m\in{\cal C},\ p_m>q^\star
\right)
\ge
1-\alpha.
\]
Equivalently, every certified checkpoint is $(\epsilon,\delta)$-correct, and, after training, any adaptively selected checkpoint
\[
\hat m\in{\cal C}
\]
is $(\epsilon,\delta)$-correct with confidence at least $1-\alpha$.
\end{proposition}

\begin{proof}
Fix a checkpoint $m$. We condition on the training history before the certification batch for checkpoint $m$, and on the frozen checkpoint $w_m$. Under this conditioning, $p_m$, $\Pi_m$, and $\alpha_m$ are fixed, and
\[
Z_{m,1},\dots,Z_{m,B}
\stackrel{\rm i.i.d.}{\sim}
{\rm Bernoulli}(p_m).
\]

Assume the null hypothesis
\[
H_{0,m}: p_m\le q^\star
\]
holds. Fix any $r\in[q^\star,1]$. We show that $\{L_{m,t}(r)\}_{t=0}^B$ is a nonnegative supermartingale under $H_{0,m}$. Let ${\cal F}_{m,t}$ be the sigma-field generated by the training history, the frozen checkpoint $w_m$, and the first $t$ certification outcomes
\[
Z_{m,1},\dots,Z_{m,t}.
\]
For $t<B$,
\[
\frac{L_{m,t+1}(r)}{L_{m,t}(r)}
=
\begin{cases}
r/q^\star, & Z_{m,t+1}=1,\\[0.5ex]
(1-r)/(1-q^\star), & Z_{m,t+1}=0.
\end{cases}
\]
Therefore,
\begin{align*}
\mathbb E\!\left[
\frac{L_{m,t+1}(r)}{L_{m,t}(r)}
\,\middle|\,
{\cal F}_{m,t}
\right]
&=
p_m\frac{r}{q^\star}
+
(1-p_m)\frac{1-r}{1-q^\star}.
\end{align*}
A direct calculation gives
\[
p_m\frac{r}{q^\star}
+
(1-p_m)\frac{1-r}{1-q^\star}
=
1+
\frac{(p_m-q^\star)(r-q^\star)}
{q^\star(1-q^\star)}.
\]
Since $p_m\le q^\star$ and $r\ge q^\star$, the final term is nonpositive. Hence
\[
\mathbb E\!\left[
\frac{L_{m,t+1}(r)}{L_{m,t}(r)}
\,\middle|\,
{\cal F}_{m,t}
\right]
\le 1.
\]
Thus
\[
\mathbb E[L_{m,t+1}(r)\mid{\cal F}_{m,t}]
\le
L_{m,t}(r),
\]
so $\{L_{m,t}(r)\}_{t=0}^B$ is a nonnegative supermartingale.

Because $M_{m,t}$ is a mixture of the nonnegative supermartingales $L_{m,t}(r)$, Tonelli's theorem gives
\[
\mathbb E[M_{m,t+1}\mid{\cal F}_{m,t}]
=
\int_{q^\star}^1
\mathbb E[L_{m,t+1}(r)\mid{\cal F}_{m,t}]
\,\Pi_m(dr)
\le
\int_{q^\star}^1 L_{m,t}(r)\,\Pi_m(dr)
=
M_{m,t}.
\]
Therefore $\{M_{m,t}\}_{t=0}^B$ is a nonnegative supermartingale with $M_{m,0}=1$. By Ville's inequality,
\[
\mathbb P\!\left(
\max_{1\le t\le B} M_{m,t}\ge \alpha_m^{-1}
\,\middle|\,
\text{past},w_m,H_{0,m}
\right)
\le
\alpha_m.
\]

Let
\[
A_m
\coloneqq
\left\{
p_m\le q^\star
\text{ and }
\max_{1\le t\le B}M_{m,t}\ge \alpha_m^{-1}
\right\}
\]
be the event that checkpoint $m$ is falsely certified. From the conditional bound above,
\[
\mathbb P(A_m\mid \text{past})\le \alpha_m.
\]
Taking expectations,
\[
\mathbb P(A_m)\le \mathbb E[\alpha_m].
\]
By the union bound,
\[
\mathbb P\!\left(\exists m\ge1:\ A_m\right)
\le
\sum_{m\ge1}\mathbb P(A_m)
\le
\mathbb E\!\left[\sum_{m\ge1}\alpha_m\right]
\le
\alpha.
\]
Thus, with probability at least $1-\alpha$, no checkpoint with $p_m\le q^\star$ is certified. Equivalently,
\[
p_m>q^\star
\qquad
\forall m\in{\cal C}.
\]
Taking $q^\star=1-\delta$ gives the stated $(\epsilon,\delta)$-correctness guarantee.
\end{proof}

\begin{remark}[What the guarantee certifies]
\label{rem:what_checkpoint_certifies}
The guarantee is simultaneous over all certified checkpoints:
\[
\mathbb P\!\left(
\forall m\in{\cal C},\ p_m>1-\delta
\right)\ge 1-\alpha.
\]
Thus the final returned model need not be the first certified model. It may be any checkpoint selected after training, as long as it belongs to ${\cal C}$. The result does not certify checkpoints that were never certified, nor does it imply that all checkpoints after the first certified checkpoint are correct.
\end{remark}

\begin{remark}[Relation to the original ICPE training-time argument]
\label{rem:comparison_original_icpe}
The original finite-hypothesis ICPE argument pools evidence across checkpoints and tests a global null of the form
\[
\sup_m p_m\le q^\star.
\]
That pooled martingale can show that training has produced evidence against the hypothesis that all checkpoints are bad. However, by itself it does not certify an arbitrary later checkpoint unless one adds a persistence or monotonicity assumption on the sequence $(p_m)$, as done in \citep{russo_learning_2025}.

The checkpointwise martingale above is different. It tests each frozen checkpoint separately, spends error probability across checkpoints, and therefore certifies the checkpoints that actually pass the test. This matches the practical workflow in which we may check a model, fail to certify it, continue training, and later certify a different checkpoint. No monotonicity or persistence assumption on the training trajectory is required.
\end{remark}

\begin{remark}[Choice of certification budgets]
\label{rem:alpha_budget_choices}
If all checkpoints are treated symmetrically and a maximum number $M_{\max}$ of checks is fixed in advance, the uniform allocation
\[
\alpha_m=\frac{\alpha}{M_{\max}}
\]
is the simplest choice. If later checkpoints are expected to be better, one can use the exponentially back-loaded allocation
\[
\alpha_m
=
\alpha\,
\frac{(\gamma-1)\gamma^{m-1}}{\gamma^{M_{\max}}-1},
\qquad \gamma>1.
\]
This preserves the same correctness theorem because the proof only uses
\[
\sum_m\alpha_m\le\alpha.
\]
The choice of schedule affects power, not validity. Larger $\alpha_m$ makes checkpoint $m$ easier to certify, so back-loading the budget gives more power to later checkpoints at the expense of earlier ones.
\end{remark}

\begin{remark}[Finite-grid implementation]
\label{rem:finite_grid_implementation}
In practice we use a finite grid
\[
q^\star\le r_1<\cdots<r_L\le 1
\]
with weights $\omega_1,\dots,\omega_L$. Then
\[
M_{m,t}
=
\sum_{\ell=1}^L
\omega_\ell
\left(\frac{r_\ell}{q^\star}\right)^{S_{m,t}}
\left(\frac{1-r_\ell}{1-q^\star}\right)^{t-S_{m,t}},
\]
which is easy to compute and remains a valid mixture martingale. The grid should place mass on plausible alternatives above $q^\star$, for example $r_\ell\in\{0.91,0.92,0.93,0.94,0.95,0.97\}$ when $q^\star=0.9$. Placing mass closer to $q^\star$ improves power for small margins but requires larger batches; placing mass farther above $q^\star$ improves power when the checkpoint is substantially better than the target.
\end{remark}

\subsection{Choice of inference model and reward modeling}
\label{app:theoretical_results:gaussian_inference_reward}

This subsection clarifies the relation between the ideal posterior quantities used in the Bellman characterization and the inference model used in the implementation. The ideal inference rule maximizes
\[
q_t(h,x)\coloneqq \mathbb P(L_\theta(x)\leq\epsilon|H_t=h),
\qquad
r_t(h)\coloneqq \sup_{x\in\X}q_t(h,x),
\]
and any measurable maximizer is denoted by
\[
I_t^\star(h)\in\argmax_{x\in\X}q_t(h,x).
\]
The implementation does not parameterize $I_t^\star$ directly. Instead, it uses a stochastic selector $\hat I_t(\cdot|h)$, represented by a diagonal Gaussian base distribution
\[
\hat I_t(\cdot|h)=\mathcal N(\mu_t(h),\Sigma_t(h)).
\] 
The  recommendation is the  mean $\mu_t(h)$, while the critic evaluates samples from $\hat I_t(\cdot|h)$. 

\paragraph{Sampled reward.}
For fixed $(h,\theta)$, define the success probability of the implemented stochastic selector by
\[
\hat r_t(h,\theta)
\coloneqq \mathbb{E}_{x\sim \hat I_t(\cdot\mid h)}[{\bf 1}\{x\in {\cal X}_\epsilon(\theta)\}].
\]
Its posterior average is
\[
\hat r_t(h)\coloneqq
\mathbb E[\hat r_t(h,\theta)|H_t=h].
\]
Given $M$ Monte Carlo samples $X_t^{(1)},\ldots,X_t^{(M)}\sim\hat I_t(\cdot|h)$, we use
\[
\hat r_{t,M}(h,\theta)
\coloneqq
\frac1M\sum_{m=1}^M
{\bf 1}\{X_t^{(m)}\in\X_\epsilon(\theta)\}.
\]

\begin{lemma}[Sampled stochastic-selector reward]
\label{lemma:sampled_selector_reward}
Conditionally on $(H_t=h,\theta)$,
\[
\mathbb E[\hat r_{t,M}(h,\theta)|H_t=h,\theta]
=
\hat r_t(h,\theta),
\qquad
{\rm Var}(\hat r_{t,M}(h,\theta)|H_t=h,\theta)\leq \frac1{4M}.
\]
Furthermore,
\[
\hat r_t(h)
=
\mathbb E[q_t(h,\hat X_t)|H_t=h]
\leq
r_t(h).
\]
\end{lemma}

\begin{proof}
The first claim follows because $\hat r_{t,M}$ is the average of $M$ Bernoulli random variables with success probability $\hat r_t(h,\theta)$. The variance bound follows from $p(1-p)\leq 1/4$. For the posterior identity,  $\hat  X_t$ is conditional independent of $\theta$ given $H_t=h$, thus
\[
\hat r_t(h)
=
\int_\X
\mathbb P(x\in\X_\epsilon(\theta)|H_t=h)\hat I_t(dx|h)
=
\int_\X q_t(h,x)\hat I_t(dx|h).
\]
The last display is bounded by $\sup_{x\in\X}q_t(h,x)=r_t(h)$.
\end{proof}

The reward $\hat r_t(h)$ evaluates the stochastic selector we actually use. It is therefore conservative relative to the ideal deterministic reward $r_t(h)$, which assumes access to the best posterior recommendation.

\paragraph{Second-moment robustness.}
The next proposition gives two sufficient conditions under which the sampled reward is close to the ideal one. The first bound uses concentration around the task-level zero-loss target $x_\theta^\star$. The second uses concentration around the posterior-optimal rule $I_t^\star(h)$.

For $z\in\X$, define
\[
D_t(h,z)
\coloneqq
\mathbb E[\|\hat X_t-z\|^2|H_t=h],
\qquad
\hat X_t\sim\hat I_t(\cdot|h).
\]

\begin{proposition}[Second-moment robustness and ideal-reward gap]
\label{prop:second_moment_robustness_and_gap}
Assume that, for posterior-a.e. $\theta$, there exists $\rho(\theta)>0$ such that
\[
B(x_\theta^\star,\rho(\theta))\cap\X\subseteq\X_\epsilon(\theta).
\]
where $B(x,r)$ is an euclidean  ball of radius $r$ around $x$.
Then
\[
\hat r_t(h,\theta)
\geq
1-\frac{D_t(h,x_\theta^\star)}{\rho(\theta)^2}.
\]
Moreover, assume $\hat X_t$ is conditionally independent of $\theta$ given $H_t=h$. If $I_t^\star(h)\in\argmax_{x\in\X}q_t(h,x)$ and $q_t(h,\cdot)$ is $L_t(h)$-Lipschitz on $B(I_t^\star(h),R_t(h))\cap\X$, then
\[
0\leq r_t(h)-\hat r_t(h)
\leq
L_t(h)\sqrt{D_t(h,I_t^\star(h))}
+
\frac{D_t(h,I_t^\star(h))}{R_t(h)^2}.
\]
If $q_t(h,\cdot)$ is globally $L_t(h)$-Lipschitz on $\X$, the second term is unnecessary.
\end{proposition}

\begin{proof}
For the first claim, the margin assumption gives
\[
\{\|\hat X_t-x_\theta^\star\|\leq \rho(\theta)\}
\subseteq
\{\hat X_t\in\X_\epsilon(\theta)\}.
\]
Therefore
\[
1-\hat r_t(h,\theta)
\leq
\mathbb P(\|\hat X_t-x_\theta^\star\|>\rho(\theta)|H_t=h,\theta).
\]
Since the law of $\hat X_t$ depends on $h$ but not on $\theta$ conditional on $h$, Markov's inequality gives
\[
1-\hat r_t(h,\theta)
\leq
\frac{D_t(h,x_\theta^\star)}{\rho(\theta)^2}.
\]

For the second claim, by \cref{lemma:sampled_selector_reward},
\[
\hat r_t(h)
=
\mathbb E[q_t(h,\hat X_t)|H_t=h].
\]
Since $I_t^\star(h)$ maximizes $q_t(h,\cdot)$,
\[
r_t(h)-\hat r_t(h)
=
\mathbb E[q_t(h,I_t^\star(h))-q_t(h,\hat X_t)|H_t=h]\geq0.
\]
Let
\[
E_h\coloneqq
\{\|\hat X_t-I_t^\star(h)\|\leq R_t(h)\}.
\]
On $E_h$, Lipschitzness gives
\[
q_t(h,I_t^\star(h))-q_t(h,\hat X_t)
\leq
L_t(h)\|\hat X_t-I_t^\star(h)\|.
\]
On $E_h^c$, the same difference is at most $1$. Hence
\[
r_t(h)-\hat r_t(h)
\leq
L_t(h)\mathbb E[\|\hat X_t-I_t^\star(h)\||H_t=h]
+
\mathbb P(E_h^c|H_t=h).
\]
Jensen's inequality gives
\[
\mathbb E[\|\hat X_t-I_t^\star(h)\||H_t=h]
\leq
\sqrt{D_t(h,I_t^\star(h))},
\]
and Markov's inequality gives
\[
\mathbb P(E_h^c|H_t=h)
\leq
\frac{D_t(h,I_t^\star(h))}{R_t(h)^2}.
\]
If $q_t(h,\cdot)$ is globally Lipschitz, take $E_h=\X$ and remove the last term.
\end{proof}

The first bound gives the margin interpretation of the sampled reward: if the stochastic selector has small second moment around $x_\theta^\star$, then its samples are likely to be $\epsilon$-optimal. The second bound is different: it compares the stochastic selector to the ideal posterior reward and is small when the selector concentrates around $I_t^\star(h)$ and $q_t(h,\cdot)$ is locally regular. Thus the sampled reward evaluates the implemented selector through the same event used for correctness; it does not assume that the NLL mean is exactly optimal.

\paragraph{Gaussian NLL.}
We next characterize the population Gaussian NLL objective. Let
\[
Z\coloneqq x_\theta^\star
\]
be the selected zero-loss target, viewed as a random variable under the posterior law $\theta|H_t=h$. For $\mu\in\mathbb R^d$ and $\Sigma\in\mathbb S_{++}^d$, define
\[
\mathcal L_h(\mu,\Sigma)
\coloneqq
\mathbb E[-\log \mathcal N(Z;\mu,\Sigma)|H_t=h].
\]
Let
\[
\mu_t^{\rm NLL}(h)
\coloneqq
\mathbb E[Z|H_t=h],
\qquad
\Sigma_t^{\rm NLL}(h)
\coloneqq
{\rm Cov}(Z|H_t=h).
\]

\begin{proposition}[Gaussian NLL moment projection]
\label{prop:gaussian_nll_moment_projection}
Assume $Z|H_t=h$ has finite second moment and positive definite covariance. Then the unique minimizer of $\mathcal L_h(\mu,\Sigma)$ over $\mu\in\mathbb R^d$ and $\Sigma\in\mathbb S_{++}^d$ is
\[
\mu=\mu_t^{\rm NLL}(h),
\qquad
\Sigma=\Sigma_t^{\rm NLL}(h).
\]
If the covariance is restricted to be diagonal, the optimal mean is still $\mu_t^{\rm NLL}(h)$ and the optimal diagonal entries are the posterior coordinate variances of $Z$.
\end{proposition}

\begin{proof}
Up to an additive constant,
\[
\mathcal L_h(\mu,\Sigma)
=
\frac12\log\det\Sigma
+
\frac12\mathbb E[(Z-\mu)^\top\Sigma^{-1}(Z-\mu)|H_t=h].
\]
For fixed $\Sigma$, the second term is minimized at
$\mu=\mathbb E[Z|H_t=h]$. With this choice, the objective becomes
\[
\frac12\log\det\Sigma
+
\frac12{\rm tr}(\Sigma^{-1}\Sigma_t^{\rm NLL}(h)).
\]
The first-order condition in $\Sigma$ gives
\[
\Sigma=\Sigma_t^{\rm NLL}(h),
\]
and strict convexity in the natural parameters gives uniqueness. The diagonal case follows by the same calculation coordinate-wise.
\end{proof}

Thus Gaussian NLL performs a moment projection of the posterior law of $x_\theta^\star$: it matches the first two posterior moments, or the coordinate variances in the diagonal case. This does not imply that $\mu_t^{\rm NLL}(h)$ maximizes $q_t(h,\cdot)$ in general. The NLL objective learns the posterior target law, while optimality for stopping is defined by the $\epsilon$-success probability.

\paragraph{Near-optimality of the NLL mean.}
The NLL objective does not directly maximize $q_t(h,x)$. It learns the posterior mean of the selected target $x_\theta^\star$ under the Gaussian moment projection. The next result gives a simple condition under which this mean is nevertheless close to the ideal rule in posterior success probability. The condition is posterior concentration relative to the margin of the success set, not exact equality between the posterior mean and the maximizer of $q_t(h,\cdot)$.

\begin{proposition}[Near-optimality of the NLL mean]
\label{prop:nll_mean_near_optimality}
Fix a history $h$ and write
\[
Z\coloneqq x_\theta^\star,
\qquad
\mu_t^{\rm NLL}(h)\coloneqq \mathbb E[Z|H_t=h],
\qquad
V_t(h)\coloneqq \mathbb E[\|Z-\mu_t^{\rm NLL}(h)\|^2|H_t=h].
\]
Assume that there exists $\rho>0$ such that, posterior-a.s.,
$
B(x_\theta^\star,\rho)\cap\X\subseteq\X_\epsilon(\theta)$, where $B(x,r)$ is an Euclidean ball of radius $r$ around $x$.
Let $\hat \mu_t(h)\in\X$ be the deployed mean recommendation and assume
\[
\|\hat \mu_t(h)-\mu_t^{\rm NLL}(h)\|\leq e_t(h).
\]
Then
\[
0\leq r_t(h)-q_t(h,\hat \mu_t(h))
\leq
\frac{V_t(h)+e_t(h)^2}{\rho^2}.
\]
Moreover, if $I_t^\star(h)\in\argmax_{x\in\X}q_t(h,x)$, $q_t(h,\cdot)$ is $L_t(h)$-Lipschitz on $B(I_t^\star(h),R_t(h))\cap\X$, and
$\|\hat \mu_t(h)-I_t^\star(h)\|\leq R_t(h)$, then
\[
0\leq r_t(h)-q_t(h,\hat \mu_t(h))
\leq
L_t(h)\|\hat \mu_t(h)-I_t^\star(h)\|.
\]
\end{proposition}

\begin{proof}
The margin assumption implies
\[
\{\|\hat \mu_t(h)-x_\theta^\star\|\leq \rho\}
\subseteq
\{\hat \mu_t(h)\in\X_\epsilon(\theta)\}.
\]
Therefore
\[
q_t(h,\hat \mu_t(h))
\geq
1-\mathbb P(\|\hat \mu_t(h)-Z\|>\rho|H_t=h).
\]
By Markov's inequality,
\[
q_t(h,\hat \mu_t(h))
\geq
1-
\frac{\mathbb E[\|\hat \mu_t(h)-Z\|^2|H_t=h]}{\rho^2}.
\]
Since $\hat \mu_t(h)$ is deterministic conditional on $H_t=h$,
\[
\mathbb E[\|\hat \mu_t(h)-Z\|^2|H_t=h]
=
V_t(h)+\|\hat \mu_t(h)-\mu_t^{\rm NLL}(h)\|^2
\leq
V_t(h)+e_t(h)^2.
\]
The first claim follows because $r_t(h)\leq1$. For the second claim, use
$r_t(h)=q_t(h,I_t^\star(h))$ and Lipschitzness:
\[
r_t(h)-q_t(h,\hat \mu_t(h))
=
q_t(h,I_t^\star(h))-q_t(h,\hat \mu_t(h))
\leq
L_t(h)\|\hat \mu_t(h)-I_t^\star(h)\|.
\]
\end{proof}
The proposition clarifies what is, and is not, implied by the Gaussian NLL. In general,
$\mu_t^{\rm NLL}(h)=\mathbb E[x_\theta^\star|H_t=h]$ need not maximize
$q_t(h,\cdot)$: if the posterior law of $x_\theta^\star$ is multimodal, the
mean can lie between modes. The first bound gives the guarantee tied to NLL
training. If the posterior uncertainty on $x_\theta^\star$ is small relative to
the margin of $\X_\epsilon(\theta)$, then the deployed mean is near-optimal in
posterior success probability. The Lipschitz bound is different: it says that
any deterministic recommendation close to an ideal maximizer $I_t^\star(h)$ is
near-optimal when $q_t(h,\cdot)$ is locally regular. Thus the Lipschitz argument
also applies to the NLL mean, but only if one separately controls its distance
to $I_t^\star(h)$.

For localization losses $L_\theta(x)=\|x-x_\theta^\star\|$, the margin condition
holds with $\rho(\theta)=\epsilon$. For smooth value-gap losses, it follows from
a local upper curvature bound near $x_\theta^\star$. For example, if
\[
f_\theta(x)-f_\theta(x_\theta^\star)
\leq
\frac{M}{2}\|x-x_\theta^\star\|^2
\]
near $x_\theta^\star$, then
\[
B\left(x_\theta^\star,\sqrt{\frac{2\epsilon}{M}}\right)
\cap\X
\subseteq
\X_\epsilon(\theta).
\]

\begin{remark}[Exact alignment under symmetric localization]
In special cases the NLL mean is exactly Bayes-optimal. Suppose the success sets
are translates of a fixed centrally symmetric convex set, i.e.,
\[
\X_\epsilon(\theta)=\{x\in\X:x-x_\theta^\star\in S_\epsilon\},
\]
and suppose $x_\theta^\star|H_t=h$ is Gaussian with mean
$\mu_t^{\rm NLL}(h)$, with $\mu_t^{\rm NLL}(h)\in\X$. Then
\[
\mu_t^{\rm NLL}(h)\in\argmax_{x\in\X}q_t(h,x).
\]
Indeed, $q_t(h,x)$ is the posterior probability that a Gaussian random variable
falls in a translate of $S_\epsilon$, and this probability is maximized when
the translate is centered at the Gaussian mean. Without this type of symmetry,
the posterior mean, posterior mode, and maximizer of $q_t(h,\cdot)$ can differ.
\end{remark}

\subsection{Robustness to prior misspecification}
\label{sec:prior_misspecification}

We now study deployment under a misspecified task prior. The controller \cicpe{} is trained under a prior $\nu$ on $\Theta$ and then frozen. At deployment, the same learned inference network, critic, action rule, cost parameter, and stopping rule are used, but the environment parameter is drawn from a different prior $\nu'$. Thus, conditional on a fixed environment $\theta$, the trajectory law $\mathbb P_\theta^{\pi_\phi}$ is unchanged; only the outer averaging measure over $\theta$ changes.

Let
\[
s(\theta)
\coloneqq
\mathbb P_\theta^{\pi_\phi}
\bigl(\mu_\phi(H_{\hat\tau})\in \X_\epsilon(\theta)\bigr)
\in[0,1],
\qquad
t(\theta)
\coloneqq
\mathbb E_\theta^{\pi_\phi}[\hat\tau]\in[0,\infty].
\]
For a prior $\eta$ on $\Theta$, define
\[
p^{(\eta)}\coloneqq \int_\Theta s(\theta)\,\eta(d\theta),
\qquad
T^{(\eta)}\coloneqq \int_\Theta t(\theta)\,\eta(d\theta).
\]
The training-prior guarantee is
\[
p^{(\nu)}\ge 1-\delta.
\]
The goal is to understand how $p^{(\nu')}$ and $T^{(\nu')}$ change when $\nu$ is replaced by $\nu'$.

\begin{assumption}[Measurable performance profiles]
\label{assump:prior_shift_profiles}
The maps $s:\Theta\to[0,1]$ and $t:\Theta\to[0,\infty]$ are Borel measurable. Moreover, $t\in L^1(\eta)$ for every prior $\eta$ considered below.
\end{assumption}

\begin{lemma}[Prior shift as reweighting]
\label{lem:prior_shift_reweighting}
Under \cref{assump:prior_shift_profiles}, correctness and stopping time under any prior $\eta$ are given by
\[
p^{(\eta)}=\mathbb E_{\theta\sim\eta}[s(\theta)],
\qquad
T^{(\eta)}=\mathbb E_{\theta\sim\eta}[t(\theta)].
\]
If $\eta\ll\nu$ with density ratio $w_\eta=d\eta/d\nu$, then
\[
p^{(\eta)}=\mathbb E_{\nu}[w_\eta s],
\qquad
T^{(\eta)}=\mathbb E_{\nu}[w_\eta t].
\]
\end{lemma}

\begin{proof}
Under prior $\eta$, the joint law factors as
\[
\eta(d\theta)\,\mathbb P_\theta^{\pi_\phi}(dh).
\]
Applying Tonelli's theorem to the success indicator gives
\[
p^{(\eta)}
=
\int_\Theta
\mathbb P_\theta^{\pi_\phi}
\bigl(\mu_\phi(H_{\hat\tau})\in \X_\epsilon(\theta)\bigr)
\,\eta(d\theta)
=
\int_\Theta s(\theta)\,\eta(d\theta).
\]
The identity for $T^{(\eta)}$ follows similarly from Tonelli applied to the nonnegative random variable $\hat\tau$. The reweighting identities follow by the change of measure $d\eta=w_\eta d\nu$.
\end{proof}

The next proposition gives several complementary ways of transferring the training-prior guarantee $p^{(\nu)}\ge 1-\delta$ to a deployment prior $\nu'$. The density-ratio and $\chi^2$ bounds are useful when $\nu'\ll\nu$; the total-variation and good-set bounds do not require absolute continuity.

\begin{proposition}[Success probability under prior shift]
\label{prop:prior_shift_success_compact}
Assume \cref{assump:prior_shift_profiles}, and write $p=p^{(\nu)}$, $p'=p^{(\nu')}$. Then:
\begin{enumerate}
    \item[\rm (a)] If $\nu'\ll\nu$ and $C=\|d\nu'/d\nu\|_\infty$, then
    $
    1-p'\le C(1-p)
    $.
    In particular, $p\ge 1-\delta$ implies $p'\ge 1-C\delta$.

    \item[\rm (b)] If $\nu'\ll\nu$ and $\chi^2(\nu'\|\nu)<\infty$, then
    $
    |p'-p|
    \le
    \sqrt{\chi^2(\nu'\|\nu)\operatorname{Var}_\nu(s)}.
    $.
    Hence $p\ge 1-\delta$ implies
    \[
    p'\ge 1-\delta-\sqrt{\chi^2(\nu'\|\nu)\delta}.
    \]

    \item[\rm (c)] For arbitrary priors,
    $
    |p'-p|\le \operatorname{TV}(\nu',\nu)
    $.
    Consequently,
    \[
    p'\ge 1-\delta-\operatorname{TV}(\nu',\nu).
    \]
    By Pinsker's inequality, this also gives
    \[
    p'\ge 1-\delta-\sqrt{\frac12 D_{\mathrm{KL}}(\nu'\|\nu)}
    \]
    whenever $D_{\mathrm{KL}}(\nu'\|\nu)<\infty$, and the analogous bound with the KL arguments reversed.

    \item[\rm (d)] Let $G\subseteq\Theta$ be measurable. If $s(\theta)\ge 1-\eta$ on $G$, then
    \[
    p'\ge (1-\eta)\nu'(G).
    \]
    Moreover, if $p\ge 1-\delta$ and
    \[
    G_\eta\coloneqq\{\theta:s(\theta)\ge 1-\eta\},
    \]
    then
    \[
    \nu(G_\eta)\ge 1-\frac{\delta}{\eta},
    \qquad
    p'\ge (1-\eta)\nu'(G_\eta).
    \]
\end{enumerate}
\end{proposition}

\begin{proof}
Let $e(\theta)=1-s(\theta)\in[0,1]$.

For (a), if $w=d\nu'/d\nu$, then
\[
1-p'=\mathbb E_{\nu'}[e]=\mathbb E_\nu[we]
\le \|w\|_\infty\mathbb E_\nu[e]
=C(1-p).
\]

For (b), since $\mathbb E_\nu[w]=1$,
\[
p'-p
=
\mathbb E_\nu[(w-1)s]
=
\mathbb E_\nu[(w-1)(s-p)].
\]
Cauchy--Schwarz yields
\[
|p'-p|
\le
\sqrt{\mathbb E_\nu[(w-1)^2]}\,
\sqrt{\operatorname{Var}_\nu(s)}
=
\sqrt{\chi^2(\nu'\|\nu)\operatorname{Var}_\nu(s)}.
\]
Since $0\le s\le1$, $s^2\le s$, hence
\[
\operatorname{Var}_\nu(s)\le p(1-p)\le 1-p.
\]
If $p\ge1-\delta$, then $\operatorname{Var}_\nu(s)\le\delta$, giving the displayed bound.

For (c), since $s\in[0,1]$,
\[
|p'-p|
=
\left|\int s\,d(\nu'-\nu)\right|
\le
\operatorname{TV}(\nu',\nu).
\]
Pinsker's inequality gives the KL consequences.

For (d),
\[
p'=\int_G s\,d\nu'+\int_{G^c}s\,d\nu'
\ge (1-\eta)\nu'(G).
\]
If $G=G_\eta$, then $G_\eta^c=\{e>\eta\}$. Since $\mathbb E_\nu[e]\le\delta$, Markov's inequality gives
\[
\nu(G_\eta^c)\le \frac{\delta}{\eta}.
\]
\end{proof}

The good-set formulation is often the most faithful explanation of empirical robustness. Average correctness under $\nu$ implies that the controller is accurate on a large $\nu$-measure set of environments. Deployment remains accurate whenever $\nu'$ continues to place most of its mass on that same set.

We now state the analogous bounds for the stopping time. Since $t$ is not bounded by one, the bounds require either density-ratio control, a second-moment assumption, or a bounded horizon.

\begin{proposition}[Stopping time under prior shift]
\label{prop:prior_shift_stopping_compact}
Assume \cref{assump:prior_shift_profiles}, and write $T=T^{(\nu)}$, $T'=T^{(\nu')}$. Then:
\begin{enumerate}
    \item[\rm (a)] If $\nu'\ll\nu$ and $C=\|d\nu'/d\nu\|_\infty$, then
    $T'\le CT$.

    \item[\rm (b)] If $\nu'\ll\nu$, $\chi^2(\nu'\|\nu)<\infty$, and $t\in L^2(\nu)$, then
    \[
    |T'-T|
    \le
    \sqrt{\chi^2(\nu'\|\nu)\operatorname{Var}_\nu(t)}.
    \]

    \item[\rm (c)] If $\hat\tau\le T_{\max}$ almost surely under every $\theta$, then
    \[
    |T'-T|\le T_{\max}\operatorname{TV}(\nu',\nu).
    \]

    \item[\rm (d)] If $\hat\tau\le T_{\max}$ almost surely and $t(\theta)\le \tau_0$ on a measurable set $F\subseteq\Theta$, then
    \[
    T'\le \tau_0+(T_{\max}-\tau_0)\nu'(F^c).
    \]
\end{enumerate}
\end{proposition}

\begin{proof}
The first claim follows from
\[
T'=\mathbb E_\nu[wt]\le \|w\|_\infty\mathbb E_\nu[t]=CT.
\]
For the second, use
\[
T'-T=\mathbb E_\nu[(w-1)t]
=
\mathbb E_\nu[(w-1)(t-T)]
\]
and apply Cauchy--Schwarz. For the third, $t/T_{\max}\in[0,1]$, so the total-variation bound applies. For the fourth, split
\[
T'=\int_F t\,d\nu'+\int_{F^c}t\,d\nu'
\]
and use $t\le \tau_0$ on $F$ and $t\le T_{\max}$ everywhere.
\end{proof}

\begin{remark}[Beta--Uniform shifts]
For one-dimensional shifts with $\nu=\mathrm{Unif}(0,1)$ and
$\nu'=\mathrm{Beta}(\alpha,\beta)$, the constants in
\cref{prop:prior_shift_success_compact} can be evaluated in closed form.
For example,
\[
\chi^2(\nu'\|\nu)
=
\frac{B(2\alpha-1,2\beta-1)}{B(\alpha,\beta)^2}-1
\]
when $\alpha,\beta>1/2$, and is infinite otherwise. The essential supremum
$\|d\nu'/d\nu\|_\infty$ is finite exactly when $\alpha,\beta\ge1$, with the usual interior-mode formula when $\alpha,\beta>1$. The reverse KL used to index the robustness tables is
\[
D_{\mathrm{KL}}(\mathrm{Unif}\|\mathrm{Beta}(\alpha,\beta))
=
\log B(\alpha,\beta)+\alpha+\beta-2.
\]
These constants are useful for interpreting the experimental tables, but the robustness mechanism itself is entirely captured by the profile bounds above.
\end{remark}

The results above are deliberately environment-agnostic and help  identify  generic failure modes: robustness can fail only when the deployment prior emphasizes regions where the frozen controller is inaccurate or slow, or when $\nu'$ leaves the support region on which the controller was effectively trained.

\subsection{Sample Complexity: Value Estimation vs Argmax Localization}\label{subsec:sample_complexity_value_vs_argmax}
In this subsection we study the sample complexity of estimating the maximum value of a gaussian process $F\sim {\rm GP}(0, k_\ell)$ on $D=[0,1]^d$, and consider an RBF kernel $k_\ell(x,x')\coloneqq \exp\left( -\frac{\|x-x'\|^2}{2\ell^2}\right) $ with $x,x'\in D$.

We let the unknown lengthscale be random:
\[
    \Lambda\sim \nu,
    \qquad
    \operatorname{supp}(\nu)\subseteq [\ell_-,\ell_+]\subset(0,\infty).
\]

The learner observes
\[
    Y_t=F(a_t)+\xi_t,
    \qquad
    \xi_t\sim \mathcal N(0,\sigma^2),
\]
where $a_t\in D$ is chosen adaptively from the past history and
$\sigma>0$ is known.

Define
\[
    F^\star=\max_{x\in D}F(x),
    \qquad
    D^\star=\argmax_{x\in D}F(x),
\]
and write $X^\star$ whenever the maximizer is unique.

\paragraph{Argmax localization complexity.}
For $r>0$ and $\delta\in(0,1)$, define the Bayesian argmax-localization
complexity
\[
    T_{\arg,\nu}(r,\delta)
    =
    \inf_{\mathcal A}\mathbb E_{\Lambda,F,\xi}[\tau],
\]
where the infimum is over all sequential algorithms $\mathcal A$ that
output $\hat X$ and satisfy
\[
    \mathbb P_{\Lambda,F,\xi}
    \left(
       \inf_{X^\star\in D^\star} \|\hat X-X^\star\|\le r
    \right)
    \ge 1-\delta .
\]

\paragraph{Max-value estimation complexity.}
Similarly, define the Bayesian max-value-estimation complexity
\[
    T_{{\rm val},\nu}(\epsilon,\delta)
    =
    \inf_{\mathcal A}\mathbb E_{\Lambda,F,\xi}[\tau],
\]
where the infimum is over all sequential algorithms that output
$\hat v$ and satisfy
\[
    \mathbb P_{\Lambda,F,\xi}
    \left(
        |\hat v-F^\star|\le \epsilon
    \right)
    \ge 1-\delta .
\]

All probabilities and expectations are under the joint hierarchical law
of $(\Lambda,F)$, the observation noise, and any internal randomness of
the algorithm.

\paragraph{Main result.}  To compare the sample complexity of argmax localization vs max-value estimation, we use the fact that under regularity assumptions we approximately have $r\sim \sqrt{\epsilon}$. This follows from a Taylor's expansion
\[
F(X)=F(X^\star)+ \frac{1}{2}(X-X^\star)^\top \nabla^2 F(X^\star) (X-X^\star)+o(\|X-X^\star\|^2).
\]
from which we find  $F(X)-F(X^\star)\approx c \|X-X^\star\|_2^2$ for a suitable constant. Therefore, localizing to radius $r$ gives an error of roughly $ cr^2$ on the value. Therefore, for $\epsilon$ accuracy on the value, we can take the radius to be $r \sim \sqrt{ \epsilon}$.

Then, we have the following  main result.
\begin{theorem}[Bayesian value--argmax separation]
\label{thm:separation}
Consider the hierarchical RBF-GP model: $\Lambda \sim \nu$ with 
$\operatorname{supp}(\nu) \subseteq [\ell_-, \ell_+] \subset (0,\infty)$, 
and $F \mid \Lambda = \ell \sim \mathrm{GP}(0, k_\ell)$ on $D = [0,1]^d$, 
with observations $Y_t = F(a_t) + \xi_t$, $\xi_t \sim \mathcal{N}(0, \sigma^2)$.

Fix $\delta \in (0,1)$. Suppose there exists $\eta > 0$ such that $
\beta_\eta = \mathbb{P}_{\Lambda, F}(\mathcal{I}_\eta^c) < \delta$, where $\beta_\eta$ is defined in \cref{lem:regular_basin_quantile}.
Then:
\begin{enumerate}
    \item \textbf{(Argmax upper bound.)} There exist constants 
    $B < \infty$ and $C < \infty$, depending on $\nu, d, \delta, \sigma$ 
    but not on $\epsilon$, such that for all sufficiently small $\epsilon > 0$,
    \[
    T_{\arg,\nu}^{\rm Bayes}(\sqrt{\epsilon}, \delta) 
    \leq B + C\,\sigma^2\,\epsilon^{-3/2}\log\log \frac{1}{\epsilon}.
    \]
    
    \item \textbf{(Value lower bound.)} There exist constants 
    $c > 0$ and $c' < \infty$, depending on $\nu, d, \delta, \sigma$ 
    but not on $\epsilon$, such that for all $\epsilon > 0$,
    \[
    T_{{\rm val},\nu}^{\rm Bayes}(\epsilon, \delta) 
    \geq c\,\frac{\sigma^2}{\epsilon^2}\log\frac{1}{\delta} - c'.
    \]
\end{enumerate}
Consequently,
\[
\lim_{\epsilon \to 0} 
\frac{T_{{\rm val},\nu}^{\rm Bayes}(\epsilon, \delta)}{
T_{\arg,\nu}^{\rm Bayes}(\sqrt{\epsilon}, \delta)} = \infty.
\]
Thus, under the high-probability interior regularity condition, 
max-value estimation is asymptotically harder than argmax localization 
in the fully Bayesian hierarchical RBF model.
\end{theorem}

\begin{proof}
Set $\alpha \in (\beta_\eta,\delta)$. Since $\beta_\eta < \alpha$, 
\cref{lem:regular_basin_quantile} provides deterministic constants 
$\rho, \mu, L, M, \Gamma > 0$ and an event 
$\mathcal{E}_{\delta}$ with 
$\mathbb{P}(\mathcal{E}_{\alpha}) \geq 1 - \delta$.

\noindent\textit{Part 1.}We apply \cref{thm:tbal_sample_complexity} applies with $r = \sqrt{\epsilon}$, giving
\[
T_{\arg,\nu}^{\rm Bayes}(\sqrt{\epsilon}, \delta) 
\leq  B + O\left(\sigma^2 \epsilon^{-3/2}\log \log \frac{1}{\epsilon}\right).
\]

\noindent\textit{Part 2.}
By \cref{thm:max_value_lowerbound}, with 
$C_- = \exp(-d/(2\ell_-^2))$ and $C_+ = \exp(d/(4\ell_-^2))$,
\[
T_{{\rm val},\nu}^{\rm Bayes}(\epsilon, \delta) 
\geq \frac{\sigma^2}{C_+^2}
\left[\frac{C_-^2\, z_{1-\delta/2}^2}{\epsilon^2} - 1\right]_+.
\]
For fixed $\delta$ and small $\epsilon$, using 
$z_{1-\delta/2}^2 \geq c_0 \log(1/\delta)$, this gives
\[
T_{{\rm val},\nu}^{\rm Bayes}(\epsilon, \delta) 
\geq c\,\frac{\sigma^2}{\epsilon^2}\log\frac{1}{\delta} - c'.
\]

\noindent\textit{Separation.}
\[
\frac{T_{{\rm val},\nu}^{\rm Bayes}(\epsilon, \delta)}{
T_{\arg,\nu}^{\rm Bayes}(\sqrt{\epsilon}, \delta)}
\geq 
\frac{c\,\sigma^2 \epsilon^{-2} \log(1/\delta) - c'}{
B + C\,\sigma^2\,\epsilon^{-3/2}\log \log(1/\epsilon)}
\longrightarrow \infty
\qquad \text{as } \epsilon \downarrow 0,
\]
since $\epsilon^{-2} \gg \epsilon^{-3/2}\log \log(1/\epsilon)$.
\end{proof}

So value estimation can be asymptotically harder than argmax localization. Even at the radius where localizing the argmax would "in principle" tell you the value to accuracy $\epsilon$, the localization itself is cheaper than directly estimating the value.
The  reason is simple: the algorithm exploits the  geometry (gradient information,  smoothness) to localize at a fast rate, but it does not help with the problem of estimating the height of a function.

\subsubsection{Max-value estimation complexity}
 To derive a lower bound on $    T_{{\rm val},\nu}(\epsilon,\delta)$, we convert the problem of estimating the maximum into the problem of estimating a parameter $\Theta$. We note that the lower bound is not tight, but for our purpose (of showing that the argmax localization is easier) this is not important, as our goal is to show that even this approximate lower bound still yields an harder problem. In particular, we assume the learner has access to a particular quantity $W_\ell$ that appears in the proof. We obtain the following result (note that the constants are not optimized, and the only goal is to show the dependency on $\epsilon$).

\begin{theorem}\label{thm:max_value_lowerbound}
Consider the problem of estimating $F^\star=\max_{x\in D}F(x)$ where $D=[0,1]^d, F\mid \Lambda \sim {\rm GP}(0,k_\Lambda)$ with kernel $k_\ell(x,y)=\exp(-\|x-y\|_2^2/(2\ell^2))$ and $\Lambda\sim \nu$  with continuous support in $[\ell_-,\ell_+]\subset (0,\infty)$. Consider any sequential algorithm ${\cal A}$ that in each round selects $a_t$ and observes $Y_t=F(a_t)+\xi_t$, where $\xi_t\sim {\cal N}(0,1)$. If the algorithm outputs $\hat v$ at some stopping time $\tau$ satisfying $\mathbb{P}_{\Lambda, F}(|\hat v-F^\star|\leq \epsilon)\geq 1-\delta$, then we say that the algorithm is $(\epsilon,\delta)$-correct. Then, for any $(\epsilon,\delta)$-correct algorithm we have that
    \[
\inf_{\mathcal{A}:(\epsilon,\delta)-\text{correct}} \mathbb{E}[\tau] \geq \frac{\sigma^2}{C_+^2} \left[\frac{C_-^2 z_{1-\delta/2}^2 }{\epsilon^2}-1\right]_+,
\]
where $z_{1-\delta/2}=\Phi^{-1}(1-\delta/2)$\footnote{Inverse of the CDF of a standard normal distribution.}, $C_-=\exp(-d/(2\ell_-^2))$ and $C_+ =\exp(d / (4\ell_-^2))$.
\end{theorem}
\begin{proof}
    The proof relies on converting the problem into that of estimating a parameter $\Theta$.
We use the property of independent Gaussian r.v.: for jointly Gaussian $X,Y$ we have that $X=\frac{{\rm Cov}(X,Y)}{{\rm Var}(Y)}Y+W$, where $W$ is an independent zero-mean Gaussian.

In this case, we take $Y=\Theta_\ell$ and $X=F$, where we define  $\Theta_\ell$ as 
\[
\Theta_\ell \coloneqq\frac{1}{s_\ell} \int_D F(x){\rm d}x,\qquad s_\ell^2 \coloneqq \int_{D\times D} k_\ell(x,y)\ {\rm d}x \ {\rm d}y.
\]
Computing the covariance, we obtain
\begin{align*}
{\rm Cov}(F(x),\Theta_\ell)&= \mathbb{E}[F(x)\Theta_\ell],\\
&=\frac{1}{s_\ell}\mathbb{E}\left[F(x)\int_D F(y){\rm d}y\right],\\
&=\frac{1}{s_\ell}\int_D\mathbb{E}\left[F(x) F(y)\right]{\rm d}y,\\
&=\frac{1}{s_\ell}\int_D k_\ell(x,y)\ {\rm d}y \eqqcolon c_\ell(x).
\end{align*}
Therefore, we can rewrite the observation using that $F(x)=c_\ell(x)\Theta_\ell + W_\ell(x)$ for some GP $W_\ell$ with $0$-mean, and thus 
\[
Y_t=[c_\ell(a_t)\Theta_\ell + W_\ell(a_t)] + \xi_t.
\]

From the expression of $F(x)$ we observe that 
\[
V(\theta)=\max_x [c_\ell(x)\theta+W_\ell(x)],
\]
is increasing in $\theta$. Denote a maximizer by $x_\theta$, then, for $\theta'\neq\theta$ we have
\[
V(\theta') \geq c_\ell(x_\theta)\theta'+W_\ell(x_\theta)= V(\theta)+c_\ell(x_\theta)(\theta'-\theta).
\]
Similarly,
\[
V(\theta) \geq c_\ell(x_{\theta'}) \theta +W_\ell(x_{\theta'})=V(\theta')+c_\ell(x_{\theta'})(\theta-\theta').
\]
We now derive bounds on $c_\ell$. From the definition
\[
c_\ell(x) = \frac{1}{s_\ell} \int_D k_\ell(x,y)\ {\rm d}y,
\]
since $D$ is compact, and $k_\ell$ is the RBF kernel, we have that $k_\ell(x,y)\leq 1 \Rightarrow c_\ell(x)\leq 1/s_\ell$ and $k_\ell(x,y) \geq \exp\left( -d/(2\ell^2)\right) \Rightarrow c_\ell(x) \geq \exp\left( -d/(2\ell^2)\right)/s_\ell$. Since $s_\ell^2 \geq  \exp\left( -d/(2\ell^2)\right)$ and $s_\ell^2\leq 1$, we find
\[
\underbrace{\exp\left( -d/2\ell_-^2)\right)}_{\eqqcolon C_-} \leq c_\ell(x)\leq  \underbrace{\frac{1}{\exp\left( -d/(4\ell_-^2)\right)}}_{\eqqcolon C_+}.
\]
Then, from the bounds above on $V$ we find
\[
V(\theta')\geq V(\theta)+ C_-(\theta'-\theta),\qquad V(\theta)\geq V(\theta')-C_+ (\theta'-\theta),
\]
leading to
\[
C_-(\theta'-\theta)\leq V(\theta')-V(\theta) \leq C_+(\theta'-\theta).
\]
Therefore $V$ is bi-Lipschitz, and thus invertible. 
Choosing $u=V(\theta)$ and $v=V(\theta')$, we obtain
\[
|V^{-1}(u)-V^{-1}(v)| \leq \frac{1}{C_-} | u-v|.
\]
Therefore, with $\hat v= V(\hat \Theta)$ and $F^\star = V(\Theta)$ we get
\[
|\hat \Theta-\Theta | \leq \frac{1}{C_-} |\hat v-F^\star|.
\]
Hence, assuming the algorithm has access to $W_\ell$, we can obtain a lower bound on the sample complexity by using \cref{lemma:scalar_estimation}. For any $(\epsilon',\delta)$-algorithm that estimates $\Theta$, we can choose $\epsilon'=\epsilon/C_-$ to obtain $|V^{-1}(\hat \Theta)-V^{-1}(\Theta) |\leq \epsilon'$ at the stopping time. Therefore,
\[
\inf_{\mathcal{A}:(\epsilon,\delta)-\text{correct}} \mathbb{E}[\tau] \geq \frac{\sigma^2}{C_+^2} \left[\frac{C_-^2 z_{1-\delta/2}^2 }{\epsilon^2}-1\right]_+,
\]
\end{proof}

\begin{lemma}[Scalar estimation lemma]\label{lemma:scalar_estimation}
Let $\Theta\sim {\cal N}(0,1)$. Suppose to observe $Z_t=C_t\Theta+\xi_t$ with $|C_t|\leq C$ chosen adaptively and $\xi_t\sim {\cal N}(0,\sigma^2)$. Let $\tau$ be a stopping time such that at $\tau$ the algorithm outputs $\hat \Theta$ satisfying $\mathbb{P}(|\hat \Theta-\Theta|\leq \epsilon)\geq 1-\delta$ for $\epsilon>0,\delta\in (0,1)$. Then
\[
\inf_{\mathcal{A}:(\epsilon,\delta)-\text{correct}} \mathbb{E}[\tau] \geq \frac{\sigma^2}{C^2} \left[\frac{z_{1-\delta/2}^2 }{\epsilon^2}-1\right]_+,
\]
where $z_{1-\delta/2}=\Phi^{-1}(1-\delta/2)$.
\end{lemma}
\begin{proof}
Let ${\cal H}_t$ denote the filtration history after $t$ observations. Since everything is Gaussian, also the posterior law $\Theta \mid {\cal H}_t$ is Gaussian, of parameter $(m_t,v_t)$. In particular, the precision is
\[
P_t= v_t^{-1}=1 +  \frac{1}{\sigma^2}\sum_{s=1}^t C_s^2.
\]
After stopping, for any estimate $u$ we have
\[
\mathbb{P}(|u-\Theta|\leq \epsilon \mid {\cal H}_\tau)=\mathbb{P}(\Theta\in [u-\epsilon,u+\epsilon] \mid {\cal H}_\tau).
\]
This quantity is maximized when the interval is centered in $m_\tau$, therefore
\[
\mathbb{P}(|u-\Theta|\leq \epsilon \mid {\cal H}_\tau)\leq \Phi\left( \frac{\epsilon}{\sqrt{v_\tau}}\right)-\Phi\left( -\frac{\epsilon}{\sqrt{v_\tau}}\right)=2\Phi\left( \frac{\epsilon}{\sqrt{v_\tau}}\right)-1.
\]
Then, for any $(\epsilon,\delta)$-PAC algorithm we have
\[
1-\delta \leq \mathbb{E}[\mathbb{P}(|u-\Theta|\leq \epsilon \mid {\cal H}_\tau)] \leq  2\mathbb{E}[\Phi(\epsilon \sqrt{P_\tau})]-1.
\]
Since $\Phi$ is concave, the argument is increasing and concave in $P_\tau$, then $s \mapsto \Phi(\epsilon \sqrt{s})$ is concave, we also obtain
\[
1-\delta \leq 2\Phi\left(\epsilon\sqrt{ \mathbb{E}[P_\tau]}\right)-1.
\]
Taking the inverse, and defining $z_{1-\delta/2}=\Phi^{-1}(1-\delta/2)$, we find
\[
z_{1-\delta/2} \leq \epsilon \sqrt{ \mathbb{E}[P_\tau]}.
\]
Hence, we are just let with bounding $P_\tau:$
\[
P_\tau \leq 1+ \frac{1}{\sigma^2} C^2 \tau,
\]
from which we get
\[
z_{1-\delta/2} \leq \epsilon \sqrt{1+\frac{C^2}{\sigma^2}\mathbb{E}[\tau]}.
\]
Therefore
\[
\mathbb{E}[\tau] \geq \frac{\sigma^2}{C^2} \left[\frac{z_{1-\delta/2}^2 }{\epsilon^2}-1\right]_+
\]
\end{proof}

\subsubsection{Argmax localization  complexity}
We now study the problem of locating the argmax of a GP.
Recall that  for $r>0$ and $\delta\in(0,1)$, define the Bayesian argmax-localization
complexity
\[
    T_{\arg,\nu}(r,\delta)
    =
    \inf_{\mathcal A}\mathbb E_{\Lambda,F,\xi}[\tau],
\]
where the infimum is over all sequential algorithms $\mathcal A$ that
output $\hat X$ and satisfy
\[
    \mathbb P_{\Lambda,F,\xi}
    \left(
       \inf_{X^\star\in D^\star} \|\hat X-X^\star\|\le r
    \right)
    \ge 1-\delta .
\]

Our goal is to provide a meaningful upperbound on $T_{\arg,\nu}(r,\delta)$. We provide an algorithm, T-BAL (Two-Stage Bayesian Argmax Localization; see \cref{alg:argmax}), that locates the argmax of a GP with a finite number of samples. The algorithm works in two phases: we first locate the nice region where the argmax lies, and then perform gradient ascent on that region. The analysis relies on the argmax being away from the boundary of $D$. Therefore, we introduce teh following regularity event.

\begin{definition}[Interior regularity event]\label{def:interior_regularity_gp}
For $\eta>0$, let ${\rm dist}(D^\star, \partial D)=\inf_{X\in D^\star} {\rm dist}(X, \partial D)$ and  define
\[
    \mathcal I_\eta =
    \left\{
        {\rm dist}(D^\star,\partial D)\ge \eta
    \right\},\qquad
    \beta_\eta=\mathbb P_{\Lambda,F}(\mathcal I_\eta^c).
\]
\end{definition}
One can show that on $\mathcal I_\eta$, the set of maximizer is a singleton almost surely (i.e., $D^\star=\{X^\star\}$), and similarly one can show that the Hessian is non-degenerate in $X^\star$. In fact, we have the following.

\begin{remark}\label{remark:gp}
An RBF-GP on $D$ is a.s. $C^\infty$, and its restriction to  the interior ${\rm int}(D)$ is a.s. a Morse function (every critical point has invertible Hessian, all critical values are distinct, and there are finitely many critical points). If the maximizer is not attained at the boundary, then it is unique almost surely.
\end{remark}

We work under $\mathcal I_\eta$: this not only allows the maximizer to be unique, but also allows us to be at-least at a distance $\eta$ from the boundary, where it is more degenerate. Then, we can show the existence of the following constants.
\begin{lemma}[Regularity under the RBF prior]
\label{lem:regular_basin_quantile}
Define $B(x,\rho)$ to be the ball centered around $x$ of radius $\rho$ with some norm $\|\cdot\|$.
Fix $\eta>0$ and suppose $\beta_\eta<1$. For every
$\alpha\in(\beta_\eta,1)$
there exist deterministic constants $
    \rho_\alpha,\mu_\alpha,L_\alpha,M_\alpha,\Gamma_\alpha>0$
and an event $\mathcal E_\alpha\subseteq \mathcal I_\eta$ such that
$
    \mathbb P_{\Lambda,F}(\mathcal E_\alpha)\ge 1-\alpha$, 
and on $\mathcal E_\alpha$ the following properties hold:

\[
    B(X^\star,\rho_\alpha)\subset D,
\]
\[
    \mu_\alpha I
    \preceq
    -\nabla^2F(x)
    \preceq
    L_\alpha I
    \qquad
    \forall x\in B(X^\star,\rho_\alpha),
\]
\[
    \sup_{x\in B(X^\star,\rho_\alpha)}
    \|\nabla^3F(x)\|_{\rm op}
    \le M_\alpha,
\]
and
\[
    F^\star
    -
    \sup_{x\notin B(X^\star,\rho_\alpha/2)}F(x)
    \ge
    \Gamma_\alpha .
\]
\end{lemma}
\begin{proof}
    Let $\gamma= \alpha-\beta_\eta$. 

    Under $\mathcal I_\eta$ the maximizer  $X^\star$ is in the interior, and unique almost surely. Define $\lambda^\star = \lambda_{\rm min}(-\nabla^2  F(X^\star)) $: under $\mathcal I_\eta$ we have that $\lambda^\star>0$ almost surely, hence, there exists $\mu_\alpha>0$ s.t. $\mathbb{P}(\mathcal I_\eta \bigcap \{ \lambda^\star/2 < \mu_\alpha\}) \leq \gamma/5$.

    Next, we use the fact that $F\in C^\infty$ and $D$ is compact to obtain an upper bound on the Hessian: 
    \[
    \sup_{x \in D} \|\nabla^2 F(x)\|_{\rm op} \leq L_0 < \infty.
    \]
    Therefore, there exists $L_\alpha<\infty$ such that $\mathbb{P}(L_\alpha< L_0) \leq \gamma/5$. With a similar reasoning, we also obtain $\sup_{x\in D} \|\nabla^3 F(x)\|_{\rm op}\leq M_0 $, and thus there exists $M_\alpha < \infty$ such that $\mathbb{P}(M_\alpha < M_0)\leq \gamma/5$.

    Next, by continuity and the margin condition of $\mathcal I_\eta$, the following quantity exists  and is strictly positive almost surely
    \[
    \rho_0(\omega) = \sup \left\{q<\eta: \sup_{X \in B(X^\star,q)}\|\nabla^2 F(X)-\nabla^2F(X^\star)\|_{\rm op} \leq \frac{\lambda^\star}{2} \right\},
    \]
    where $\omega$ denotes a realization of $F$. Therefore, there exists $\rho_\alpha>0$ such that $\mathbb{P}(\mathcal I_\eta \bigcap \{ \rho_0 < \rho_\alpha\})\leq \gamma/5$.

    Define then $\Gamma(\rho_\alpha) = F^\star-\sup_{X\notin B(X^\star,\rho_\alpha/2)} F(X)$: under $\mathcal  I_\eta$ this is strictly positive since  the supremum does not attain $F^\star$. Hence, there exists $\Gamma_\alpha>0$ such that $\mathbb{P}(\mathcal I_\eta \bigcap \{\Gamma(\rho_\alpha)<\Gamma_\alpha\})\leq \gamma/5$.

    Define
    \[
    {\cal E} = \mathcal I_\eta \cap \{\lambda^\star/2 \geq \mu_\alpha\} \cap  \{L_\alpha \geq L_0\} \cap\{M_\alpha \geq  M_0\} \cap \{\rho_0\geq \rho_\alpha\} \cap \{  \Gamma(\rho_\alpha) \geq \Gamma_\alpha  \}.
    \]
    Then, since
    \[
    {\cal E}^c = \mathcal I_\eta^c \cup \{\lambda^\star/2 < \mu_\alpha \} \cup \cdots =  \mathcal I_\eta^c \cup  \left( \mathcal{I}_\eta \cap\{\lambda^\star/2 < \mu_\alpha \}\right) \cup \cdots
    \]
    we have $\mathbb{P}({\cal E}^c) \leq \beta_\eta + 5 \frac{\gamma}{5}=\alpha$.
    Hence, under $\mathcal E$ we have that $B(X^\star,\rho_\alpha)\subset D$,  and all the properties follow  quite immediately. We only show the lower bound on the Hessian: for any $X\in B(X^\star, \rho_\alpha)$:
    \[
    -\frac{\lambda^\star}{2}I \preceq \nabla^2 F(x)-\nabla^2 F(X^\star) \preceq  \frac{\lambda^\star}{2}I.
    \]
    Hence
    \[
    -\nabla^2 F(x) \succeq -\nabla^2 F(X^\star)- \frac{\lambda^\star}{2}I \succeq \frac{\lambda^\star}{2}I \succeq \mu_\alpha I.
    \]
    
\end{proof}

\begin{algorithm}[t]
\caption{Two-stage Bayesian argmax localization (T-BAL)}
\label{alg:argmax}
\begin{algorithmic}[1]
\Require Target radius $r$, confidence $\delta$, noise level $\sigma^2$, 
constants $\rho, \mu, L, M, \Gamma$ from \cref{lem:regular_basin_quantile}
\Statex
\Comment{\textbf{Stage 1: Coarse grid search}}
\State Set $h \gets  \min\{\rho/8, \sqrt{3\Gamma/(2L)}\}$
\State Construct $h$-net $\mathcal{G}$ of $D$ with $ |\mathcal{G}| \leq C_d h^{-d}$
\State Set $n_0 \gets \lceil 128\sigma^2 \Gamma^{-2} \log(2|\mathcal{G}| /\delta_0) \rceil$
\For{each $g \in \mathcal{G}$}
    \State Query $F(g)$ exactly $n_0$ times; compute sample mean $\hat{F}(g)$
\EndFor
\State $x_0 \gets \arg\max_{g \in \mathcal{G}} \hat{F}(g)$
\Statex
\Comment{\textbf{Stage 2: Local finite-difference gradient ascent}}
\State Set $K \gets \lceil \frac{4L}{3\mu}\log\frac{\rho}{2r} \rceil$, \quad $q \gets 1 - \frac{3\mu}{4L}$, \quad $e_0 \gets \rho/2$
\For{$k = 0, \ldots, K-1$}
    \State $e_k \gets e_0 \, q^k$
    \State $s_k \gets \min\left\{\rho/8,\; \sqrt{3\mu e_k / (4\sqrt{d}\, M)}\right\}$
    \State $n_k \gets \lceil 64 d\cdot\sigma^2 \mu^{-2} e_k^{-2} s_k^{-2} \log(2dK/\delta_1) \rceil$
    \For{$j = 1, \ldots, d$}
        \State Query $F(x_k + s_k e_j)$ and $F(x_k - s_k e_j)$ each $n_k$ times
        \State $\hat{g}_{k,j} \gets \frac{\bar{Y}(x_k + s_k e_j) - \bar{Y}(x_k - s_k e_j)}{2s_k}$
    \EndFor
    \State $x_{k+1} \gets x_k + \frac{1}{L}\hat{g}_k$
\EndFor
\State \Return $\hat{X} = x_K$
\end{algorithmic}
\end{algorithm}

\paragraph{Argmax upper bound.} Under the event $\mathcal I_\eta$, we want to construct an algorithm that localizes $X^\star$ to within radius $r$ with probability $1-\delta$. The idea is to show an upper bound on the minimal lower bound of the type
\[
T_{\arg,\nu}(r,\delta)\coloneqq\inf_{{\cal A}}\mathbb{E}[\tau] \leq B_\alpha + O(\sigma^2 r^{-\gamma} \log(1/r)),
\]
for some $\gamma>0$.

Under $\mathcal I_\eta$, we have the guarantees from \cref{lem:regular_basin_quantile}, and thus
\[
F(X)=F(X^\star)+ \frac{1}{2}(X-X^\star)^\top \nabla^2 F(X^\star) (X-X^\star)+o(\|X-X^\star\|^2).
\]
from which we find  $F(X)-F(X^\star)\approx \frac{\mu_\alpha}{2 } \|X-X^\star\|_2^2$. Therefore, localizing to radius $r$ gives an error of roughly $ \frac{\mu_\alpha r^2}{2}$ on the value. Therefore, for $\epsilon$ accuracy on the value, we can take the radius to be $r \sim \sqrt{ \epsilon}$

Therefore, under $\mathcal I_\eta$ if $\beta_\eta <1$, and 
\[
\liminf_{\epsilon\to 0}\frac{T_{{\rm val},\nu}(\epsilon,\delta)}{T_{\arg,\nu}(\sqrt{\epsilon},\delta)} \to \infty 
\]
one can argue that the problem of estimating the max-value is intrinsically harder than the problem of estimating the argmax as the accuracy radius decreases. In particular, considering  the lower bound on the max-value estimation problem \cref{thm:max_value_lowerbound}, and the proposed upper bound on $T_{\arg,\nu}(r,\delta)$, we obtain that
\[
\frac{T_{{\rm val},\nu}(\epsilon,\delta)}{T_{\arg,\nu}(\sqrt{\epsilon},\delta)}\geq \frac{\Omega(\epsilon^{-2})}{O( \epsilon^{-\gamma/2} \log\log(1/\sqrt{\epsilon}))},
\]
which diverges for $\gamma\in (0,4)$ as $\epsilon\to 0$.

So value estimation is asymptotically harder than argmax localization. Even at the radius where localizing the argmax would "in principle" tell you the value to accuracy $\epsilon$, the localization itself is cheaper than directly estimating the value.
The  reason is simple: the algorithm exploits the  geometry (gradient information,  smoothness) to localize at a fast rate, but it does not help with the problem of estimating the height of a function.

\paragraph{Analysis of T-BAL (\cref{alg:argmax}).} We provide now an algorithm for argmax localization in Gaussian processes. The algorithm, Two-Stage Bayesian Argmax Localization (T-BAL), outlined in \cref{alg:argmax}, works in two phases. In the first phase we try to find a point inside $B(X^\star,\rho/2)$, where $\rho$ is described in \cref{lem:regular_basin_quantile}.  In the second phase, assuming we are inside the above ball, we perform gradient ascent to find the maximum using finite differences to approximate the gradients.

The second phase analysed gradient descent when $x_0\in B(X^\star,\rho)$.

\begin{theorem}[Sample Complexity of T-BAL] \label{thm:tbal_sample_complexity} 
Consider the event ${\cal E}_\alpha $ in \cref{lem:regular_basin_quantile}. Set $r>0,\delta\in (0,1), \alpha\in (\beta_\eta,\delta)$ and $\delta_0=\delta_1=(\delta-\alpha)/2$. Then, T-BAL (\cref{alg:argmax})  satisfies $\mathbb{P}_{F,\Lambda}(\|\hat X-X^\star\|\leq r)\geq 1-\delta$, using at-most
    \[
    B_\alpha+O\left(\frac{ \sigma^2}{\mu^2 } \max(2,L/\mu)  \left[ \frac{ d^2}{\rho^2 r^2} + \frac{ d^{5/2} M }{ \mu r^3}\right] \log \frac{ Ld \log(\rho/r)}{\mu\delta_1}.\right)
    \]
    samples, where $B_\alpha$ is an appropriate finite constant for each $\alpha$ that does not depend on $r$.
\end{theorem}
\begin{proof}

In this proof we consider the  analysis of the second stage, while the first stage and the constant $B_\alpha$ are provided in \cref{prop:phase_1_sample_complexity}.
    The idea is to prove a bound on gradient ascent. We first show that one step of the ideal gradient ascent brings us closer to $X^\star$. We then find the value of $s_k$ and $n_k$ to compute the approximate gradient up to the desired accuracy. After that, we estimate how many iterations we need, and compute the total number of required samples.

    \noindent \underline{\it One step gradient ascent.} Let $x^+=x+\frac{1}{L}\nabla F(x)$. 

    Define $\phi(t)=  \nabla F(X^\star  +t(x-X^\star))$.
    Since $\nabla F(X^\star)=0$, using the fundamental theorem of calculus we have
    \[
    \nabla F(x)=\phi(1)-\phi(0)=\int_0^1 \phi'(t){\rm d}t= (x-X^\star)\underbrace{\int_0^1 \nabla^2 F(x^\star + t(x-X^\star)){\rm d}t}_{\eqqcolon -A(x)}.
    \]
    Then
    \[
    x^+ -X^\star= x - \frac{1}{L} (x-X^\star)A(x) -X^\star = (I- \frac{A(x)}{L})(x-X^\star).
    \]
    The matrix $(I- \frac{A(x)}{L})$ has eigenvalues in $[0,1-\mu/L] $ for $x\in B(X^\star,\rho)$. Therefore,
    \[
    \|x^+-X^\star \|\leq \left(1-\frac{\mu}{L}\right)\|x-X^\star\|
    \]
    so the contraction factor is $q_0=1-\mu/L$.

    \noindent \underline{\it Noisy gradients.} Suppose we do not have access to the exact gradient, but only to a noisy gradient $\hat g_k$ in round $k$. Assume the noisy gradient satisfy
    \[
    \| \hat g_k - \nabla F(x_k)\| \leq \frac{\mu e_k}{4},
    \]
    where $e_k=e_0 q^k$. Then
    \begin{align*}
    \| x_{k+1}-X^\star\|&= \|x_k + \frac{1}{L}\hat g_k - X^\star\|,\\
    &=\|x_k + \frac{1}{L}\nabla F(x_k)+\frac{1}{L}\hat g_k -\frac{1}{L}\nabla F(x_k)- X^\star\|,\\
    &\leq \|x_k + \frac{1}{L}\nabla F(x_k)-X^\star\|+\frac{1}{L}\|\hat g_k -\nabla F(x_k)\|,\\
    &\leq \left(1-\frac{\mu}{L}\right)\|x_k-X^\star\| + \frac{\mu e_k}{4L},\\
    &\leq  \left(1-\frac{3\mu}{4L}\right)e_k.
    \end{align*}
    So noisy gradients still guarantee convergence as long as we can show $ \| \hat g_k - \nabla F(x_k)\| \leq \frac{\mu e_k}{4}$.

    \noindent \underline{\it Finite-difference gradient bound: bias term.} We now bound $\|\hat g_k -\nabla F(x_k)\|$ through a bias-variance decomposition:
    \[
    \|\hat g_k -\nabla F(x_k)\| \leq \|\hat g_k - \mathbb{E}[\hat g_k]\|+\|\mathbb{E}[\hat g_k] -\nabla F(x_k)\|.
    \]

    We begin with the bias term  $\|\mathbb{E}[\hat g_k] -\nabla F(x_k)\|$. Note that the noisy gradients are computed as follow for each direction $e_j$:
    \[ \hat{g}_{k,j} \gets \frac{\bar{Y}(x_k + s_k e_j) - \bar{Y}(x_k - s_k e_j)}{2s_k},\]
    where $\bar Y$ is an average over $n_k$ samples of the values observed. The idea is to bound $|\mathbb{E}[\hat g_{k,j}] -\partial_j F(x_k)|$ using Taylor series and the fact that $\nabla^3 F$ is bounded  in $B(X^\star,\rho)$ by \cref{lem:regular_basin_quantile}.

    Then, fix a direction $e_j$ and let $\phi(t)=F(x_k+te_j)$. Then
    \begin{align*}
    \phi(s_k)&=\phi(0)+\phi'(0)s_k + \frac{1}{2}\phi''(0)s_k^2 + \frac{s_k^3}{6}\phi'''(\xi_+),\\
     \phi(-s_k)&=\phi(0)-\phi'(0)s_k + \frac{1}{2}\phi''(0)s_k^2 - \frac{s_k^3}{6}\phi'''(\xi_-),
    \end{align*}
    where $\xi_+ \in (0,s_k),\xi_- \in (-s_k,0)$. Then
    \[
    \frac{\phi(s_k)-\phi(-s_k)}{2s_k} = \phi'(0) + \frac{s_k^2}{12} [\phi'''(\xi_+)+\phi'''(\xi_-)].
    \]
    Hence, for $s_k$ sufficiently small, $s_k< \rho/2$  (we choose $s_k<\rho/8$) we have $|\phi'''|\leq M$, and therefore
    \[
    \Big|\frac{\phi(s_k)-\phi(-s_k)}{2s_k} - \phi'(0) \Big|
 \leq \frac{s_k^2 M}{6},   \]
 leading to 
 \[\|\mathbb{E}[\hat g_k]-\nabla F(x_k)\|\leq \frac{\sqrt{d}s_k^2 M}{6}.\]

 Setting $\frac{\sqrt{d}s_k^2 M}{6} \leq \mu e_k/8$ yields
 \[
 s_k  \leq  \sqrt{\frac{3 \mu e_k}{4\sqrt{d} M}}
 \]
 and we set $s_k=\min\left\{\rho/8, \sqrt{\frac{3 \mu e_k}{4\sqrt{d} M}}\right\}$. Furthermore, note that for this choice of $s_k$ we can guarantee that we stay inside the nice region $B(X^\star,\rho)$: since $\|x_k-X^\star\|\leq e_k \leq \rho/2$ and $s_k\leq \rho/8$, then
 \[
 \|x_k\pm s_k e_j - X^\star\| \leq  \frac{5\rho}{8} < \rho.
 \]

 \noindent \underline{\it Finite-difference gradient bound: variance term.} We now bound the variance term $\|\hat g_k -\mathbb{E}[\hat g_k]\|$. We have
 \[
 \hat g_{k,j} -\mathbb{E}[\hat g_{k,j}] = \frac{\bar \xi(x_k +s_ke_j)-\bar \xi (x_k -s_k e_j)}{2s_k},\qquad \bar\xi(x)=\frac{1}{n_k}\sum_{i=1}^{n_k}\xi_i.
 \]
 Therefore, ${\rm Var}(\hat g_{k,j})= \frac{2\sigma^2}{4s_k^2 n_k}.$ Then
 \[
 \mathbb{P}(|\hat g_{k,j} -\mathbb{E}[\hat g_{k,j}]| >  \mu e_k/(8\sqrt{d})) \leq 2\exp\left(-\frac{\mu^2 e_k^2 s_k^2 n_k}{64 d\cdot \sigma^2}\right).
 \]
 Set the right hand-side smaller than $\delta_1/(d K)$, where $K$ is the total number of iterations $k=0,\dots, K-1$, to obtain
 \[
n_k \geq \frac{64 d\cdot\sigma^2}{\mu^2e_k^2s_k^2}\log \frac{2d K}{\delta_1}.
 \]
 Then, a union bound over $j=1,\dots, d$ and $k=0,\dots, K-1$, yields
 \[
 \mathbb{P}(\exists j\in \{1,\dots, d\},k\in \{0,\dots,K-1\}: |\hat g_{k,j} -\mathbb{E}[\hat g_{k,j}]| >  \mu e_k/(8\sqrt{d})) \leq \delta_1.
 \]
 Hence, with probability $1-\delta_1$ we have that
 \[
  \|\hat g_k -\nabla F(x_k)\| \leq  \frac{\mu e_k}{8}+ \frac{\mu e_k}{8}= \frac{\mu e_k}{4},
 \]
 which is what we wanted to show.

 \noindent \underline{\it Sample complexity.}  Since we need $e_K \leq r$, and $e_K=e_0 q^K$, we have
 \[
 K=\left\lceil \frac{\log( e_0/r)}{\log(1/q)}\right\rceil.
 \]
 Since $\log(1/q)-\log q=-\log (1-3\mu/4L)\geq 3\mu/4L$, we have
 \[
 K\leq 
 4L\frac{\log(e_0/r)}{3\mu}
 \leq 
 4L\frac{\log(\rho/(2r))}{3\mu},
 \]
 where we used $e_0\leq \rho/2$.
 Then, summing the number of samples from $k=0, \dots, K-1$, we obtain
 \begin{align*}
 \sum_{k=0}^{K-1} 2d n_k= \sum_{k=0}^{K-1}   \frac{128 d^2\cdot\sigma^2}{\mu^2e_k^2s_k^2}\log \frac{2d K}{\delta_1}.
 \end{align*}
 Use that $\frac{1}{s_k^2} \leq \frac{64}{\rho^2} +  \frac{4\sqrt{d}M}{3\mu e_k}$ and that $e_k=e_0q^k$ with $q\in (0,1)$:
\begin{align*}
 \sum_{k=0}^{K-1} 2d n_k&\leq  \sum_{k=0}^{K-1}   \left[\frac{128\cdot 64\cdot  d^2\cdot\sigma^2}{\mu^2e_k^2 \rho^2}+\frac{512 d^{5/2}M\cdot\sigma^2}{3\mu^3e_k^3} \right]\log \frac{2d K}{\delta_1},\\
   &\underset{\sim}{< }\log \frac{2d K}{\delta_1}\sum_{k=0}^{K-1}   \left[\frac{1}{e_k^2}+\frac{1}{e_k^3} \right].
 \end{align*}
 Regarding the series, 
 \[
 \sum_{k=0}^{K-1}   \frac{1}{e_k^p}=\frac{1}{e_0^p}\sum_{k=0}^{K-1}   \frac{1}{q^{pk}}\leq e_0^{-p}\frac{1-q^{-pK}}{1-q^{-p}} \leq \frac{ e_0^{-p}q^{-pK}}{q^{-p}-1} = \frac{e_K^{-p}}{q^{-p}-1}.
 \]
 Now, note  that since $K$ is the smallest integer achieving $e_K=e_0q^K \leq r$, we have $e_0 q^{K-1} \geq r$, and thus $e_K =qe_{K-1}\geq qr \Rightarrow (qr)^{-p} \geq e_K^{-p}$ . We obtain
 \[
 \sum_{k=0}^{K-1}   \frac{1}{e_k^p}\leq  \frac{(qr)^{-p}}{q^{-p}-1} = \frac{r^{-p}}{1-q^p}.
 \]
 To lower bound the denominator, recall that $q$ is of the form $(1-x)$. Using that $(1-x)^p \leq e^{-xp}$ for $x\in (0,1)$, we have $1-q^p \geq 1-e^{-p 3\mu/(4L)}$.

 Since $1-e^{-t} \geq t/2$ for $t\in (0,1)$, for $p$ sufficiently small we obtain $1-q^p \geq  \frac{3 p \mu}{8L}$. If $p$ is large, such that $t\geq 1$, then $1-e^{-t}\geq 1-e^{-1}\geq 1/2$. Therefore, $\frac{1}{1-q^p} \leq \max(2, \frac{8L}{3p\mu})$, and 
 \[
  \sum_{k=0}^{K-1}   \frac{1}{e_k^p}\leq   r^{-p}\max\left(2, \frac{8L}{3p\mu}\right).
 \]
 Since $p\in \{2,3\}$ we also have $  \sum_{k=0}^{K-1}   \frac{1}{e_k^p}\leq   r^{-p}\max\left(2, \frac{8L}{6\mu}\right)$.
 
 Then, we conclude that the second phase sample complexity is upper bounded by
 \begin{align*}
  \sum_{k=0}^{K-1} 2d n_k &\leq \frac{ 128 \sigma^2}{\mu^2 }\cdot \max\left(2, \frac{8L}{6\mu}\right)\cdot  \left[ \frac{64 d^2}{\rho^2 r^2} + \frac{4 d^{5/2} M }{3 \mu r^3}\right] \log \frac{2dK}{\delta_1},\\
  &\leq\frac{ 128 \sigma^2}{\mu^2 }\cdot \max\left(2, \frac{8L}{6\mu}\right)\cdot  \left[ \frac{64 d^2}{\rho^2 r^2} + \frac{4 d^{5/2} M }{3 \mu r^3}\right] \log \frac{8 Ld \log(\rho/(2r))}{3\mu\delta_1}.
 \end{align*}

  \noindent \underline{\it Connecting everything together.} In phase $1$ we have probablity $\delta_0$ of failure, while in phase 2 we have probability $\delta_1$ of failure. Since we work under the event ${\cal E}_\alpha$ with failure probability $\alpha$, we have $\mathbb{P}(\text{failure}) \leq \delta_0+\delta_1+\alpha  = \delta$.
\end{proof}

The first phase analysis is provided in the following proposition.
We construct an $h$-net ${\cal G}$ of $D$ that depends on the geometric  of the problem (see \cref{lem:regular_basin_quantile}). For each point in ${\cal G}$, we sample $F(g)$ exactly $n_0$ times, such that we have good concentration, and we return the point that achieves the maximum.
\begin{proposition}\label{prop:phase_1_sample_complexity}
Consider \cref{alg:argmax}, and let $\delta_0\in (0,1)$. Under ${\cal E}_\alpha$ (see \cref{lem:regular_basin_quantile}),  the first phase samples $B_\alpha\leq 2^dh^{-d} \lceil 128\sigma^2 \Gamma^{-2} \log(2^{d+1}h^{-d}/\delta_0)\rceil$ queries. Furthermore, we have that $x_0 \in B(X^\star,\rho/2)$ 
with probability $1-\delta_0$.
\end{proposition}
\begin{proof}
Consider the event ${\cal E}_\alpha$ in \cref{lem:regular_basin_quantile}, and omit the subscript $\alpha$ for simplicity.
We construct an $h$-net ${\cal G}$ of $D$ such that for every $x\in D$ there exists $g\in {\cal D}$ satisfying $\|x-g\|\leq h$. Hence, there exists a point $g^\star\in {\cal G}$ satisfying $\|X^\star-g^\star\|\leq h$.

At each grid point $g$ take $n_0$ samples and compute the sample mean $\hat F(g)$: by the Gaussian tail bound, we have
\[
\mathbb{P}(|\hat F(g)-F(g)| > \Gamma/8) \leq 2\exp\left( -\frac{n_0\Gamma^2}{128\sigma^2}\right).
\]
Choosing $n_0= \frac{128\sigma^2}{\Gamma^2} \log\left(\frac{2|{\cal G}|}{\delta_0}\right)$, for some $\delta_0\in (0,1)$, and taking a union bound over $g$, we obtain
\[
\mathbb{P}(\exists g\in {\cal G}:|\hat F(g)-F(g)| > \Gamma/8) \leq \delta_0.
\]

We now choose $h$ small enough such that a lower bound on $\hat F(g^\star)$ upper bounds a valid upper bound on $\hat F(g)$ for $g\notin B(X^\star,\rho/2)$.
 Under the event ${\cal E}=\{\forall g\in {\cal G}: |\hat F(g)- F(g)| \leq \Gamma/8\}$, we have that for $g\notin B(X^\star, \rho/2)$
\[
\hat F(g^\star)\geq F(g^\star)-\Gamma/8,\qquad \hbox{and} \qquad \hat F(g)\leq F(g)+\Gamma/8 \leq F^\star -\Gamma +\Gamma/8 = F^\star-\frac{7}{8}\Gamma,
\]
where we used the fact from \cref{lem:regular_basin_quantile} that $F^\star - F(g)\geq \Gamma$ for $g\notin B(X^\star,\rho/2)$.
To construct a lower bound on $F(g^\star)$, we use the gradient properties of $F$ in  $B(X^\star,\rho)$. To that aim, we need to ensure $g^\star$ is sufficiently inside the ball, that is, choose $h$ small enough.

We require $h\leq \rho/2$, and for simplicity, we just set $h\leq \rho/8$. Then, $g^\star \in B(X^\star,\rho/2)$, nd since  in the ball the function is  smooth, with $\nabla F(X^\star)=0$, we have
\[
F(g^\star)\geq F(X^\star)- \frac{L}{2} \|X^\star - g^\star\|^2\geq F^\star -\frac{L}{2}h^2 .
\]
Hence, we require
\[
F^\star -\frac{L}{2}h^2 -\frac{\Gamma}{8} \geq  F^\star-\frac{7}{8}\Gamma,
\]
which is satisfied if $h^2 \leq \frac{3}{2L}\Gamma$. Hence, any point outside $B(X^\star,\rho/2)$ cannot upper bound $\hat F(g^\star)$. 

Therefore, $x_0=\argmax_{g\in {\cal G}} \hat F(g)$ satisfies $x_0\in B(X^\star,\rho/2)$ and
\[
\hat F(x_0)\geq \hat F(g^\star)\geq  F^\star- \frac{L}{2}\left(\min\{\rho/8, \sqrt{3\Gamma/(2L)}\}\right)^2-\frac{\Gamma}{8},
\]

Lastly, the number of points sampled depends on the number of points in ${\cal G}$ is bounded by $(1/h+1)^d$. Since $h\leq 1$, then $|{\cal G}|\leq 2^d h^{-d}$.
\end{proof}

%% file: sections/appendix/algorithms.tex
\section{Appendix: Algorithms}\label{sec:appendix:algorithms}
This appendix describes the implementation of \cicpe{} used in the experiments.  We keep the notation of the main text: $H_t$ is the current history, $I_\phi(\cdot|H_t)$ is the inference distribution over the target $x_\theta^\star$, $Q_\psi(H_t,a)$ is the critic for a continuation action $a\in\A$, and $Q_\psi(H_t,a_{\rm stop})$ is the value of stopping.  The implementation follows the Lagrangian view in \cref{app:theoretical_results:icpe:fixed_confidence:dual}: the inference model learns a stochastic selector, the critic learns the stop/continue Bellman comparison, and the actor rule proposes the next continuation action.  The main practical point is that all three objects are trained from replay.  For this reason we use target networks, conservative critic targets, and simple regularizers that keep early noisy estimates from determining the stopping boundary.

\subsection{History Encoder and Time Pooling Layer}
The inference network, the critic, and the learned TD3 actors use the same sequential template, although their parameters are separate.  Each interaction step is embedded as a token by concatenating the query and the next observation, $u_s=[A_s;Y_{s+1}]$, and passing it through a small embedding network.  The resulting sequence $(e_1,\ldots,e_t)$ is processed by a causal sequential backbone.  In the code this backbone can be an LSTM, an attention stack, or a recurrent/linear-attention variant.  Padding is masked throughout, so replay batches can contain histories of different lengths without leaking future observations.

The readout is not simply the last hidden state.  After obtaining hidden states $h_1,\ldots,h_t\in\mathbb R^d$, we use a query-conditioned pooling layer.  Given a query $q$, it computes
\[
\alpha_i(q,H_t)
=
\frac{\exp(\langle W_q q,W_k h_i\rangle/\sqrt d)}
{\sum_{j=1}^t \exp(\langle W_q q,W_k h_j\rangle/\sqrt d)},
\qquad
v(q,H_t)=\sum_{i=1}^t \alpha_i(q,H_t)W_v h_i,
\]
and then applies a feature-wise gate,
$
 z(q,H_t)=v(q,H_t)\odot \mathrm{silu}(W_m q)
$.
This is useful because the relevant part of the history depends on what the model is asked to do.  For inference and stopping, the query is a time embedding, since we need a representation of the current prefix.  For the continuation critic, the query is the candidate action $a$, so the same history can be read differently when evaluating different future measurements.  The critic therefore has two readouts from the same encoded history: a time-conditioned readout for $Q_\psi(H_t,a_{\rm stop})$ and an action-conditioned readout for $Q_\psi(H_t,a)$.

\begin{figure}[b]
    \centering
    \includegraphics[width=\linewidth]{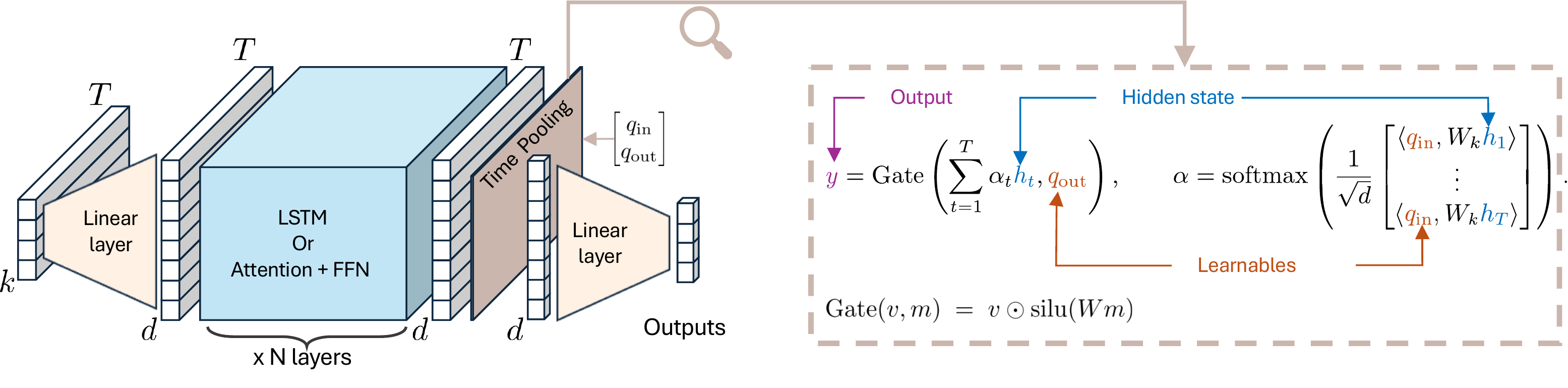}
    \caption{Sequential encoder and query-pooling readout used by the inference network $I_\phi$, the critic $Q_\psi$, and the learned TD3 actors.}
    \label{fig:architecture_network}
\end{figure}

\subsection{Replay Buffer and Prefix Sampling}
Training is off-policy.  A rollout samples a task $\theta\sim\nu$, interacts with the corresponding environment until the critic stops or the maximum horizon is reached, and stores the whole trajectory together with $x_\theta^\star$ and a flag indicating whether the trajectory ended by executing the stop action.  Minibatches are built by first sampling trajectories and then sampling prefixes inside them.  The prefix sampler uses a mixture of uniformly sampled prefixes, prefixes around the current average stopping time, and terminal prefixes.  This gives the inference network examples from all time scales, but also gives the critic many prefixes close to the point where the decision changes from continue to stop.

There is one convention that matters.  If a trajectory terminates because the agent stops at time $t$, the training prefix for that terminal state is $H_t$, because no new observation is collected after the stop action.  If the trajectory terminates only because the maximum horizon is reached, the last collected observation is included.  This makes replay consistent with deployment: stopping is a decision made before paying for the next sample.

\subsection{Inference Update and Its Regularization}
For a prefix $h$ and target $x_\theta^\star$, the inference model outputs a diagonal Gaussian
\[
I_\phi(\cdot|h)=\mathcal N(\mu_\phi(h),\operatorname{diag}(\sigma_\phi^2(h))).
\]
The ideal objective in the main text is the negative log-likelihood in \cref{eq:loss_inference}.  In the implementation we use a robust version.  Define the per-coordinate averaged NLL
\[
\ell_\phi(h,x_\theta^\star)
= -\frac{1}{d_{\X}}\sum_{j=1}^{d_{\X}}
\log \mathcal N\big((x_\theta^\star)_j;\mu_{\phi,j}(h),\sigma_{\phi,j}^2(h)\big),
\]
and
\[
\rho_\tau(u)=\min\{u,0\}+\tau\log\left(1+\frac{[u]_+}{\tau}\right),
\qquad [u]_+=\max\{u,0\}.
\]

For Euclidean recommendation problems we train with
\begin{equation}\label{eq:appendix_robust_nll_loss}
\mathcal L_{\rm inf}(\phi)
=
\mathbb E_{(h,x_\theta^\star)\sim{\cal B}}
\left[
\rho_\tau\big(\ell_\phi(h,x_\theta^\star)\big)
+\alpha_{\rm anc}\operatorname{SmoothL1}\big(\mu_\phi(h),x_\theta^\star\big)
\right].
\end{equation}
The robust transform is a simple way of saying that early bad prefixes should not dominate the variance head.  A standard Gaussian NLL can become very large when the model underestimates uncertainty for one minibatch, and this can push the log-variance to extremes.  The logarithmic tail keeps the ranking of ordinary examples, but reduces the influence of rare outliers.  The SmoothL1 anchor is also important: the final decision is the mean $\mu_\phi(H_\tau)$, so we want the mean to remain a good deterministic recommendation even when the Gaussian still has large uncertainty.  We additionally clamp the predicted log-standard deviations to fixed lower and upper bounds  for numerical stability.

For the $\epsilon$-best-arm problem, the target is directional.  The magnitude of a vector is not meaningful once arms are represented on the sphere, and correctness is measured by cosine distance.  Therefore recommendations and samples are normalized before they are evaluated.  In this case we keep the robust NLL and replace the Euclidean anchor by a spherical alignment term.  With $Z\sim I_\phi(\cdot|h)$,
\begin{equation}\label{eq:appendix_eps_best_arm_loss}
\mathcal L_{\rm sph}(\phi)
= -\mathbb E\left[
\left\langle \frac{Z}{\|Z\|_2},\frac{x_\theta^\star}{\|x_\theta^\star\|_2}\right\rangle\right].
\end{equation}
The reason is that an Euclidean anchor would penalize harmless radial errors, while the bandit loss only cares about the direction.   

A target copy $I_{\bar\phi}$ is maintained by Polyak averaging.  The critic never uses the online inference model inside its TD targets; it uses $I_{\bar\phi}$.  This separation is important because the reward itself is learned through $I_\phi$, and bootstrapping from a rapidly moving reward makes the stop/continue comparison unstable.

\subsection{Reward, Critic Update, and Critic Regularization}
The critic learns two quantities from the same replay prefixes.  The stopping head learns the value of recommending now, and the continuation head learns the value of paying for one more observation and then acting optimally.  Since the ideal reward $r_t(h)=\max_x q_t(h,x)$ is not available, we use the target inference model to estimate how likely the implemented stochastic selector is to be already $\epsilon$-correct:
\begin{equation}\label{eq:appendix_sampled_reward}
\hat r_{\bar\phi,m}(h,\theta)
=
\frac{1}{m}\sum_{k=1}^m
{\bf 1}\left\{L_\theta\left(X^{(k)}\right)\leq \epsilon\right\},
\qquad
X^{(k)}\sim I_{\bar\phi}(\cdot|h).
\end{equation}
   The reason for sampling is that the stopping decision should depend on posterior concentration, not only on the posterior mean.  If $\mu_\phi(h)$ is close to $x_\theta^\star$ but $\sigma_\phi(h)$ is still large, stopping is risky.  The sampled reward makes this visible to the critic, and this is exactly the gap controlled by \cref{prop:second_moment_robustness_and_gap}.

Let $d$ be the terminal flag for a replay transition, and let $a^+(h')$ be the target continuation action at the next prefix.  This target action depends on the actor rule.  For TS and TTPS it is the target inference mean plus small smoothing noise, projected to the feasible action set.  For TD3 it is the target actor action, again with small target-policy smoothing noise.  The smoothing noise is much smaller than the posterior sampling noise used during rollouts.

When using twin critics, we scalarize target values by the conservative minimum
\[
\bar Q_{\bar\psi}(h,a)=\min\{Q_{\bar\psi,1}(h,a),Q_{\bar\psi,2}(h,a)\},
\]
and define
\[
V_{\bar\psi}(h')=
\max\left\{\bar Q_{\bar\psi}(h',a_{\rm stop}),\bar Q_{\bar\psi}(h',a^+(h'))\right\}.
\]
The TD targets are
\begin{align}
 y_{\rm stop}(h,\theta)
 &= \hat r_{\bar\phi,m}(h,\theta),\label{eq:appendix_stop_target}\\
 y_{\rm cont}(h,a,h',d,\theta)
 &= -c(1-d)+d\,\hat r_{\bar\phi,m}(h',\theta)
 +\gamma(1-d)V_{\bar\psi}(h').\label{eq:appendix_cont_target}
\end{align}
Thus the stopping head is directly supervised by the current confidence, while the continuation head is supervised by the gain from collecting the next observation.  The action-head loss is applied only to prefixes that did not already stop, because a stopped transition has no genuine continuation action.  The stopping head is trained on every prefix, since stopping is a valid action at every prefix.  With $M=1$ for non-stopped replay transitions and $M=0$ for stopped transitions, the critic loss is
\begin{equation}\label{eq:appendix_critic_loss}
\begin{aligned}
\mathcal L_Q(\psi)
=&\frac{1}{2}\mathbb E_{\cal B}\left[
M\sum_{j=1}^2\left(Q_{\psi,j}(h,a)-y_{\rm cont}\right)^2
\right]+\frac{w_{\rm stop}}{2}\mathbb E_{\cal B}\left[
\sum_{j=1}^2\left(Q_{\psi,j}(h,a_{\rm stop})-y_{\rm stop}\right)^2
\right],
\end{aligned}
\end{equation}
with the obvious single-critic version when twin critics are disabled.  We also optionally clip the bootstrap values in a minibatch to moderate quantiles before taking the maximum in $V_{\bar\psi}$. It prevents a few very optimistic target values from propagating through replay and moving the stopping boundary too early.


At rollout time the default stopping test is
\[
Q_\psi(H_t,a_{\rm stop})\geq Q_\psi(H_t,A_t),
\]
where both quantities are scalarized by the conservative twin rule.  This is exactly the learned version of the Bellman comparison in \cref{thm:main_bellman_greedy}.  

\subsection{Actor Rules: TS, TTPS, and TD3}
The actor rule only chooses continuation actions.  The stop action is handled by the critic comparison above.  We use three actor rules depending on the relation between the query space $\A$ and the recommendation space $\X$.

\noindent{\bf Thompson sampling.}
When $\A=\X$, we can use the inference distribution itself as the actor.  The TS rule samples
\begin{equation}\label{eq:appendix_ts_rule}
A_t=\mu_\phi(H_t)+\sigma_\phi(H_t)\odot \xi_t,
\qquad \xi_t\sim\mathcal N(0,I).
\end{equation}
At evaluation in greedy mode we set $\xi_t=0$, so the action is the projected posterior mean.  This choice is deliberately simple.  Early in a task, the learned posterior is broad and TS explores.  Later, when the inference distribution contracts, the same rule automatically becomes exploitative.  No separate actor is trained, which removes one source of approximation error and is appropriate when informative queries are themselves plausible recommendations.  In the critic target, the next action for TS is the target inference mean, with only the small target-smoothing noise described above.

\noindent{\bf Top-two posterior sampling.}
TTPS is also inference-based and therefore also assumes $\A=\X$.  It first draws a Thompson sample $Z_1$.  With probability $1/2$ this sample is used directly.  With the remaining probability, the rule draws a challenger $Z_2$ and keeps the sample that is farther from the posterior mean, provided the current sample is too close to the mean.  In Euclidean tasks the distance is $\|z-\mu_\phi(H_t)\|_2$; in $\epsilon$-best-arm it is $1-\langle z,\mu_\phi(H_t)\rangle$.  The intuition is  that the posterior mean is the current recommendation, while farther posterior samples represent plausible alternatives.  TTPS spends part of the sampling budget checking those alternatives before the critic decides that stopping is safe.  As for TS, the target action used by the critic is the target inference mean plus small smoothing noise.

\noindent{\bf TD3 actor.}
TD3 is the actor rule used when we want to learn the continuation action from the critic, and especially when $\A$ and $\X$ are different objects.  A target actor $\pi_{\bar\vartheta}$ is used in the critic target, while the online actor $\pi_\vartheta$ is updated on a delayed schedule.  The target action is
\[
a^+(h')=\tanh(\pi_{\bar\vartheta}(h'))+\zeta,
\qquad \zeta\sim\mathcal N(0,\sigma_{\rm tgt}^2 I),
\]
with clipping or normalization depending on the domain.  The target-policy smoothing noise $\zeta$ prevents the critic from learning sharp artificial peaks in action space, and the delayed actor update prevents the actor from chasing a critic that is still changing after every minibatch.

For a deterministic actor, the main update is
\begin{equation}\label{eq:appendix_td3_actor_loss}
\mathcal L_{\rm act}^{\rm det}(\vartheta)
= -\mathbb E_{h\sim{\cal B}}\left[Q_{\psi,1}\left(h,\tanh(\pi_\vartheta(h))\right)\right].
\end{equation}
The actor uses the first critic for the policy gradient, while the target uses the first critic.  

When $\A=\X$ and the actor is Gaussian, we regularize the TD3 actor toward the inference distribution,
\begin{equation}\label{eq:appendix_td3_kl_actor_loss}
\mathcal L_{\rm act}^{\rm KL}(\vartheta)
=
-\mathbb E_{h\sim{\cal B}}\left[Q_{\psi,1}(h,A_\vartheta(h))\right]
+
\beta_{\rm KL}\mathbb E_{h\sim{\cal B}}\left[
{\rm KL}\left(I_\phi(\cdot|h)\|\pi_\vartheta(\cdot|h)\right)
\right],
\qquad A_\vartheta(h)\sim\pi_\vartheta(\cdot|h).
\end{equation}
\input{sections/appendix/full_algo}

During training we also perform random exploration to encourage parametric exploration (with small probability we sample a random action), and void collapsing of the actor.
\subsection{Cost Update for Fixed Confidence}
The scalar cost $c$ is the implemented Lagrange tradeoff between confidence and sample complexity.  We update it from the empirical stopped success rate.  For a batch of completed rollouts, let
\[
D_i=
\begin{cases}
\|\mu_\phi(H_{\tau_i}^{(i)})-x_{\theta_i}^\star\|_2, & \text{Euclidean tasks},\\
1-\left\langle \mu_\phi(H_{\tau_i}^{(i)}),x_{\theta_i}^\star\right\rangle, & \epsilon\text{-best-arm tasks},
\end{cases}
\]
and use the smooth accuracy proxy
\[
\hat p=\frac{1}{B}\sum_{i=1}^B
\sigma\left(\frac{\epsilon-D_i}{\kappa\epsilon}\right).
\]
The update is
\begin{equation}\label{eq:appendix_cost_update}
 c\leftarrow \operatorname{Proj}_{[0,1]}
 \left(c-\eta_c\big((1-\delta)-\hat p\big)\right).
\end{equation}
Thus, if the observed correctness is below $1-\delta$, the cost decreases and continuing becomes cheaper, so trajectories become longer.  If correctness is above the target, the cost increases and stopping becomes more aggressive.    The sign of the update is still exactly the dual intuition in \cref{eq:cost_update}.

%% file: sections/appendix/full_algo.tex
\begin{algorithm}[t!]
 \footnotesize
 \caption{Implementation of \cicpe{}}
 \label{algo:appendix_cicpe_implementation}
\begin{algorithmic}[1]
 \State Initialize replay buffer ${\cal B}$, inference network $I_\phi$, critic $Q_\psi$, actor rule ${\rm Act}\in\{\mathrm{TS},\mathrm{TTPS},\mathrm{TD3}\}$, target networks $I_{\bar\phi},Q_{\bar\psi}$, and cost $c$.
 \Statex \texttt{\color{blue}// Training phase}
 \While{training is not over}
   \State Sample a batch of tasks $\theta\sim\nu$ and initialize their histories $H_1$.
   \For{$t=1,\ldots,T_{\max}$}
      \State Propose continuation actions $A_t$ using TS, TTPS, or the TD3 actor with its current exploration schedule.
      \State Stop tasks satisfying $Q_\psi(H_t,a_{\rm stop})\geq Q_\psi(H_t,A_t)$ after the warmup and minimum-time gates.
      \State Execute $A_t$ on the remaining tasks, observe $Y_{t+1}$, and append $(A_t,Y_{t+1})$ to the histories.
   \EndFor
   \State Store the completed trajectories, stop flags, terminal times, and targets $x_\theta^\star$ in ${\cal B}$.
   \State Sample replay prefixes and update $I_\phi$ with \cref{eq:appendix_robust_nll_loss}; for $\epsilon$-best-arm also use \cref{eq:appendix_eps_best_arm_loss}.
   \State Estimate sampled rewards with $I_{\bar\phi}$ using \cref{eq:appendix_sampled_reward} and update $Q_\psi$ using \cref{eq:appendix_stop_target,eq:appendix_cont_target,eq:appendix_critic_loss}.
   \State If using TD3, update $\pi_\vartheta$ on the delayed actor schedule using \cref{eq:appendix_td3_actor_loss,eq:appendix_td3_kl_actor_loss}.
   \State Polyak-update $I_{\bar\phi}$, $Q_{\bar\psi}$, and, when present, $\pi_{\bar\vartheta}$; update $c$ with \cref{eq:appendix_cost_update}.
 \EndWhile
 \Statex \hrulefill
 \Statex \texttt{\color{blue}// Deployment phase}
 \State Freeze the learned networks and initialize a fresh test task.
 \For{$t=1,\ldots,T_{\max}$}
    \State Propose $A_t$ using the selected actor rule.
    \If{$Q_\psi(H_t,a_{\rm stop})\geq Q_\psi(H_t,A_t)$}
       \State {\bf return} $\hat x=\mu_\phi(H_t)$.
    \EndIf
    \State Execute $A_t$, observe $Y_{t+1}$, and update $H_{t+1}$.
 \EndFor
 \State {\bf return} $\hat x=\mu_\phi(H_{T_{\max}+1})$.
\end{algorithmic}
\end{algorithm}

%% file: sections/appendix/results.tex
\section{Appendix: Numerical Results}\label{sec:appendix:numerical_results}
In this section we present more details on the numerical results. We refer the reader to the code for more details (see the \texttt{README.md} file), especially regarding the hyperparameters. We now present the synthetic benchmarks with additional numerical results. These additional results include sweeps over various values of $\epsilon,\sigma$, as well as checking robustness to prior misspecification. We conclude with details and additional results regarding the geochemical exploration task.

\paragraph{Computational resources.}
All experiments were run on NVIDIA V100 GPU or NVIDIA L40S GPU.
For the synthetic benchmarks, each C-ICPE training run takes
approximately 12 hours. With 3 random seeds, at least 4
$(\varepsilon, \sigma)$ configurations, and 3 dimensionalities per
benchmark, the total training budget per synthetic task is
$12 \times 3 \times 4 \times 3 = 432$ GPU-hours. For the geochemical
exploration task, each training run takes approximately 70 hours;
with 3 seeds, 2 $\varepsilon$ configurations, and a single
dimensionality ($d=2$), the total is $70 \times 3 \times 2 = 420$
GPU-hours.

\paragraph{Confidence intervals via hierarchical bootstrap.}
To account for variability across both task instances and trajectory
randomness, we report 95\% confidence intervals computed via
hierarchical bootstrap \citep{efron1992bootstrap}.
For each trained model (seed), we sample 300 test environments
$\theta \sim \nu$ and collect 15 independent trajectories per
environment. The total
variance of a statistic $\hat\mu$ (e.g., accuracy) decomposes as
\[
\mathrm{Var}(\hat{\mu}) =
\underbrace{\mathrm{Var}_{s}\bigl(\mathbb{E}[\hat{\mu} \mid s]\bigr)}_{\text{between-seed}}
+ \underbrace{\mathbb{E}_{s}\bigl[\mathrm{Var}_{\theta}\bigl(
\mathbb{E}[\hat{\mu} \mid s, \theta]\bigr)\bigr]}_{\text{between-environment}}
+ \underbrace{\mathbb{E}_{s,\theta}\bigl[\mathrm{Var}(\hat{\mu}
\mid s, \theta)\bigr]}_{\text{within-environment}},
\]
where the first term captures variability due  to training, the second due to which tasks are drawn
from $\nu$, and the third due to observation noise and policy
stochasticity within a fixed task. A single bootstrap replicate is
constructed by (i) resampling seeds with replacement, then (ii) for
each resampled seed, resampling environments with replacement, then
(iii) for each resampled environment, resampling trajectories with
replacement, and computing the statistic on the resampled dataset.
This three-level resampling preserves all components of variance. We
draw 10{,}000 bootstrap replicates and report the 2.5\% and 97.5\%
percentiles as the confidence interval.

\begin{remark}[On testing the inference model]
    Several of our synthetic localization benchmarks fall close to the symmetric
case described in \cref{app:theoretical_results:gaussian_inference_reward} (after \cref{prop:nll_mean_near_optimality}): the
loss is a distance to a selected target $x_\theta^\star$, so
$\X_\epsilon(\theta)$ is a ball, interval, or cap around this target. In the
optimization benchmarks, such as the value estimation task and the geochemical task, the loss is
instead induced by value gaps and the success sets need not be symmetric or
convex. These experiments therefore test the method beyond the setting where
the Gaussian NLL mean has an exact Bayes-optimality interpretation.
\end{remark}

\subsection{Synthetic Benchmarks: description }\label{sec:appendix_synthetic_benchmarks}
The synthetic benchmarks in \cref{sec:empirical_evaluation} are designed to isolate different aspects of the continuous fixed-confidence problem.  Binary search is the cleanest localization problem: every query returns a noisy comparison with the unknown target.  The $\epsilon$-best-arm problem keeps the same idea of identifying $x_\theta^\star$, but changes the geometry to a sphere and makes the loss directional rather than Euclidean.  Ackley minimization adds nuisance parameters and a multimodal response surface, so the agent has to learn an exploration rule that is not purely local.  Finally, GP max-value estimation is included because it is the case where the query space and the recommendation space are genuinely different: the agent queries a location, but it recommends a scalar value.  This is the setting where a TD3 actor is necessary, since posterior samples from the inference network are no longer valid actions.

\paragraph{Common protocol.}
Each training episode starts by sampling a fresh task parameter $\theta$ from the task prior $\nu$.  The agent then observes a sequential history $H_t=(A_1,Y_2,\ldots,A_{t-1},Y_t)$ and either stops or selects a new action.  We use a maximum horizon $t_{\max}$; if the learned stopping rule does not stop before this horizon, the episode is truncated and the final recommendation is still evaluated.  For the synthetic experiments reported in the survival plots and correctness tables, we use $t_{\max}=100$ unless otherwise stated.

All fixed-confidence runs use $\delta=0.1$.  The reported accuracy is
\[
\widehat{\mathrm{Acc}}
=\frac{1}{n}\sum_{i=1}^n {\bf 1}\left\{L_{\theta_i}(\hat x_i)\leq \epsilon\right\},
\]
and the goal is to achieve accuracy at least $1-\delta$ while minimizing the expected stopping time $\mathbb E[\tau]$.  For Euclidean localization tasks we use $L_\theta(x)=\|x-x_\theta^\star\|_2$.  For $\epsilon$-best-arm we use the directional loss $L_\theta(x)=1-\langle x,x_\theta^\star\rangle$ after normalizing both vectors to the unit sphere.  For GP max-value estimation the loss is the scalar absolute error $L_\theta(x)=|x-v_\theta^\star|$.  Confidence intervals are computed with hierarchical bootstrap over test episodes and random seeds.

\subsubsection{Noisy binary search.}
In binary search the unknown target is $x_\theta^\star=\theta\in[-1,1]^d$, sampled uniformly.  A query $a\in[-1,1]^d$ returns one noisy comparison per coordinate,
\begin{equation}\label{eq:appendix_binary_search_observation}
Y_{t,i}=\xi_{t,i}\operatorname{sign}(\theta_i-A_{t,i}),
\qquad
\mathbb P(\xi_{t,i}=1)=1-p,
\qquad
\mathbb P(\xi_{t,i}=-1)=p,
\end{equation}
independently over $i$ and $t$.  We use this problem because the statistically useful action is interpretable: a good policy should place queries near the current posterior median in each coordinate and shrink the feasible region.  This makes binary search a sanity check for the inference network and critic.  If the inference model does not contract its posterior, or if the critic cannot recognize when the posterior radius is below $\epsilon$, the method will fail even in this simple setting.  Since $\A=\X=[-1,1]^d$, TS and TTPS can act directly by sampling from the learned posterior over the target.

\begin{figure}[t]
\centering\includegraphics[width=\textwidth]{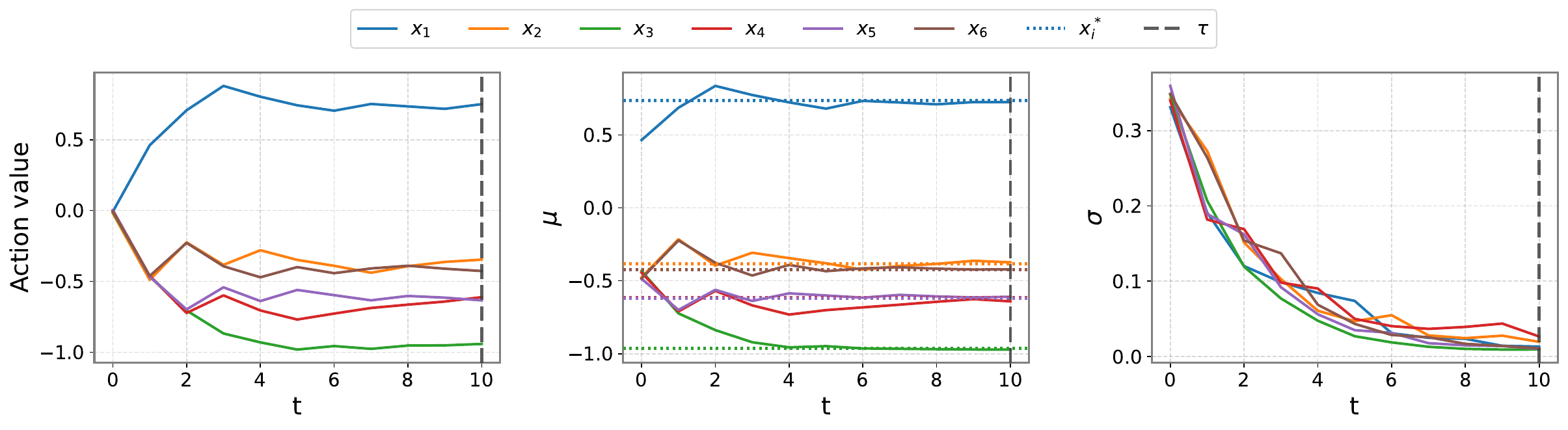}
\caption{Binary search: visualization of how \cicpe{} explores in $6$ dimensions. From left to right: the query action (left), posterior mean (middle), and posterior standard deviation (right) along an exploration trajectory in the noisy binary search problem.}
\label{tab:binary_search_visualization}
\end{figure}

\subsubsection{$\epsilon$-best-arm identification on the sphere.}
For the continuous $\epsilon$-best-arm problem, the task parameter is a direction $x_\theta^\star=\theta\in\mathbb{S}^{d-1}$ sampled by normalizing a standard Gaussian vector.  The agent queries a vector $a$ and observes
\begin{equation}\label{eq:appendix_eps_best_arm_observation}
Y_t=\theta^\top A_t+\xi_t,
\qquad
\xi_t\sim\mathcal N(0,\sigma^2).
\end{equation}
The recommendation is correct if $\theta^\top \hat x\geq 1-\epsilon$.  Although the implementation stores the enclosing action bounds as $[-1,1]^d$, the inference mean, posterior samples, TTPS candidates, and uniform baseline actions are projected to $\mathbb{S}^{d-1}$ for this benchmark.  This projection is important: Euclidean uncertainty is not the right object near the sphere, and two vectors with the same direction but different norms should not be treated as different hypotheses.  The reason for this benchmark is that each observation is a scalar projection.  The agent must choose directions that disambiguate the posterior over $\theta$, while the stopping rule must reason in terms of cosine error rather than Euclidean error.

This experiment also lets us compare to a specialized frequentist fixed-confidence baseline.  Lazy Track-and-Stop uses the known linear observation structure and an analytic generalized-likelihood-ratio stopping rule.  In our implementation it queries canonical directions and keeps a least-squares estimate of $\theta$, stopping only after both the likelihood-ratio condition and the spectral coverage condition are satisfied.  This is not a general baseline for all our tasks, but it is a useful reference point on the one benchmark where a specialized fixed-confidence method is available. However, note that in this problem the optimal exploration strategy is uniform \citep{jedra2020optimal}. Therefore, it shows to what degree \cicpe{} is able to learn a good inference model.

\subsubsection{Ackley minimizer identification.}
The Ackley task is a shifted and randomly parametrized global-optimization problem.  The task parameter is
\[
\theta=(a,b,c,\theta^\star),
\qquad
\theta^\star\sim\mathrm{Unif}([-1,1]^d),
\]
where $\theta^\star$ is the global minimizer and $(a,b,c)$ control the shape of the response surface.  In the reported runs we fix $a=10$ and sample $b\sim\mathrm{Unif}[0.1,0.5]$ and $c\sim\mathrm{Unif}[\pi,4\pi]$.  Given $u=A_t-\theta^\star$, the unnormalized Ackley value is
\begin{equation}\label{eq:appendix_ackley_function}
F_{a,b,c}(u)
= a+e
-a\exp\left(-b\sqrt{\frac{1}{d}\sum_{j=1}^d u_j^2}\right)
-\exp\left(\frac{1}{d}\sum_{j=1}^d \cos(cu_j)\right).
\end{equation}
We observe the sign-inverted and normalized value (more on this in the next page)
\begin{equation}\label{eq:appendix_ackley_observation}
Y_t=1-2\frac{F_{a,b,c}(A_t-\theta^\star)}{Z_{\rm norm}(b,c,d)}+\xi_t,
\qquad
\xi_t\sim\mathcal N(0,\sigma^2),
\end{equation}
so that larger observations are better and the target remains the minimizer $\theta^\star$.  The nuisance parameters are not provided to the agent.  Thus, across episodes, \cicpe{} must infer not only where the optimum is but also how observations should be interpreted for that episode.

Ackley is included because it is deliberately hostile to naive local search.  The function has many oscillations near the optimum and a broad outer region where observations can be weakly informative.  The active policy therefore has to balance broad exploration with local refinement, and the critic has to stop based on whether the inferred minimizer is accurate, not based on whether the last observed function value was high. 

\begin{figure}[t]
    \centering
    \begin{subfigure}{0.45\textwidth}
        \centering
        \includegraphics[width=\textwidth]{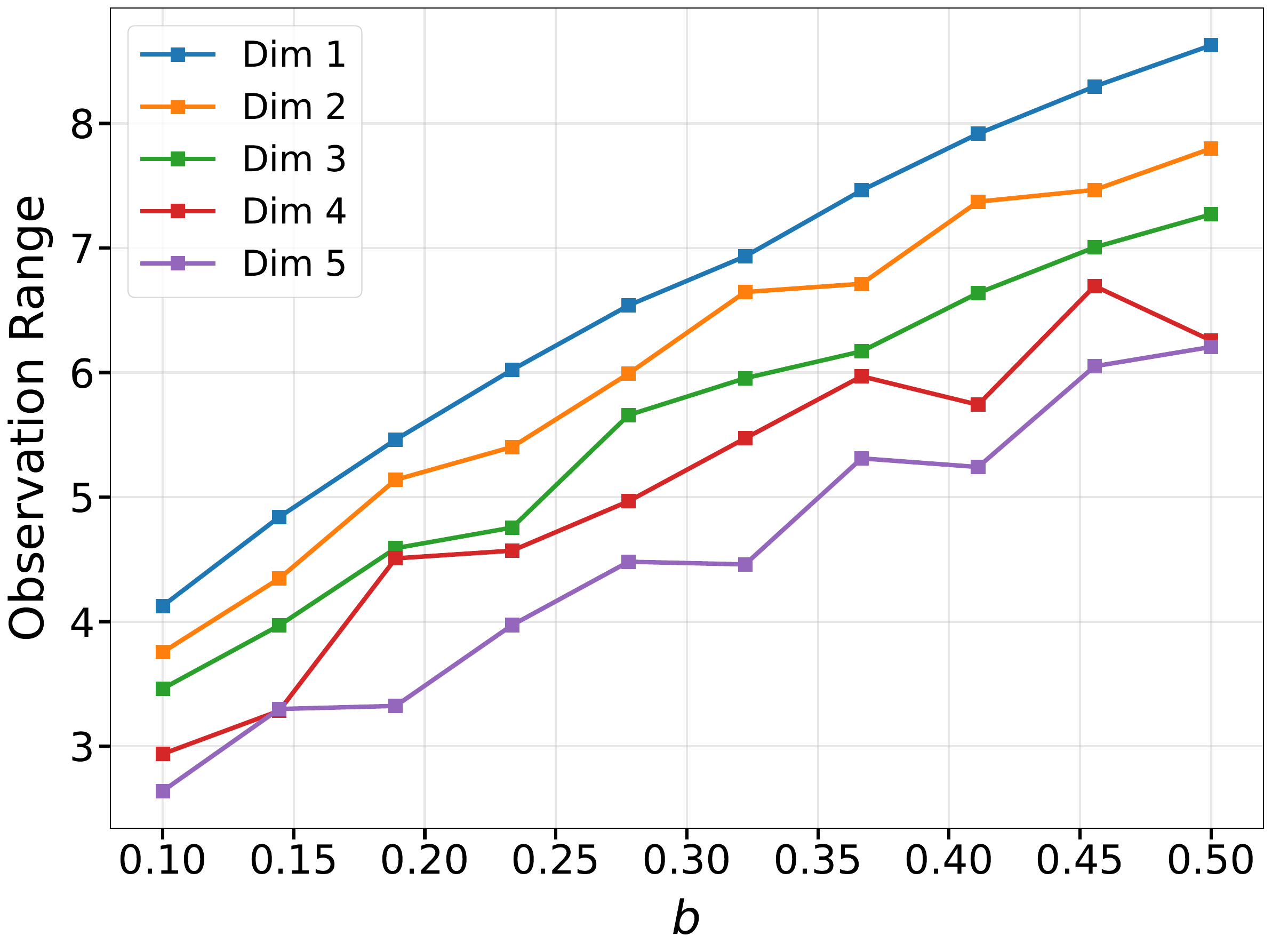}
        \caption{}
        \label{fig:ackley_b_vs_range}
    \end{subfigure}
    \hfill
    \begin{subfigure}{0.45\textwidth}
        \centering
        \includegraphics[width=\textwidth]{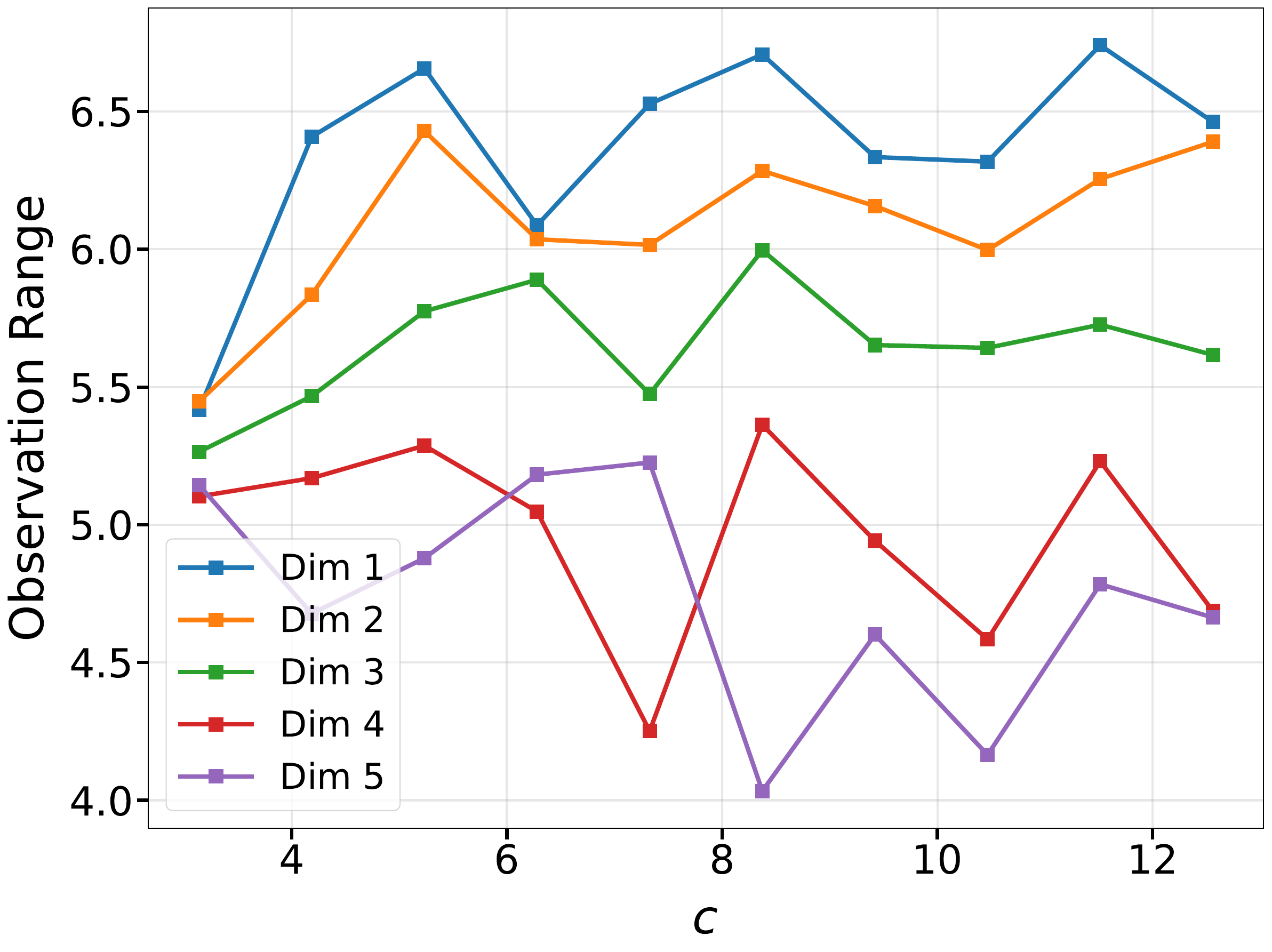}
        \caption{}
        \label{fig:ackley_c_vs_range}
    \end{subfigure}
    \caption{Effect of Ackley function's parameters on output range across multiple dimensions: (a) range vs $b$; (b) range vs $c$.}
    \label{fig:ackley_parameters_ranges_effect}
\end{figure}

\paragraph{Ackley function output normalization.}
The Ackley function's global minimum is always at the origin with a value of 0, but the maximum value within our defined recommendation space $\mathcal{X}$ depends on the function parameters and dimensionality. \cref{fig:ackley_b_vs_range,fig:ackley_c_vs_range} show how the values of $b$ and $c$ affect the function output ranges. Larger $b$ values consistently increase the range, while $c$ has less significant effect on the output ranges. From the figures, we also see that the ranges depend on dimensionality, where lower dimensions tend to have larger ranges. The issue with varying output ranges is that the influence of noise can vary across different priors sampled, and we want the noise effect to be on the same scale. Additionally, without normalization, \cicpe{} must not only learn the relative patterns from $H_t$ but also account for the scale differences across different $H_t$. For these reasons, we derive a normalization constant empirically from multiple samples across different $b$, $c$ values and dimensionalities: $Z_{\text{norm}}(b,c,d) = \pi - 0.21 \cdot D + 9.68 \cdot b + 0.04 \cdot c$.

\subsubsection{GP max-value estimation}
The GP benchmark separates the action and recommendation spaces.  At the beginning of an episode we sample a latent function
\[
f\sim\mathrm{GP}(0,k_{\mathrm{RBF}}(\ell,\sigma_f)),
\qquad
\ell\sim\mathrm{Unif}[0.05,0.2],
\qquad
\sigma_f=1,
\]
on $[0,1]^d$, with $d=1$ in the implementation.  The sample path is generated on a dense grid using circulant embedding.  Queries are continuous points $A_t\in[0,1]^d$; the observed value is obtained by linear interpolation in one dimension or bilinear interpolation in two dimensions, followed by Gaussian noise:
\begin{figure}[t]
\centering
        \includegraphics[width=\textwidth]{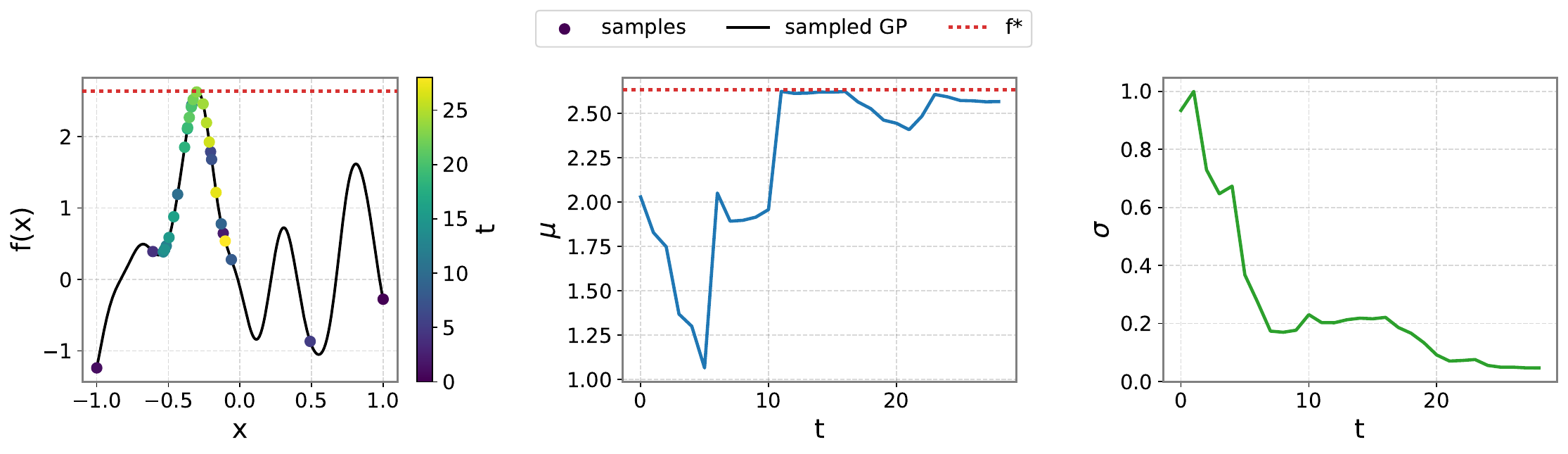}
\caption{Visualization of how \cicpe{} explores in the max-value estimation problem. From left to right: the query action (left), posterior mean (middle), and posterior standard deviation (right) along an exploration trajectory in the GP value estimation problem. Darker samples are queried earlier in the trajectory, while lighter samples are queried later.}
\label{tab:GP_visualization}
\end{figure}
\begin{equation}\label{eq:appendix_gp_value_observation}
Y_t=f(A_t)+\xi_t,
\qquad
\xi_t\sim\mathcal N(0,\sigma^2).
\end{equation}
The target is the scalar maximum value
\[
v_\theta^\star=\max_{u\in[0,1]^d} f(u),
\qquad
\X\subseteq\mathbb R,
\qquad
\A=[0,1]^d,
\]
where the maximum is computed on the same grid used to generate the episode.  The final recommendation is the inference mean for this scalar value, and success is $|\hat x-v_\theta^\star|\leq\epsilon$.

The reason for using max-value estimation rather than another argmax-localization task is that it tests the decoupling of inference and exploration.  In binary search, $\epsilon$-best-arm, and Ackley, a posterior sample of $x_\theta^\star$ is itself a reasonable query, so TS and TTPS can be implemented directly from the inference distribution.  For GP value estimation this would be meaningless: a sample from the inference model is a scalar value, not a point in $[0,1]^d$.  Therefore the action must be learned through the critic.  We use the TD3 actor for this benchmark because the critic can assign value to a query according to how much it is expected to improve the future estimate of $v_\theta^\star$, even though the query is not itself a recommendation.  This experiment is consequently the main empirical check that \cicpe{} handles the general $\A\neq\X$.

\subsection{Synthetic Benchmarks: baselines }\label{sec:appendix_baselines_reporting}
The most important baseline is \cicpeuniform{}, which keeps the same inference network, critic, stopping rule, replay buffer, and fixed-confidence cost update as \cicpe{}, but replaces the learned active query rule by uniform exploration.  This isolates the value of active experimentation: if \cicpe{} improves over \cicpeuniform{}, the gain cannot be explained by the inference model alone, because both methods use the same form of inference and the same stopping mechanism.

For tasks with $\A=\X$, we evaluate TS and TTPS because they use the learned posterior in the most direct way.  TS samples a plausible target and queries it.  TTPS keeps the posterior mean as the current recommendation and intentionally samples a plausible challenger that is sufficiently different.  This is useful when many posterior samples are small perturbations around the current mean: such samples do not test the remaining uncertainty, whereas a challenger query can reveal whether another region is still plausible.  For GP value estimation, TS and TTPS are not the right action rules for the reason described above, so we use TD3.

For all benchmarks, we set the sample budget to $t_{\max}=100$ by
default. If the trained \cicpe{} policy failed to reach the target ($1-\delta$)-accuracy, we extended the sample budget by $50$. For easier setting, such as low dimension, we instead used a smaller budget. The settings in which $t_{max}\neq100$ are listed in Table \ref{tab:icpe_budgets}, all other configurations use $t_{\max}=100$.

\begin{table}[h]
\centering
\begin{tabular}{lccc}
\toprule
Environment & $d$ & $\varepsilon$ & $t_{\max}$ \\
\midrule
$\varepsilon$-Best-Arm   & 15 & 0.005 & 150 \\
GP-value estimation       &  1 & 0.2   &  60 \\
Geochemical               &  2 & 0.15  & 150 \\
\bottomrule
\end{tabular}
\vspace{0.1cm}
\caption{Sample budget $t_{\max}$ for setting where $t_{max} \neq100$}
\label{tab:icpe_budgets}
\end{table}

We also compare against standard fixed-budget optimization methods implemented through Optuna \citep{akiba2019optuna}: TPE, GP-based Bayesian optimization, and CMA-ES.  We also compare with GP-UCB via BoTorch \citep{balandat2020botorch}. These baselines do not have a learned stopping rule and are not optimized for $(\epsilon,\delta)$-correctness.  To make the comparison conservative, we give them budgets tied to the empirical stopping time of the corresponding \cicpe{} variant.

They are then evaluated under the same success criterion $L_\theta(\hat x)\leq\epsilon$.  This comparison asks whether a generic optimizer, with a fixed budget equal to \cicpe{}'s expected sample complexity, already reaches the fixed-confidence target. 

\begin{itemize}
    \item \textbf{TPE} \citep{bergstra2011algorithms}: splits observations into good and bad groups based on a quantile threshold and models the objective by building a density estimator for each group, then selects candidates that maximize the ratio of good-to-bad.
    \item \textbf{CMA-ES} \citep{hansen2016cma}: an evolutionary algorithm that samples a population of candidates and iteratively updates a multivariate Gaussian distribution by adapting its covariance matrix based on the successful candidates.
    \item \textbf{GP-logEI} \citep{ament2023unexpected}: updates the Mat\'ern kernel's hyperparameters by maximizing the marginal log-likelihood on the past observations, and uses log expected improvement as the acquisition function.
    \item \textbf{GP-UCB} \citep{srinivas2010gaussian}: a variant of GP-based Bayesian Optimization with a Mat\'ern kernel that uses upper confidence bound as the acquisition function.
    \item \textbf{Uniform bin}: partitions the query space into $\left\lceil \sqrt{\mathbb{E}[\tau]} \right\rceil$ bins, where $\mathbb{E}[\tau]$ is the corresponding \cicpe{} variant's expected sample complexity, and uniformly sample within each bin. We compute the average value of each bin and report the maximum.
    \item \textbf{Uniform top 5\%}: queries uniformly and return the mean of the top 5\% values. The number of query matches the corresponding \cicpe{} variant's expected sample complexity.

    \item \textbf{Lazy Track-and-Stop round robin} \citep{jedra2020optimal}: queries the canonical basis in a round-robin fashion. The method maintains a least square estimate $\hat \theta_t$, and stops whenever
    
    \begin{equation*}
    Z_t \;\ge\; \beta(\delta, t)
    \quad\text{and}\quad
    \min_{j} N_t(j)
    \;\ge\;
    \max\!\Bigl(c,\;
    \tfrac{\rho(\delta, t)}{\lVert \hat\theta_t \rVert^{2}}\Bigr),
    \end{equation*}
    
    where $Z_t$ is the generalized likelihood ratio (GLR) for the
    $\varepsilon_t$ best-arm hypothesis evaluated against the
    worst-case competitor on the $\varepsilon_t$-boundary. This measures how much the current guess is better than the closest alternative. $N_t(j)$ is the number of pulls of the  $j$-th element of the canonical basis, and $\beta(\delta,t)$, $\rho(\delta,t)$, $\varepsilon_t$ are the threshold rule, spectral threshold, and gap-relaxation threshold. Lastly $c$ is a constant defined in \citep{jedra2020optimal}. See also \citep{jedra2020optimal} for more details. Upon stopping, the agent recommends $\hat a_t = \hat\theta_t / \lVert \hat\theta_t \rVert$.


    \item \textbf{Lazy Track-and-Stop uniform} \citep{jedra2020optimal}: Same as Lazy Track-and-Stop round robin except it samples the basis uniformly.
\end{itemize}

\clearpage
\subsection{Synthetic Benchmarks:  numerical results }\label{sec:appendix_additional_numerical_results}

We now present detailed accuracy and sample complexity results across
all benchmarks, sweeping over $(\varepsilon, \sigma, d)$
configurations. Tables 
report mean accuracy and sample complexity with 95\% confidence
intervals for every method and parameter combination. To
complement these aggregate statistics, we examine the stopping
behavior of each actor through two diagnostics: (i)~the survival function $P(\tau > t)$, which reveals how
quickly the learned critic commits to stopping, and (ii)~the
standard deviation of the inference model over the horizon, which
tracks how rapidly the posterior uncertainty around $x_\theta^\star$
contracts.
Across all benchmarks, \cicpe{} with learned exploration (TS, TTPS,
or TD3) consistently meets the $1 - \delta$ accuracy target while
\cicpeuniform{} degrades as dimension increases, particularly on
Ackley and binary search where directed exploration is essential.

\subsubsection{Noisy Binary Search}

\begin{figure}[htbp]
    \centering
        \includegraphics[width=\textwidth]{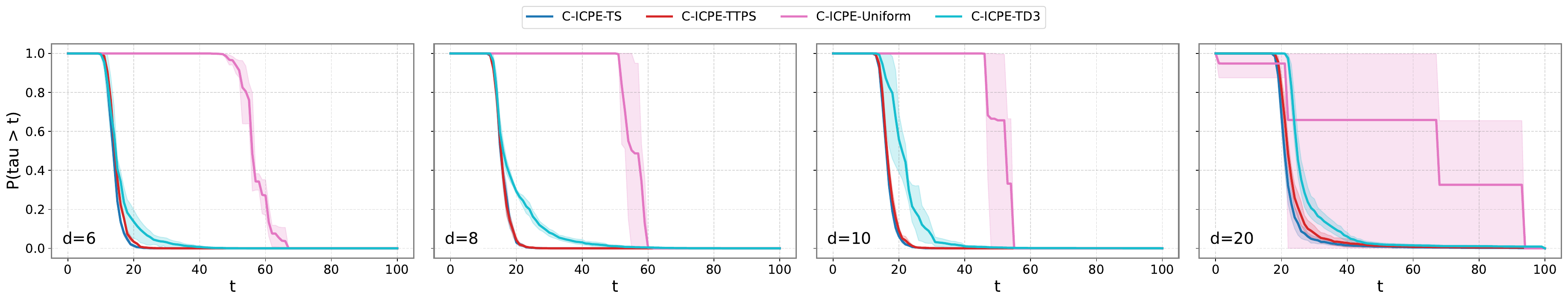}
       
        \includegraphics[width=\textwidth]{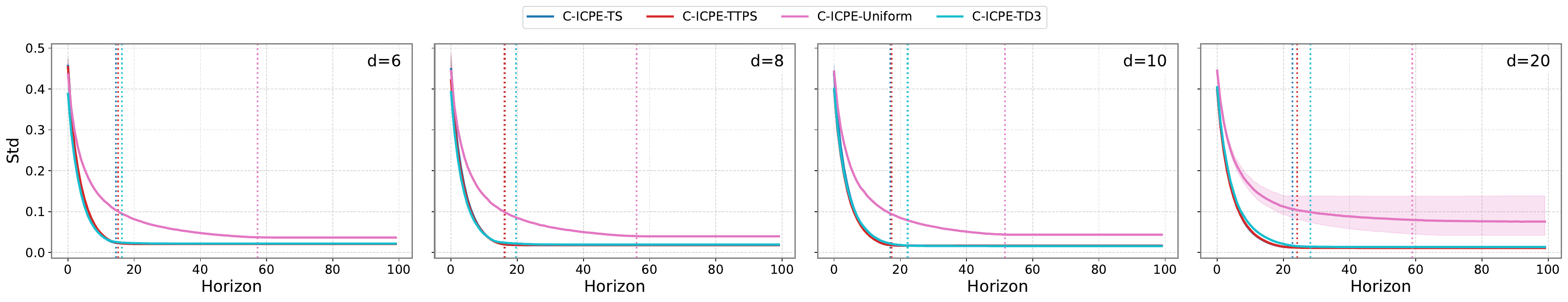}
    \caption{Results for Binary Search Problem with fixed confidence $\delta=0.1$ and $N=100$ across different dimensions at the most hardest $(\varepsilon, \sigma)$ setting: (top) survival function of $\tau$; (bottom) inference uncertainty convergence.}
    \label{fig:binary_search_survival_function_std}
\end{figure}

Tables~\ref{tab:binary-search-accuracy} and~\ref{tab:binary-search-sample-complexity} report accuracy
and sample complexity across all $(d, \varepsilon, \sigma)$
configurations; Figure~\ref{fig:binary_search_survival_function_std}
shows the survival function of $\tau$ and the convergence of the
inference model's standard deviation. In \cref{tab:binary_search_visualization} we also show ho \cicpe{} explores in $5$ dimensions, depicting the queries chosen by the actor, the posterior mean, and the posterior standard deviation over timesteps.

\textit{Accuracy.} All three active actors (TS, TTPS, TD3) meet the
$1 - \delta = 0.90$ target across every configuration tested, with
mean accuracy between $0.895$ and $0.916$. The confidence intervals
confirm that the target is met reliably: even the lower bounds
remain at or above $0.886$. \cicpeuniform{} matches the active
methods at $d = 6$ with $\varepsilon = 0.2$ (accuracy $\geq 0.901$)
but degrades sharply as either $d$ increases or $\varepsilon$
decreases. At $(\varepsilon, \sigma) = (0.1, 0.05)$, accuracy drops
from $0.688$ at $d = 6$ to $0.440$ at $d = 8$, $0.171$ at $d = 10$,
and $0.006$ at $d = 20$. This confirms that passive exploration
cannot accumulate sufficient directional information per coordinate
to localize the target within the allowed horizon in high dimensions.

\textit{Sample complexity.} Among active methods, \cicpets{} and
\cicpettps{} achieve comparable sample complexity across all
settings: at $d = 20$, $(\varepsilon, \sigma) = (0.1, 0.05)$,
\cicpets{} stops in $22.7$ queries on average and \cicpettps{} in
$24.2$, both with tight confidence intervals. \cicpetdthree{} is
competitive at moderate dimensions ($d \leq 10$) but exhibits higher
mean stopping times and substantially wider confidence intervals at
$d = 20$ (e.g., $35.8\ [19.2, 67.6]$ at $\varepsilon = 0.2$,
$\sigma = 0.05$), suggesting that the learned actor is less stable
in high dimensions. Sample complexity scales sublinearly in $d$ for
the active actors: \cicpets{} increases from $11.1$ ($d = 6$) to
$16.4$ ($d = 20$) at $(\varepsilon, \sigma) = (0.2, 0.05)$.

\textit{Stopping behavior.} The survival functions
(Figure~\ref{fig:binary_search_survival_function_std}, top)
corroborate the sample complexity results. At $d \leq 10$, \cicpets{}
and \cicpettps{} exhibit sharp transitions: $P(\tau > t)$ drops from
$1$ to $0$ within a narrow window, indicating that the critic
identifies a consistent stopping point. \cicpeuniform{} has a
heavy-tailed survival function that extends to the horizon, and at
$d = 20$ it rarely stops before $t_{\max}$. The inference standard
deviation (bottom row) confirms that the active actors' posteriors
contract rapidly, reaching near-zero uncertainty before the median
stopping time (dashed vertical lines), while \cicpeuniform{} at
$d = 20$ retains high residual uncertainty throughout the episode.

\begin{table}[htbp]
\centering
\small
\setlength{\tabcolsep}{2.5pt}
\renewcommand{\arraystretch}{1.08}
\begin{tabular}{@{}llrrrr@{}}
\toprule
$d$ & Method & \multicolumn{2}{c}{$\varepsilon=0.2$} & \multicolumn{2}{c}{$\varepsilon=0.1$} \\
\cmidrule(lr){3-4}\cmidrule(lr){5-6}
 &  & $\sigma=0.05$ & $\sigma=0.1$ & $\sigma=0.025$ & $\sigma=0.05$ \\
\midrule
\multirow[t]{4}{*}{6} & C-ICPE-TD3 & $0.895$ {\tiny [.886,.902]} & $0.898$ {\tiny [.886,.908]} & $0.903$ {\tiny [.893,.914]} & \textbf{0.901} {\tiny [.889,.912]} \\
 & C-ICPE-TS & \textbf{0.904} {\tiny [.895,.912]} & \textbf{0.907} {\tiny [.899,.913]} & \textbf{0.909} {\tiny [.899,.920]} & $0.900$ {\tiny [.893,.909]} \\
 & C-ICPE-TTPS & $0.897$ {\tiny [.890,.905]} & $0.903$ {\tiny [.893,.913]} & $0.900$ {\tiny [.893,.907]} & $0.898$ {\tiny [.889,.905]} \\
 & C-ICPE-uniform & $0.901$ {\tiny [.894,.908]} & $0.904$ {\tiny [.893,.914]} & $0.854$ {\tiny [.835,.874]} & $0.688$ {\tiny [.666,.709]} \\
\midrule
\multirow[t]{4}{*}{8} & C-ICPE-TD3 & $0.903$ {\tiny [.897,.911]} & $0.898$ {\tiny [.890,.907]} & $0.901$ {\tiny [.891,.912]} & $0.899$ {\tiny [.890,.906]} \\
 & C-ICPE-TS & $0.905$ {\tiny [.892,.917]} & $0.898$ {\tiny [.889,.905]} & \textbf{0.916} {\tiny [.902,.930]} & \textbf{0.909} {\tiny [.899,.918]} \\
 & C-ICPE-TTPS & $0.903$ {\tiny [.896,.911]} & \textbf{0.905} {\tiny [.893,.915]} & $0.907$ {\tiny [.895,.917]} & $0.901$ {\tiny [.891,.909]} \\
 & C-ICPE-uniform & \textbf{0.907} {\tiny [.897,.918]} & $0.894$ {\tiny [.886,.901]} & $0.724$ {\tiny [.693,.756]} & $0.440$ {\tiny [.406,.477]} \\
\midrule
\multirow[t]{4}{*}{10} & C-ICPE-TD3 & $0.901$ {\tiny [.892,.908]} & $0.898$ {\tiny [.889,.904]} & $0.904$ {\tiny [.893,.914]} & \textbf{0.915} {\tiny [.891,.948]} \\
 & C-ICPE-TS & $0.896$ {\tiny [.886,.903]} & \textbf{0.903} {\tiny [.890,.916]} & \textbf{0.916} {\tiny [.905,.925]} & $0.904$ {\tiny [.892,.915]} \\
 & C-ICPE-TTPS & \textbf{0.902} {\tiny [.895,.910]} & $0.900$ {\tiny [.891,.906]} & $0.901$ {\tiny [.891,.911]} & $0.904$ {\tiny [.897,.912]} \\
 & C-ICPE-uniform & $0.873$ {\tiny [.853,.891]} & $0.860$ {\tiny [.847,.873]} & $0.493$ {\tiny [.471,.514]} & $0.171$ {\tiny [.136,.201]} \\
\midrule
\multirow[t]{4}{*}{20} & C-ICPE-TD3 & $0.897$ {\tiny [.888,.905]} & $0.905$ {\tiny [.896,.913]} & \textbf{0.925} {\tiny [.908,.940]} & $0.903$ {\tiny [.895,.912]} \\
 & C-ICPE-TS & \textbf{0.911} {\tiny [.893,.931]} & $0.903$ {\tiny [.892,.914]} & $0.901$ {\tiny [.891,.910]} & $0.903$ {\tiny [.889,.919]} \\
 & C-ICPE-TTPS & $0.896$ {\tiny [.887,.903]} & \textbf{0.907} {\tiny [.897,.917]} & $0.900$ {\tiny [.890,.909]} & \textbf{0.913} {\tiny [.906,.921]} \\
 & C-ICPE-uniform & $0.662$ {\tiny [.630,.698]} & $0.079$ {\tiny [.001,.136]} & $0.036$ {\tiny [.022,.049]} & $0.006$ {\tiny [.000,.013]} \\
\bottomrule
\end{tabular}
\vspace{0.1cm}
\caption{Binary search: accuracy (mean and 95\% CI) for every $(d, \varepsilon, \sigma)$ configuration.}
\label{tab:binary-search-accuracy}
\end{table}

\begin{table}[htbp]
\centering
\small
\setlength{\tabcolsep}{2.5pt}
\renewcommand{\arraystretch}{1.08}
\begin{tabular}{@{}llrrrr@{}}
\toprule
$d$ & Method & \multicolumn{2}{c}{$\varepsilon=0.2$} & \multicolumn{2}{c}{$\varepsilon=0.1$} \\
\cmidrule(lr){3-4}\cmidrule(lr){5-6}
 &  & $\sigma=0.05$ & $\sigma=0.1$ & $\sigma=0.025$ & $\sigma=0.05$ \\
\midrule
\multirow[t]{4}{*}{6} & C-ICPE-TD3 & \textbf{10.4} {\tiny [10.2,10.6]} & $15.5$ {\tiny [15.2,15.9]} & $13.3$ {\tiny [12.5,14.2]} & $16.3$ {\tiny [15.7,16.8]} \\
 & C-ICPE-TS & $11.1$ {\tiny [11.0,11.3]} & \textbf{15.3} {\tiny [15.2,15.4]} & \textbf{12.0} {\tiny [11.9,12.0]} & \textbf{14.4} {\tiny [14.3,14.6]} \\
 & C-ICPE-TTPS & $11.3$ {\tiny [11.0,11.6]} & $16.2$ {\tiny [16.1,16.3]} & $12.4$ {\tiny [12.1,12.6]} & $15.1$ {\tiny [14.9,15.4]} \\
 & C-ICPE-uniform & $36.3$ {\tiny [36.0,36.7]} & $54.1$ {\tiny [52.4,56.9]} & $60.0$ {\tiny [58.2,62.3]} & $57.3$ {\tiny [56.1,58.1]} \\
\midrule
\multirow[t]{4}{*}{8} & C-ICPE-TD3 & \textbf{11.8} {\tiny [11.7,11.9]} & \textbf{17.3} {\tiny [17.0,17.5]} & $15.8$ {\tiny [15.0,16.8]} & $19.7$ {\tiny [19.1,20.2]} \\
 & C-ICPE-TS & $12.0$ {\tiny [11.5,12.4]} & $17.4$ {\tiny [16.9,17.7]} & \textbf{13.0} {\tiny [12.6,13.4]} & $16.2$ {\tiny [15.8,16.5]} \\
 & C-ICPE-TTPS & $12.3$ {\tiny [12.2,12.4]} & $17.8$ {\tiny [17.6,18.0]} & $13.4$ {\tiny [13.2,13.6]} & \textbf{16.2} {\tiny [15.9,16.5]} \\
 & C-ICPE-uniform & $44.0$ {\tiny [43.4,44.7]} & $62.1$ {\tiny [61.6,62.7]} & $62.1$ {\tiny [60.0,64.4]} & $56.0$ {\tiny [53.8,58.8]} \\
\midrule
\multirow[t]{4}{*}{10} & C-ICPE-TD3 & \textbf{12.7} {\tiny [12.5,12.8]} & $19.4$ {\tiny [18.5,20.7]} & $18.2$ {\tiny [16.2,20.4]} & $22.2$ {\tiny [20.5,23.6]} \\
 & C-ICPE-TS & $12.8$ {\tiny [12.7,12.8]} & \textbf{18.5} {\tiny [18.1,18.8]} & \textbf{14.0} {\tiny [13.8,14.1]} & \textbf{17.0} {\tiny [16.7,17.3]} \\
 & C-ICPE-TTPS & $13.4$ {\tiny [12.9,13.8]} & $18.8$ {\tiny [18.7,18.9]} & $14.1$ {\tiny [14.0,14.2]} & $17.3$ {\tiny [17.0,17.7]} \\
 & C-ICPE-uniform & $47.6$ {\tiny [45.2,49.3]} & $68.1$ {\tiny [67.2,68.9]} & $59.2$ {\tiny [57.9,60.4]} & $51.7$ {\tiny [47.1,55.0]} \\
\midrule
\multirow[t]{4}{*}{20} & C-ICPE-TD3 & $35.8$ {\tiny [19.2,67.6]} & $33.0$ {\tiny [27.2,42.2]} & $26.7$ {\tiny [23.1,32.3]} & $28.2$ {\tiny [26.9,29.5]} \\
 & C-ICPE-TS & \textbf{16.4} {\tiny [16.2,16.6]} & \textbf{23.3} {\tiny [23.0,23.5]} & \textbf{17.6} {\tiny [17.3,17.9]} & \textbf{22.7} {\tiny [21.8,23.5]} \\
 & C-ICPE-TTPS & $16.4$ {\tiny [16.1,16.6]} & $24.4$ {\tiny [24.1,24.7]} & $18.7$ {\tiny [18.0,19.4]} & $24.2$ {\tiny [23.2,25.5]} \\
 & C-ICPE-uniform & $64.0$ {\tiny [62.0,66.0]} & $47.5$ {\tiny [29.5,62.4]} & $81.1$ {\tiny [71.6,88.2]} & $59.6$ {\tiny [19.4,91.9]} \\
\bottomrule
\end{tabular}
\vspace{0.1cm}
\caption{Binary search: sample complexity (mean and 95\% CI) for every $(d, \varepsilon, \sigma)$ configuration.}
\label{tab:binary-search-sample-complexity}
\end{table}

\clearpage
\subsubsection{$\epsilon$-best arm problem}

\begin{figure}[htbp]
    \centering
        \includegraphics[width=\textwidth]{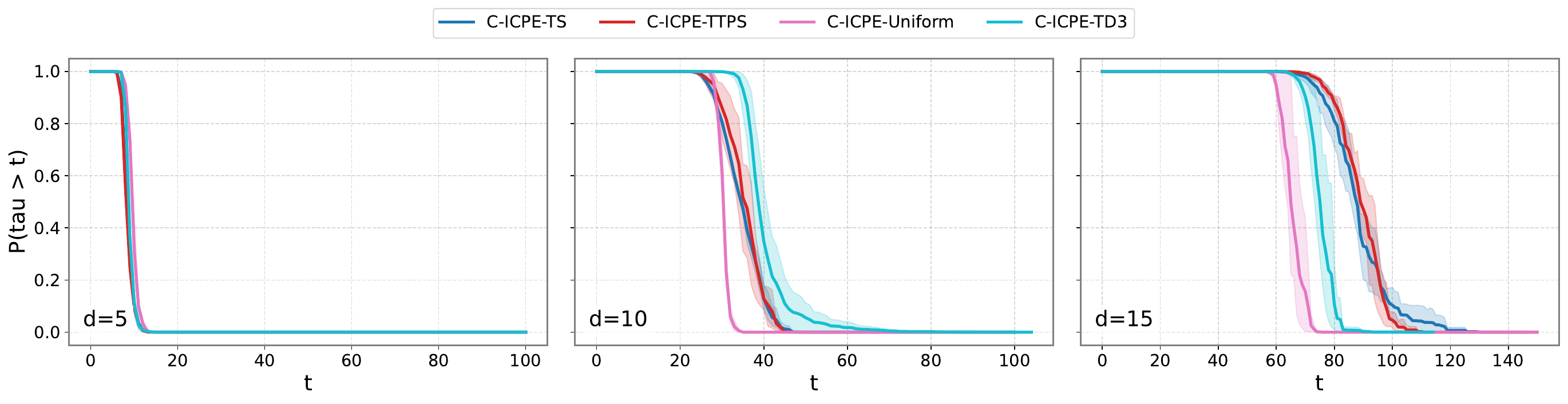}
    
        \includegraphics[width=\textwidth]{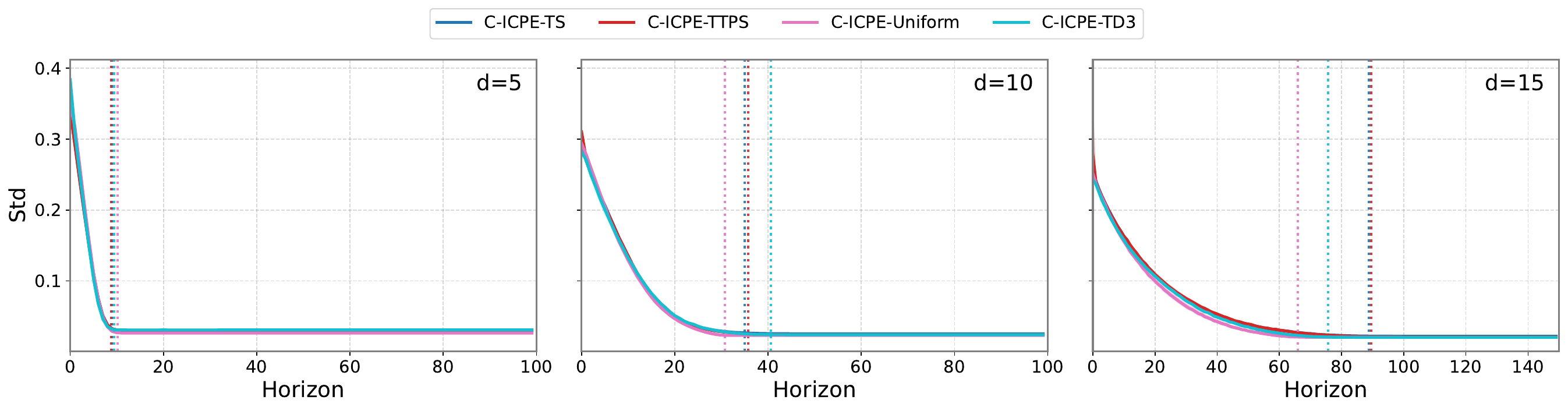}

    \caption{Results for $\epsilon$-Best-Arm Identification Problem with fixed confidence $\delta=0.005$ and $N=100/150$ across different dimensions at the most hardest $(\varepsilon, \sigma)$ setting: (top) survival function of $\tau$; (bottom) inference uncertainty convergence.}
    \label{fig:eps_best_arm_survival_function_std}
\end{figure}

Tables~\ref{tab:eps-best-arm-accuracy} and~\ref{tab:eps-best-arm-sample-complexity} report accuracy
and sample complexity;
Figure~\ref{fig:eps_best_arm_survival_function_std} shows the survival
function and inference uncertainty convergence.

\textit{Accuracy.} All \cicpe{} variants meet the $1 - \delta$
target across every $(d, \varepsilon, \sigma)$ configuration, with
mean accuracy between $0.897$ and $0.939$. Notably,
\cicpeuniform{} also meets the target throughout (accuracy
$0.897$--$0.936$), consistent with the rotational symmetry of the
problem: since the loss $L_\theta(x) = 1 - \theta^\top x$ is
invariant to orthogonal transformations, no query direction is
intrinsically more informative than another, and isotropic
exploration is  in general optimal. Lazy Track-and-Stop achieves
perfect accuracy ($1.000$) in all configurations by exploiting the
known linear observation structure and an analytic likelihood-ratio
stopping rule.

\textit{Sample complexity.} The gap between \cicpe{} and LT\&S is
substantial and grows with dimension. LT\&S queries exactly $d$
samples in every configuration, one per canonical direction,
achieving the minimum for a full-rank linear estimator. \cicpe{}
uses $2$--$6\times$ more queries: at $d = 15$,
$(\varepsilon, \sigma) = (0.005, 0.0025)$, \cicpets{} stops at
$88.8$ and \cicpettps{} at $89.5$, compared to $15.0$ for LT\&S.
This gap reflects the cost of a model-agnostic stopping rule:
\cicpe{} does not know the observation model is linear and must
learn when to stop from interaction data alone. Among \cicpe{}
variants, \cicpeuniform{} is competitive with and sometimes more
sample-efficient than the active actors. At $d = 15$,
$(\varepsilon, \sigma) = (0.005, 0.0025)$, \cicpeuniform{} stops at
$65.9$ while \cicpets{} requires $88.8$ and \cicpettps{} $89.5$.
This inversion occurs because the problem's symmetry makes directed
exploration unnecessary, and the overhead of posterior-driven action
selection, which occasionally concentrates queries in already
well-estimated directions, slows convergence relative to uniform
coverage. We also note that the current noise levels are low relative to
$\varepsilon$, placing the problem in a regime where LT\&S resolves
the target direction in a single pass of $d$ orthogonal queries. At
higher noise levels, where multiple measurement rounds are necessary, we expect the relative performance of \cicpe{} to improve.

\textit{Stopping behavior.} The survival functions
(Figure~\ref{fig:eps_best_arm_survival_function_std}, top) show that all
\cicpe{} actors have similar stopping profiles at $d = 5$, where
episodes terminate within $t \approx 10$. At $d = 15$, the curves
spread: \cicpeuniform{} stops earlier (median $\approx 60$) than
\cicpets{} and \cicpettps{} (median $\approx 75$--$85$), again
reflecting the advantage of isotropic coverage in this symmetric
problem. The inference standard deviation (bottom row) converges at
comparable rates across all actors, confirming that the posterior
contracts uniformly regardless of the exploration strategy, the
sample complexity differences are driven by stopping calibration,
not by differences in information acquisition.

\begin{table}[htbp]
\centering
\small
\setlength{\tabcolsep}{2.5pt}
\renewcommand{\arraystretch}{1.08}
\begin{tabular}{@{}llrrrr@{}}
\toprule
$d$ & Method & \multicolumn{2}{c}{$\varepsilon=0.02$} & \multicolumn{2}{c}{$\varepsilon=0.005$} \\
\cmidrule(lr){3-4}\cmidrule(lr){5-6}
 &  & $\sigma=0.005$ & $\sigma=0.01$ & $\sigma=0.00125$ & $\sigma=0.0025$ \\
\midrule
\multirow[t]{6}{*}{5} & C-ICPE-TD3 & $0.925$ {\tiny [.907,.943]} & $0.915$ {\tiny [.905,.926]} & $0.898$ {\tiny [.890,.907]} & $0.902$ {\tiny [.893,.912]} \\
 & C-ICPE-TS & $0.939$ {\tiny [.912,.958]} & $0.939$ {\tiny [.929,.947]} & $0.908$ {\tiny [.890,.930]} & $0.904$ {\tiny [.891,.917]} \\
 & C-ICPE-TTPS & $0.934$ {\tiny [.922,.944]} & $0.938$ {\tiny [.930,.944]} & $0.902$ {\tiny [.891,.911]} & $0.900$ {\tiny [.893,.909]} \\
 & C-ICPE-uniform & $0.936$ {\tiny [.925,.946]} & $0.931$ {\tiny [.919,.943]} & $0.922$ {\tiny [.914,.928]} & $0.927$ {\tiny [.918,.935]} \\
 & LT\&S round-robin & \textbf{1.000} {\tiny [1.000,1.000]} & \textbf{1.000} {\tiny [1.000,1.000]} & \textbf{1.000} {\tiny [1.000,1.000]} & \textbf{1.000} {\tiny [1.000,1.000]} \\
 & LT\&S uniform & $1.000$ {\tiny [1.000,1.000]} & $1.000$ {\tiny [1.000,1.000]} & $1.000$ {\tiny [1.000,1.000]} & $1.000$ {\tiny [1.000,1.000]} \\
\midrule
\multirow[t]{6}{*}{10} & C-ICPE-TD3 & $0.920$ {\tiny [.908,.931]} & $0.901$ {\tiny [.891,.909]} & $0.908$ {\tiny [.896,.919]} & $0.919$ {\tiny [.906,.932]} \\
 & C-ICPE-TS & $0.902$ {\tiny [.893,.911]} & $0.909$ {\tiny [.897,.920]} & $0.908$ {\tiny [.892,.922]} & $0.915$ {\tiny [.899,.930]} \\
 & C-ICPE-TTPS & $0.917$ {\tiny [.907,.926]} & $0.904$ {\tiny [.897,.911]} & $0.911$ {\tiny [.900,.921]} & $0.906$ {\tiny [.892,.920]} \\
 & C-ICPE-uniform & $0.911$ {\tiny [.901,.921]} & $0.905$ {\tiny [.895,.913]} & $0.905$ {\tiny [.894,.916]} & $0.900$ {\tiny [.890,.910]} \\
 & LT\&S round-robin & \textbf{1.000} {\tiny [1.000,1.000]} & \textbf{1.000} {\tiny [1.000,1.000]} & \textbf{1.000} {\tiny [1.000,1.000]} & \textbf{1.000} {\tiny [1.000,1.000]} \\
 & LT\&S uniform & $1.000$ {\tiny [1.000,1.000]} & $1.000$ {\tiny [1.000,1.000]} & $1.000$ {\tiny [1.000,1.000]} & $1.000$ {\tiny [1.000,1.000]} \\
\midrule
\multirow[t]{6}{*}{15} & C-ICPE-TD3 & $0.904$ {\tiny [.894,.913]} & $0.906$ {\tiny [.896,.915]} & $0.909$ {\tiny [.898,.920]} & $0.910$ {\tiny [.902,.917]} \\
 & C-ICPE-TS & $0.918$ {\tiny [.898,.937]} & $0.918$ {\tiny [.900,.932]} & $0.908$ {\tiny [.895,.921]} & $0.898$ {\tiny [.889,.907]} \\
 & C-ICPE-TTPS & $0.907$ {\tiny [.894,.919]} & $0.907$ {\tiny [.894,.919]} & $0.910$ {\tiny [.887,.934]} & $0.923$ {\tiny [.893,.959]} \\
 & C-ICPE-uniform & $0.897$ {\tiny [.887,.905]} & $0.902$ {\tiny [.889,.913]} & $0.906$ {\tiny [.893,.919]} & $0.901$ {\tiny [.892,.910]} \\
 & LT\&S round-robin & \textbf{1.000} {\tiny [1.000,1.000]} & \textbf{1.000} {\tiny [1.000,1.000]} & \textbf{1.000} {\tiny [1.000,1.000]} & \textbf{1.000} {\tiny [1.000,1.000]} \\
 & LT\&S uniform & $1.000$ {\tiny [1.000,1.000]} & $1.000$ {\tiny [1.000,1.000]} & $1.000$ {\tiny [1.000,1.000]} & $1.000$ {\tiny [1.000,1.000]} \\
\bottomrule
\end{tabular}
\vspace{0.1cm}
\caption{$\epsilon$-best arm: accuracy (mean and 95\% CI) for every $(d, \varepsilon, \sigma)$ configuration.}
\label{tab:eps-best-arm-accuracy}
\end{table}

\begin{table}[htbp]
\centering
\small
\setlength{\tabcolsep}{2.5pt}
\renewcommand{\arraystretch}{1.08}
\begin{tabular}{@{}llrrrr@{}}
\toprule
$d$ & Method & \multicolumn{2}{c}{$\varepsilon=0.02$} & \multicolumn{2}{c}{$\varepsilon=0.005$} \\
\cmidrule(lr){3-4}\cmidrule(lr){5-6}
 &  & $\sigma=0.005$ & $\sigma=0.01$ & $\sigma=0.00125$ & $\sigma=0.0025$ \\
\midrule
\multirow[t]{6}{*}{5} & C-ICPE-TD3 & $7.4$ {\tiny [7.3,7.5]} & $7.5$ {\tiny [7.4,7.5]} & $9.4$ {\tiny [9.2,9.6]} & $9.4$ {\tiny [9.3,9.5]} \\
 & C-ICPE-TS & $7.3$ {\tiny [7.2,7.5]} & $7.6$ {\tiny [7.5,7.6]} & $8.9$ {\tiny [8.6,9.1]} & $9.0$ {\tiny [8.8,9.2]} \\
 & C-ICPE-TTPS & $7.3$ {\tiny [7.1,7.4]} & $7.5$ {\tiny [7.4,7.5]} & $8.7$ {\tiny [8.6,8.9]} & $8.8$ {\tiny [8.7,9.0]} \\
 & C-ICPE-uniform & $8.3$ {\tiny [8.2,8.4]} & $8.2$ {\tiny [8.0,8.4]} & $10.0$ {\tiny [9.9,10.1]} & $10.1$ {\tiny [10.0,10.3]} \\
 & LT\&S round-robin & \textbf{5.0} {\tiny [5.0,5.0]} & \textbf{5.0} {\tiny [5.0,5.0]} & \textbf{5.0} {\tiny [5.0,5.0]} & \textbf{5.0} {\tiny [5.0,5.0]} \\
 & LT\&S uniform & $5.0$ {\tiny [5.0,5.0]} & $5.0$ {\tiny [5.0,5.0]} & $5.0$ {\tiny [5.0,5.0]} & $5.0$ {\tiny [5.0,5.0]} \\
\midrule
\multirow[t]{6}{*}{10} & C-ICPE-TD3 & $21.7$ {\tiny [21.1,22.6]} & $21.9$ {\tiny [21.6,22.3]} & $37.3$ {\tiny [35.3,39.0]} & $40.6$ {\tiny [38.2,43.3]} \\
 & C-ICPE-TS & $20.0$ {\tiny [19.5,20.7]} & $23.9$ {\tiny [22.8,25.3]} & $31.4$ {\tiny [30.1,33.7]} & $35.0$ {\tiny [34.7,35.5]} \\
 & C-ICPE-TTPS & $20.7$ {\tiny [20.4,21.0]} & $22.8$ {\tiny [22.6,23.0]} & $32.7$ {\tiny [30.9,34.6]} & $35.8$ {\tiny [33.8,37.3]} \\
 & C-ICPE-uniform & $22.1$ {\tiny [21.9,22.3]} & $22.4$ {\tiny [22.2,22.5]} & $31.1$ {\tiny [30.8,31.4]} & $30.7$ {\tiny [30.5,31.1]} \\
 & LT\&S round-robin & \textbf{10.0} {\tiny [10.0,10.0]} & \textbf{10.0} {\tiny [10.0,10.0]} & \textbf{10.0} {\tiny [10.0,10.0]} & \textbf{10.0} {\tiny [10.0,10.0]} \\
 & LT\&S uniform & $10.0$ {\tiny [10.0,10.0]} & $10.0$ {\tiny [10.0,10.0]} & $10.0$ {\tiny [10.0,10.0]} & $10.0$ {\tiny [10.0,10.0]} \\
\midrule
\multirow[t]{6}{*}{15} & C-ICPE-TD3 & $46.9$ {\tiny [46.5,47.2]} & $48.6$ {\tiny [48.1,49.2]} & $71.4$ {\tiny [70.2,72.6]} & $75.7$ {\tiny [73.4,78.9]} \\
 & C-ICPE-TS & $46.5$ {\tiny [46.0,47.1]} & $53.8$ {\tiny [53.1,54.7]} & $74.8$ {\tiny [73.3,76.3]} & $88.8$ {\tiny [85.4,92.7]} \\
 & C-ICPE-TTPS & $45.2$ {\tiny [44.0,47.0]} & $52.7$ {\tiny [50.9,54.1]} & $78.9$ {\tiny [74.6,87.0]} & $89.5$ {\tiny [87.7,91.7]} \\
 & C-ICPE-uniform & $44.5$ {\tiny [43.3,45.8]} & $45.1$ {\tiny [43.6,46.9]} & $68.1$ {\tiny [65.6,70.8]} & $65.9$ {\tiny [64.1,69.4]} \\
 & LT\&S round-robin & \textbf{15.0} {\tiny [15.0,15.0]} & \textbf{15.0} {\tiny [15.0,15.0]} & \textbf{15.0} {\tiny [15.0,15.0]} & \textbf{15.0} {\tiny [15.0,15.0]} \\
 & LT\&S uniform & $15.0$ {\tiny [15.0,15.0]} & $15.0$ {\tiny [15.0,15.0]} & $15.0$ {\tiny [15.0,15.0]} & $15.0$ {\tiny [15.0,15.0]} \\
\bottomrule
\end{tabular}
\vspace{0.1cm}
\caption{$\epsilon$-best arm: sample complexity (mean and 95\% CI) for every $(d, \varepsilon, \sigma)$ configuration.}
\label{tab:eps-best-arm-sample-complexity}
\end{table}

\clearpage
\subsubsection{Ackley minimization}

\begin{figure}[htbp]
    \centering
        \includegraphics[width=\textwidth]{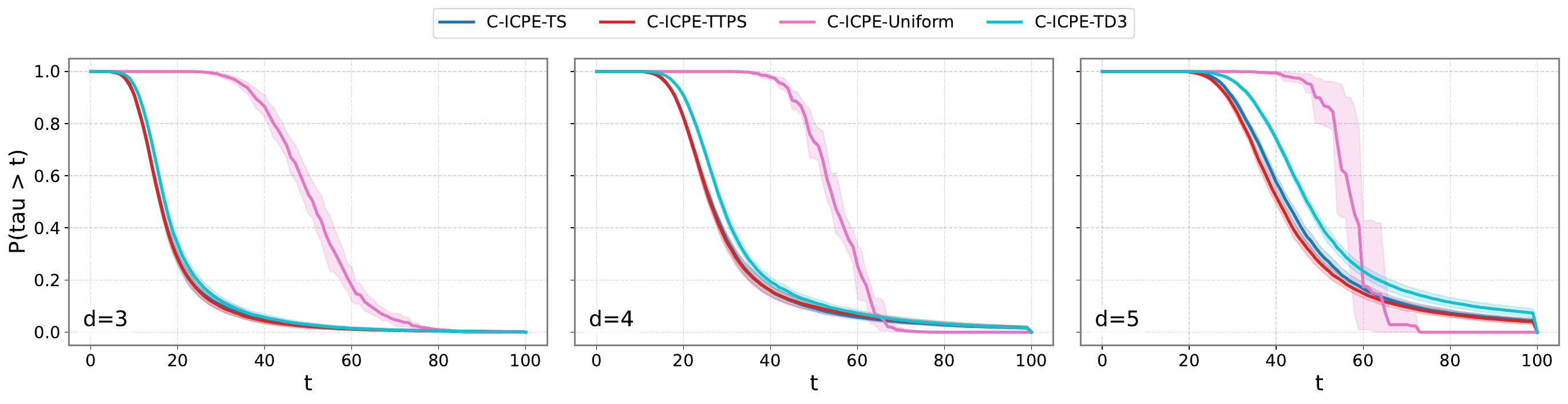}
        \includegraphics[width=\textwidth]{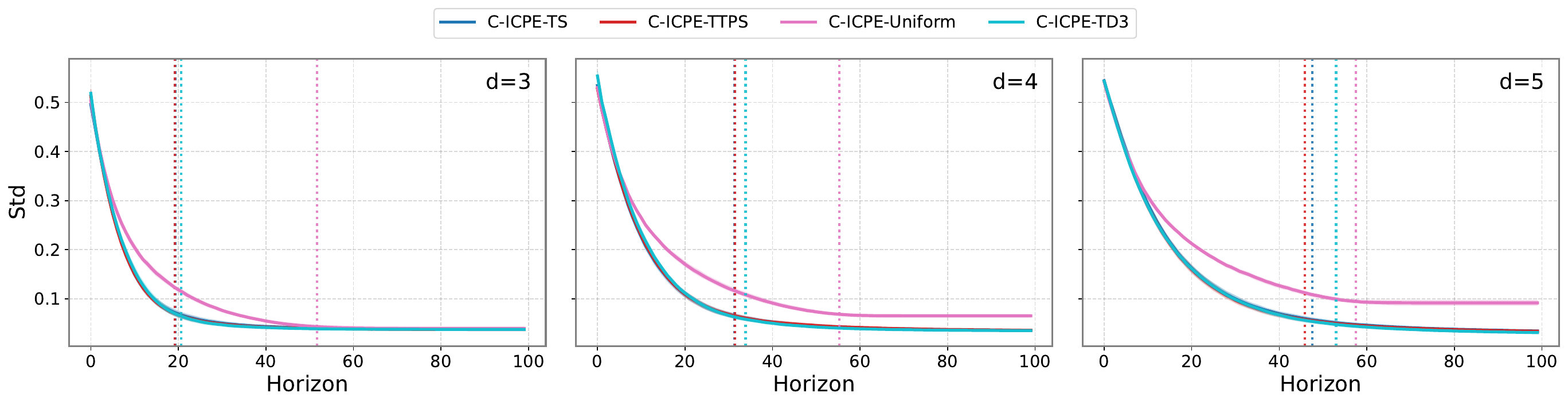}
    \caption{Results for Ackley function with fixed confidence $\delta=0.1$ and $N=100$ across different dimensions at the most hardest $(\varepsilon, \sigma)$ setting: (top) survival function of $\tau$; (bottom) inference uncertainty convergence.}
    \label{fig:ackley_survival_function_std}
\end{figure}

Tables~\ref{tab:ackley-accuracy} and~\ref{tab:ackley-sample-complexity} report accuracy
and sample complexity;
Figure~\ref{fig:ackley_survival_function_std} shows the survival
function and inference uncertainty convergence.

\textit{Accuracy.} All active \cicpe{} variants meet the
$1 - \delta = 0.90$ target across every $(d, \varepsilon, \sigma)$
configuration, with mean accuracy between $0.896$ and $0.928$.
\cicpeuniform{} meets the target at $d = 3$ (accuracy $\geq 0.890$)
but degrades as dimension increases: at
$(\varepsilon, \sigma) = (0.1, 0.05)$, accuracy drops from $0.890$
at $d = 3$ to $0.589$ at $d = 4$ and $0.237$ at $d = 5$. Unlike the
$\varepsilon$-best arm problem, the Ackley function has no symmetry
that makes uniform exploration competitive: the multimodal landscape
and flat outer region require directed queries to distinguish the
global minimizer from local optima. The Bayesian optimization
baselines fail across the board. GP-UCB achieves the highest
accuracy among them ($0.344$--$0.426$ at $d = 3$) but remains far
below the $0.90$ target even at the easiest configuration. TPE and
CMA-ES are near zero at $d \geq 4$. These methods optimize a
fixed-budget objective without a stopping rule and are not designed
for $(\varepsilon, \delta)$-correctness; the comparison confirms
that standard continuous optimization does not yield
fixed-confidence guarantees at comparable sample budgets.

\textit{Sample complexity.} Active actors use roughly half the
samples of \cicpeuniform{} across all dimensions. At $d = 5$,
$(\varepsilon, \sigma) = (0.1, 0.05)$, \cicpets{} stops at $47.5$
and \cicpettps{} at $45.8$, compared to $57.5$ for
\cicpeuniform{}. \cicpets{} and \cicpettps{} achieve similar sample
complexity throughout, while \cicpetdthree{} is slightly higher
(e.g., $53.0$ at the same setting). Sample complexity grows with
dimension for all methods: \cicpets{} increases from $14.9$ ($d=3$)
to $34.4$ ($d=5$) at $(\varepsilon, \sigma) = (0.2, 0.05)$.
Increasing $\sigma$ at fixed $\varepsilon$ consistently raises
sample complexity, as expected from the noisier observations.

\textit{Stopping behavior.} The survival functions
(Figure~\ref{fig:ackley_survival_function_std}, top) show a clear
separation between active and passive exploration. At $d = 3$,
\cicpets{} and \cicpettps{} exhibit sharp transitions around
$t \approx 15$--$25$, while \cicpeuniform{} has a heavy tail
extending past $t = 80$. At $d = 5$, the active actors stop around
$t \approx 35$--$50$ while \cicpeuniform{} rarely stops before
$t = 60$ and retains substantial probability mass near the horizon.
The inference standard deviation (bottom row) reveals a qualitative
difference from binary search: the posterior uncertainty plateaus
around $0.05$--$0.1$ rather than converging to zero. This reflects
the inherent difficulty of the Ackley landscape, the multimodal
structure and observation noise prevent the inference model from
achieving the same posterior concentration as in the unimodal binary
search setting. Nevertheless, the critic learns to stop at an
appropriate uncertainty level that is sufficient for
$\varepsilon$-correctness.

\begin{table}[htbp]
\centering
\small
\setlength{\tabcolsep}{2.5pt}
\renewcommand{\arraystretch}{1.08}
\begin{tabular}{@{}llrrrr@{}}
\toprule
$d$ & Method & \multicolumn{2}{c}{$\varepsilon=0.2$} & \multicolumn{2}{c}{$\varepsilon=0.1$} \\
\cmidrule(lr){3-4}\cmidrule(lr){5-6}
 &  & $\sigma=0.05$ & $\sigma=0.1$ & $\sigma=0.025$ & $\sigma=0.05$ \\
\midrule
\multirow[t]{4}{*}{3} & C-ICPE-TD3 & $15.3$ {\tiny [14.6,15.9]} & $18.9$ {\tiny [18.2,19.6]} & $18.9$ {\tiny [18.2,19.8]} & $20.6$ {\tiny [19.9,21.3]} \\
 & C-ICPE-TS & \textbf{14.9} {\tiny [14.5,15.5]} & $18.8$ {\tiny [18.1,19.7]} & \textbf{17.5} {\tiny [16.9,18.0]} & $19.3$ {\tiny [18.5,20.2]} \\
 & C-ICPE-TTPS & $15.4$ {\tiny [14.8,16.1]} & \textbf{18.7} {\tiny [18.1,19.4]} & $18.0$ {\tiny [17.3,18.7]} & \textbf{19.2} {\tiny [18.6,19.9]} \\
 & C-ICPE-uniform & $37.6$ {\tiny [34.4,40.2]} & $47.9$ {\tiny [46.3,49.4]} & $50.5$ {\tiny [48.6,52.2]} & $51.7$ {\tiny [49.8,53.7]} \\
\midrule
\multirow[t]{4}{*}{4} & C-ICPE-TD3 & $25.9$ {\tiny [25.0,26.7]} & $32.8$ {\tiny [31.8,33.9]} & $31.5$ {\tiny [30.5,32.6]} & $33.9$ {\tiny [33.1,34.7]} \\
 & C-ICPE-TS & $24.4$ {\tiny [23.2,25.8]} & \textbf{31.1} {\tiny [29.7,32.6]} & \textbf{27.5} {\tiny [26.3,28.9]} & \textbf{31.3} {\tiny [30.4,32.1]} \\
 & C-ICPE-TTPS & \textbf{24.2} {\tiny [22.7,25.8]} & $32.2$ {\tiny [31.0,33.6]} & $27.9$ {\tiny [27.0,29.0]} & $31.4$ {\tiny [30.4,32.4]} \\
 & C-ICPE-uniform & $62.4$ {\tiny [60.3,64.6]} & $59.4$ {\tiny [56.8,61.4]} & $57.8$ {\tiny [57.3,58.3]} & $55.2$ {\tiny [54.0,56.5]} \\
\midrule
\multirow[t]{4}{*}{5} & C-ICPE-TD3 & $38.1$ {\tiny [37.0,39.3]} & $50.7$ {\tiny [49.3,52.3]} & $44.1$ {\tiny [42.2,46.2]} & $53.0$ {\tiny [51.9,54.3]} \\
 & C-ICPE-TS & \textbf{34.4} {\tiny [33.3,35.9]} & \textbf{45.8} {\tiny [44.2,47.1]} & $38.5$ {\tiny [37.5,39.5]} & $47.5$ {\tiny [46.3,48.6]} \\
 & C-ICPE-TTPS & $35.6$ {\tiny [33.4,38.0]} & $47.0$ {\tiny [45.7,48.1]} & \textbf{37.7} {\tiny [36.9,38.9]} & \textbf{45.8} {\tiny [44.8,46.9]} \\
 & C-ICPE-uniform & $70.0$ {\tiny [63.9,79.2]} & $58.2$ {\tiny [57.1,59.7]} & $61.0$ {\tiny [55.3,68.6]} & $57.5$ {\tiny [55.1,60.1]} \\
\bottomrule
\end{tabular}
\vspace{0.1cm}
\caption{Ackley: sample complexity (mean and 95\% CI) for every $(d, \varepsilon, \sigma)$ configuration.}
\label{tab:ackley-sample-complexity}
\end{table}

\begin{table}[t]
\centering
\small
\setlength{\tabcolsep}{2.5pt}
\renewcommand{\arraystretch}{1.08}
\begin{tabular}{@{}llrrrr@{}}
\toprule
$d$ & Method & \multicolumn{2}{c}{$\varepsilon=0.2$} & \multicolumn{2}{c}{$\varepsilon=0.1$} \\
\cmidrule(lr){3-4}\cmidrule(lr){5-6}
 &  & $\sigma=0.05$ & $\sigma=0.1$ & $\sigma=0.025$ & $\sigma=0.05$ \\
\midrule
\multirow[t]{8}{*}{3} & C-ICPE-TD3 & \textbf{0.915} {\tiny [.902,.927]} & $0.904$ {\tiny [.894,.914]} & $0.898$ {\tiny [.888,.906]} & $0.908$ {\tiny [.899,.918]} \\
 & C-ICPE-TS & $0.896$ {\tiny [.886,.904]} & $0.899$ {\tiny [.890,.909]} & $0.899$ {\tiny [.887,.906]} & $0.907$ {\tiny [.895,.917]} \\
 & C-ICPE-TTPS & $0.902$ {\tiny [.894,.913]} & \textbf{0.907} {\tiny [.894,.919]} & \textbf{0.911} {\tiny [.900,.920]} & \textbf{0.912} {\tiny [.899,.924]} \\
 & C-ICPE-uniform & $0.902$ {\tiny [.886,.915]} & $0.905$ {\tiny [.890,.918]} & $0.906$ {\tiny [.894,.914]} & $0.890$ {\tiny [.868,.910]} \\
 & TPE & $0.098$ {\tiny [.072,.124]} & $0.136$ {\tiny [.106,.166]} & $0.016$ {\tiny [.005,.027]} & $0.028$ {\tiny [.014,.042]} \\
 & CMA-ES & $0.086$ {\tiny [.061,.111]} & $0.124$ {\tiny [.095,.153]} & $0.020$ {\tiny [.008,.032]} & $0.034$ {\tiny [.018,.050]} \\
 & GP-logEI & $0.322$ {\tiny [.281,.363]} & $0.356$ {\tiny [.314,.398]} & $0.218$ {\tiny [.182,.254]} & $0.194$ {\tiny [.159,.229]} \\
 & GP-UCB & $0.400$ {\tiny [.357,.443]} & $0.426$ {\tiny [.383,.469]} & $0.344$ {\tiny [.302,.386]} & $0.350$ {\tiny [.308,.392]} \\
\midrule
\multirow[t]{8}{*}{4} & C-ICPE-TD3 & $0.907$ {\tiny [.895,.919]} & \textbf{0.906} {\tiny [.895,.914]} & $0.899$ {\tiny [.888,.908]} & $0.899$ {\tiny [.891,.907]} \\
 & C-ICPE-TS & $0.907$ {\tiny [.899,.918]} & $0.904$ {\tiny [.892,.916]} & \textbf{0.914} {\tiny [.905,.924]} & \textbf{0.913} {\tiny [.900,.923]} \\
 & C-ICPE-TTPS & $0.898$ {\tiny [.886,.908]} & $0.898$ {\tiny [.888,.908]} & $0.908$ {\tiny [.897,.919]} & $0.906$ {\tiny [.894,.915]} \\
 & C-ICPE-uniform & \textbf{0.908} {\tiny [.897,.919]} & $0.739$ {\tiny [.713,.765]} & $0.735$ {\tiny [.712,.762]} & $0.589$ {\tiny [.560,.616]} \\
 & TPE & $0.048$ {\tiny [.029,.067]} & $0.084$ {\tiny [.060,.108]} & $0.006$ {\tiny [.000,.013]} & $0.006$ {\tiny [.000,.013]} \\
 & CMA-ES & $0.020$ {\tiny [.008,.032]} & $0.056$ {\tiny [.036,.076]} & $0.002$ {\tiny [.000,.006]} & $0.002$ {\tiny [.000,.006]} \\
 & GP-logEI & $0.270$ {\tiny [.231,.309]} & $0.206$ {\tiny [.171,.241]} & $0.132$ {\tiny [.102,.162]} & $0.104$ {\tiny [.077,.131]} \\
 & GP-UCB & $0.344$ {\tiny [.302,.386]} & $0.308$ {\tiny [.267,.349]} & $0.336$ {\tiny [.295,.377]} & $0.276$ {\tiny [.237,.315]} \\
\midrule
\multirow[t]{8}{*}{5} & C-ICPE-TD3 & $0.904$ {\tiny [.888,.918]} & $0.904$ {\tiny [.891,.916]} & $0.907$ {\tiny [.892,.920]} & $0.904$ {\tiny [.891,.916]} \\
 & C-ICPE-TS & \textbf{0.909} {\tiny [.900,.920]} & \textbf{0.912} {\tiny [.899,.921]} & $0.913$ {\tiny [.905,.925]} & $0.911$ {\tiny [.897,.921]} \\
 & C-ICPE-TTPS & $0.905$ {\tiny [.893,.915]} & $0.904$ {\tiny [.889,.915]} & \textbf{0.928} {\tiny [.916,.939]} & \textbf{0.912} {\tiny [.895,.927]} \\
 & C-ICPE-uniform & $0.690$ {\tiny [.618,.756]} & $0.416$ {\tiny [.401,.437]} & $0.428$ {\tiny [.332,.538]} & $0.237$ {\tiny [.203,.279]} \\
 & TPE & $0.030$ {\tiny [.015,.045]} & $0.042$ {\tiny [.024,.060]} & $0.000$ {\tiny [.000,.000]} & $0.002$ {\tiny [.000,.006]} \\
 & CMA-ES & $0.014$ {\tiny [.004,.024]} & $0.040$ {\tiny [.023,.057]} & $0.002$ {\tiny [.000,.006]} & $0.004$ {\tiny [.000,.010]} \\
 & GP-logEI & $0.154$ {\tiny [.122,.186]} & $0.074$ {\tiny [.051,.097]} & $0.068$ {\tiny [.046,.090]} & $0.044$ {\tiny [.026,.062]} \\
 & GP-UCB & $0.344$ {\tiny [.302,.386]} & $0.214$ {\tiny [.178,.250]} & $0.284$ {\tiny [.244,.324]} & $0.220$ {\tiny [.184,.256]} \\
\bottomrule
\end{tabular}
\vspace{0.1cm}
\caption{Ackley: accuracy (mean and 95\% CI) for every $(d, \varepsilon, \sigma)$ configuration.}
\label{tab:ackley-accuracy}
\end{table}

\clearpage
\subsubsection{GP max-value estimation}

\begin{figure}[htbp]
    \centering
    \begin{subfigure}{0.48\textwidth}
        \centering
        \includegraphics[width=\textwidth]{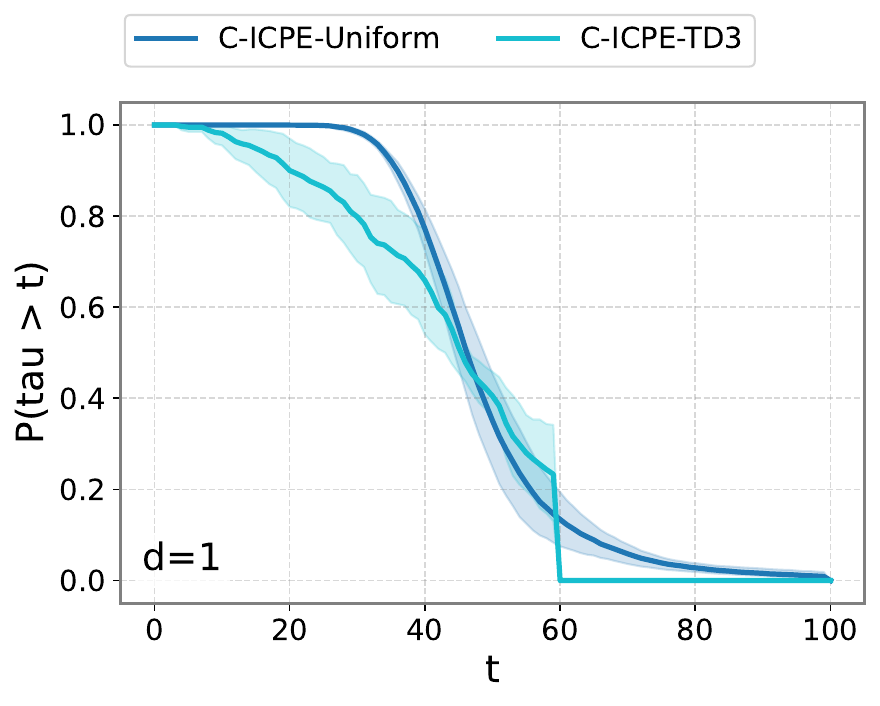}
        \caption{}
        \label{fig:gpValue_subfig-a}
    \end{subfigure}
    \hfill
    \begin{subfigure}{0.48\textwidth}
        \centering
        \includegraphics[width=\textwidth]{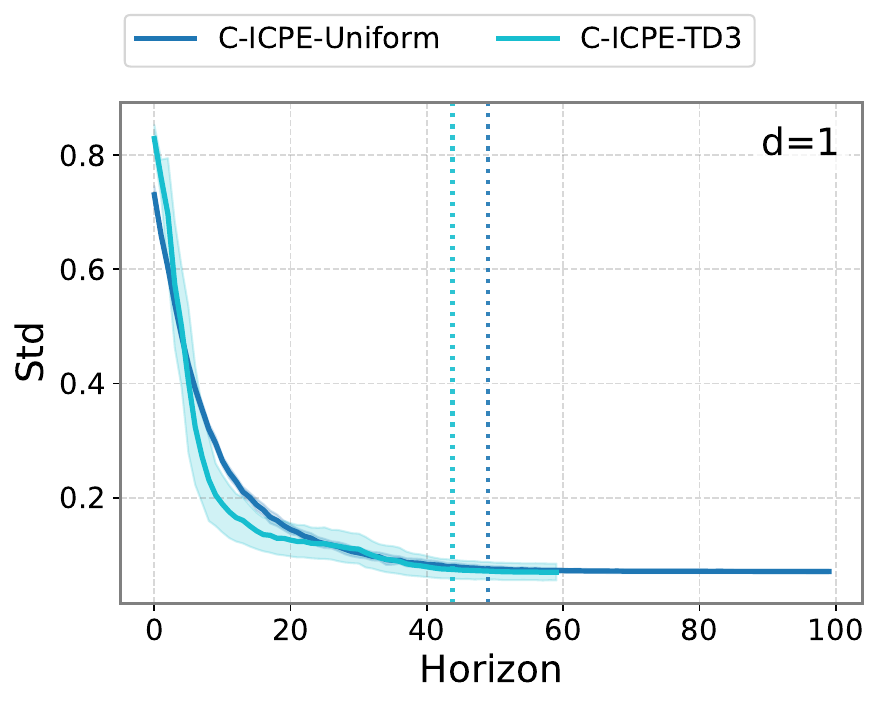}
        \caption{}
        \label{fig:gpValue_subfig-b}
    \end{subfigure}
    \caption{Results for GP value estimation with fixed confidence $\delta=0.1$ and $N=100$ across different dimensions at the most hardest $(\varepsilon, \sigma)$ setting: (a) survival function of $\tau$; (b) inference uncertainty convergence.}
    \label{fig:gpValue_survival_function_std}
\end{figure}

Tables~\ref{tab:gp-value-estimation-accuracy} and~\ref{tab:gp-value-estimation-sample-complexity} report accuracy and
sample complexity;
Figure~\ref{fig:gpValue_survival_function_std} shows the survival
function and inference uncertainty convergence. This is the
$\mathcal{X} \neq \mathcal{A}$ setting: the agent queries locations
$a \in [0,1]^d$ but recommends a scalar value estimate
$\hat{x} \in \mathbb{R}$. Only \cicpetdthree{} and
\cicpeuniform{} are evaluated, since TS and TTPS require
$\mathcal{X} = \mathcal{A}$.  In \cref{tab:GP_visualization} we also show ho \cicpe{} explores in this problem, depicting the queries chosen by the actor, the posterior mean, and the posterior standard deviation over timesteps.

\textit{Accuracy.} Both \cicpe{} variants meet the $1 - \delta$
target across all $(\varepsilon, \sigma)$ configurations.
\cicpetdthree{} achieves $0.906$--$0.930$ and \cicpeuniform{}
$0.901$--$0.931$; the two methods are comparable in accuracy, with
\cicpeuniform{} slightly higher at the noisier settings (e.g.,
$0.927$ vs.\ $0.910$ at $\varepsilon = 0.2$, $\sigma = 0.1$). The
non-parametric baselines fail to meet the target: Uniform bin
achieves $0.130$--$0.230$ and Uniform top~5\% reaches
$0.648$--$0.762$. Bayesian optimization baselines are not applicable
to this task, as they return locations rather than value estimates.

\textit{Sample complexity.} \cicpetdthree{} uses substantially fewer
samples than \cicpeuniform{}, particularly at the larger tolerance.
At $(\varepsilon, \sigma) = (0.2, 0.05)$, \cicpetdthree{} stops at
$13.7$ queries on average compared to $31.7$ for \cicpeuniform{} —
a $2.3\times$ reduction. At $(\varepsilon, \sigma) = (0.2, 0.1)$
the ratio is $1.8\times$ ($22.7$ vs.\ $41.0$). The gap narrows at
$\varepsilon = 0.15$: $39.9$ vs.\ $41.6$ at $\sigma = 0.0375$ and
$43.7$ vs.\ $49.0$ at $\sigma = 0.075$. This pattern indicates
that the learned exploration policy provides the largest benefit
when the tolerance is generous enough that a well-chosen sequence
of queries can resolve the max-value quickly, whereas at tighter
tolerances both methods require extensive coverage and the advantage
of directed exploration diminishes.

\textit{Stopping behavior.} The survival function
(left plot in \cref{fig:gpValue_survival_function_std}) shows that
\cicpetdthree{} stops earlier than \cicpeuniform{}, with the bulk
of episodes terminating between $t = 20$ and $t = 60$. Both methods
exhibit gradual transitions rather than the sharp drops observed in
binary search, reflecting the greater variability in task difficulty
under the GP prior (functions with short lengthscales require more
queries to resolve the peak value). The inference standard deviation
(right plot in \cref{fig:gpValue_survival_function_std}) starts high
($\approx 0.8$) and contracts rapidly in the first $20-40$ queries,
then plateaus. This residual uncertainty is
consistent with the difficulty of estimating a function's global
maximum from noisy pointwise observations: even after localizing
the region of high values, the precise peak height remains uncertain
until sufficient samples accumulate near the optimum.

\begin{table}[htbp]
\centering
\small
\setlength{\tabcolsep}{2.5pt}
\renewcommand{\arraystretch}{1.08}
\begin{tabular}{@{}llrrrr@{}}
\toprule
$d$ & Method & \multicolumn{2}{c}{$\varepsilon=0.2$} & \multicolumn{2}{c}{$\varepsilon=0.15$} \\
\cmidrule(lr){3-4}\cmidrule(lr){5-6}
 &  & $\sigma=0.05$ & $\sigma=0.1$ & $\sigma=0.0375$ & $\sigma=0.075$ \\
\midrule
\multirow[t]{4}{*}{1} & C-ICPE-TD3 & \textbf{0.906} {\tiny [.878,.928]} & $0.910$ {\tiny [.883,.932]} & \textbf{0.920} {\tiny [.890,.942]} & $0.930$ {\tiny [.903,.948]} \\
 & C-ICPE-uniform & $0.902$ {\tiny [.889,.914]} & \textbf{0.927} {\tiny [.915,.938]} & $0.901$ {\tiny [.887,.916]} & \textbf{0.931} {\tiny [.915,.946]} \\
 & Uniform bin & $0.130$ {\tiny [.100,.160]} & $0.176$ {\tiny [.143,.209]} & $0.230$ {\tiny [.193,.267]} & $0.226$ {\tiny [.189,.263]} \\
 & Uniform top 5\% & $0.648$ {\tiny [.606,.690]} & $0.672$ {\tiny [.631,.713]} & $0.762$ {\tiny [.725,.799]} & $0.696$ {\tiny [.656,.736]} \\
\bottomrule
\end{tabular}
\vspace{0.1cm}
\caption{GP value estimation: accuracy (mean and 95\% CI) for every $(d, \varepsilon, \sigma)$ configuration.}
\label{tab:gp-value-estimation-accuracy}
\end{table}

\begin{table}[htbp]
\centering
\small
\setlength{\tabcolsep}{2.5pt}
\renewcommand{\arraystretch}{1.08}
\begin{tabular}{@{}llrrrr@{}}
\toprule
$d$ & Method & \multicolumn{2}{c}{$\varepsilon=0.2$} & \multicolumn{2}{c}{$\varepsilon=0.15$} \\
\cmidrule(lr){3-4}\cmidrule(lr){5-6}
 &  & $\sigma=0.05$ & $\sigma=0.1$ & $\sigma=0.0375$ & $\sigma=0.075$ \\
\midrule
\multirow[t]{2}{*}{1} & C-ICPE-TD3 & \textbf{13.7} {\tiny [12.3,15.4]} & \textbf{22.7} {\tiny [21.8,23.8]} & \textbf{39.9} {\tiny [36.6,43.4]} & \textbf{43.7} {\tiny [41.3,45.8]} \\
 & C-ICPE-uniform & $31.7$ {\tiny [30.6,32.8]} & $41.0$ {\tiny [39.8,42.1]} & $41.6$ {\tiny [39.5,43.9]} & $49.0$ {\tiny [46.7,51.2]} \\
\bottomrule
\end{tabular}
\vspace{0.1cm}
\caption{GP value estimation: sample complexity (mean and 95\% CI) for every $(d, \varepsilon, \sigma)$ configuration.}
\label{tab:gp-value-estimation-sample-complexity}
\end{table}

\clearpage
\subsection{Synthetic Benchmarks: robustness }
When training \cicpe{}, the task prior $\nu$ is uniform over the
parameter space. We investigate robustness to prior misspecification
by evaluating frozen \cicpe{} models under Beta$(\alpha, \beta)$
deployment priors with varying concentration parameters.
\cref{tab:binary_search_robustness,tab:eps_best_arm_robustness,tab:ackley_robustness,tab:gp_val_robustness} report accuracy and
confidence intervals for all actor variants at the most challenging
$(\varepsilon, \sigma)$ configuration and highest dimensionality per
benchmark.

\subsubsection{Noisy Binary Search}

The training prior is $\mathrm{Uniform}[-1,1]^{20}$, corresponding
to $\alpha = \beta = 1$ (white box). In \cref{tab:binary_search_robustness} we report the results. \cicpets{} and \cicpettps{}
are robust across the full grid of Beta priors: accuracy remains
above $0.83$ in all configurations, even under substantial
distributional shift ($\alpha = 0.5, \beta = 7$ or vice versa).
Performance improves mildly when both $\alpha, \beta \geq 3$, since
concentrated priors place more mass in the interior of
$[-1,1]^d$, where localization is easier. \cicpetdthree{} matches
or exceeds the other actors when the deployment prior is
concentrated ($\alpha, \beta \geq 3$, reaching $0.943$), but
degrades sharply when either parameter is small: at $\alpha = 0.5,
\beta = 7$ accuracy drops to $0.612$. Small $\alpha$ or $\beta$
produces a U-shaped Beta distribution that concentrates mass near
the boundary of the domain, where targets are harder to
disambiguate and the learned actor generalizes poorly.
\cicpeuniform{} fails entirely in $d = 20$ (accuracy $< 0.01$),
confirming that passive exploration cannot localize a target in
high-dimensional binary search within the allowed horizon.

\begin{figure}[htbp]
\centering
    \begin{subfigure}{0.9\textwidth}
        \centering
        \includegraphics[width=\textwidth]{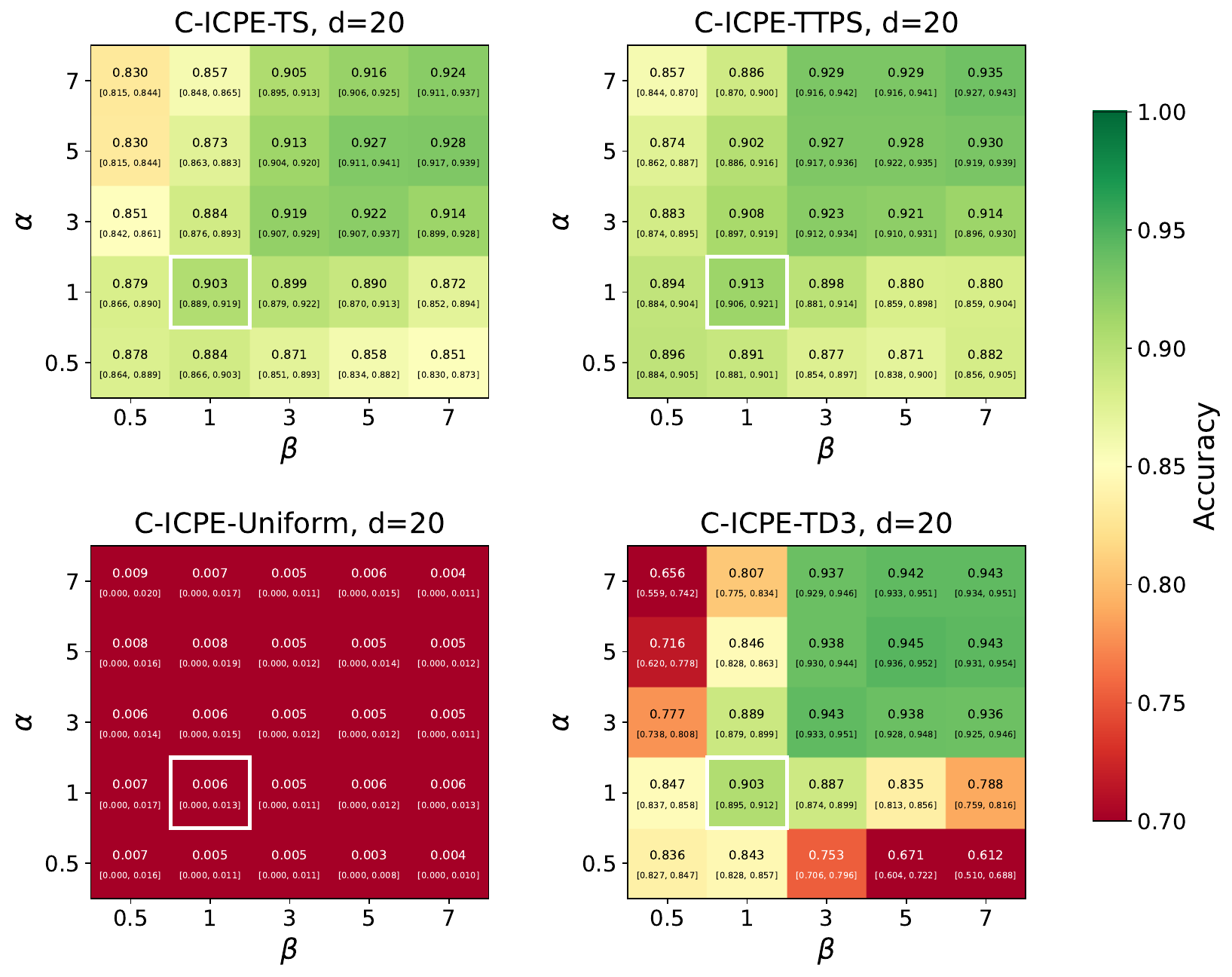}
    \end{subfigure}
\vspace{0.2cm}
\caption{Robustness to prior misspecification on the 20D noisy binary search problem ($\varepsilon=0.1$ and $\sigma=0.05$). Each heatmap reports the mean accuracy and the confidence intervals under varying Beta prior hyperparameters $(\alpha, \beta) \in \{0.5, 1, 3, 5, 7\}$. The white box indicates the matched prior ($\alpha=\beta=1$) during training.}
\label{tab:binary_search_robustness}
\end{figure}

\subsubsection{$\epsilon$-best arm problem}
We report results in \cref{tab:eps_best_arm_robustness}.
All four \cicpe{} variants are remarkably stable across the full
Beta prior grid at $d = 15$: accuracy varies by less than $0.04$
across all $(\alpha, \beta)$ configurations for each actor. This
robustness is a direct consequence of the rotational symmetry of the
problem. Since $\theta$ is drawn on $\mathbb{S}^{d-1}$ and the loss
$L_\theta(x) = 1 - \theta^\top x$ is invariant to orthogonal
transformations, the intrinsic difficulty of each task instance does
not depend on the location of $\theta$ on the sphere. Reweighting
the prior therefore has little effect on the distribution of problem
difficulty, unlike binary search or Ackley where boundary effects
create heterogeneous difficulty across the parameter space. Notably,
\cicpeuniform{} performs well here ($\approx 0.89$--$0.90$ uniformly),
consistent with the observation that isotropic exploration is optimal \citep{jedra2020optimal}. \cicpettps{} achieves the
highest accuracy overall ($0.92$--$0.94$), suggesting that the
challenger mechanism provides a modest benefit even in a setting
where uniform exploration is already near-optimal.

\begin{figure}[htbp]
\centering
    \begin{subfigure}{0.9\textwidth}
        \centering
        \includegraphics[width=\textwidth]{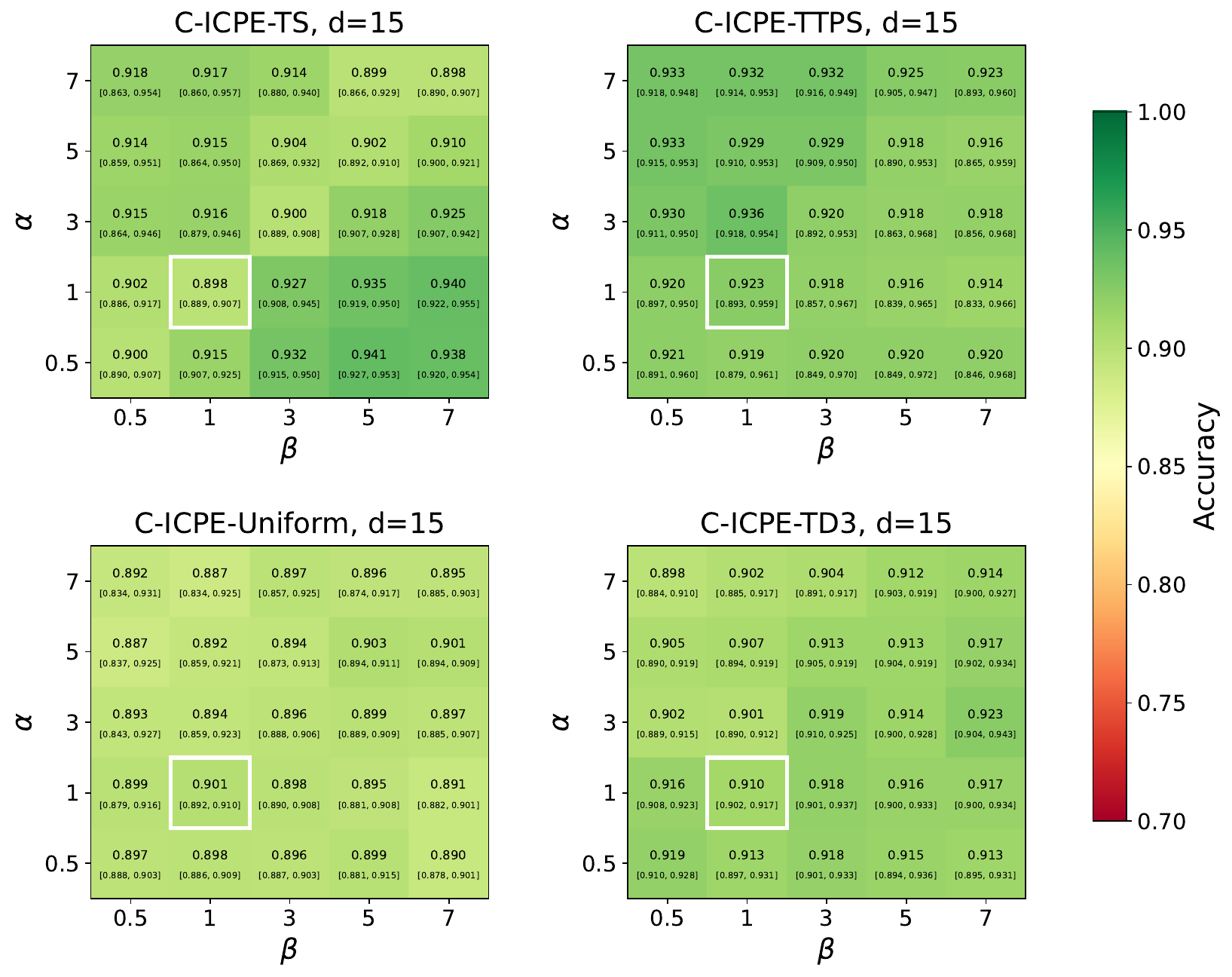}
    \end{subfigure}
\vspace{0.2cm}
\caption{Robustness to prior misspecification on the 15D $\varepsilon$-best-arm identification problem ($\varepsilon=0.005$ and $\sigma=0.0025$). Each heatmap reports the mean accuracy and the confidence intervals under varying Beta prior hyperparameters $(\alpha, \beta) \in \{0.5, 1, 3, 5, 7\}$. The white box indicates the matched prior ($\alpha=\beta=1$) during training.}
\label{tab:eps_best_arm_robustness}
\end{figure}

\subsubsection{Ackley minimization}

Results are reported in \cref{tab:ackley_robustness}.
The Ackley benchmark at $d = 5$ exhibits the strongest sensitivity
to prior misspecification among all tasks. At the training prior
($\alpha = \beta = 1$), \cicpets{} achieves $0.911$ and \cicpettps{}
$0.912$. When both $\alpha, \beta \geq 3$, i.e., the deployment
prior concentrates mass toward the interior of $[-1,1]^d$, all
active methods improve substantially, with \cicpets{} reaching
$0.975$ and \cicpetdthree{} $0.963$. The gains reflect the
structure of the Ackley function: targets near the center of the
domain sit in a region of higher curvature where observations are
more informative, making identification easier.

Conversely, when either $\alpha$ or $\beta$ is small ($0.5$), the
Beta prior becomes U-shaped or boundary-skewed, placing significant
mass on targets near the edges of $[-1,1]^d$. In the flat outer
region of the Ackley function, observations carry little signal, and
the learned policies, trained under a uniform prior that rarely
produces such extreme configurations, degrade. The effect is most
pronounced for $\alpha=0.5,\beta=7$. This asymmetry between interior and boundary
targets is specific to the Ackley geometry and is absent in the
rotationally symmetric $\varepsilon$-best arm problem.
\cicpeuniform{} fails across the board (accuracy $\leq 0.45$),
confirming that directed exploration is essential for this
multimodal benchmark regardless of the prior.

\begin{figure}[htbp]
\centering
    \begin{subfigure}{0.9\textwidth}
        \centering
        \includegraphics[width=\textwidth]{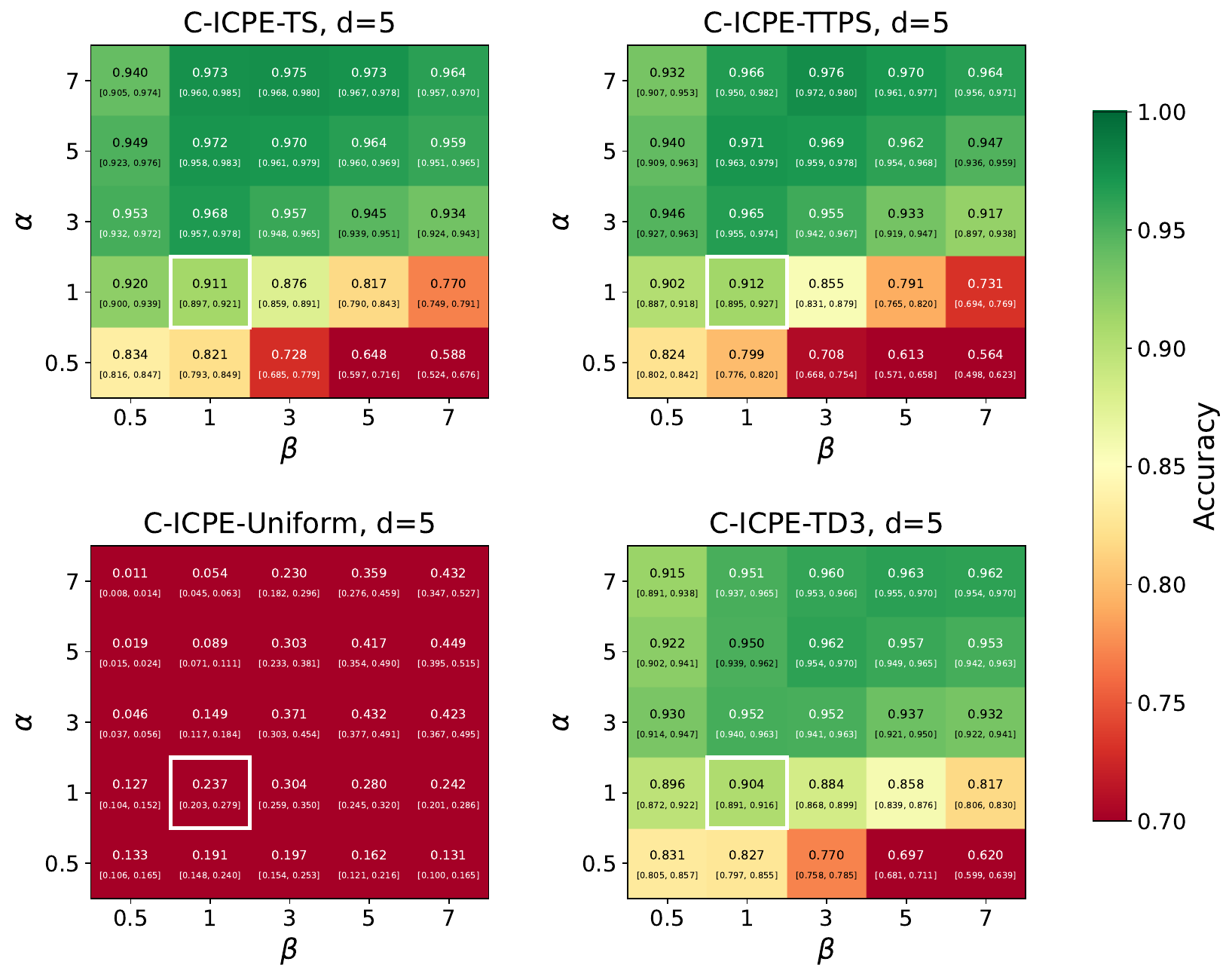}
    \end{subfigure}
\vspace{0.2cm}
\caption{Robustness to prior misspecification on the 5D Ackley function ($\varepsilon=0.1$ and $\sigma=0.05$). Each heatmap reports the mean accuracy and the confidence intervals under varying Beta prior hyperparameters $(\alpha, \beta) \in \{0.5, 1, 3, 5, 7\}$. The white box indicates the matched prior ($\alpha=\beta=1$) during training.}
\label{tab:ackley_robustness}
\end{figure}

\subsubsection{GP max-value estimation}
We report results in \cref{tab:gp_val_robustness}.
This is the $\mathcal{X} \neq \mathcal{A}$ setting ($d = 1$), so
only \cicpeuniform{} and \cicpetdthree{} are applicable.
\cicpetdthree{} achieves the highest accuracy when the deployment
prior is well-matched or concentrated with small $\beta$: it reaches
$0.992$ at $(\alpha, \beta) = (7, 0.5)$ and remains above $0.94$
throughout the upper-left triangle of the grid. However, it degrades
when $\beta$ is large ($0.772$ at $\alpha = 0.5, \beta = 7$;
$0.817$ at $\alpha = 1, \beta = 7$), indicating that the learned
exploration policy is sensitive to prior shifts that alter the
distribution of task difficulty. \cicpeuniform{}, by contrast, is
more robust: accuracy stays between $0.837$ and $0.965$ across the
entire grid, with a milder gradient from upper-left to lower-right.
Because the uniform actor does not depend on a learned exploration
policy, its performance varies only through the stopping criterion
and inference model, both of which appear stable under moderate
prior shift. At the training prior ($\alpha = \beta = 1$), the two
methods are comparable ($0.920$ vs.\ $0.913$), but \cicpetdthree{}
offers a clear advantage when the deployment prior concentrates mass
on tasks where directed exploration helps most ($\alpha \geq 3$,
$\beta \leq 1$). The overall pattern suggests that the TD3 actor
learns an exploration strategy well-adapted to the training
distribution but with limited extrapolation to deployment priors
that shift the typical task structure.
\begin{figure}[t]
\centering
    \begin{subfigure}{0.9\textwidth}
        \centering
        \includegraphics[width=\textwidth]{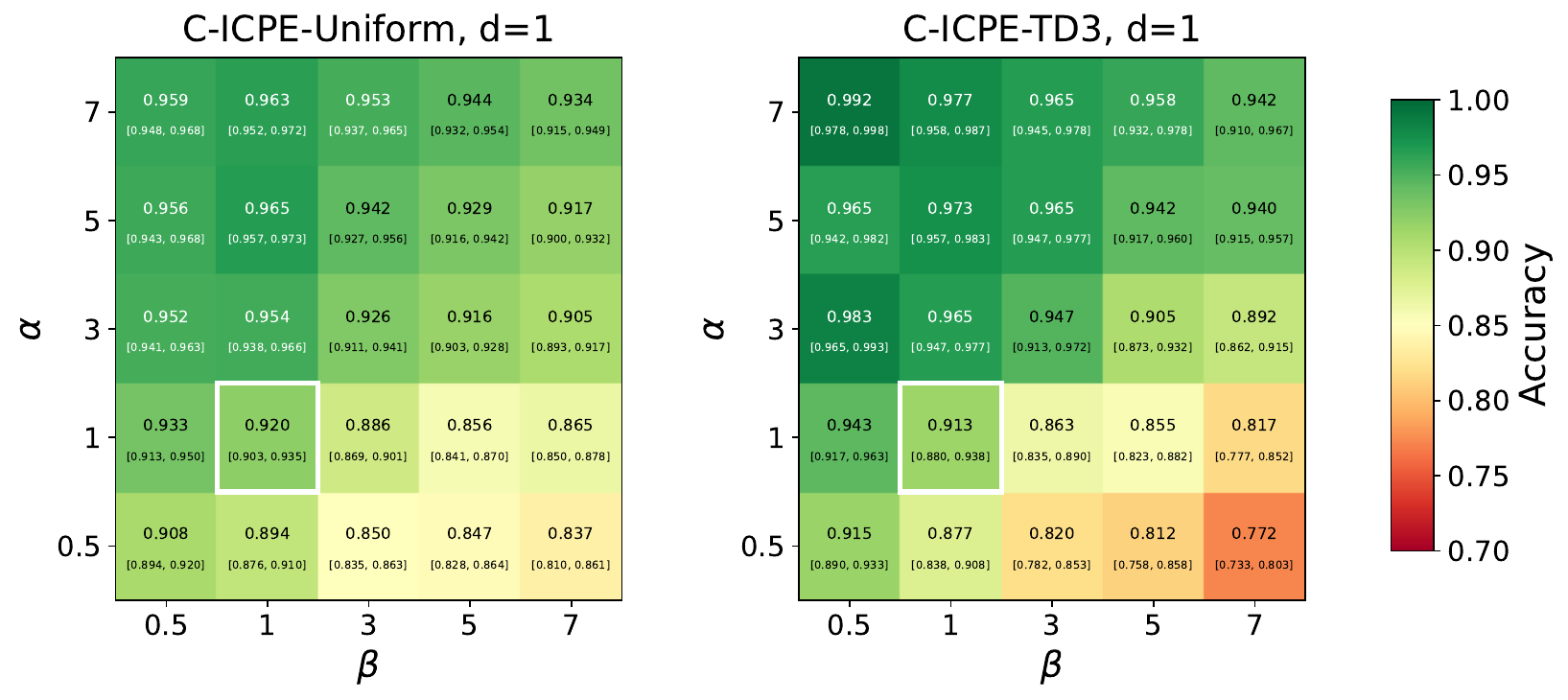}
    \end{subfigure}
\vspace{0.2cm}
\caption{Robustness to prior misspecification on the GP value estimation problem ($\varepsilon=0.15$ and $\sigma=0.075$). Each heatmap reports the mean accuracy and the confidence intervals under varying Beta prior hyperparameters $(\alpha, \beta) \in \{0.5, 1, 3, 5, 7\}$. The white box indicates the matched prior ($\alpha=\beta=1$) during training.}
\label{tab:gp_val_robustness}
\end{figure}

\clearpage
\subsection{Geochemical Exploration: Experimental Details and Numerical Results}\label{app:geochem}

\begin{figure}[b]
   \centering
         \includegraphics[width=.35\linewidth]{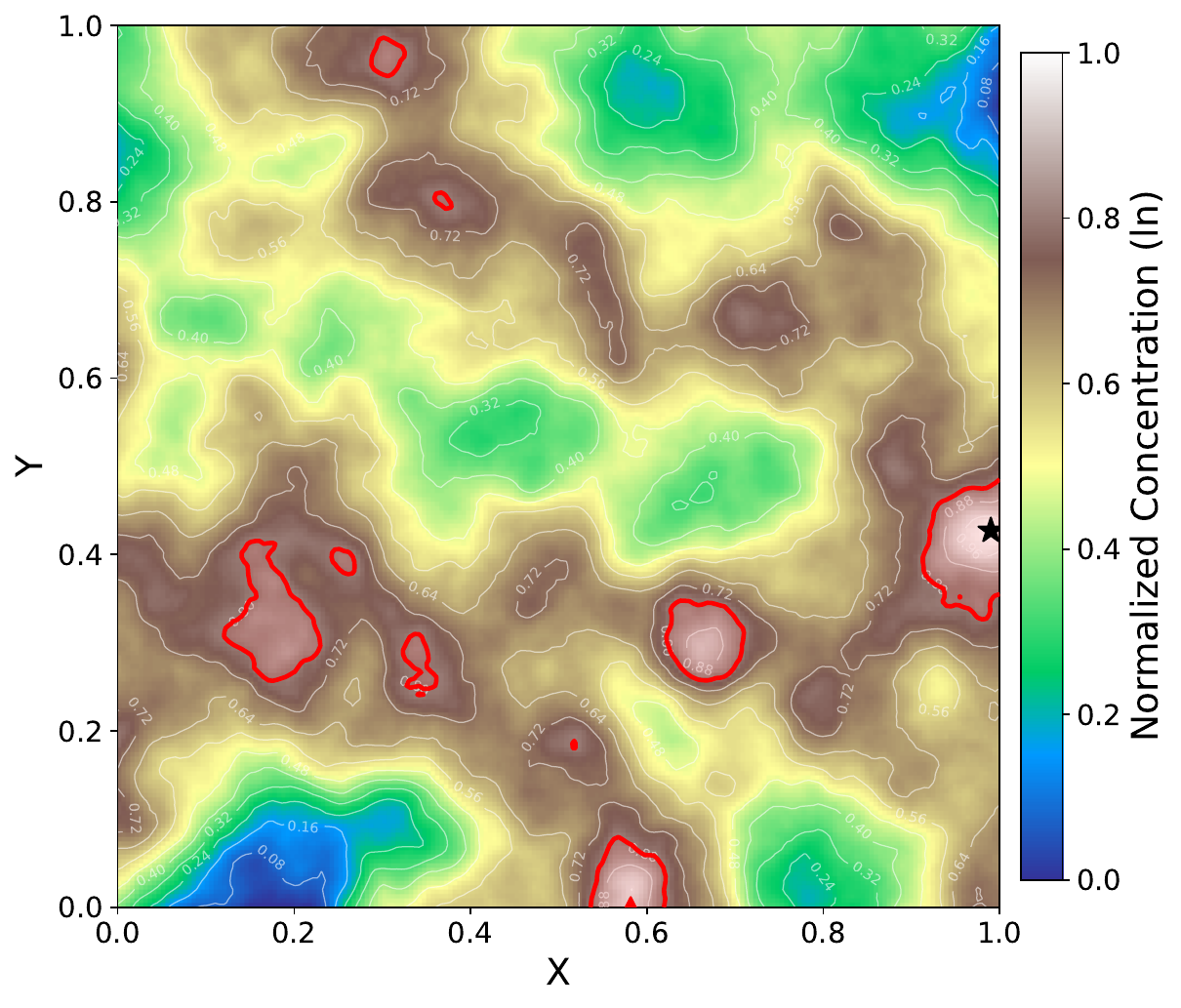}
         \caption{Example copper concentration in a 2D region in the geochemical exploration task. Red regions indicate concentration of copper within $\epsilon$ of the maximum value.}
         \label{fig:placeholder_2}
\end{figure}

The geochemical exploration task is a stylized version of a real problem in mineral exploration: given a budget of field samples (each requiring physical collection, transport, and laboratory analysis), identify the most promising location for further investigation. In practice, each sample costs hundreds to thousands of dollars, and field campaigns are logistically constrained. A method that can identify the target location with $(\epsilon,\delta)$-guarantees while minimizing the number of samples has direct economic value.
\subsubsection{Dataset and Motivation}

We use data from the USGS National Geochemical Survey \citep{usgs_geochem}, which provides soil and sediment measurements of element concentrations across the continental United States. We focus on copper (Cu) concentrations, which are of direct interest in mineral exploration: copper deposits are spatially heterogeneous, and identifying regions of peak concentration from sparse, noisy field measurements is a costly sequential decision problem. 

The dataset contains point measurements at irregularly spaced locations, each reporting the concentration of multiple elements. We extract copper concentration values and apply a log-transform followed by z-score normalization per region, yielding standardized log-concentrations that serve as observations.

As a concrete example, \cref{fig:kingman_region} shows the Kingman region in southeastern California and southern Nevada. This region contains the Mountain Pass rare earth mine (35.5$^\circ$N, 115.5$^\circ$W), an open-pit mine of rare earth elements in the Mojave Desert. A satellite image of the mine and surrounding terrain, acquired by the Advanced Spaceborne Thermal Emission and Reflection Radiometer (ASTER) instrument on NASA's Terra spacecraft on March 28, 2010, is shown in \cref{fig:mountain_pass_satellite} \citep{nasa_mountain_pass}. The region exhibits spatially varying copper concentrations with a clear peak near the mining district, making it a representative example of the kind of localization problem \cicpe{} is designed to solve.

\subsubsection{Region Partitioning and GP Fitting}

We partition the geochemical survey data into 17 geographic regions, each covering approximately $1^\circ \times 2^\circ$ in latitude and longitude. The regions span diverse geological settings across the western and southeastern United States: Gadsden, Bozeman, Billings, Wells, Needles, Jenkins, Montgomery, Millett, Prescott, Lovelock, Aurora, Holbrook, Atlanta, Ely, Kingman, Winnemucca, and Baker. \cref{fig:region_grid} shows all 17 regions with sample locations colored by normalized log-copper concentration. The regions vary substantially in sample density (from ${\sim}50$ to ${\sim}500$ measurements), spatial structure, and concentration range, providing a diverse task distribution for meta-training.

For each region, we fit a sparse variational Gaussian process (SVGP) with a Mat\'ern-3/2 ARD kernel and Gaussian likelihood. The model is:
\begin{equation}
    f \sim \mathrm{GP}(0, \sigma_f^2 \, k_{\mathrm{Mat\acute{e}rn\text{-}3/2}}(\cdot, \cdot \,; \ell_1, \ell_2)), \qquad y_i = f(\mathbf{s}_i) + \xi_i, \quad \xi_i \sim \mathcal{N}(0, \sigma_n^2),
\end{equation}
where $\mathbf{s}_i \in \mathbb{R}^2$ are UTM coordinates (normalized to $[0,1]^2$ for numerical stability), $y_i$ is the standardized log-copper concentration, and $(\ell_1, \ell_2, \sigma_f, \sigma_n)$ are per-region hyperparameters learned by maximizing the variational evidence lower bound (ELBO). We use $M = 500$ inducing points initialized via $k$-means clustering of the observation locations, and optimize for 1000 Adam iterations at learning rate $0.01$.

\cref{fig:elbo_curves} shows the ELBO training curves for all 17 regions. All regions converge smoothly, confirming that the SVGP fits are well-behaved. The fitted hyperparameters, in particular the  lengthscales $(\ell_1, \ell_2)$, vary across regions, reflecting different spatial correlation structures: some regions exhibit short-range variability (small lengthscales) while others have smoother concentration surfaces (large lengthscales).

\subsubsection{Ground Truth Construction}

For each fitted SVGP, we evaluate the posterior mean on a dense $200 \times 200$ grid over the normalized domain $[0,1]^2$. The ground truth target is defined as:
\begin{equation}
    \theta^\star = \arg\max_{\mathbf{s} \in \mathrm{grid}} \mu_{\mathrm{GP}}(\mathbf{s}),
\end{equation}
i.e., the grid location with the highest posterior mean copper concentration. 

\subsubsection{Task Prior and Train/Test Split}

The 17 regions are split into training and evaluation sets. Training regions define the task prior $\nu$: during meta-training, each episode samples a region uniformly from the training set and presents \cicpe{} with the corresponding fitted GP as the unknown function. The agent queries 2D locations $a \in [0,1]^2$ and observes noisy evaluations $y = \mu_{\mathrm{GP}}(a) + \xi$, $\xi \sim \mathcal{N}(0, \sigma_n^2)$, where $\sigma_n$ is the fitted noise standard deviation for that region. The goal is to identify the location $\theta^\star$ of peak concentration to within $\epsilon$ with probability at least $1 - \delta$.

Evaluation is performed on held-out regions whose spatial structure, lengthscales, and concentration patterns were not seen during training. This tests two properties simultaneously:
\begin{enumerate}
    \item \textbf{$(\epsilon,\delta)$-correct identification on realistic functions.} The GP posterior means are spatially structured, non-stationary (due to irregular sampling), and vary in smoothness across regions, a substantial departure from the synthetic benchmarks.
    \item \textbf{Robustness to distribution shift.} The evaluation regions have different hyperparameters $(\ell_1, \ell_2, \sigma_f, \sigma_n)$ from the training regions, so the agent must generalize across spatial correlation structures it has not encountered during meta-training.
\end{enumerate}

\begin{figure}[htbp]
    \centering
    \includegraphics[width=0.45\textwidth]{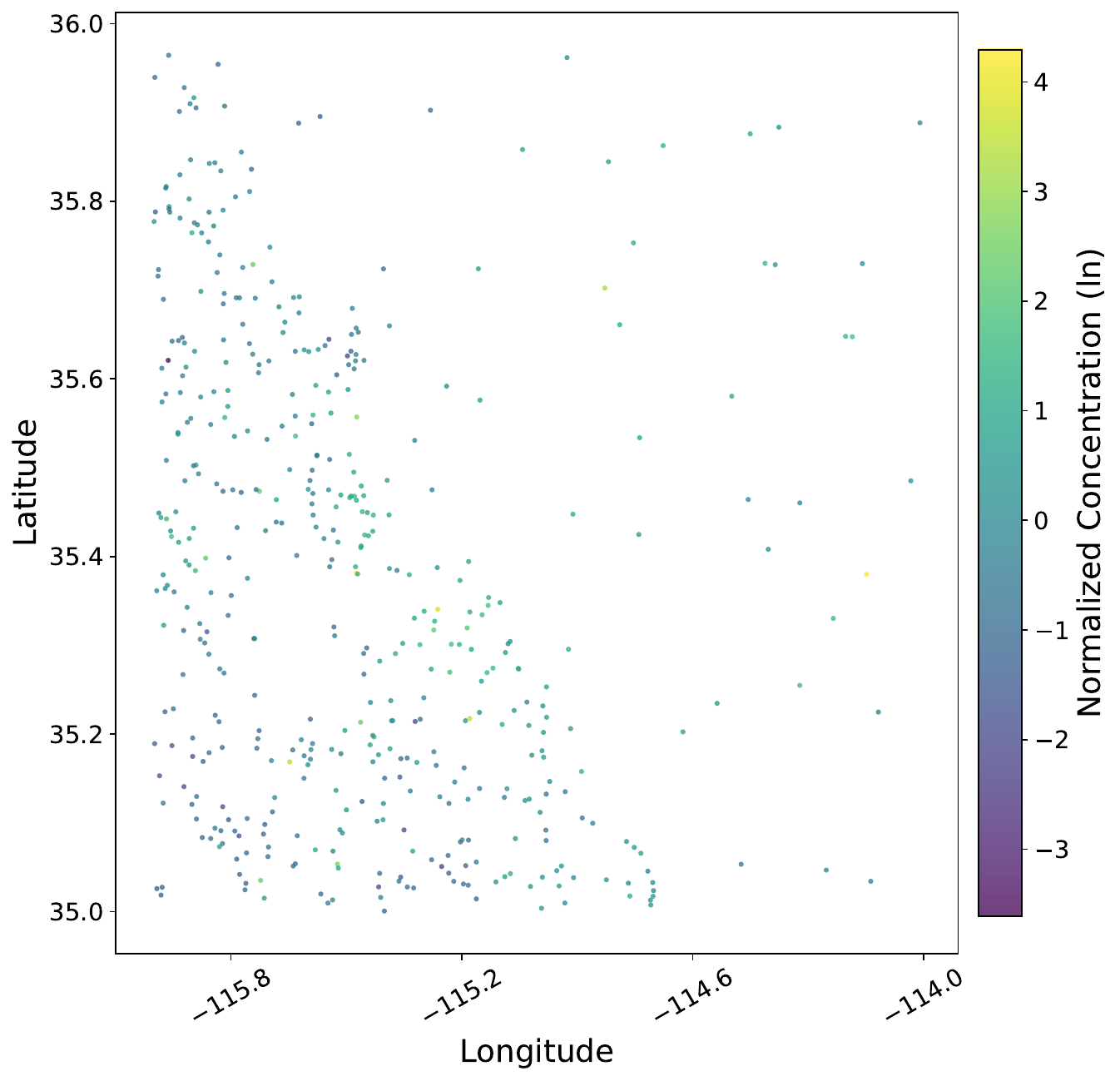}
    \hfill
    \includegraphics[width=0.42\textwidth]{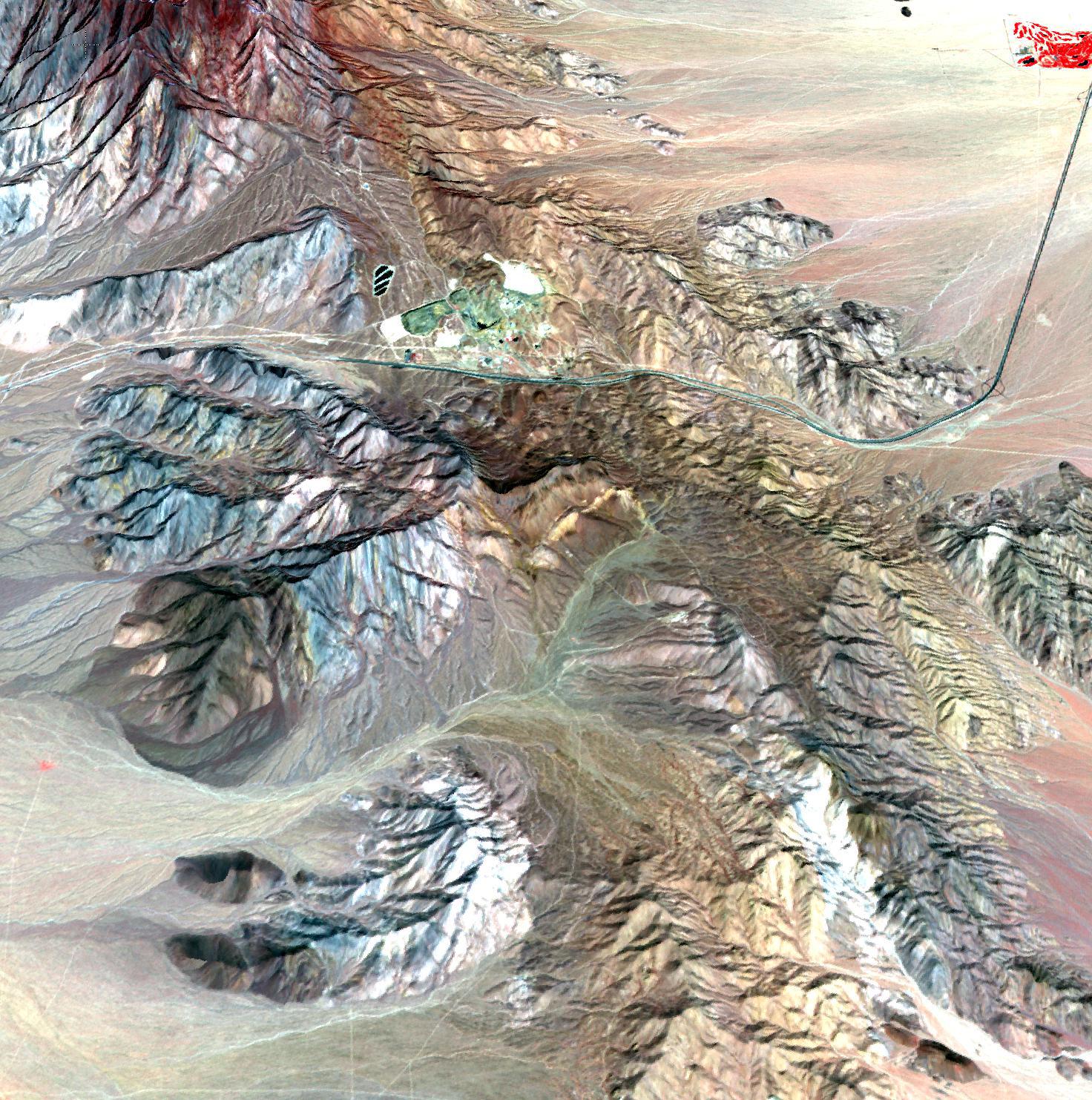}
    \caption{\textbf{Left:} Normalized log-copper concentration in the Kingman region (southeastern California / southern Nevada). Each point is a soil sample from the USGS National Geochemical Survey \citep{usgs_geochem}; the red star marks the location of peak GP posterior mean. \textbf{Right:} ASTER satellite image of the Mountain Pass rare earth mine (35.5$^\circ$N, 115.5$^\circ$W) within this region, acquired March 28, 2010. The mine area is visible as the light-colored open pit in the upper center of the image. Credit: NASA/GSFC/METI/ERSDAC/JAROS, and U.S./Japan ASTER Science Team \citep{nasa_mountain_pass}.}
    \label{fig:kingman_region}
    \label{fig:mountain_pass_satellite}
\end{figure}

\begin{figure}[htbp]
    \centering
    \includegraphics[width=0.8\textwidth]{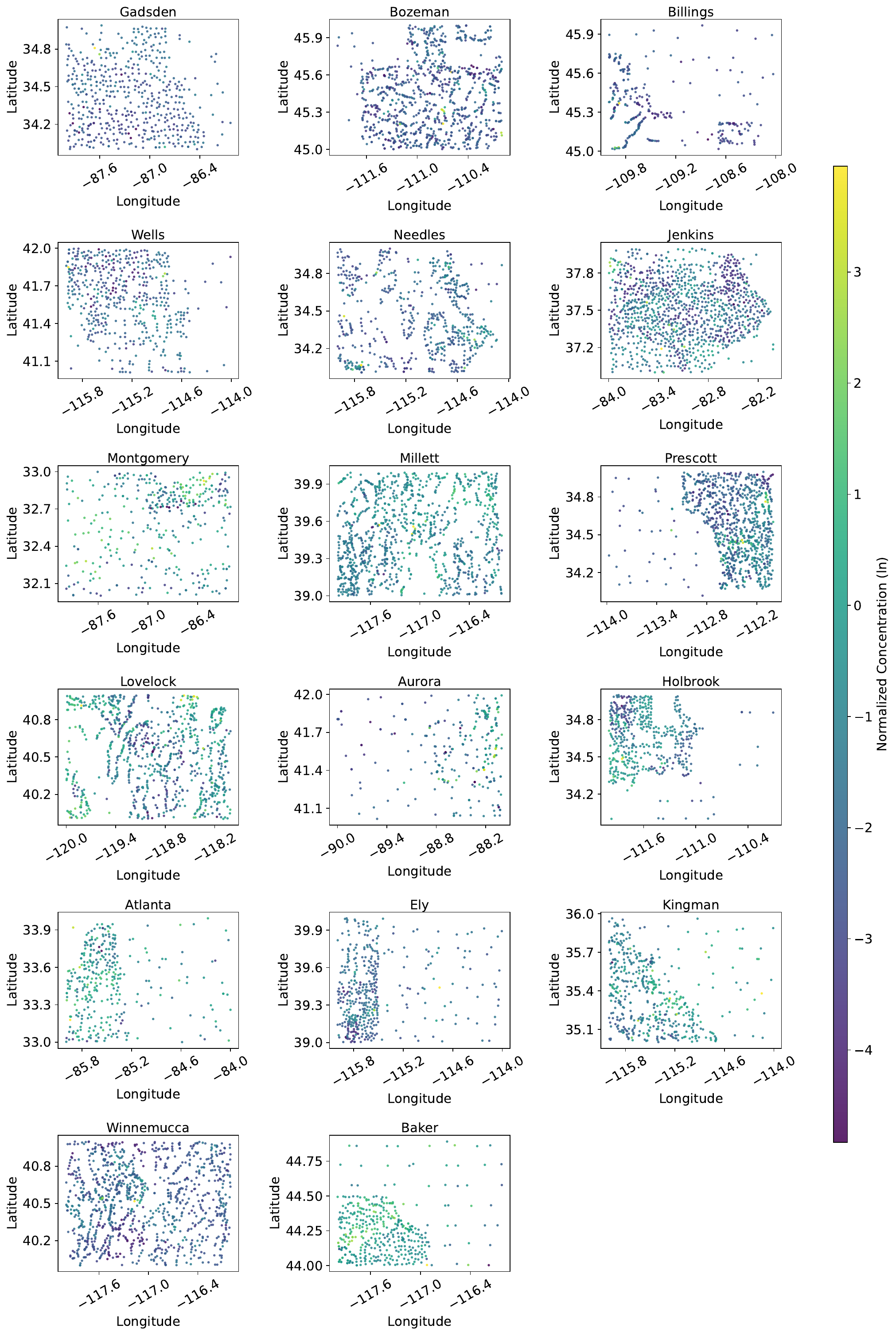}
    \caption{All 17 geographic regions used in the geochemical experiment. Each panel shows soil sample locations colored by normalized log-copper concentration. Regions are split into training and evaluation sets; evaluation regions have spatial structure not seen during meta-training.}
    \label{fig:region_grid}
\end{figure}

\begin{figure}[htbp]
    \centering
    \includegraphics[width=0.8\textwidth]{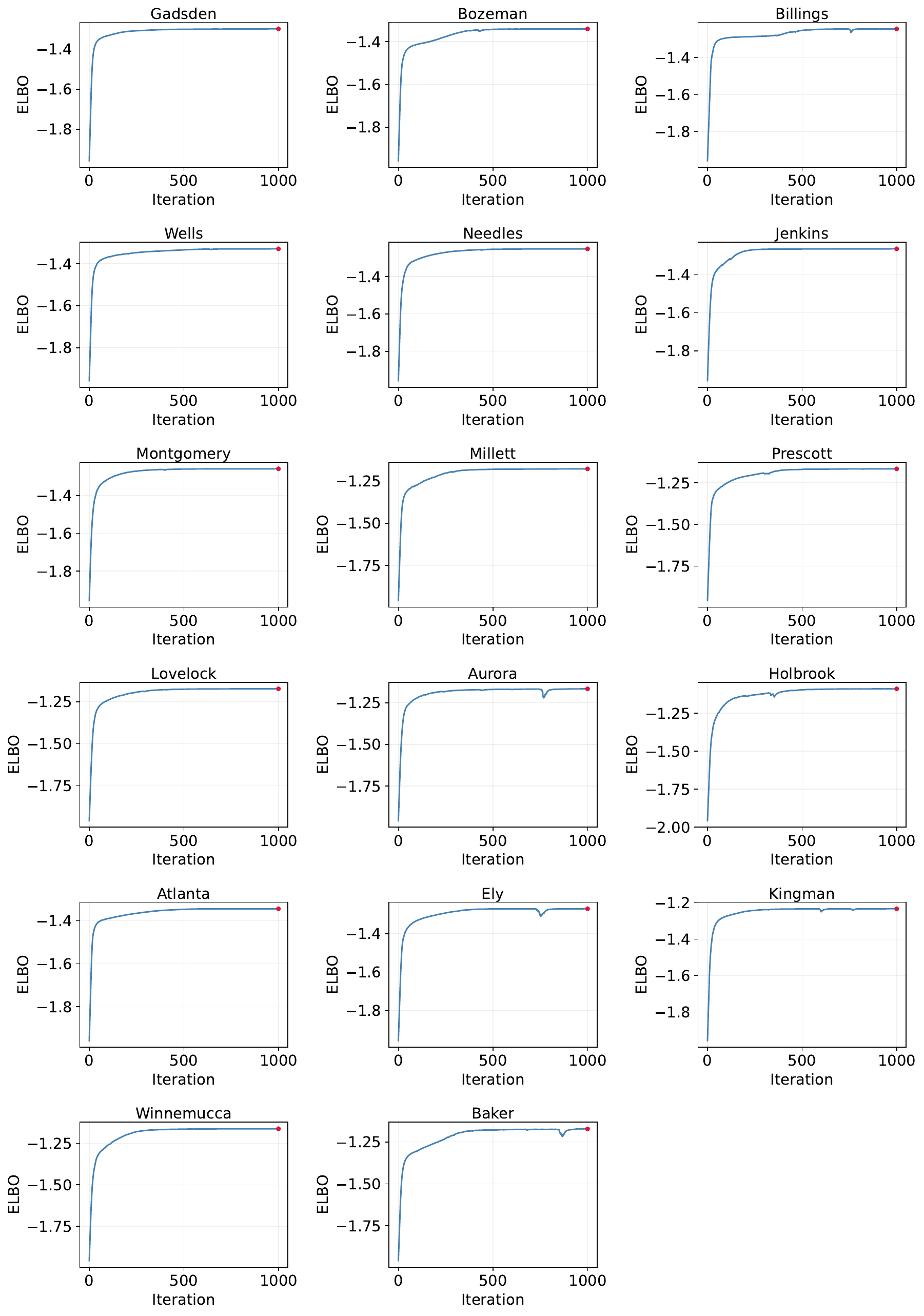}
    \caption{ELBO training curves for sparse variational GP fits across all 17 regions (1000 Adam iterations, $M=500$ inducing points). All regions converge smoothly, indicating well-behaved GP fits.}
    \label{fig:elbo_curves}
\end{figure}

\clearpage
\subsubsection{Numerical Results}

\paragraph{Baselines.}
We compare C-ICPE-TD3 and C-ICPE-TS against four black-box optimization
baselines that do not possess a stopping rule for $(\varepsilon,\delta)$-correct
identification: TPE~\cite{bergstra2011algorithms},
CMA-ES~\cite{hansen2016cma}, GP-logEI~\cite{ament2023unexpected}, and
GP-UCB~\cite{srinivas2010gaussian}. Since these methods optimize a
fixed-budget objective and have no principled mechanism for adaptive
stopping, we allocate each a fixed sample budget calibrated to the mean
sample complexity of C-ICPE-TS: $N=22$ for $\varepsilon=0.2$ and $N=39$
for $\varepsilon=0.15$. After exhausting this budget, each baseline returns
the best-observed location as its recommendation. This protocol is
deliberately generous: the baselines receive as many samples as
\texttt{C-ICPE-TS} typically needs on average, yet bear no cost for
deciding when to stop. 

\begin{table}[htbp]
\centering
\small
\setlength{\tabcolsep}{2.5pt}
\renewcommand{\arraystretch}{1.08}
\begin{tabular}{@{}llrr@{}}
\toprule
 &  & $\varepsilon=0.2$ & $\varepsilon=0.15$ \\
 &  & $\sigma=0$ & $\sigma=0$ \\
\midrule
\multirow[t]{6}{*}{2} & C-ICPE-TD3 & \textbf{0.913} {\tiny [.894,.931]} & $0.916$ {\tiny [.904,.927]} \\
 & C-ICPE-TS & $0.913$ {\tiny [.894,.930]} & \textbf{0.925} {\tiny [.905,.944]} \\
 & TPE & $0.438$ {\tiny [.394,.482]} & $0.526$ {\tiny [.482,.570]} \\
 & CMA-ES & $0.484$ {\tiny [.440,.528]} & $0.514$ {\tiny [.470,.558]} \\
 & GP-logEI & $0.564$ {\tiny [.520,.608]} & $0.724$ {\tiny [.685,.763]} \\
 & GP-UCB & $0.730$ {\tiny [.691,.769]} & $0.720$ {\tiny [.681,.759]} \\
\bottomrule
\end{tabular}
\vspace{0.1cm}
\caption{Geochem: accuracy (mean and 95\% CI) for every $(d, \varepsilon, \sigma)$ configuration.}
\label{tab:geochem-accuracy}
\end{table}

\begin{table}[htbp]
\centering
\small
\setlength{\tabcolsep}{2.5pt}
\renewcommand{\arraystretch}{1.08}
\begin{tabular}{@{}llrr@{}}
\toprule
 &  & $\varepsilon=0.2$ & $\varepsilon=0.15$ \\
 &  & $\sigma=0$ & $\sigma=0$ \\
\midrule
\multirow[t]{2}{*}{2} & C-ICPE-TD3 & $23.9$ {\tiny [21.6,26.2]} & $44.4$ {\tiny [42.4,46.5]} \\
 & C-ICPE-TS & \textbf{22.0} {\tiny [20.0,23.9]} & \textbf{38.7} {\tiny [34.2,43.5]} \\
\bottomrule
\end{tabular}
\vspace{0.1cm}
\caption{Geochem: sample complexity (mean and 95\% CI) for every $(d, \varepsilon, \sigma)$ configuration.}
\label{tab:geochem-sample-complexity}
\end{table}

\paragraph{Accuracy and sample complexity.}
Table~\ref{tab:geochem-accuracy} reports identification accuracy (mean and
95\% CI over held-out regions) at the two tolerance levels, with
$\delta = 0.1$ and maximum horizon $N = 150$. Both C-ICPE variants
exceed the $1 - \delta = 0.90$ correctness target in every
configuration. At $\varepsilon = 0.2$, C-ICPE-TD3 and C-ICPE-TS are
tied at $0.913$ $[.894, .931]$ and $0.913$ $[.894, .930]$,
respectively. At the harder $\varepsilon = 0.15$, C-ICPE-TS pulls
ahead with $0.925$ $[.905, .944]$ versus $0.916$ $[.904, .927]$ for
C-ICPE-TD3. Among the baselines, GP-UCB is the strongest, reaching
$0.730$ $[.691, .769]$ at $\varepsilon = 0.2$ and $0.720$
$[.681, .759]$ at $\varepsilon = 0.15$, still roughly $19$--$20$
percentage points below C-ICPE despite receiving a comparable sample
budget. GP-logEI performs comparably to GP-UCB at $\varepsilon = 0.15$
($0.724$) but falls to $0.564$ at $\varepsilon = 0.2$. TPE and CMA-ES
remain below $0.53$ in both settings, indicating that gradient-free
search without a surrogate is ineffective on these spatially structured
surfaces.

Table~\ref{tab:geochem-sample-complexity} reports sample complexity for the two
C-ICPE variants. C-ICPE-TS is more sample-efficient in both settings:
$22.0$ $[20.0, 23.9]$ versus $23.9$ $[21.6, 26.2]$ at
$\varepsilon = 0.2$, and $38.7$ $[34.2, 43.5]$ versus $44.4$
$[42.4, 46.5]$ at $\varepsilon = 0.15$. The gap widens at the tighter
tolerance, suggesting that Thompson sampling's implicit exploration
adapts more efficiently to the difficulty of each region.
\begin{figure}[t]
    \centering
    \begin{subfigure}{0.48\textwidth}
        \centering
        \includegraphics[width=\textwidth]{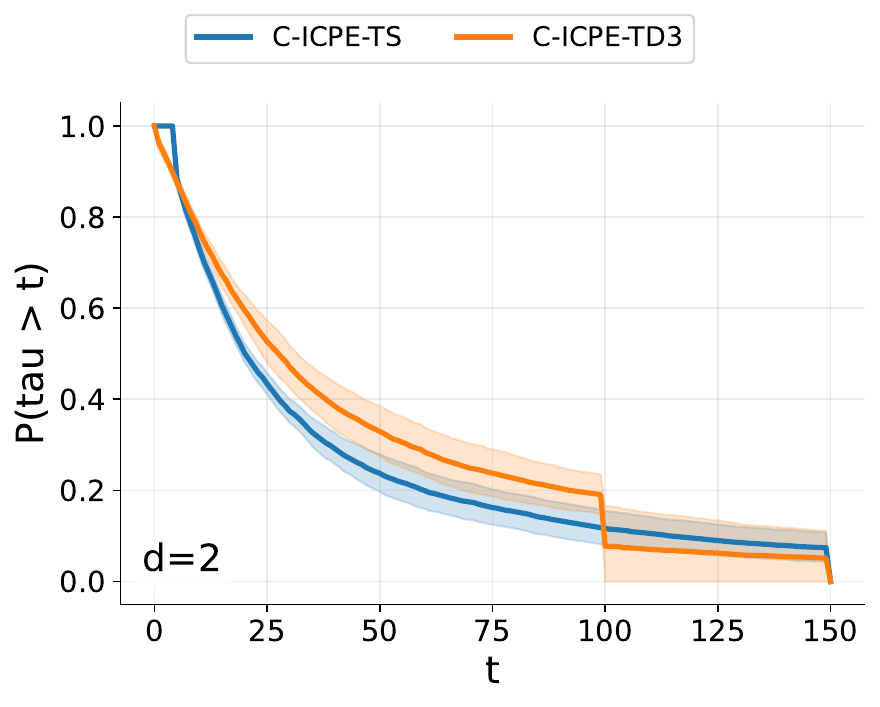}
        \caption{}
        \label{fig:geochem_subfig-a}
    \end{subfigure}
    \hfill
    \begin{subfigure}{0.48\textwidth}
        \centering
        \includegraphics[width=\textwidth]{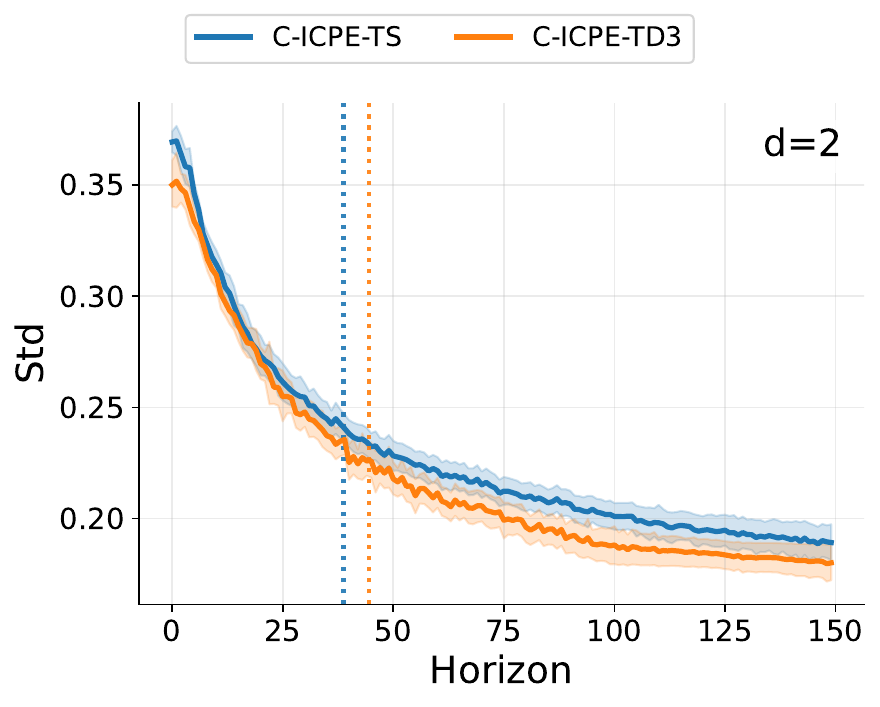}
        \caption{}
        \label{fig:geochem_subfig-b}
    \end{subfigure}
    \caption{Results for geochemical problem with fixed confidence $\delta=0.1$ and $N=150$ across different dimensions at the most hardest $\varepsilon$ setting: (a) survival function of $\tau$; (b) inference uncertainty convergence.}
    \label{fig:geochemical_survival_function_std}
\end{figure}

\paragraph{Survival function and uncertainty convergence.}
Figure~\ref{fig:geochem_subfig-a} displays the survival function
$\mathbb{P}(\tau > t)$ at the hardest setting ($\varepsilon = 0.15$,
$d = 2$). Both variants exhibit a rapid initial decline. The shaded confidence bands for C-ICPE-TD3 are
noticeably slightly wider, reflecting higher variance in stopping times, consistent
with the wider confidence interval in Table~\ref{tab:geochem-sample-complexity}. 

Figure~\ref{fig:geochem_subfig-b} shows the posterior standard
deviation of the recommendation as a function of the horizon. Both
methods converge to approximately $\mathrm{Std} \approx 0.19$ by
$t = 150$. The vertical dashed lines mark each method's median
stopping time; C-ICPE-TS stops earlier than C-ICPE-TD3, and at both
stopping points the standard deviation has already dropped below $0.25$.
This confirms that C-ICPE-TS stops at a point where the inference
model's posterior is sufficiently concentrated, rather than stopping
prematurely. However, it's interesting to note that the posterior variance has an overall larger decrease with TD3, confirming that C-ICPE-TD3 is learning a good explorative policy. The slightly larger stopping time may then be due to the fact that training was stopped early, and one could have trained for longer for better performance of C-ICPE-TD3. 

\paragraph{Robustness to prior misspecification.}
Figure~\ref{tab:geochem_robustness} reports accuracy under misspecified Beta
priors $(\alpha, \beta) \in \{0.5, 1, 3, 5, 7\}^2$ on the normalized
$[0,1]^2$ domain, with $\varepsilon = 0.15$. The matched prior used
during meta-training corresponds to $\alpha = \beta = 1$ (uniform).
C-ICPE-TS maintains accuracy between $0.90$ and $0.93$ across the
entire $5 \times 5$ grid, with no discernible degradation even at
extreme configurations such as $(\alpha, \beta) = (0.5, 7)$ or
$(7, 0.5)$. C-ICPE-TD3 is similarly stable in the upper portion of the
grid ($\alpha \geq 3$), but shows a mild decline to $0.90$ at
$(\alpha, \beta) = (0.5, 5)$ and $(0.5, 7)$. Across all 25
configurations, every cell remains at or above $0.90$, satisfying the
$1 - \delta$ correctness target. This degree of robustness is notably
stronger than what is observed on the Ackley benchmark
(Figure~\ref{tab:ackley_robustness}), where boundary-skewed priors cause
substantial degradation.

\begin{figure}[htbp]
\centering
    \begin{subfigure}{0.9\textwidth}
        \centering
        \includegraphics[width=\textwidth]{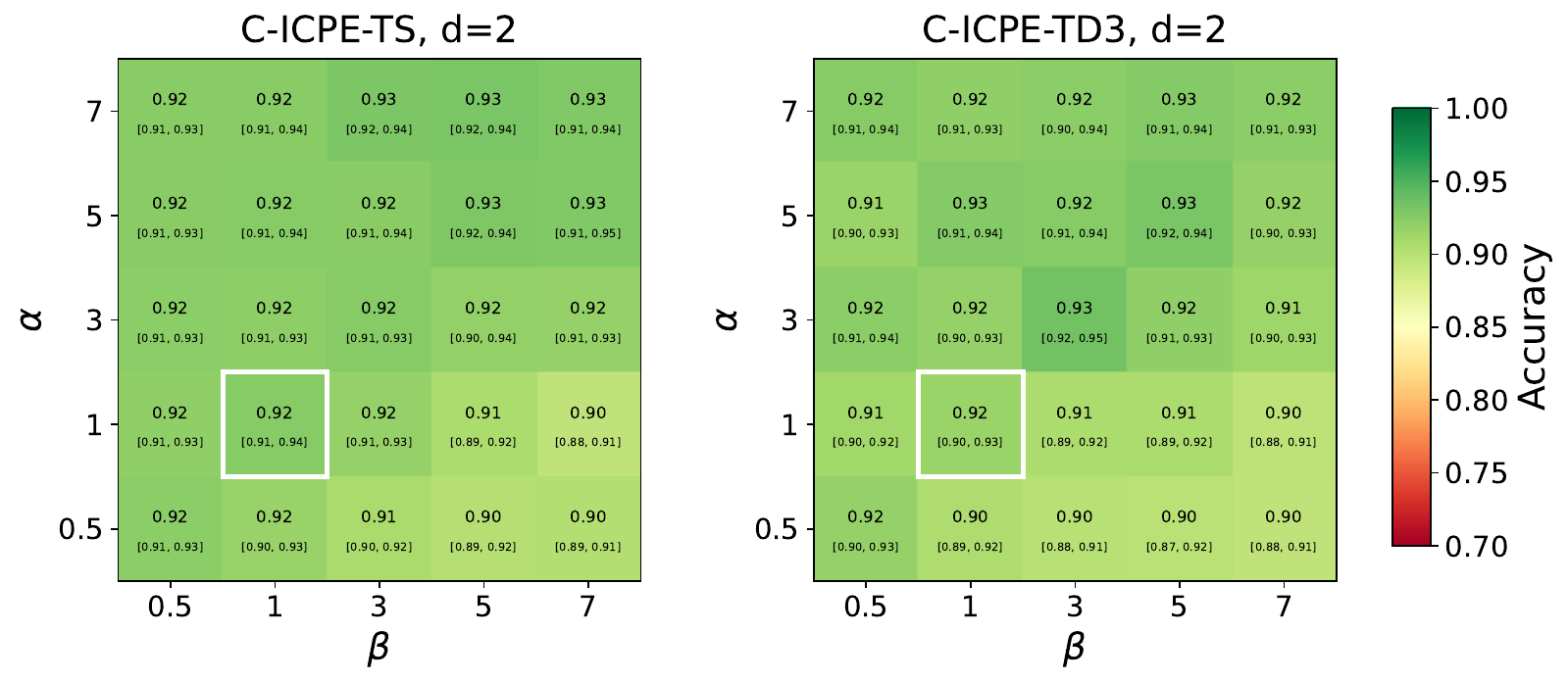}
    \end{subfigure}
\vspace{0.2cm}
\caption{Robustness to prior misspecification on the Geochemical exploration ($\varepsilon=0.15$). Each heatmap reports the mean accuracy and the confidence intervals under varying Beta prior hyperparameters $(\alpha, \beta) \in \{0.5, 1, 3, 5, 7\}$. The white box indicates the matched prior ($\alpha=\beta=1$) during training.}
\label{tab:geochem_robustness}
\end{figure}

Taken together, these results demonstrate that C-ICPE transfers to a
real-data task involving genuine distribution shift: the evaluation
regions have spatial correlation structures, lengthscales, and noise
levels not seen during meta-training, yet both variants maintain
$(\varepsilon, \delta)$-correctness while using fewer samples than
fixed-budget baselines that fail to meet the accuracy target. This
validates \cicpe{} as a practical tool for sequential experimental design.

%% file: ref.bib
@inproceedings{takemori2025instance,
  title={Instance-Optimal Pure Exploration for Linear Bandits on Continuous Arms},
  author={Takemori, Sho and Umeda, Yuhei and Gopalan, Aditya},
  booktitle={Forty-second International Conference on Machine Learning},
  year={2025}
}

@inproceedings{shang2020ttps,
  title={Fixed-confidence guarantees for bayesian best-arm identification},
  author={Shang, Xuedong and Heide, Rianne and Menard, Pierre and Kaufmann, Emilie and Valko, Michal},
  booktitle={International Conference on Artificial Intelligence and Statistics},
  pages={1823--1832},
  year={2020},
  organization={PMLR}
}

@inproceedings{srinivas2010gaussian,
  title={Gaussian process optimization in the bandit setting: no regret and experimental design},
  author={Srinivas, Niranjan and Krause, Andreas and Kakade, Sham and Seeger, Matthias},
  booktitle={Proceedings of the 27th International Conference on International Conference on Machine Learning},
  pages={1015--1022},
  year={2010}
}

@article{wilson2024stopping,
  title={Stopping Bayesian optimization with probabilistic regret bounds},
  author={Wilson, James T},
  journal={Advances in Neural Information Processing Systems},
  volume={37},
  pages={98264--98296},
  year={2024}
}

@misc{nasa_mountain_pass,
   author       = {{NASA/GSFC/METI/ERSDAC/JAROS, and U.S./Japan ASTER Science Team}},
   title        = {Mountain Pass Mine, California},
   year         = {2010},
   note         = {Image PIA13979, acquired March 28, 2010},
   url          = {https://www.jpl.nasa.gov/images/pia13979-mountain-pass-mine-california}
 }

@techreport{usgs_geochem,
  type = {Report},
  title = {The {{National Geochemical Survey}}: {{Database}} and Documentation},
  year = 2004,
  edition = {Version 1.0},
  number = {2004-1001},
  doi = {10.3133/ofr20041001},
  langid = {english}
}

@article{kaufmann2021mixture,
  title={Mixture martingales revisited with applications to sequential tests and confidence intervals},
  author={Kaufmann, Emilie and Koolen, Wouter M},
  journal={Journal of Machine Learning Research},
  volume={22},
  number={246},
  pages={1--44},
  year={2021}
}

@article{liu2017materials,
  title={Materials discovery and design using machine learning},
  author={Liu, Yue and Zhao, Tianlu and Ju, Wangwei and Shi, Siqi},
  journal={Journal of Materiomics},
  volume={3},
  number={3},
  pages={159--177},
  year={2017},
  publisher={Elsevier}
}

@misc{poiani2025Pure,
  title = {Pure {{Exploration}} with {{Infinite Answers}}},
  author = {Poiani, Riccardo and Bernasconi, Martino and Celli, Andrea},
  year = 2025,
  month = may,
  number = {arXiv:2505.22473},
  eprint = {2505.22473},
  primaryclass = {cs},
  publisher = {arXiv},
  doi = {10.48550/arXiv.2505.22473},
  urldate = {2026-02-12},
  archiveprefix = {arXiv}
}

@inproceedings{wang2021Fast,
  title = {Fast {{Pure Exploration}} via {{Frank-Wolfe}}},
  booktitle = {Advances in {{Neural Information Processing Systems}}},
  author = {Wang, Po-An and Tzeng, Ruo-Chun and Proutiere, Alexandre},
  year = 2021,
  volume = {34},
  pages = {5810--5821},
  publisher = {Curran Associates, Inc.},
  urldate = {2026-02-12}
}

@article{hantoute2008Characterizations,
  title = {Characterizations of the Subdifferential of the Supremum of Convex Functions},
  author = {Hantoute, A and L{\'o}pez, {\relax MA}},
  year = {2008},
  journal = {Journal of Convex Analysis},
  volume = {15},
  pages = {831--858}
}

@article{naghshvar2013Active,
  title = {Active {{Sequential Hypothesis Testing}}},
  author = {Naghshvar, Mohammad and Javidi, Tara},
  year = 2013,
  journal = {The Annals of Statistics},
  volume = {41},
  number = {6},
  eprint = {23566746},
  eprinttype = {jstor},
  pages = {2703--2738},
  publisher = {Institute of Mathematical Statistics},
  issn = {0090-5364},
  urldate = {2026-01-27}
}

@misc{russo_learning_2025,
    title = {Learning to {Explore}: {An} {In}-{Context} {Learning} {Approach} for {Pure} {Exploration}},
    shorttitle = {Learning to {Explore}},
    doi = {10.48550/arXiv.2506.01876},
    language = {en},
    urldate = {2025-07-12},
    publisher = {arXiv},
    author = {Russo, Alessio and Welch, Ryan and Pacchiano, Aldo},
    month = jun,
    year = {2025},
    note = {arXiv:2506.01876 [cs]},
}

@article{chernoff_sequential_1959,
    title = {Sequential design of experiments},
    volume = {30},
    doi = {10.1214/aoms/1177706205},
    number = {3},
    journal = {The Annals of Mathematical Statistics},
    author = {Chernoff, Herman},
    year = {1959},
    note = {Publisher: Institute of Mathematical Statistics},
    pages = {755 -- 770},
}

@article{rainforth2024modern,
    title = {Modern {Bayesian} experimental design},
    volume = {39},
    number = {1},
    journal = {Statistical Science},
    author = {Rainforth, Tom and Foster, Adam and Ivanova, Desi R and Bickford Smith, Freddie},
    year = {2024},
    note = {Publisher: Institute of Mathematical Statistics},
    pages = {100--114},
}

@article{garivier2016optimal,
  title = {Optimal Best Arm Identification with Fixed Confidence},
  author = {Garivier, Aur{\'e}lien and Kaufmann, Emilie},
  year = {2016},
  journal = {Proceedings of the 29th Conference on Learning Theory},
  series = {Proceedings of Machine Learning Research},
  volume = {49},
  pages = {998--1027},
}

@book{puterman2014markov,
    title = {Markov decision processes: discrete stochastic dynamic programming},
    publisher = {John Wiley \& Sons},
    author = {Puterman, Martin L},
    year = {2014},
}

@inproceedings{poiani_best-arm_2025,
    title = {Best-{Arm} {Identification} in {Unimodal} {Bandits}},
    language = {en},
    urldate = {2025-06-01},
    booktitle = {Proceedings of {The} 28th {International} {Conference} on {Artificial} {Intelligence} and {Statistics}},
    publisher = {PMLR},
    author = {Poiani, Riccardo and Jourdan, Marc and Kaufmann, Emilie and Degenne, Rémy},
    month = apr,
    year = {2025},
    note = {ISSN: 2640-3498},
    pages = {2233--2241},
}

@inproceedings{degenne_non-asymptotic_2019,
    title = {Non-asymptotic pure exploration by solving games},
    volume = {32},
    booktitle = {Advances in neural information processing systems},
    publisher = {Curran Associates, Inc.},
    author = {Degenne, Rémy and Koolen, Wouter M and Ménard, Pierre},
    editor = {Wallach, H. and Larochelle, H. and Beygelzimer, A. and dAlché-Buc, F. and Fox, E. and Garnett, R.},
    year = {2019},
}

@inproceedings{pmlr-v258-russo25a,
    series = {Proceedings of machine learning research},
    title = {Pure exploration with feedback graphs},
    volume = {258},
    booktitle = {Proceedings of the 28th international conference on artificial intelligence and statistics},
    publisher = {PMLR},
    author = {Russo, Alessio and Song, Yichen and Pacchiano, Aldo},
    year = {2025},
    pages = {1810--1818},
}

@inproceedings{russomulti,
    title = {Multi-reward best policy identification},
    volume = {37},
    booktitle = {Advances in neural information processing systems},
    author = {Russo, Alessio and Vannella, Filippo},
    year = {2024},
    pages = {105583--105662},
}

@inproceedings{russo_adaptive_2025,
    series = {Proceedings of machine learning research},
    title = {Adaptive exploration for multi-reward multi-policy evaluation},
    volume = {267},
    booktitle = {Proceedings of the 42nd international conference on machine learning},
    publisher = {PMLR},
    author = {Russo, Alessio and Pacchiano, Aldo},
    editor = {Singh, Aarti and Fazel, Maryam and Hsu, Daniel and Lacoste-Julien, Simon and Berkenkamp, Felix and Maharaj, Tegan and Wagstaff, Kiri and Zhu, Jerry},
    month = jul,
    year = {2025},
    pages = {52382--52421},
}

@article{garivier2021nonasymptotic,
    title = {Nonasymptotic sequential tests for overlapping hypotheses applied to near-optimal arm identification in bandit models},
    volume = {40},
    number = {1},
    journal = {Sequential Analysis},
    author = {Garivier, Aurélien and Kaufmann, Emilie},
    year = {2021},
    note = {Publisher: Taylor \& Francis},
    pages = {61--96},
}

@inproceedings{al2021navigating,
    title = {Navigating to the best policy in markov decision processes},
    volume = {34},
    booktitle = {Advances in neural information processing systems},
    author = {Al Marjani, Aymen and Garivier, Aurélien and Proutiere, Alexandre},
    year = {2021},
    pages = {25852--25864},
}

@inproceedings{audibert2010best,
    title = {Best arm identification in multi-armed bandits},
    booktitle = {{COLT}-23th {Conference} on learning theory-2010},
    author = {Audibert, Jean-Yves and Bubeck, Sébastien},
    year = {2010},
    pages = {13--p},
}

@article{thompson1933likelihood,
    title = {On the likelihood that one unknown probability exceeds another in view of the evidence of two samples},
    volume = {25},
    number = {3-4},
    journal = {Biometrika},
    publisher = {Oxford University Press},
    author = {Thompson, William R},
    year = {1933},
    pages = {285--294},
}

@book{lattimore2020bandit,
    title = {Bandit algorithms},
    publisher = {Cambridge University Press},
    author = {Lattimore, Tor and Szepesvári, Csaba},
    year = {2020},
}

@inproceedings{kocak_best_2021,
    address = {Yokohama, Yokohama, Japan},
    series = {{IJCAI}'20},
    title = {Best arm identification in spectral bandits},
    isbn = {978-0-9992411-6-5},
    booktitle = {Proceedings of the twenty-ninth international joint conference on artificial intelligence},
    author = {Kocák, Tomáš and Garivier, Aurélien},
    year = {2021},
    note = {Number of pages: 7
tex.articleno: 307},
}

@inproceedings{wang_optimal_2020,
    series = {Proceedings of machine learning research},
    title = {Optimal algorithms for multiplayer multi-armed bandits},
    volume = {108},
    booktitle = {Proceedings of the twenty third international conference on artificial intelligence and statistics},
    publisher = {PMLR},
    author = {Wang, Po-An and Proutiere, Alexandre and Ariu, Kaito and Jedra, Yassir and Russo, Alessio},
    editor = {Chiappa, Silvia and Calandra, Roberto},
    month = aug,
    year = {2020},
    pages = {4120--4129},
}

@inproceedings{jedra2020optimal,
    title = {Optimal best-arm identification in linear bandits},
    volume = {33},
    booktitle = {Advances in neural information processing systems},
    author = {Jedra, Yassir and Proutiere, Alexandre},
    year = {2020},
    pages = {10007--10017},
}

@inproceedings{jourdan_top_2022,
    title = {Top {Two} {Algorithms} {Revisited}},
    volume = {35},
    language = {en},
    urldate = {2026-02-12},
    booktitle = {Advances in {Neural} {Information} {Processing} {Systems}},
    author = {Jourdan, Marc and Degenne, Rémy and Baudry, Dorian and de Heide, Rianne and Kaufmann, Emilie},
    month = dec,
    year = {2022},
    pages = {26791--26803},
}

@inproceedings{russo_simple_2016,
    title = {Simple {Bayesian} {Algorithms} for {Best} {Arm} {Identification}},
    issn = {1938-7228},
    language = {en},
    urldate = {2026-02-12},
    booktitle = {Conference on {Learning} {Theory}},
    publisher = {PMLR},
    author = {Russo, Daniel},
    month = jun,
    year = {2016},
    pages = {1417--1418},
}

@article{wald_optimum_1948,
    title = {Optimum character of the sequential probability ratio test},
    volume = {19},
    doi = {10.1214/aoms/1177730197},
    number = {3},
    journal = {The Annals of Mathematical Statistics},
    publisher = {Institute of Mathematical Statistics},
    author = {Wald, A. and Wolfowitz, J.},
    year = {1948},
    pages = {326 -- 339},
}

@book{garnett_bayesoptbook_2023,
  author    = {Garnett, Roman},
  title     = {{Bayesian Optimization}},
  year      = {2023},
  publisher = {Cambridge University Press}
}

@article{hennig2012Entropy,
  title = {Entropy {{Search}} for {{Information-Efficient Global Optimization}}},
  author = {Hennig, Philipp and Schuler, Christian J.},
  year = 2012,
  journal = {Journal of Machine Learning Research},
  volume = {13},
  number = {57},
  pages = {1809--1837},
  issn = {1533-7928},
  urldate = {2026-02-12},
  langid = {english}
}

@inproceedings{amini_complexity_2023,
    address = {Cham},
    title = {On the {Complexity} of {All} $\epsilon$-{Best} {Arms} {Identification}},
    volume = {13716},
    isbn = {978-3-031-26411-5 978-3-031-26412-2},
    doi = {10.1007/978-3-031-26412-2_20},
    language = {en},
    urldate = {2025-08-31},
    booktitle = {Machine {Learning} and {Knowledge} {Discovery} in {Databases}},
    publisher = {Springer Nature Switzerland},
    author = {Al Marjani, Aymen and Kocak, Tomas and Garivier, Aurélien},
    editor = {Amini, Massih-Reza and Canu, Stéphane and Fischer, Asja and Guns, Tias and Kralj Novak, Petra and Tsoumakas, Grigorios},
    year = {2023},
    note = {Series Title: Lecture Notes in Computer Science},
    pages = {317--332},
}

@inproceedings{russo2023model,
    title = {Model-free active exploration in reinforcement learning},
    volume = {36},
    booktitle = {Advances in neural information processing systems},
    author = {Russo, Alessio and Proutiere, Alexandre},
    year = {2023},
    pages = {54740--54753},
}

@inproceedings{tirinzoni2022near,
    title = {Near instance-optimal pac reinforcement learning for deterministic mdps},
    volume = {35},
    booktitle = {Advances in neural information processing systems},
    author = {Tirinzoni, Andrea and Al Marjani, Aymen and Kaufmann, Emilie},
    year = {2022},
    pages = {8785--8798},
}

@article{ghosh1991brief,
    title = {A brief history of sequential analysis},
    volume = {1},
    journal = {Handbook of sequential analysis},
    publisher = {Marcel Dekker New York},
    author = {Ghosh, Bashkar K},
    year = {1991},
}

@article{lindley1956measure,
    title = {On a measure of the information provided by an experiment},
    volume = {27},
    number = {4},
    journal = {The Annals of Mathematical Statistics},
    publisher = {Institute of Mathematical Statistics},
    author = {Lindley, Dennis V},
    year = {1956},
    pages = {986--1005},
}

@inproceedings{naghshvar2012noisy,
    title = {Noisy bayesian active learning},
    booktitle = {2012 50th annual allerton conference on communication, control, and computing (allerton)},
    publisher = {IEEE},
    author = {Naghshvar, Mohammad and Javidi, Tara and Chaudhuri, Kamalika},
    year = {2012},
    pages = {1626--1633},
}

@article{golovin2011adaptive,
    title = {Adaptive submodularity: {Theory} and applications in active learning and stochastic optimization},
    volume = {42},
    journal = {Journal of Artificial Intelligence Research},
    author = {Golovin, Daniel and Krause, Andreas},
    year = {2011},
    pages = {427--486},
}

@article{cohn1996active,
    title = {Active learning with statistical models},
    volume = {4},
    journal = {Journal of artificial intelligence research},
    author = {Cohn, David A and Ghahramani, Zoubin and Jordan, Michael I},
    year = {1996},
    pages = {129--145},
}

@article{cecchi2017adaptive,
    title = {Adaptive active hypothesis testing under limited information},
    volume = {30},
    journal = {Advances in Neural Information Processing Systems},
    author = {Cecchi, Fabio and Hegde, Nidhi},
    year = {2017},
}

@article{russo2014learning,
    title = {Learning to optimize via posterior sampling},
    volume = {39},
    number = {4},
    journal = {Mathematics of Operations Research},
    publisher = {INFORMS},
    author = {Russo, Daniel and Van Roy, Benjamin},
    year = {2014},
    pages = {1221--1243},
}

@inproceedings{hernandez-lobato_predictive_2014,
    title = {Predictive {Entropy} {Search} for {Efficient} {Global} {Optimization} of {Black}-box {Functions}},
    volume = {27},
    urldate = {2026-02-12},
    booktitle = {Advances in {Neural} {Information} {Processing} {Systems}},
    publisher = {Curran Associates, Inc.},
    author = {Hernández-Lobato, José Miguel and Hoffman, Matthew W. and Ghahramani, Zoubin},
    year = {2014},
}

@inproceedings{russo2023sample,
    title = {On the sample complexity of representation learning in multi-task bandits with global and local structure},
    volume = {37},
    booktitle = {Proceedings of the {AAAI} conference on artificial intelligence},
    author = {Russo, Alessio and Proutiere, Alexandre},
    year = {2023},
    pages = {9658--9667},
}

@inproceedings{taupin2023best,
    title = {Best policy identification in discounted linear {MDPs}},
    booktitle = {Sixteenth european workshop on reinforcement learning},
    author = {Taupin, Jérôme and Jedra, Yassir and Proutiere, Alexandre},
    year = {2023},
}

@article{bergstra2011algorithms,
  title={Algorithms for hyper-parameter optimization},
  author={Bergstra, James and Bardenet, R{\'e}mi and Bengio, Yoshua and K{\'e}gl, Bal{\'a}zs},
  journal={Advances in neural information processing systems},
  volume={24},
  year={2011}
}

@article{ament2023unexpected,
  title={Unexpected improvements to expected improvement for bayesian optimization},
  author={Ament, Sebastian and Daulton, Samuel and Eriksson, David and Balandat, Maximilian and Bakshy, Eytan},
  journal={Advances in neural information processing systems},
  volume={36},
  pages={20577--20612},
  year={2023}
}

@article{hansen2016cma,
  title={The CMA evolution strategy: A tutorial},
  author={Hansen, Nikolaus},
  journal={arXiv preprint arXiv:1604.00772},
  year={2016}
}

@inproceedings{akiba2019optuna,
  title={Optuna: A next-generation hyperparameter optimization framework},
  author={Akiba, Takuya and Sano, Shotaro and Yanase, Toshihiko and Ohta, Takeru and Koyama, Masanori},
  booktitle={Proceedings of the 25th ACM SIGKDD international conference on knowledge discovery \& data mining},
  pages={2623--2631},
  year={2019}
}

@article{naser2025review,
  title={A review of benchmark and test functions for global optimization algorithms and metaheuristics},
  author={Naser, MZ and Al-Bashiti, Mohammad Khaled and Tapeh, Arash Teymori Gharah and Naser, Ahmad and Kodur, Venkatesh and Hawileh, Rami and Abdalla, Jamal and Khodadadi, Nima and Gandomi, Amir H and Eslamlou, Armin Dadras},
  journal={Wiley Interdisciplinary Reviews: Computational Statistics},
  volume={17},
  number={2},
  pages={e70028},
  year={2025},
  publisher={Wiley Online Library}
}

@book{hernandez-lerma1996DiscreteTime,
    address = {New York, NY},
    title = {Discrete-{Time} {Markov} {Control} {Processes}},
    copyright = {http://www.springer.com/tdm},
    isbn = {978-1-4612-6884-0 978-1-4612-0729-0},
    doi = {10.1007/978-1-4612-0729-0},
    urldate = {2026-04-14},
    publisher = {Springer},
    author = {Hernández-Lerma, Onésimo and Lasserre, Jean Bernard},
    year = {1996},
}

@article{feinberg2021mdps,
    title = {{MDPs} with setwise continuous transition probabilities},
    volume = {49},
    number = {5},
    journal = {Operations Research Letters},
    publisher = {Elsevier},
    author = {Feinberg, Eugene A and Kasyanov, Pavlo O},
    year = {2021},
    pages = {734--740},
}

@book{rockafellar1974conjugate,
  title={Conjugate duality and optimization},
  author={Rockafellar, R Tyrrell},
  year={1974},
  publisher = {Society for Industrial and Applied Mathematics},
}

@book{ekeland1976convex,
  title={Convex analysis and variational problems},
  author={Ekeland, Ivar and Temam, Roger},
  year={1999},
  publisher={SIAM}
}

@book{rockafellar1998variational,
  title={Variational analysis},
  author={Rockafellar, R Tyrrell and Wets, Roger JB},
  year={1998},
  publisher={Springer}
}

@book{borwein2006convex,
  title={Convex Analysis and Nonlinear Optimization: Theoryand Examples},
  author={Borwein, Jonathan and Lewis, Adrian},
  year={2006},
  publisher={Springer}
}

@book{bertsekas1996stochastic,
    title = {Stochastic optimal control: the discrete-time case},
    volume = {5},
    publisher = {Athena Scientific},
    author = {Bertsekas, Dimitri and Shreve, Steven E},
    year = {1996},
}

@inproceedings{fujimoto2018addressing,
    title = {Addressing function approximation error in actor-critic methods},
    booktitle = {International conference on machine learning},
    publisher = {PMLR},
    author = {Fujimoto, Scott and Hoof, Herke and Meger, David},
    year = {2018},
    pages = {1587--1596},
}

@inproceedings{jang2024Fixed,
    title = {Fixed {Confidence} {Best} {Arm} {Identification} in the {Bayesian} {Setting}},
    volume = {37},
    booktitle = {Advances in {Neural} {Information} {Processing} {Systems}},
    publisher = {Curran Associates, Inc.},
    author = {Jang, Kyoungseok and Komiyama, Junpei and Yamazaki, Kazutoshi},
    editor = {Globerson, A. and Mackey, L. and Belgrave, D. and Fan, A. and Paquet, U. and Tomczak, J. and Zhang, C.},
    year = {2024},
    pages = {17789--17829},
}

@article{oquigley1990Continual,
    title = {Continual reassessment method: a practical design for phase 1 clinical trials in cancer},
    volume = {46},
    issn = {0006-341X},
    shorttitle = {Continual reassessment method},
    language = {eng},
    number = {1},
    journal = {Biometrics},
    author = {O'Quigley, J. and Pepe, M. and Fisher, L.},
    month = mar,
    year = {1990},
    pages = {33--48},
}

@incollection{efron1992bootstrap,
    title = {Bootstrap methods: another look at the jackknife},
    booktitle = {Breakthroughs in statistics: {Methodology} and distribution},
    publisher = {Springer},
    author = {Efron, Bradley},
    year = {1992},
    pages = {569--593},
}

@article{balandat2020botorch,
  title={BoTorch: A framework for efficient Monte-Carlo Bayesian optimization},
  author={Balandat, Maximilian and Karrer, Brian and Jiang, Daniel and Daulton, Samuel and Letham, Ben and Wilson, Andrew G and Bakshy, Eytan},
  journal={Advances in neural information processing systems},
  volume={33},
  pages={21524--21538},
  year={2020}
}
